\documentclass[journal]{IEEEtran}

\usepackage{amsthm}
\usepackage{amssymb}
\usepackage{cancel}
\usepackage{xfrac}
\usepackage{multirow}
\usepackage{graphicx,stackengine,scalerel}
\usepackage{array}
\usepackage{algorithm}
\usepackage{algpseudocode}
\usepackage{csquotes}
\usepackage{diagbox}
\usepackage{struktex}
\usepackage{empheq}
\usepackage{xcolor}
\usepackage{setspace}
\usepackage{etoolbox}
\usepackage{units}
\usepackage{hyperref} 
\usepackage{float}
\usepackage{textcomp}

\newtheorem{theorem}{Theorem}[section]
\newtheorem{corollary}{Corollary}[theorem]
\newtheorem{lemma}[theorem]{Lemma}
\theoremstyle{definition}
\newtheorem{definition}{Definition}[section]
\theoremstyle{remark}
\newtheorem*{remark}{Remark}

\makeatletter
\patchcmd{\@makecaption}
  {\scshape}
  {}
  {}
  {}
\makeatother

\algrenewcommand\algorithmicindent{1.0em}%

\begin{document}

\title{On Solving the Dynamics of Constrained Rigid Multi-Body Systems with Kinematic Loops}

\author{Vassilios~Tsounis,~Ruben~Grandia,~and~Moritz~B\"acher~
\thanks{All authors are with Disney Research Imagineering,\\ Zurich, Switzerland. e-mail: \textit{vassilios.tsounis@disney.com}}}

\maketitle

\begin{abstract}
This technical report provides an in-depth evaluation of both established and state-of-the-art methods for simulating constrained rigid multi-body systems with hard-contact dynamics, using formulations of Nonlinear Complementarity Problems (NCPs). We are particularly interest in examining the simulation of highly coupled mechanical systems with multitudes of closed-loop bilateral kinematic joint constraints in the presence of additional unilateral constraints such as joint limits and frictional contacts with restitutive impacts. This work thus presents an up-to-date literature survey of the relevant fields, as well as an in-depth description of the approaches used for the formulation and solving of the numerical time-integration problem in a maximal coordinate setting. More specifically, our focus lies on a version of the overall problem that decomposes it into the forward dynamics problem followed by a time-integration using the states of the bodies and the constraint reactions rendered by the former. We then proceed to elaborate on the formulations used to model frictional contact dynamics and define a set of solvers that are representative of those currently employed in the majority of the established physics engines. A key aspect of this work is the definition of a benchmarking framework that we propose as a means to both qualitatively and quantitatively evaluate the performance envelopes of the set of solvers on a diverse set of challenging simulation scenarios. We thus present an extensive set of experiments that aim at highlighting the absolute and relative performance of all solvers on particular problems of interest as well as aggravatingly over the complete set defined in the suite. 
\end{abstract}

\IEEEpeerreviewmaketitle

\section{Introduction}
\label{sec:introduction}

\noindent
Simulating complex mechanical systems such as multi-limb robots that can walk and manipulate is a key element in today's workflows for the development and testing of such systems. Indeed the availability of fast and accurate physical simulation has recently proved to be fundamental in applying state-of-the-art methods such as Deep Reinforcement Learning (DRL) for their control~\cite{hwangbo2019,isaacgym2023}. One important challenge the field currently faces is being able to design and control robots with highly coupled mechanical assemblies that exhibit intrinsic closed kinematic loops~\cite{dematteis2025optimal}. Such configurations are often used as means for transferring power and motion with reduced requirements on actuation compared to serial chains. 

Traditionally, however, the latter approach has been employed and the systems have been designed to be as simple as possible, both mechanically and morphologically. Examples include: a) actuating all joints, i.e. having no passive Degrees-of-Freedom (DoF), b) simplifying mechanical power transmission by placing actuators directly at or as near to the joint DoF as possible, and c) avoiding the use of parallel actuation that induces closed-loop kinematics. Besides the obvious benefits in terms of hardware design, in large part, the aforementioned conventions have also served to simplify simulating and controlling walking systems. In the presence of singularities and hyperstaticity, the system dynamics can become exceptionally challenging to solve, whether it be their forward dynamics for simulation, or inverse dynamics for control. As more advanced control methods that circumvent these challenges by employing nonlinear optimization and/or DRL have become widely available, it remains an open problem of how can we make our physics engines capable of accurately and efficiently simulating systems with arbitrary mechanical assemblies. 

Such a capability would have the potential to revolutionize the kinds of robots that could be developed. In the case of DRL in particular, which makes no assumptions on the model of the system, it is only limited in practice by the throughput of the simulation due to the large sample complexity required for convergence~\cite{hwangbo2019, rudin2022, hoeller2024}. Therefore if physically accurate and fast simulation for complex mechanisms can achieve an effective throughput on par with that of articulated systems, it would open up the possibility of applying DRL to any type of robot, and even more types of mechanical systems in general.

Broadly speaking, the requirements placed on such universal simulators could be considered in terms of two categories: 
\begin{itemize}
\item Modeling requirements: these are the types of constraints that can be supported. These include bilateral kinematic constraints such as passive, actuated, binary or unary joints that are common in mechanical assemblies, closed-loop kinematics and redundancy, as well as unilateral constraints such as frictional contacts and joint limits.
\item Performance requirements: physical plausibility\footnote{Physical realism is a difficult assertion to make when simulating rigid-body systems with discrete contacts. A more conservative and fair view would be to assert physical plausibility in terms of respecting kinematic constraints and based on the regularity of the computed constraint forces and torques.} and the relative trade-off between accuracy and speed that a given approach would exhibit. In this context, accuracy is defined strictly in terms of constraint satisfaction and sensitivity of the resulting constraint reactions (i.e. forces and torques) w.r.t the state of the system. Simulation speed is often conflated with sample throughput. In this work we consider speed strictly in the single-instance sense, i.e. in the absence of mass parallelization, for example on a GPU.
\end{itemize}

Satisfying all of the aforementioned modeling requirements, however, bears significant numerical challenges. Firstly, the forward dynamics problem can be ill-conditioned in the presence of large mass ratios and hyperstaticity due to highly coupled constraints such as closed-loop kinematics and underactuation. Secondly, constraint satisfaction at configuration-level (i.e. joints breaking, limits being exceeded and body inter-penetrations) can occur as a result of inaccurately estimating constraint reactions, and can accumulate prohibitively over multiple simulation steps. Moreover, another consequence of the ill-conditioning is the potential irregularity of the yielded constraint reactions, which can become extremely sensitive to small changes in the state of the system and thus change erratically over subsequent time-steps although the state has experienced changes on scales close to the numerical precision. Lastly, the stability of the solution can deteriorate when using larger time-steps, which is often necessary for increasing simulation throughput in applications such as DRL.

In this paper we present an empirical evaluation of existing approaches for solving the forward dynamics of constrained rigid-body systems. We consider cases that include bilateral kinematic constraints, of arbitrary connectivity and under/over actuation, as well as additional unilateral constraints that model restitutive impacts, frictional contacts and joint limits. Specifically, we compare both established and state-of-the-art algorithms for solving the Nonlinear Complementarity Problem (NCP)~\cite{facchinei2003finite} derived from the \textit{dual problem} of forward dynamics, i.e. the resolution of constraint reactions, and posited as a Nonlinear Second-Order Cone Program (NSOCP)~\cite{acary2018comparisons}.

Our evaluation consists of a benchmark suite of systems of varying complexity that range from primitive toy problems to full-scale robotic and \textit{Audio-Animatronic}\textregistered \,\,systems. Moreover, we also evaluate techniques such such as constraint relaxation and proximal optimization in order to deal with the aforementioned types of ill-conditioning. Our comparisons are performed using set of performance metrics based on first principles as well as practical considerations, that summarily speaking, cover physical plausibility, accuracy, speed and robustness in as fair terms as possible.

As done in previous work~\cite{schumacher2021versatile, maloisel2023optimal, maloisel2025versatile}, we employ a \textit{maximal-coordinate} formulation of Constrained Rigid-Body Dynamics (CRBD). In contrast to using \textit{minimal-coordinates}, this results in expressing kinematic constraints explicitly and makes it very straight-forward to model systems with kinematic loops and passive joints. In addition, it facilitates a versatile parameterization of the system kinematics that can support both unary and binary joints over a very extensive set of mechanically feasible joint types. Conversely, a potentially more computationally efficient approach could be presented in combining minimal-coordinates with constraints, i.e. Constrained Articulated-Body Dynamics (CABD)~\cite{featherstone2014rigid, todorov2014analytical, carpentier2021proximal, sathya2024constrained, sathya2025constrained}. However, as the scope of this study lies mainly in comparing algorithms specialized for dynamics problems, using CRBD instead of CABD serves merely as a benchmark to better highlight differences between solvers. We assert that if a method can be made to work effectively in the more challenging case of CRBD, then employing CABD can only serve to further simplify the problem by reducing its overall dimensionality.\\

\paragraph*{\textbf{Contributions}}
\label{par:contributions}
This work thus makes the following contributions within the context of robotics and control:
\begin{enumerate}
\item We provide an extensive survey of the current state-of-the-art in physical simulation of rigid-body systems and evaluate the reproducibility of highly-cited works.
\item We define a suite of benchmark problems and performance metrics that can be used to systematically evaluate the relative performance of forward dynamics solvers. The suite consists of a set of problems that are relevant and meaningful within the context of robotics but have applicability to even broader classes of systems.
\item We present an extensive evaluation of relevant algorithms on the aforementioned benchmark suite and describe practical recommendations on using them to realize physics simulators for constrained systems.
\end{enumerate}

\section{Constrained Rigid-Body Mechanics}
\label{sec:constrained-rigidbody-mechanics}

This section provides an overview of the mechanics of constrained rigid-body systems, and outlines the problems that arise within the context of physical simulation. The process of physical simulation itself, is most often viewed as solving a single high-level problem, comprised of three constituent sub-problems, namely: Event Detection (ED), Forward Dynamics (FD), and Time Integration (TI). Indeed, this decomposition is ubiquitous across all known physics engines~\cite{erez2015simulation,siggraph2022contact}. The first essentially corresponds to Collision Detection (CD), but can be viewed in the general sense of detecting events where the set of active constraints changes. In this way we can also incorporate detection of active joint limits, and possibly other elements that are evaluated at configuration-level. However, this work does not cover CD methods themselves. We recommend interested readers to consider~\cite{ericson2004real, pan2012fcl, coumans2022bullet3, coal2024hithub} for more details on the topic of CD. The FD and TI sub-problems are described in the next section.

\subsection{Forward Dynamics}
\label{sec:constrained-rigidbody-mechanics:dynamics}
The forward dynamics problem is one of evaluating how the motion of the system will evolve at a certain time $t \in \mathbb{R}$ given its state. Broadly speaking, assuming the set of constraints is immutable, and for a given choice of generalized coordinates $\mathbf{s} := \mathbf{s}(t) \in \mathbb{R}^{n_s}$ and generalized velocities $\mathbf{u} := \mathbf{u}(t) \in \mathbb{R}^{n_u}$ with $n_s \geq n_u$, the FD problem can be stated generically as:\\

\textbf{Find} $\,\, \dot{\mathbf{u}} \in \mathbb{R}^{n_u}\,\,,\dot{\mathbf{s}} \in \mathbb{R}^{n_s} \,\,$ s.t. :
\begin{subequations}
\begin{equation}
\dot{\mathbf{s}} = \mathbf{H}(\mathbf{s}) \, \mathbf{u}
\end{equation}
\begin{equation}
\mathbf{M}(\mathbf{s}) \, \dot{\mathbf{u}}
= \mathbf{h}(\mathbf{s},\mathbf{u}) + \mathbf{J}(\mathbf{s})^{T} \, \boldsymbol{\lambda}
\end{equation}
\begin{equation}
\mathbf{f}(\mathbf{s},\mathbf{u}) = 0
\end{equation}
\begin{equation}
\mathcal{K}^{*} \ni \mathbf{g}(\mathbf{s},\mathbf{u}) \perp \boldsymbol{\lambda} \in \mathcal{K}
\,\,\,\,.
\end{equation}
\label{eq:prelim:constrained-system-ncp}
\end{subequations}
$\mathbf{H}(\mathbf{s}) \in \mathbb{R}^{n_s \times n_u}$ is a matrix mapping generalized velocities to the time-derivative of the generalized coordinates. Such a mapping is often necessary for floating-base systems, where $\dot{\mathbf{s}} \neq \mathbf{u}$ due to the parameterization of rotations, e.g. quaternions, Euler angles, etc. $\mathbf{M}(\mathbf{s}) \in \mathbb{S}_{++}^{n_u}$ is the generalized mass matrix, where $\mathbb{S}_{++}^{n_u}$ denotes the set of $n \times n$ symmetric positive-definite (SPD)\footnote{When $\mathbf{s}$ is defined with either minimal (independent) coordinates or maximal coordinates then $\mathbf{M}(\mathbf{s})$ is always SPD. However, if the coordinates of $\mathbf{s}$ are dependent or redundant, then it is possible for $\mathbf{M}(\mathbf{s})$ to be singular, and thus positive semi-definite. See \cite{udwadia2006explicit} for a detailed analysis on the matter.} matrices. $\mathbf{h}(\mathbf{s}, \mathbf{u}) \in \mathbb{R}^{n_u}$ is the vector of non-linear generalized force terms, which, can include gravity, Coriolis, centrifugal, actuation and external effects. $\mathbf{f} : \mathbb{R}^{n_s} \times \mathbb{R}^{n_u} \rightarrow \mathbb{R}^{n_e}$ and $\mathbf{g} : \mathbb{R}^{n_s} \times \mathbb{R}^{n_u} \rightarrow \mathbb{R}^{n_i}$ are respectively the set of equality and inequality constraints in implicit form. $\mathbf{J}(\mathbf{s}) \in \mathbb{R}^{n_d \times n_u}$ is the constraint Jacobian\footnote{With the given form for $\mathbf{f}(\mathbf{s},\mathbf{u})$ and $\mathbf{g}(\mathbf{s},\mathbf{u})$, $\mathbf{J}(\mathbf{s}) = \begin{bmatrix} \frac{\partial \mathbf{f}}{\partial \mathbf{u}} & \frac{\partial \mathbf{g}}{\partial \mathbf{u}} \end{bmatrix}$. Alternatively, if the constraints were formed at configuration level, i.e. were in the form of $\mathbf{f}(\mathbf{s})$ and $\mathbf{g}(\mathbf{s})$, and $\mathbf{s}$ did not employ reduced parameterizations of rotation, then the Jacobian would be defined as $\mathbf{J}(\mathbf{s}) = \begin{bmatrix} \frac{\partial \mathbf{f}}{\partial \mathbf{s}} & \frac{\partial \mathbf{g}}{\partial \mathbf{s}} \end{bmatrix}$.} matrix, where $n_d = n_e + n_i$. $\boldsymbol{\lambda} \in \mathcal{K} \subseteq \mathbb{R}^{n_d}$ is the vector of Lagrange multipliers corresponding to constraint forces, where $\mathcal{K}$ is a closed and convex cone and $\mathcal{K}^{*}$ denotes its dual.

Fundamentally, (\ref{eq:prelim:constrained-system-ncp}a-b) form a system of Differential Algebraic Equations (DAE). However, since $\boldsymbol{\lambda} \in \mathcal{K}$ is bound by set-valued force laws~\cite{glocker2001setvalued} (e.g. unilateral contact and Coulomb friction), the system becomes a Differential Inclusion (DI). Thus, (\ref{eq:prelim:constrained-system-ncp}) overall is a Variational Inequality (VI)~\cite{facchinei2003finite}, that due to the feasible set being a \textit{cone}, forms a Complementarity Problem (CP). Lastly, due to the structure of the complementarity constraints (\ref{eq:prelim:constrained-system-ncp}d), which in general may be nonlinear, the FD problem (\ref{eq:prelim:constrained-system-ncp}) is therefore an NCP. 

\subsection{Impulsive Dynamics}
\label{sec:constrained-rigidbody-mechanics:impulsive}
One of the primary difficulties in modeling the contact dynamics of rigid bodies is that impacts and friction induce instantaneous changes to $\mathbf{u}(t)$ and $\dot{\mathbf{u}}(t)$ as well as to the set of active constraints. The discontinuity of the former thus impacts their integrability, as they cease to be continuous functions in $t$. In order to deal with this non-smoothness, we can employ techniques from \textit{non-smooth dynamics} \cite{glocker2001setvalued, stewart2011dynamics}. Specifically, \textit{differential measures}~\cite{moreau1988unilateral} provide a means to decompose the velocity differential into smooth and impulsive terms as
\begin{equation}
{d}\mathbf{u}
= 
\dot{\mathbf{u}}(t)\,{dt} 
+ 
(\mathbf{u}^{+} - \mathbf{u}^{-}) \, {d\eta}
\,\, ,
\label{eq:prelim:impulsive-velocity}
\end{equation}
where $\dot{\mathbf{u}}(t)$ is the smooth acceleration and $\mathbf{u}^{-}, \mathbf{u}^{+}$ are the pre- and post-event impulsive velocities, respectively. The key ingredient in this model is the impulsive differential ${d\eta}$. Assuming $n_p \in \mathbb{N}_{\geq 0}$ impulsive events occur at times $t_i \in [t_0, t] \,\,\forall \, i \in \{1, \dots, n_p\}$, the set of discrete impulsive events can be denoted as $I_p = \bigcup_{i=1}^{n_p} \{ t_{i} \}$. Then, ${d\eta}$ represents a finite sum of so-called \textit{Dirac point measures} over $t \in \mathbb{R}$:
\begin{equation}
{d\eta} := \sum_{t_i \in I_p} d\delta_{t_i}
\,\,,\,
\,\,
\int_{t_0}^{t} d\delta_{t_i} = 
\begin{cases} 
1 \,,\,\, t_i \in [t_0,\,t]  \\
0 \,,\,\, t_i \not\in [t_0,\,t]
\end{cases}
\,\,,
\end{equation}
thus allowing the time-integration of the impulsive velocity as 
\begin{equation}
\int_{t_0}^{t} \, \left(\mathbf{u}^{+}(\tau) - \mathbf{u}^{-}(\tau)\right) \, {d\eta} 
= 
\sum_{t_i \in I_p} \left( \mathbf{u}^{+}(t_i) - \mathbf{u}^{-}(t_i) \right)
\,\,.
\end{equation}
Equivalently, the differential measure over generalized forces
\begin{equation}
{d\mathbf{p}} = \mathbf{w}_{s}\,{dt} + \mathbf{w}_{ns} \, {d\eta}
\,\, ,
\label{eq:impulsive-forces}
\end{equation}
is often referred to as \textit{percussion}, 
where $\mathbf{w}_{s}, \mathbf{w}_{ns}$ correspond to the smooth and non-smooth (i.e. impulsive) generalized forces, respectively. Thus, (\ref{eq:prelim:impulsive-velocity}) and (\ref{eq:impulsive-forces}) enable the decomposition of the system dynamics into respective smooth and impulsive parts
\begin{subequations}
\begin{equation}
\mathbf{M}(\mathbf{s}) \, \dot{\mathbf{u}} \, {dt} 
=
\mathbf{h}(\mathbf{s}, \mathbf{u}) \, {dt} 
+
\mathbf{J}_{s}(\mathbf{s})^{T} \, \boldsymbol{\lambda}_{s} \, {dt} 
\,\,\,,
\label{eq:prelim:smooth-dynamics}
\end{equation}
\begin{equation}
\mathbf{M}(\mathbf{s}) \, 
\left( \mathbf{u}^{+} - \mathbf{u}^{-} \right)
\, {d\eta} 
=
\mathbf{J}_{ns}(\mathbf{s})^{T} \, \boldsymbol{\lambda}_{ns} \, {d\eta}
\,\,\,,
\label{eq:prelim:impuslive-dynamics}
\end{equation}
\label{eq:prelim:precussion-dynamics}
\end{subequations}
where $\mathbf{w}_{s} = \mathbf{J}_{s}(\mathbf{s})^{T} \, \boldsymbol{\lambda}_{s}$ and $\mathbf{w}_{ns} = \mathbf{J}_{ns}(\mathbf{s})^{T} \, \boldsymbol{\lambda}_{ns}$ are necessarily distinct as the set of active constraints may differ between the continuous dynamics and those at the instants of discontinuity. Moreover, while $\boldsymbol{\lambda}_{s}$ represent forces, $\boldsymbol{\lambda}_{ns}$ now represent \textit{impulses}. Herein we refer to both as \textit{constraint reactions}.

Therefore, whenever an impulsive event occurs at some time $t_i$, the FD problem (\ref{eq:prelim:constrained-system-ncp}) must be rephrased to employ (\ref{eq:prelim:impuslive-dynamics}) in-place of (\ref{eq:prelim:smooth-dynamics}). This leads to the impulsive FD problem\\

\textbf{Find} $\,\, \mathbf{u}^{+} \in \mathbb{R}^{n_u}\,\,,\dot{\mathbf{s}} \in \mathbb{R}^{n_s} \,\,$ s.t. :
\begin{subequations}
\begin{equation}
\dot{\mathbf{s}} = \mathbf{H}(\mathbf{s}) \, \mathbf{u}^{+}
\end{equation}
\begin{equation}
\mathbf{M}(\mathbf{s}) \, \left( \mathbf{u}^{+} - \mathbf{u}^{-} \right)
= \mathbf{J}(\mathbf{s})^{T} \, \boldsymbol{\lambda}
\end{equation}
\begin{equation}
\mathbf{f}(\mathbf{s},\mathbf{u}^{-},\mathbf{u}^{+}) = 0
\end{equation}
\begin{equation}
\mathcal{K}^{*} \ni \mathbf{g}(\mathbf{s},\mathbf{u}^{-},\mathbf{u}^{+}) \perp \boldsymbol{\lambda} \in \mathcal{K}
\,\,\,\,,
\end{equation}
\label{eq:prelim:impulsive-constrained-system-ncp}
\end{subequations}
where all time-dependent configuration-level quantities are evaluated at time $t_i$, e.g. $\mathbf{s} := \mathbf{s}(t_i)$, while the pre-event generalized velocity $\mathbf{u}^{-} = \mathbf{u}^{-}(t_i)$ is corresponds to the lower limit value $\lim_{\tau  \rightarrow 0} \mathbf{u}(t_i - \tau)$ resulting from the time-integration of the smooth dynamics (\ref{eq:prelim:constrained-system-ncp}). Crucially, (\ref{eq:prelim:impulsive-constrained-system-ncp}) can be understood as a boundary-value problem w.r.t time, as it is only admissible at the discrete atomic time-points $t_i$. Overall, the impulsive FD problem (\ref{eq:prelim:impulsive-constrained-system-ncp}), like its smooth counterpart (\ref{eq:prelim:constrained-system-ncp}), is also a NCP, but one that now yields the system's generalized velocities directly, instead of the generalized accelerations.
\begin{figure}[!b]
\centering
\includegraphics[width=1.0\linewidth]{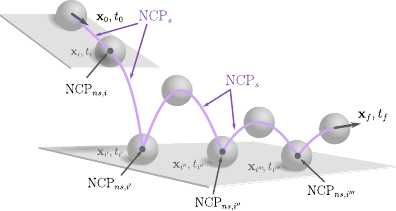}
\caption{A visual depiction of the non-smooth initial-value problem (\ref{eq:prelim:initial-value-problem}) using the example of a rigid sphere rolling-off of one plane and then bouncing on another. When the sphere drops the first plane, the impact would result in a contact that is initially open, then closes momentarily and opens once again. Although the position and orientation of the sphere would remain continuous functions of time, the linear and angular velocities would change instantaneously in order to satisfy the contact constraints. This means that the velocities are in effect not continuous, and nor are the positions and orientations smooth functions of time. Therefore, the sphere's trajectory consists of the piece-wise smooth segments (violet lines) and atomic points $\bigcup_{i}{t_{i}}$ at the moments of impact (gray nodes). The former is determined by the smooth dynamics represented by the acceleration-level $\text{NCP}_{s,t}$ of (\ref{eq:prelim:constrained-system-ncp}), while the latter defined at each instance $t_{i}$ is determined by the impulsive $\text{NCP}_{ns,t_{i}}$ of (\ref{eq:prelim:impulsive-constrained-system-ncp}). Direct time-integration is intractable analytically thus necessitating the application of time-discretization to approximate the integral in (\ref{eq:prelim:time-integration-problem}).}
\label{fig:ball-drop}
\end{figure}

\newpage
\subsection{Time Integration}
\label{sec:constrained-rigidbody-mechanics:simulation}

Given means to evaluate the forward dynamics of the constrained rigid-body system, pointwise for any time $t$, simulation amounts to time-integrating the former over some finite time interval $I = [t_0, t_f]$. This corresponds to solving an \textit{Initial Value Problem} (IVP) in the form of
\begin{subequations}
\begin{equation*}
\textbf{For} \,\, t \in I := [t_0, t_f], 
\,\, \text{and} \,\, I_p := \bigcup_{i=1}^{n_p} \{ t_{i} \} \subset I \\[4pt]
\end{equation*}
\begin{equation}
\textbf{With} \,\,\, 
\mathbf{x}(t) :=
\begin{bmatrix}
    \mathbf{s}(t)\\
    \mathbf{u}(t)
\end{bmatrix} 
\\
\end{equation}
\begin{equation*}
\textbf{Find} \,\, \mathbf{x}(t) \,, \\
\,\,
\textrm{s.t.} 
\end{equation*}
\begin{equation}
\mathbf{x}(t_0) = \mathbf{x}_{0} \\[4pt]
\end{equation}
\begin{equation}
\mathbf{x}(t)
= 
\mathbf{x}_{0} 
+ 
\begin{bmatrix}
\displaystyle \int_{\,t_0}^{\,t} \, \mathbf{H}(\mathbf{s}(\tau)) \, \mathbf{u}(\tau) \, {d}{\tau} \\
\displaystyle \sum_{t_i \in I_p}{\left(\mathbf{u}^{+}(t_i) - \mathbf{u}^{-}(t_i)\right)}
+
\int_{\,t_0}^{\,t} \, \dot{\mathbf{u}}(t) \, {d}{\tau}    
\end{bmatrix}\\[10pt]
\label{eq:prelim:time-integration-problem}
\end{equation}
\begin{empheq}[left={\text{NCP}_{ns,t_i} \, (\ref{eq:prelim:impulsive-constrained-system-ncp}) }\empheqlbrace\,\,\,\,]{align}
\mathbf{M} \, (\mathbf{u}^{+}(t_i) - \mathbf{u}^{-}(t_i)) 
= \mathbf{J}_{ns}^{T} \, \boldsymbol{\lambda}_{ns}\\[4pt]
\mathbf{f}_{ns}(\mathbf{s}(t_i), \mathbf{u}^{+}(t_i)) = 0\\[4pt]
\mathcal{K}_{ns}^{*} \ni \mathbf{g}_{ns}(\mathbf{s}(t_i), \mathbf{u}^{+}(t_i)) \perp \boldsymbol{\lambda}_{ns} \in \mathcal{K}_{ns}
\end{empheq}
\begin{empheq}[left={\text{NCP}_{s,t} \, (\ref{eq:prelim:constrained-system-ncp})}\empheqlbrace\,\,\,\,]{align}
\mathbf{M} \, \dot{\mathbf{u}}(t) 
= \mathbf{h} + \mathbf{J}_{s}^{T} \, \boldsymbol{\lambda}_{s}\\[4pt]
\mathbf{f}_{s}(\mathbf{s}(t), \mathbf{u}(t), \dot{\mathbf{u}}(t)) = 0\\[4pt]
\mathcal{K}_{s}^{*} \ni \mathbf{g}_{s}(\mathbf{s}(t), \mathbf{u}(t), \dot{\mathbf{u}}(t)) \perp \boldsymbol{\lambda}_{s} \in \mathcal{K}_{s}
\end{empheq}
\label{eq:prelim:initial-value-problem}
\end{subequations}
Fig.~\ref{fig:ball-drop} provides a visual reference for understanding the structure of the IVP (\ref{eq:prelim:initial-value-problem}). Essentially, it involves three distinct sub-problems. The first and second, we recognize to be the aforedescribed FD problems $\text{NCP}_{ns,t_i}$ (\ref{eq:prelim:impulsive-constrained-system-ncp}) and $\text{NCP}_{s,t}$ (\ref{eq:prelim:constrained-system-ncp}), which yield $\dot{\mathbf{u}}(\tau)$ and $\mathbf{u}^{+}(t_i)$ pointwise for any $t$ and $t_i$, respectively, and the third corresponds to evaluating the time-integrals in (\ref{eq:prelim:time-integration-problem}). However, as the $\text{NCP}_{ns,t_i}$ and $\text{NCP}_{s,t}$ can only be evaluated pointwise, they cannot be solved in closed-form to yield $\dot{\mathbf{u}}(\tau)$ and $\mathbf{u}(t)$ as integrable functions. Therefore, the IVP (\ref{eq:prelim:initial-value-problem}) is analytically intractable in general, and solving it necessitates the use of time-discretization and the application of numerical techniques.

\section{Problem Formulation}
\label{sec:problem-formulation}
\noindent
This section outlines the derivation of a computationally tractable approximation of the IVP (\ref{eq:prelim:initial-value-problem}) summarized previously. Specifically, we describe the set of approximations, w.r.t time-discretization, that are required to transcribe the integration and NCP sub-problems, and introduce a reduced form of the latter that is amenable to numerical optimization techniques.

\subsection{Time Discretization}
\label{sec:formulation:time-discretization}
\noindent
In order to render the IVP (\ref{eq:prelim:initial-value-problem}) tractable, appropriate time-discretization and time-integration schemes are required. Among the most widely adopted are the so-called \textit{time-stepping methods}~\cite{moreau1988unilateral,stewart1996implicit}. These progress the state over series of pre-determined time-intervals and evaluate the system dynamics only on impulse-velocity level. In doing so, every simulation step is able to handle all types of impulsive events in a unified fashion. They are also able to resolve phenomena such as Painlev\'e's paradox~\cite{genot1999} that occur when evaluating the dynamics at force-acceleration level.

Assuming a fixed time-step ${\Delta}{t} \in \mathbb{R}_{\geq 0}$, and considering all impulsive events occurring within the closed interval $[t_0, t_0 + {\Delta}{t}]$ to happen simultaneously, the smooth and impulsive dynamics of (\ref{eq:prelim:precussion-dynamics}) can be combined as
\begin{equation}
\mathbf{M} \, \left( \mathbf{u}^{+} - \mathbf{u}^{-} \right) =  
\Delta{t} \, \mathbf{h}
+ \mathbf{J}^{T} \, \boldsymbol{\lambda}
\,\, ,
\label{eq:formulation:impulsive-constrained-dynamics}
\end{equation}
and the IVP over a single-step interval $[t_0, t_0 + {\Delta}{t}]$ becomes
\begin{subequations}
\begin{equation*}
\textbf{For} \,\, t \in I := [t_0, t_0 + {\Delta}{t}]\\
\end{equation*}
\begin{equation*}
\textbf{Find} \,\, \mathbf{x}(t_0 + {\Delta}{t}) \,, \\
\,\,
\textrm{s.t.} 
\end{equation*}
\begin{equation}
\mathbf{x}(t_0) = \mathbf{x}_{0} \\[4pt]
\end{equation}
\begin{equation}
\mathbf{x}(t_0 + {\Delta}{t})
= 
\begin{bmatrix}
\mathbf{s}(t_0) \, \boxplus \,\mathbf{H}(\mathbf{s}(t_0)) \, \mathbf{u}^{+}(t_0 + {\Delta}{t})   \\
\mathbf{u}^{+}(t_0 + {\Delta}{t})    
\end{bmatrix}\\[10pt]
\label{eq:formulation:time-stepping-euler}
\end{equation}
\begin{empheq}[left={\text{NCP}_{{\Delta}{t}}} \empheqlbrace\,\,\,\,]{align}
\label{eq:formulation:time-stepping-ncp}
\mathbf{M} \, (\mathbf{u}^{+} - \mathbf{u}^{-}) 
= \Delta{t} \, \mathbf{h} + \mathbf{J}^{T} \, \boldsymbol{\lambda}\\[4pt]
\mathbf{f}(\mathbf{s}, \mathbf{u}^{+}) = 0\\[4pt]
\mathcal{K}^{*} \ni \mathbf{g}(\mathbf{s}, \mathbf{u}^{+}) \perp \boldsymbol{\lambda} \in \mathcal{K}
\,\,,
\end{empheq}
\label{eq:formulation:single-step-initial-value-problem}
\end{subequations}
where the $\boxplus$ denotes the addition operator specific to the parameterization of rotations used to define $\mathbf{s}$. In this work we make use of the \textit{exponential map} to compute integrals of orientation as the concatenation of rotation operators~\cite{bloesch2016primer}.

Thus the single-step IVP (\ref{eq:formulation:single-step-initial-value-problem}) formulated over the interval $[t_0, t_0 + {\Delta}{t}]$, amounts to solving a single NCP to yield the generalized velocities $\mathbf{u}(t_0 + {\Delta}{t})$ and then forward integrating to render the next-step generalized coordinates $\mathbf{s}(t_0 + {\Delta}{t})$. This integration scheme is often referred to as the \textit{semi-implicit Euler} method, as velocities are computed implicitly via the NCP, while generalized positions are integrated explicitly. 

Although in this work we only consider the first-order forward integration of the generalized coordinates, higher-order integration methods such as Runge-Kutta can be conceived, that would solve the NCP at multiple internal steps to render more accurate, and possibly more stable state integrators. Furthermore, another integration method commonly used in non-smooth mechanics formulations is the mid-point time-stepping scheme of J.J. Moreau~\cite{moreau1988unilateral}. It is similar to the semi-implicit Euler method, except that all intermediate quantities, including event detection, are evaluated from the state at the middle of the time-step computed using backward Euler.

\subsection{The Dual Problem}
\label{sec:formulation:dual-problem}
\noindent
A common strategy for solving the NCP sub-problem (\ref{eq:formulation:time-stepping-ncp}), is by taking an alternative view that rewrites the problem in terms of the constraint reactions $\boldsymbol{\lambda}$. Generally speaking, this alternative formulation can be referred to as the \textit{dual problem of forward dynamics}. Doing so, the FD problem can be solved by simple back-substitution to the respective equations of motion (\ref{eq:formulation:impulsive-constrained-dynamics}), to yield $\mathbf{u}^{+}$. Although a multitude of approaches exist for deriving the dual problem, we consider the most principled to be that which employs \textit{Gauss' principle of duality} and the \textit{extended principle of least constraint}~\cite{udwadia1992,glocker2001setvalued}. At this stage, we will only outline the formulation of the dual FD problem, while its construction will be detailed in Sec.~\ref{sec:construction}, and a complete derivation is provided in Appendix.~\ref{sec:apndx:ncp-derivation}. 

Projecting (\ref{eq:formulation:time-stepping-ncp}) to the \textit{space of the constraints}, the dynamics of the system are summarized by the following quantities:
\begin{enumerate}
\item the inverse apparent inertia, a.k.a the \textit{Delassus} matrix
\begin{equation}
\mathbf{D} := \mathbf{J}^{T} \, \mathbf{M}^{-1} \, \mathbf{J}
\label{eq:delassus-matrix}
\end{equation}
\item the unconstrained, a.k.a free, velocity
\begin{equation}
\mathbf{v}_{f} 
:= \mathbf{J} \, \left( \mathbf{u}^{-} + \Delta{t} \,\mathbf{M}^{-1} \mathbf{h}\right) + \mathbf{v}^{*}
\label{eq:free-velocity}
\end{equation}
\item post-event constraint velocity 
\begin{equation}
\mathbf{v}^{+}(\boldsymbol{\lambda}) := \mathbf{D} \, \boldsymbol{\lambda} + \mathbf{v}_{f}
\label{eq:post-event-constraint-velocity}
\end{equation}
\item augmented post-event constraint velocity
\begin{equation}
\hat{\mathbf{v}}(\boldsymbol{\lambda}) := \mathbf{v}^{+}(\boldsymbol{\lambda}) + \boldsymbol{\Gamma}(\mathbf{v}^{+}(\boldsymbol{\lambda}))
\label{eq:augmented-constraint-velocity}
\end{equation}
\end{enumerate}
where $\mathbf{v}^{*}$ is an auxiliary bias term that allows us to introduce additional elements such as a model of impacts and constraint stabilization, and $\boldsymbol{\Gamma}(\mathbf{v}^{+}(\boldsymbol{\lambda}))$ is the non-linear De Saxc\'e correction~\cite{desaxce1998bipotential,acary2011formulation}. The latter is an operator over the post-event constraint velocity and is the principle source of non-linearity in the problem, as detailed in Sec.~\ref{sec:models:contacts}. Thus, the NCP formulated for the dual FD is stated, rather concisely, as
\begin{flalign}
& \text{NCP}(\mathbf{D}, \mathbf{v}_{f}, \mathcal{K}): & \nonumber\\[4pt]
& \,\, \textbf{Find} \,\, \boldsymbol{\lambda}, \,\,
\textrm{s.t.} \quad
\mathcal{K}^{*} \ni
\hat{\mathbf{v}}(\boldsymbol{\lambda})
\, \perp \,
\boldsymbol{\lambda} \in \mathcal{K}
\,\,.
\label{eq:formulation:dual-forward-dynamics-ncp}
\end{flalign}

Furthermore, (\ref{eq:post-event-constraint-velocity})-(\ref{eq:formulation:dual-forward-dynamics-ncp}) correspond to the KKT conditions of a Nonlinear Second-Order Cone Program (NSOCP). This perspective motivates the application of relevant techniques which can be used to solve the NCP via the transcription 
\begin{flalign}
& \text{NSOCP}(\mathbf{D}, \mathbf{v}_{f}, \mathcal{K}): & \nonumber\\
& \,\, \textbf{Find} \,\, \boldsymbol{\lambda} =
\displaystyle
\operatorname*{argmin}_{\mathbf{x} \in \mathcal{K}} \,\,
\frac{1}{2} \,\mathbf{x}^{T} \, \mathbf{D} \, \mathbf{x}
+ \mathbf{x}^{T} 
\left(
\mathbf{v}_{f} + \boldsymbol{\Gamma}(\mathbf{v}^{+}(\mathbf{x}))
\right)
\label{eq:formulation:dual-forward-dynamics-nsocp}
\end{flalign}
Carpentier et al provided a coarse proof of the equivalence between the NCP and the NSOCP in~\cite{carpentier2024unified}, and it has also been outlined in other works such as~\cite{acary2018comparisons}. Given this context, we herein identify the constraint reactions $\boldsymbol{\lambda}$ and the constraint velocities $\mathbf{v}^{+}$ as the primal and dual variables, respectively, of the dual FD problem. We will also find it useful to decompose the objective function of (\ref{eq:formulation:dual-forward-dynamics-nsocp}) into the individual terms
\begin{subequations}
\begin{equation}
f(\mathbf{x}) := 
\frac{1}{2} \, \mathbf{x}^{T} \, \mathbf{D} \, \mathbf{x}
+ \mathbf{x}^{T} \, \mathbf{v}_{f}
\,\,,
\label{eq:formulation:quadratic-objective}
\end{equation}
\begin{equation}
f_{\text{DS}}(\mathbf{x}) := 
\mathbf{x}^{T} \,  \boldsymbol{\Gamma}(\mathbf{v}^{+}(\mathbf{x}))
\,\,,
\label{eq:formulation:ds-objective}
\end{equation}
\begin{equation}
f_{\text{NCP}}(\mathbf{x}) := 
f(\mathbf{x}) + f_{\text{DS}}(\mathbf{x})
\,\,.
\label{eq:formulation:ncp-objective}
\end{equation}
\label{eq:formulation:nsocp-objectives}
\end{subequations}
The total NSOCP objective (\ref{eq:formulation:ncp-objective}) thus consists of the purely quadratic objective (\ref{eq:formulation:quadratic-objective}) as well as the non-linear De Saxc\'e term (\ref{eq:formulation:ds-objective}). This decomposition will enable a relaxation of the problem that will be discussed next.

\subsection{Problem Relaxations}
\label{sec:formulation:relaxations}
\noindent
Observing the decomposition of the NSOCP objective function defined in (\ref{eq:formulation:nsocp-objectives}), we may consider the case where the non-linear term (\ref{eq:formulation:quadratic-objective}) is omitted. Doing so turns out to be equivalent to approximating the original problem with a Cone Complementarity Problem (CCP) in the form of 
\begin{flalign}
& \text{CCP}(\mathbf{D}, \mathbf{v}_{f}, \mathcal{K}): & \nonumber\\[4pt]
& \,\, \textbf{Find} \,\, \boldsymbol{\lambda}, \,\,
\textrm{s.t.} \quad
\mathcal{K}^{*} \ni
\mathbf{v}^{+}(\boldsymbol{\lambda})
\, \perp \,
\boldsymbol{\lambda} \in \mathcal{K}
\,\,.
\label{eq:formulation:dual-forward-dynamics-ccp}
\end{flalign}
The primary effect of this relaxation is that complementarity is now asserted directly between the constraint velocities and the constraint reactions. Equivalently, a transcription of (\ref{eq:formulation:dual-forward-dynamics-ccp}) as an optimization problem results in a convex Second-Order Cone Program (SOCP) which can be stated concisely as
\begin{flalign}
& \text{SOCP}(\mathbf{D}, \mathbf{v}_{f}, \mathcal{K}): & \nonumber\\
& \,\, \textbf{Find} \,\, \boldsymbol{\lambda} =
\displaystyle
\operatorname*{argmin}_{\mathbf{x} \in \mathcal{K}} \,\,
\frac{1}{2} \,\mathbf{x}^{T} \, \mathbf{D} \, \mathbf{x}
+ \mathbf{x}^{T} \, \mathbf{v}_{f}
\label{eq:formulation:dual-forward-dynamics-socp}
\end{flalign}

This alternate formulation of the FD problem can be traced back to the work of Redon et al in~\cite{redon2002}, Anitescu et al in~\cite{anitescu2006}, and Drumwright et al in~\cite{drumwright2011modeling}, although not stated as concisely as in the now established form of (\ref{eq:formulation:dual-forward-dynamics-socp}). Undoubtedly, the CCP approach as we know it today, has mainly been popularized by the work of Todorov in~\cite{todorov2011convex} in developing the MuJoCo simulator~\cite{todorov2012mujoco}. More information on the CCP and its implications can be found in~\cite{todorov2014analytical, siggraph2022contact, lidec2024models}.

However, the CCP formulation as been the source of significant debate among researchers, as discussed in~\cite{chatterjee1999, drumwright2011modeling}. Due to the mathematical rigor by which the NCP is derived from VI theory, the CCP is widely thought to be an approximation that results in undesirable artifacts such as forces acting at a distance~\cite{anitescu2006,hwangbo2018percontact,lidec2024models}. Conversely, Chaterjee~\cite{chatterjee1999} and Drumwright~\cite{drumwright2011modeling}, as well as the authors of MuJoCo~\cite{mujoco2024docs}, note that NCP complementarity does not necessarily agree with physical experiments, as the rigid-body model is also fundamentally an approximation of real physical bodies. In any case, the CCP model has proven to be a useful mechanism to realize physics engines with impressive throughput and versatility, such as MuJoCo and project Chrono~\cite{mazhar2015,tasora2016chrono}.

\subsection{First-Order Dual Solvers}
\label{sec:formulation:first-order-dual-solvers}
\noindent
Broadly speaking, the vast majority of algorithms used in the various popular physics engines and simulators, are \textit{first-order} methods and follow a common pattern. To outline their commonalities, we can coarsely fit them to the template defined in Alg.~\ref{alg:generic-first-order-dual-solver} that presents a generic form of first-order algorithm for solving the dual problem. From a high-level perspective, the major common components of these include:
\begin{itemize}
\item $\mathbf{D},\mathbf{v}_{f},\mathcal{K}$: definition of the dual problem 
\item $\boldsymbol{f}_{BLAS}(\cdot)$: operator to represent the linear system solver\footnote{BLAS: Basic Linear Algebra Subprograms} 
\item $\mathcal{P}^{\mathcal{K}}(\cdot)$: operator to project iterates to the feasible set $\mathcal{K}$
\item $\boldsymbol{f}_{stop}(\cdot)$: operator to represent the \textit{termination criteria}
\item $\boldsymbol{\lambda}^{0}$: (optional) initial guess to \textit{warmstart} the solution
\end{itemize}

\begin{algorithm}[!t]
\caption{Generic First-Order Dual Solver}
\label{alg:generic-first-order-dual-solver}
\begin{algorithmic}
\Require $\boldsymbol{\lambda}^{0},\,\,\mathbf{D},\,\,\mathbf{v}_{f},\,\,\mathcal{K},\,\,\boldsymbol{f}_{BLAS}(\cdot)\,\,\mathcal{P}^{\mathcal{K}}(\cdot),\,\,\boldsymbol{f}_{stop}(\cdot)$
\For{$i = 1$ to $N$}
\State $\mathbf{s}^{i} \gets \boldsymbol{\Gamma}(\mathbf{v}^{+}(\boldsymbol{\lambda}^{i-1}))$ \textcolor{gray}{\Comment{estimate nonlinearity}}
\State ${\hat{\mathbf{v}}}^{i} \gets \mathbf{D} \, \boldsymbol{\lambda}^{i-1} + \mathbf{v}_{f} + \mathbf{s}^{i}$ \textcolor{gray}{\Comment{compute gradient}}
\State $\boldsymbol{\lambda}_{0}^{i} \gets \boldsymbol{f}_{BLAS}(\mathbf{D}\,,\,\, -{\hat{\mathbf{v}}}^{i})$ \textcolor{gray}{\Comment{perform descent step}}
\State $\boldsymbol{\lambda}^{i} \gets \mathcal{P}^{\mathcal{K}}(\boldsymbol{\lambda}_{0}^{i})$ \textcolor{gray}{\Comment{project to feasible set}}
\State $\boldsymbol{\lambda}^{*} \gets \boldsymbol{f}_{stop}(\boldsymbol{\lambda}^{i})$ \textcolor{gray}{\Comment{check termination criteria}}
\EndFor \\
\Return $\boldsymbol{\lambda}^{*}$
\end{algorithmic}
\end{algorithm}

The generic BLAS operator $\boldsymbol{f}_{BLAS}(\cdot)$ expresses the method used to approximately perform the inversion of the Delassus matrix $\mathbf{D}$. It often is the case that this is not realized in a literal sense, but rather, it represents a factorization of $\mathbf{D}$, e.g. Cholesky or LU decompositions, that facilitate computing an \textit{unconstrained solution} $\boldsymbol{\lambda}_{0}^{i}$ at each iteration.

In every iteration, the projection operator $\mathcal{P}^{\mathcal{K}}(\cdot)$, a.k.a \textit{projector}, projects the unconstrained solution $\boldsymbol{\lambda}_{0}^{i}$ onto the feasible set $\mathcal{K}$, yielding a solution iterate $\boldsymbol{\lambda}^{i}$ that is feasible according to $\mathcal{K}$. Such an operation is common in optimization algorithms that deal with inequality constraints. It therefore attempts to ensure that each iterate $\boldsymbol{\lambda}^{i+1}$, and therefore also the final solution $\boldsymbol{\lambda}^{*}$, satisfy the inequalities as well as other conditions that define $\mathcal{K}$. $\mathcal{P}^{\mathcal{K}}(\cdot)$ is often one of the most important components of algorithms that follow the template of Alg.~\ref{alg:generic-first-order-dual-solver}. As we will see in Sec.~\ref{sec:solvers:projectors}, choosing how to realize it depends on the contact model used, and thus plays a crucial role in solving the dual problem (\ref{eq:formulation:dual-forward-dynamics-ncp}) or (\ref{eq:formulation:dual-forward-dynamics-ccp}).

Each algorithm necessarily defines its own specialized termination operator $\boldsymbol{f}_{stop}(\cdot)$ that is closely tied to the numerical scheme used. Common elements among various dual solvers include the use of absolute and/or relative numerical error between sequential solution iterates. But it does not, and sometimes cannot, be based solely on such criteria. In addition, algorithms must also check for constraint satisfaction, i.e. whether solution iterates lie within the feasible set, and by how much they violate the respective constraints. This ensures that solvers will converge only when both the numerical error as well as constraint violation are within desired bounds.

Lastly, the use of an initial guess $\boldsymbol{\lambda}^{0}$ is a common trait of many efficient numeric solvers. In most cases it significantly speeds up convergence to a solution, however, choosing how to construct one can prove challenging~\cite{wang2016warm,tasora2021admm,lidec2024models}. If the contact configuration has not changed, i.e. the dimensionality of the constraint forces is the same, the most common approach is to use the solution of the previous time-step. Otherwise, initialization to zero is necessary, and is the default behavior when $\boldsymbol{\lambda}^{0}$ is unspecified. A more robust approach would be to compute a constraint-invariant initial guess based on the current state of the system~\cite{wang2016warm,tasora2021admm}. 

\section{Modeling Constrained Systems}
\label{sec:models}
\noindent
Until this point we have been able to state the FD problem in rather generic terms, without being specific about what is involved in the transcription of the NCP (\ref{eq:formulation:dual-forward-dynamics-ncp}). This section introduces the elements of physical modeling that will provide the foundation from which we will later build the NCP concretely in Sec.~\ref{sec:construction}. Specifically, we will detail the maximal-coordinate CRBD formulation and the modeling of constraints, i.e. joints, limits and contacts.

For this purpose, we introduce some additional notation to denote the \textit{index sets} of the system's DoFs and constraints~\cite{glocker2001setvalued}. Consider a constrained rigid-body system, such as the one depicted in Fig.~\ref{fig:constraint-system-with-contacts}. It can be defined as the collection of bodies $\mathcal{B}_{i}\,, i \in \{1, \dots, n_b\}$, joints $\mathcal{J}_{j}\,, j \in \{1, \dots, n_j\}$, limits $\mathcal{L}_{l}\,, l \in \{1, \dots, n_l\}$ and contacts $\mathcal{C}_{k}\,, k \in \{1, \dots, n_c\}$. The symbols $\mathcal{B}_{i}$, $\mathcal{J}_{j}$, $\mathcal{L}_{l}$, $\mathcal{C}_{k}$ are used to refer abstractly to each element; implying its dimensionality, associated quantities, and relevant local coordinate frames. 

Moreover, to remove ambiguity from representing quantities in different Cartesian frames of reference we will employ the following conventions: Global coordinates are defined in the world frame $W$. Given two local reference frames $A$ and $B$, moving w.r.t $W$, the local relative position of $B$ w.r.t $A$ expressed in $B$ is denoted as $_{B}\mathbf{r}_{AB}$, and the relative orientation from $B$ to $A$ as $\mathbf{R}_{AB}$. For absolute quantities, i.e. those expressed in global coordinates, the $W$ subscript is omitted for brevity, e.g. $_{W}\mathbf{r}_{WA} \equiv \mathbf{r}_{A}$ and $\mathbf{R}_{WA} \equiv \mathbf{R}_{A}$. Thus, Cartesian coordinate transformations are expressed in the form $\mathbf{r}_{B} = \mathbf{r}_{A} + \mathbf{r}_{AB} = \mathbf{r}_{A} + \mathbf{R}_{A} \, _{A}\mathbf{r}_{AB} = \mathbf{r}_{A} + \mathbf{R}_{A} \, \mathbf{R}_{AB} \, _{B}\mathbf{r}_{AB}$.\\

\begin{figure}[!t]
\centering
\includegraphics[width=1.0\linewidth]{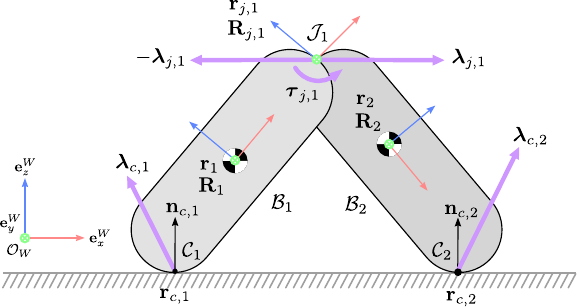}
\caption{Free-body diagram of a simple constrained system with: (1) $n_{b}=2$ rigid bodies $\mathcal{B}_{1}$ and $\mathcal{B}_{2}$, (2) $n_{j}=1$ joints $\mathcal{J}_{1}$ with revolute constraints, and (3) $n_{c}=2$ contacts $\mathcal{C}_{1}$ and $\mathcal{C}_{2}$ between the bodies and the ground.}
\label{fig:constraint-system-with-contacts}
\end{figure}

\subsection{Rigid Bodies}
\label{sec:models:bodies}
\noindent
Each rigid body $\mathcal{B}_{i}$, using a body-fixed coordinate frame located at its respective Center-of-Mass (CoM), is parameterized by the following collection of intrinsic and extrinsic properties:
\begin{itemize}
\item mass $m_{i}$
\item moment-of-inertia $_{i}\mathbf{I}_{i}$
\item surface geometry $\mathcal{G}_{\mathcal{B}, i}$
\item CoM position $\mathbf{r}_{i} \in \mathbb{R}^{3}$
\item CoM orientation $\mathbf{R}_{i} \in \mathrm{SO(3)}$
\item CoM linear velocity $\mathbf{v}_{i} \in \mathbb{R}^{3}$
\item CoM angular velocity $\boldsymbol{\omega}_{i} \in \mathbb{R}^{3}$
\end{itemize}

Each body is intrinsically characterized by its inertial and geometric properties. The former comprise its mass $m_{i}$ and moment-of-inertia $_{i}\mathbf{I}_{i}$ expressed in local body-fixed coordinates. The latter is represented by the parameterized set $\mathcal{G}_{\mathcal{B}, i}$\footnote{Parametric geometry sets can take the form of primitive geometric shapes such as boxes, spheres etc, which are parameterized by their respective constants (width, height, length and radius, respectively), or more explicit forms such as triangle-based meshes which are collections of vertices.} containing all points belonging to its exterior surface geometry. 

The body's extrinsic property is its mutable state, comprising its Cartesian pose and twist defined by the aforementioned CoM-centric properties $\mathbf{r}_{i}, \mathbf{R}_{i}, \mathbf{v}_{i},\boldsymbol{\omega}_{i}$. In addition, we may option for a more concise parameterization of the body's orientation such as quaternions. Specifically, using Hamiltonian unit quaternions $\mathbf{q}_{i} \in \mathbb{H}$, we can express the rotation matrix as $\mathbf{R}_{i} = \mathbf{R}(\mathbf{q}_{i})$. This choice does not affect the derivation of the dynamics in later sections; it only impacts the representation of configuration-dependent and configuration-level constraints. The pose and twist vectors of $\mathcal{B}_{i}$ are respectively defined as
\begin{equation}
\mathbf{s}_{i} =
\begin{bmatrix}
    \mathbf{r}_{i}\\
    \mathbf{q}_{i}
\end{bmatrix}
\in \mathbb{S}
\,,\,\,
\mathbf{u}_{i} =
\begin{bmatrix}
    \mathbf{v}_{i}\\
    \boldsymbol{\omega}_{i}
\end{bmatrix}
\in \mathbb{R}^{6} \, ,
\label{eq:models:rigid-body-state}
\end{equation}
where the compound set $\mathbb{S} := \mathbb{R}^{3} \times \mathbb{H}$ is constructed using the \textit{Cartesian product}\footnote{A Cartesian product of two sets $A$ and $B$, is defined as the set of all ordered pairs formed from their elements, i.e. $A \times B := \{ (a,b) \,|\, a \in A \,,\,\, b \in B\}$. Effectively $A \times B$ is a mechanism through which we can construct compound sets/spaces over multiple scalar fields, vector spaces etc.} of the sets of positions and quaternions. $\mathbb{S}$ in this case is merely a quaternion-based parameterization of Cartesian poses, which, are technically elements of $\mathrm{SE}(3)$.

\subsection{Joints}
\label{sec:models:joints}
\noindent
Each joint $\mathcal{J}_{j}$ introduces a set of bilateral (i.e. equality) constraints that restricts the motion of the bodies it acts upon. These can either take the form of unary constraints acting on a single body, anchoring it to the world, or binary constraints coupling the motion between a body pair. By convention, we designate base $B_j$ and follower $F_j$ coordinate frames for each $\mathcal{J}_{j}$. For unary joints $B_j \equiv W$ and $F_j$ is coincident with the body-fixed frame of the body. For binary joints, both $B_j$ and $F_j$ are coincident with the body-fixed CoM frames of the corresponding bodies. With the aforementioned conventions, a universal parameterization of joints consists of the:
\begin{itemize}
    \item relative position w.r.t the base body $_{B}\mathbf{x}_{Bj} \in \mathbb{R}^{3}$
    \item relative position w.r.t the follower body $_{F}\mathbf{x}_{Fj} \in \mathbb{R}^{3}$
    \item frame axes $\mathbf{X}_{j} \in \mathrm{SO(3)}$
    \item selection matrix $\mathbf{S}_{j} \in \mathbb{R}^{6 \times 6}$
    \item constraint dimensions $m_j \in [1,6] \subset \mathbb{N}_{+}$
    \item DoF dimensions $d_j = 6 - m_j$
    \item position $\mathbf{r}_{j} \in \mathbb{R}^{3}$
    \item coordinate frame $\mathbf{R}_{j} \in \mathrm{SO(3)}$
    \item generalized configuration $\mathbf{q}_{j} \in \mathbb{R}^{d_j}$
\end{itemize}
In a maximal-coordinate setting, the intrinsic parameters of a joint are the constraint dimensions $m_j$, the relative positions $_{B}\mathbf{x}_{Bj}, \, _{B}\mathbf{x}_{Bj}$, the frame axes $\mathbf{X}_{j}$, and the selection matrix $\mathbf{S}_{j}$. $m_j$ is effectively the number of constraint equations introduced, and together with the $\mathbf{X}_{j}$ and $\mathbf{S}_{j}$ determine the kinematic DoFs it encodes, i.e. prismatic, revolute, spherical, etc. All of the aforementioned quantities are essentially constants that fully specify its parameterization. Thus, all other quantities such as the absolute position $\mathbf{r}_{j}$, orientation $\mathbf{R}_{j}$, and generalized configuration $\mathbf{q}_{j}$ are derived from the joint's parameters and the configurations of the associated bodies.

According to its kinematic type, each joint introduces a set of $m_j$ configuration-level implicit equations in the form of
\begin{equation}
\mathbf{f}_{j}(\mathbf{s}_{B_{j}}, \mathbf{s}_{F_{j}}) = 0,
\label{eq:models:joints:configuration-level-implicit-constraints}
\end{equation}
where $\mathbf{f}_{j} : \mathbb{S} \times  \mathbb{S} \rightarrow \mathbb{R}^{m_j}$, and $\mathbf{s}_{B_{j}}, \mathbf{s}_{F_{j}}$ are the poses of the base and follower bodies, respectively. It will be necessary to also express the constraints at velocity-level in Pfaffian form as
\begin{equation}
\dot{\mathbf{f}}_{j}(\mathbf{s}_{B_{j}}, \mathbf{s}_{F_{j}}, \mathbf{u}_{B_{j}}, \mathbf{u}_{F_{j}}) = 0.
\label{eq:models:joints:velocity-level-implicit-constraints}
\end{equation}
Taking the first-order derivative w.r.t time of (\ref{eq:models:joints:configuration-level-implicit-constraints}) is referred to as \textit{index reduction} in the context of DAEs~\cite{griepentrog1991index}, and allows us to express the constraints as functions of the body twists. Moreover, (\ref{eq:models:joints:velocity-level-implicit-constraints}) also admits an interpretation that is fundamental to formulating constrained dynamics: it defines the so-called \textit{constraint-space velocities} $\mathbf{v}_{j} = \dot{\mathbf{f}}_{j}(\mathbf{s}_{B_{j}}, \mathbf{s}_{F_{j}}, \mathbf{u}_{B_{j}}, \mathbf{u}_{F_{j}})$, which are closely related to the concept of \textit{virtual displacements}~\cite{baruh1999analytical}.

The constraints are enforced by corresponding generalized forces that are represented by a vector of Lagrange multipliers $\boldsymbol{\lambda}_{j} \in \mathbb{R}^{m_j}$ that, according to the principle of d'Alembert, should conserve the energy of the system. The generalized forces acting along the DoFs of the joint are defined as the vector $\boldsymbol{\tau}_{j} \in \mathbb{R}^{d_j}$. Crucially, both $\boldsymbol{\lambda}_{j}$ and $\boldsymbol{\tau}_{j}$ are quantities defined in the local coordinates of the joint frame defined by $\mathbf{R}_{j}$. However, in order to define the dynamics of a joint, we need to express the forces and torques, i.e. \textit{wrenches}, enacted upon the bodies. This is the principle function of the selection matrix $\mathbf{S}_{j}$ and the frame axes $\mathbf{X}_{j}$. Specifically, the former maps the constraint and actuation generalized force vectors to 6D wrenches as
\begin{equation}
_{j}\mathbf{w}_{j} = 
\mathbf{S}_{j} \,
\begin{bmatrix}
    \boldsymbol{\lambda}_{j}\\
    \boldsymbol{\tau}_{j}
\end{bmatrix}
=
\begin{bmatrix}
    \mathbf{S}_{c, j} & \mathbf{S}_{a, j}
\end{bmatrix}
\,
\begin{bmatrix}
    \boldsymbol{\lambda}_{j}\\
    \boldsymbol{\tau}_{j}
\end{bmatrix}
\,\, ,
\label{eq:models:joint-body-wrenches-actuated}
\end{equation}
$\mathbf{S}_{c, j} \in \mathbb{R}^{6 \times m_j}$ and $\mathbf{S}_{a, j} \in \mathbb{R}^{6 \times d_j}$ are the component selection matrices that respectively map the generalized forces along the constraint and DoF dimensions. If the joint is passive then $\boldsymbol{\tau}_{j}$ is always zero, $\mathbf{S}_{j} \equiv \mathbf{S}_{c,j}$ and (\ref{eq:models:joint-body-wrenches-actuated}) effectively reduces to 
\begin{equation}
_{j}\mathbf{w}_{j} = 
\mathbf{S}_{j} \, \boldsymbol{\lambda}_{j}
\,\, .
\label{eq:models:joint-body-wrenches-passive}
\end{equation}
The base-follow convention defines the joint wrench as acting on the follower $F_j$ by the base $B_j$ at position $\mathbf{r}_{j}$. To compute the body-wise wrenches about their respective CoMs we must first retrieve the joint wrench, acting at $\mathbf{r}_{j}$ and expressed in world coordinates, using the joint's axes and frame matrices:
\begin{equation}
\mathbf{w}_{j} = \bar{\mathbf{R}}_{j} \, \bar{\mathbf{X}}_{j} \, _{j}\mathbf{w}_{j} 
\,,
\label{eq:models:joint-wrench-in-world}
\end{equation}
where $\bar{\mathbf{R}}_{j}$ and $\bar{\mathbf{X}}_{j}$ are the block-diagonal versions of $\mathbf{R}_{j}$ and $\mathbf{X}_{j}$ expanded to 6D. 
Finally, in order to express the wrench effected upon each body, screw transformation matrices must be applied. These are quantities that transform 6D wrenches and twists from one point and frame of application to another. In the context of this work, we only need to consider screw transformations of wrenches, so that we can express 
the effect of joint wrench $\mathbf{w}_{j}$ when acting on body $\mathcal{B}_{i}$, that may correspond to the base $B_j$ or follower $F_j$. Referring to this wrench as $\mathbf{w}_{i,j}$, and denoting the skew-symmetric operator as $\left[ \cdot \right] _{\times}$ , we can express it as
\begin{subequations}
\begin{equation}
\mathbf{w}_{i,j} = \mathbf{W}_{i,j}(\mathbf{r}_{j}, \mathbf{r}_{i}) \, \mathbf{w}_{j}
\end{equation}
\begin{equation}
\mathbf{W}_{i,j}(\mathbf{r}_{j}, \mathbf{r}_{i}) =
\begin{bmatrix}
    \mathbb{I}_{3} & \mathbf{0}_{3} \\
    \left[ \mathbf{r}_{i} - \mathbf{r}_{j} \right]_{\times}  & \mathbb{I}_{3}
\end{bmatrix}
\,\,.
\end{equation}
\label{eq:models:screw-transforms}
\end{subequations}

\subsection{Limits}
\label{sec:models:limits}
\noindent
Each joint limit $\mathcal{L}_{l}$ introduces an additional unilateral constraint that is defined by the kinematic limits of the associated joint $\mathcal{J}_{j}$, that may represent mechanical end-stops or other restrictions to the motion along the admissible DoFs. Each $\mathcal{L}_{l}$ is thus explicitly dependent on the parameterization of the corresponding $\mathcal{J}_{j}$, but in addition requires the specification of:
\begin{itemize}
\item the DoF selection vector $\mathbf{s}_{l} \in \mathbb{R}^{6}$
\item the DoF configuration $q_{l} \in \mathbb{R}$
\item the minimum DoF limit $q_{l}^{min} \in \mathbb{R}$
\item the maximum DoF limit $q_{l}^{max} \in \mathbb{R}$
\end{itemize}

Fundamentally, each joint with prescribed configuration limits defines a set implicit inequalities in the form of
\begin{equation}
\mathbf{g}_{j}(\mathbf{s}_{B_{j}}, \mathbf{s}_{F_{j}}) \geq 0
\,\,,
\label{eq:models:limits:configuration-level-implicit-constraints}
\end{equation}
where $\mathbf{g}_{j} : \mathbb{S} \times \mathbb{S} \rightarrow \mathbb{R}^{m_{j,l}}$ and $m_{j,l} \leq 2\,d_j$, because they may be imposed on either lower and/or upper orthant of each joint DoF. However, only one side may ever be active at any point in time, so the dimensionality of limit constraints effectively reduces to $m_{j,l} \leq d_j$. To understand why we must consider how limits can become active in the first place. Given the instantaneous joint DoF configuration $\mathbf{q}_{j} \in \mathbb{R}^{m_{j}}$ and lower/upper DoF limits denoted as $\mathbf{q}_{j}^{min}, \mathbf{q}_{j}^{max} \in \mathbb{R}^{m_{j}}$, we can define the respective \textit{joint-limit gap functions}
\begin{subequations}
\begin{equation}
\mathbf{g}_{j}^{min}(\mathbf{q}_{j}) = \mathbf{q}_{j} - \mathbf{q}_{j}^{min}
\end{equation}
\begin{equation}
\mathbf{g}_{j}^{max}(\mathbf{q}_{j}) = \mathbf{q}_{j}^{max} - \mathbf{q}_{j}
\,\,.
\end{equation}
\label{eq:models:limits:minmax-gap-functions}
\end{subequations}
If either $g_{j,i}^{min}(q_{j,i}) \leq 0$ or $g_{j,i}^{max}(q_{j,i}) \leq 0$, then the lower, or respectively upper,  limit of the joint has been reached and thus the corresponding limit constraint becomes active. Thus for every joint DoF coordinate $i \in [1, d_{j}]$ of $\mathcal{J}_{j}$ where a limit is active, a single limit entity $\mathcal{L}_{l}$ is defined with $q_{l} := q_{j,i}$, $q_{l}^{min} := q_{j,i}^{min}$ and $q_{l}^{max} := q_{j,i}^{max}$. Moreover, employing an index mapping in the form of $g_{l}^{min}(q_{l}) := g_{j,i}^{min}(q_{j,i})$ and $g_{l}^{max}(q_{l}) := g_{j,i}^{max}(q_{j,i})$, we can define the unilateral constraint introduced by each $\mathcal{L}_{l}$ as the implicit inequality
\begin{equation}
g_{l}(q_{l}) := 
\begin{cases}
g_{l}^{min}(q_{l}) &,\,\, g_{l}^{min}(q_{l}) \leq g_{l}^{max}(q_{l}) \\
g_{l}^{max}(q_{l}) &,\,\, g_{l}^{max}(q_{l}) \leq g_{l}^{min}(q_{l})
\end{cases}
\,\,.
\label{eq:models:limits:implicit-configuration-constraint}
\end{equation}
Note that (\ref{eq:models:limits:implicit-configuration-constraint}) does not actually express a piecewise continuous function, rather it just expresses the fact that the constraint corresponds to the gap function of only one of the lower/upper limits at any point in time. 

The corresponding constraint reaction of $\mathcal{L}_{l}$ is represented by the Lagrange multiplier $\lambda_{l} \geq 0$, and together with (\ref{eq:models:limits:implicit-configuration-constraint}), define the complementarity conditions
\begin{equation}
g_{l}(q_{l}) \geq 0
\,\, \perp \,\, 
\lambda_{l} \geq 0
\,\,.
\label{eq:limits:signorini-conditions-configuration}
\end{equation}
By slightly abusing notation, we can denote the constraint-space velocity of $\mathcal{L}_{l}$ as $\text{v}_{l}(\dot{q}_{l}) := \dot{f}_l(\mathbf{s}_{B_{j}}, \mathbf{s}_{F_{j}}, \mathbf{u}_{B_{j}}, \mathbf{u}_{F_{j}})$, which enables us to express the velocity-level formulation of (\ref{eq:limits:signorini-conditions-configuration}) as
\begin{equation}
\text{v}_{l}(\dot{q}_{l}) \geq 0
\,\, \perp \,\, 
\lambda_{l} \geq 0
\,\,.
\label{eq:limits:signorini-conditions-velocity}
\end{equation}

Lastly, we must express the wrenches enacting on the bodies associated with $\mathcal{L}_{l}$. Fortunately, this works out to be exceptionally straightforward since $\mathcal{L}_{l}$ is by definition acting along only a single DoF of the corresponding joint $\mathcal{J}_{j}$. Invoking once more the index remapping used to define $q_{l}$, the DoF selection vector of each $\mathcal{L}_{l}$ is exactly the $i$-th column of the joint's DoF selection matrix $\mathbf{S}_{a,j}$, i.e. $\mathbf{s}_{l} := (\mathbf{S}_{j,a})_{i}$. Thus the wrench applied by an active joint limit, acting at and about the joint position $\mathbf{r}_{j}$ and expressed in world coordinates, is
\begin{equation}
\mathbf{w}_{l} = \bar{\mathbf{R}}_{j} \, \bar{\mathbf{X}}_{j} \, \mathbf{s}_{l} \, \lambda_{l}
\,\,,
\label{eq:models:limit-wrench-in-world}
\end{equation}

\subsection{Contacts}
\label{sec:models:contacts}
\noindent
Each discrete contact $\mathcal{C}_{k}$, introduces a set of unilateral (i.e. inequality) constraints that act to simultaneously: (a) prevent inter-penetration between it's associated bodies when they collide, and (b) define the material interactions between them in the form of \textit{friction} and \textit{restitution}. Contact constraints are therefore more complicated than joints and limits, as they act on both the kinematics and dynamics of the system in a way that depends on the material properties of the bodies; not just their geometry and inertial properties. Similarly to joints, they may also be unary or binary, depending on whether a collision occurs between a body and the world or between bodies. Like limits, they introduce set-valued force laws, i.e. constraints on the contact reactions themselves. 

In order to specify the constraints contributed by each contact element $\mathcal{C}_{k}$, let's first define the intrinsic quantities that parameterize it. Namely, these are the:
\begin{itemize}
    \item position $\mathbf{r}_{c,k} \in \mathbb{R}^{3}$
    \item normal vector $\mathbf{n}_{c,k} \in \mathbb{R}^{3}, \,\, \Vert \mathbf{n}_{c,k} \Vert_{2} = 1$
    \item gap distance $d_{c,k} \in \mathbb{R}$
    \item coefficient of friction $\mu_k \in \mathbb{R}_{+}$
    \item coefficient of restitution $\epsilon_k \in [0, 1]$
\end{itemize}
These quantities are thus the minimal amount of information required to define a model of contacts and their corresponding constraint sets. We will now detail how we can define these constraints precisely and the physical phenomena that they model, so that we can incorporate them into our CRBD model.\\
\begin{figure}[!t]
\centering
\includegraphics[width=1.0\linewidth]{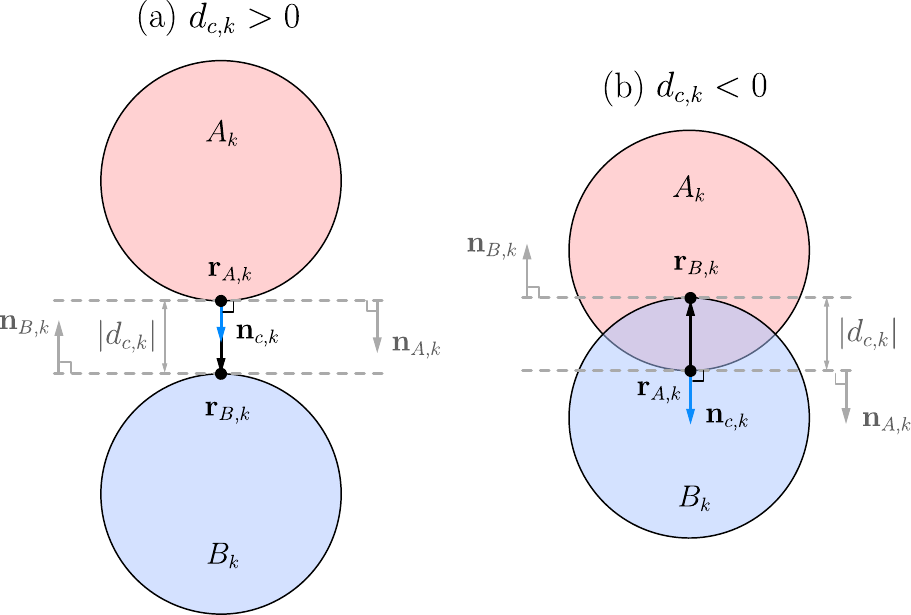}
\caption{Binary collisions between bodies $A_{k}$ and $B_{k}$ for arbitrary contact with index $k$. In all configurations, the gap-function (\ref{eq:gap-function-def}) can be used to compute the signed distance $d_{c,k}\,\mathbf{n}_{c,k}$ (black arrows) of $B_{k}$ w.r.t. $A_{k}$, where $\mathbf{n}_{c,k}$ is the normal vector (blue arrows) w.r.t the plane that is tangent to the two contacting points of the bodies. Choosing $A_{k}$ as the reference body is a merely a matter of convention. When $d_{c,k} > 0$ (case (a)), then there is no interpenetration, and the distance is interpreted as the distance between the nearest points between the bodies. Conversely, when $d_{c,k} \leq 0$ (case (b)), interpenetration occurs, and $d_{c,k}$ becomes the penetration depth measured between the most interpenetrating points.}
\label{fig:binary-collisions}
\end{figure}
\subsubsection*{\textbf{Collisions}}
\label{sec:models:contacts:collisions}
Although often used synonymously, \textit{collisions} and \textit{contacts} correspond to different aspects of the interactions between bodies. For real physical bodies, each collision defines a continuum of body contact distributed over potentially multiple surfaces. However, for simulating multi-body systems, and especially those with rigid bodies, the most straight-forward way to model such interactions is using \textit{discrete contacts}, i.e. geometrically discretizing these surfaces as collections of points. Therefore, a collision occurring between a pair of bodies, characterizes the geometry of the interaction, and is the progenitor of the set of discrete contacts that we use to represent it. If the colliding bodies have convex geometries, then a single contact can represent the collision, otherwise multiple may be required. Moreover, particularly challenging yet important, are the cases of multiple co-planar contacts, e.g. a box on a plane, as the geometry is indeed convex, yet multiple contacts are required to represent the interaction. Within this context, we will describe how collisions between bodies define the geometric properties of contacts.

A crucial aspect of defining the contact position $\mathbf{r}_{c,k}$ is that it does not necessarily coincide with the exact position on each body where the collision occurs. Consider the case depicted in Fig.\ref{fig:binary-collisions}b, where the two (convex) bodies $A_{k}$ and $B_{k}$ are interpenetrating. If we define the contact normal to lie along the line connecting the inner-most penetrating points on the bodies, then the contact should be defined using both points. Instead, we arbitrarily designate one of these as \textit{the contact position} $\mathbf{r}_{c,k} := \mathbf{r}_{A,k}$ and determine the location of the other using the surface normal $\mathbf{n}_{c,k}$ and penetration depth $d_{c,k}$ as $\mathbf{r}_{B,k} = \mathbf{r}_{c,k} + d_{c,k} \, \mathbf{n}_{c,k}$. The relative position between bodies $A_{k}$ and $B_{k}$ is used to define the so-called \textit{gap function}
\begin{equation}
\mathbf{g}_{k}(\mathbf{r}_{A,k}, \mathbf{r}_{B,k}) 
= \mathbf{r}_{B,k} - \mathbf{r}_{A,k}
= d_{c,k} \, \mathbf{n}_{c,k}
\label{eq:gap-function-def}
\end{equation}
of each contact, where $\mathbf{g}_{k} : \mathbb{R}^{3} \times \mathbb{R}^{3} \rightarrow \mathbb{R}^{3}$. Essentially, the relative position expressed via (\ref{eq:gap-function-def}) corresponds to a vector-valued signed-distance function, and defines the configuration-level constraint function of $\mathcal{C}_{k}$.\\

\subsubsection*{\textbf{Contact Modes}}
\label{sec:models:contacts:modes}
Noting how the penetration depth $d_{c,k}$ can take zero, positive and negative values, we can use this to define the so-called \textit{contact mode} $m_k$, a discrete valued scalar that summarizes the state of $\mathcal{C}_{k}$. When $d_{c,k} > 0$, we define the contact to be \textit{open} ($m_k = \text{Open}$) and $d_{c,k}$ is interpreted to be the minimum Euclidean distance between the two bodies. Conversely, when $d_{c,k} \leq 0$ the contact is defined as \textit{closed} ($m_k = \text{Stick} \lor \text{Slip}$) and $d_{c,k}$ becomes the penetration depth. Therefore, for convex geometries, $\mathbf{r}_{A,k}, \mathbf{r}_{B,k}$ are either the nearest points between the two bodies or the most-penetrating, respectively, and are always functions of the corresponding body poses $\mathbf{s}_{A,k}$ $\mathbf{s}_{B,k}$. When the geometries are not convex, or multiple contacts per body pair are admissible, the situation is slightly more complicated\footnote{However, an easy way to work around this is by realizing that contacts are defined only once they occur (i.e. become closed) in the first-place, so then the consideration of $\mathbf{r}_{A,k}, \mathbf{r}_{B,k}$, only matters locally while the contact persists and until they open.}.\\

\subsubsection*{\textbf{Contact Space}}
\label{sec:models:contacts:assumptions}
In this work we only consider isotropic 3D Coulomb-type friction, where the contact constraint reaction and respective constraint-space velocity are respectively 3D vectors, i.e. $\boldsymbol{\lambda}_{k}$, $\mathbf{v}_{c,k} \in \mathbb{R}^{3}$. Although more advanced models of friction exist, such as those that can also incorporate torsional effects~\cite{leine2003, mujoco2024docs}, we consider the simpler case a sufficient baseline for the purposes of this evaluation.\\

\subsubsection*{\textbf{Contact Frames}}
\label{sec:models:contacts:frames}
The space in which all contact-related quantities can most easily be represented is that of a local contact-specific coordinate frame. Employing such a frame allows for the re-parameterization of the contact reaction $\boldsymbol{\lambda}_{k}$ into a form that delineates between the normal $N$ and tangent $T$ components. By applying certain conventions, the normal vector $\mathbf{n}_{c,k}$, can be used to represent the corresponding tangent plane as the tangent and orthogonal (i.e. bi-normal) vectors $\mathbf{t}_{c,k},\mathbf{o}_{c,k} \in \mathbb{R}^{3}$, respectively. The triplet $\mathbf{n}_{c,k}, \mathbf{t}_{c,k},\mathbf{o}_{c,k}$ thus enables the construction of a rotation matrix $\mathbf{R}_{k} \in \mathrm{SO(3)}$ to represent the coordinate frame of $\mathcal{C}_{k}$. This construction is expressed by the contact-frame function $\mathbf{C} : \mathbb{R}^{3} \rightarrow \mathrm{SO(3)}$ as
\begin{equation}
\mathbf{R}_{k} 
= 
\mathbf{C}(\mathbf{n}_{c,k})
=:
\begin{bmatrix}
\mathbf{t}_{c,k} & \mathbf{o}_{c,k} & \mathbf{n}_{c,k}
\end{bmatrix}
\,,
\end{equation}
In this work we place the normal along the contact-local Z-axis, and the XY-plane as the tangent space. This results in parameterizing the contact reaction in local coordinates as
\begin{equation}
\boldsymbol{\lambda}_{k} := 
\begin{bmatrix}
    \boldsymbol{\lambda}_{T,k}\\
    \lambda_{N,k}
\end{bmatrix}
\, , \,\,
\boldsymbol{\lambda}_{T, k} := 
\begin{bmatrix}
    \lambda_{t,k}\\
    \lambda_{o,k}
\end{bmatrix}\, ,
\end{equation}
where $\lambda_{N,k}$, $\lambda_{t,k}$, and $\lambda_{o,k}$ are respectively the normal, tangent, and orthogonal components. Thus $\boldsymbol{\lambda}_{k}$, and its normal-tangent decomposition, are expressed in global coordinates as 
\begin{subequations}
\begin{equation}
_{W}\boldsymbol{\lambda}_{k} := \mathbf{R}_{k} \, \boldsymbol{\lambda}_{k}
\label{eq:models:contacts:contact-reaction-in-world}
\end{equation}
\begin{equation}
_{W}\boldsymbol{\lambda}_{T, k} :=  \mathbf{R}_{k} \,
\begin{bmatrix}
    \boldsymbol{\lambda}_{T, k} \\
    0
\end{bmatrix}
\end{equation} 
\begin{equation}
_{W}\boldsymbol{\lambda}_{N, k} :=  \lambda_{N,k} \, \mathbf{n}_{c,k} \, .
\end{equation}
\end{subequations}

\subsubsection*{\textbf{Contact Wrench}}
\label{sec:models:contacts:wrench}
Similarly to joints, the wrench applied to contacting bodies is also computed using the screw transformations described in (\ref{eq:models:screw-transforms}). However, given the aforementioned assumption that contact reactions contain only linear 3D terms, only the left block-column of (\ref{eq:models:screw-transforms}b) is needed. Thus, the contact wrench applied to body $\mathcal{B}_{i}$ by $\mathcal{C}_{k}$ is therefore
\begin{subequations}
\begin{equation}
\mathbf{w}_{c,k,i} = \mathbf{W}_{c,k,i}(\mathbf{r}_{k}, \mathbf{r}_{i}) \, _{W}\boldsymbol{\lambda}_{k}
\end{equation}
\begin{equation}
\mathbf{W}_{c,k,i}(\mathbf{r}_{k}, \mathbf{r}_{i}) =
\begin{bmatrix}
    \mathbb{I}_{3} \\
    \left[ \mathbf{r}_{i} - \mathbf{r}_{k} \right]_{\times} 
\end{bmatrix}
\end{equation}
\label{eq:contact-wrench-transforms}
\end{subequations}

\subsubsection*{\textbf{Contact Motion}}
\label{sec:models:contacts:velocities}
The duality of twists and wrenches allows us to also express the linear velocity of the point of contact on $\mathcal{B}_{i}$ as an affine function of the respective body twist in the form of
\begin{equation}
_{W}\mathbf{v}_{c,k,i} = \mathbf{W}_{c,k,i}(\mathbf{r}_{k}, \mathbf{r}_{i})^{T} \, \mathbf{u}_{i}
\label{eq:body-contact-velocities}
\end{equation}
However, the quantity we are actually interested in is the relative velocity computed at the point of contact $\mathbf{r}_{k}$, between the contacting bodies $A_{k}$ and $B_{k}$. Referring to this simply as \textit{the contact velocity}, (\ref{eq:body-contact-velocities})\footnote{An alternative derivation can be found, if instead, we form the first-order time-derivative of the gap function (\ref{eq:gap-function-def}).} is used to define it precisely using the relations
\begin{subequations}
\begin{equation}
_{W}\mathbf{v}_{c,k,A} = \mathbf{W}_{c,k,A}(\mathbf{r}_{k}, \mathbf{r}_{A,k})^{T} \, \mathbf{u}_{A,k}
\end{equation}
\begin{equation}
_{W}\mathbf{v}_{c,k,B} = \mathbf{W}_{c,k,B}(\mathbf{r}_{k}, \mathbf{r}_{B,k})^{T} \, \mathbf{u}_{B,k}
\end{equation}
\begin{equation}
_{W}\mathbf{v}_{c,k} = _{W}\mathbf{v}_{c,k,A} - _{W}\mathbf{v}_{c,k,B}
\end{equation}
\begin{equation}
\mathbf{v}_{c,k} = \mathbf{R}_{k}^{T} \, _{W}\mathbf{v}_{c,k}
\,\, .
\end{equation}
\label{eq:relative-contact-velocities}
\end{subequations}

\subsubsection*{\textbf{Signorini Conditions}}
\label{sec:models:contacts:singorini}
Also known as the \textit{unilateral contact hypothesis}, the Signorini conditions assert that contact reactions preventing interpenetration between bodies must act only while the contact is closed, i.e. when the gap constraint is active. Using the definition of the gap function from (\ref{eq:gap-function-def}), this statement is formalized as the complementarity problem 
\begin{subequations}
\begin{equation}
g_{c,N,k} := \mathbf{n}_{c,k}^{T} \, \mathbf{g}_{k}(\mathbf{r}_{A,k}, \mathbf{r}_{B,k}) \geq 0
\label{eq:contacts:gap-constraint}
\end{equation}
\begin{equation}
g_{c,N,k} \geq 0
\,\, \perp \,\, 
\lambda_{N,k} \geq 0
\,\,.
\end{equation}
\label{eq:contacts:signorini-conditions-position}
\end{subequations}
However, in order to introduce these complementarity conditions to the formulation of the system dynamics, it is often necessary to state them at the velocity and acceleration level, \cite{chatterjee1999,glocker2001setvalued}. Denoting the first and second derivative of (\ref{eq:contacts:gap-constraint}) w.r.t time as $\text{v}_{c,N, k} := \text{v}_{c,N, k}(\mathbf{s}_{A,k}, \mathbf{s}_{B,k}, \mathbf{u}_{A,k}, \mathbf{u}_{B,k})$, and $a_{c,N, k} := a_{c,N, k}(\mathbf{s}_{A,k}, \mathbf{s}_{B,k}, \mathbf{u}_{A,k}, \mathbf{u}_{B,k}, \dot{\mathbf{u}}_{A,k}, \dot{\mathbf{u}}_{B,k})$, they become
\begin{subequations}
\begin{equation}
\text{v}_{c,N,k} \geq 0
\,\, \perp \,\, 
\lambda_{N,k} \geq 0
\,\,,
\label{eq:contacts:signorini-conditions-velocity}
\end{equation}
\begin{equation}
a_{c,N,k} \geq 0
\,\, \perp \,\, 
\lambda_{N,k} \geq 0
\,\,.
\label{eq:contacts:signorini-conditions-acceleration}
\end{equation}
\end{subequations}

\begin{figure}[t]
\centering
\includegraphics[width=0.9\linewidth]{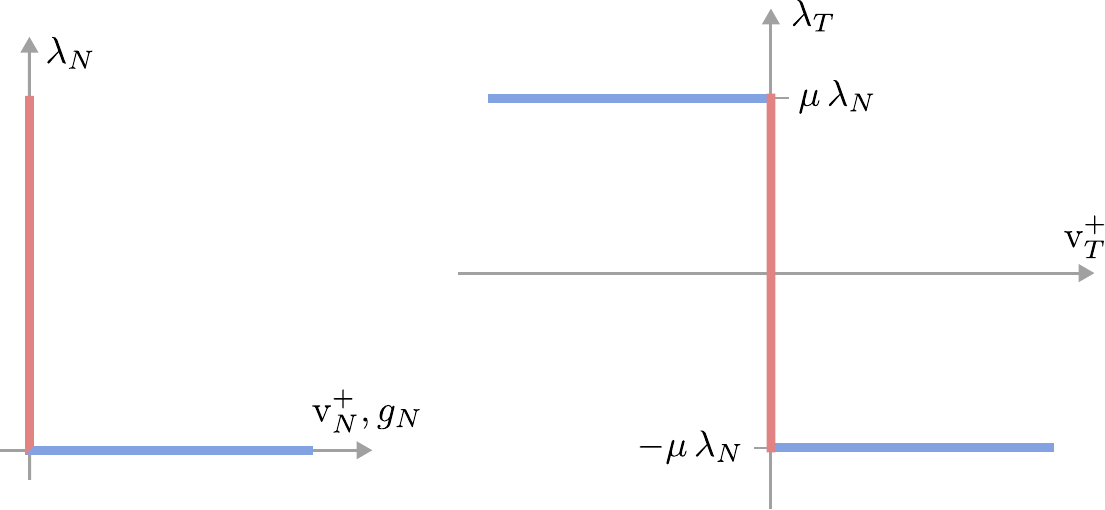}
\caption{The set-valued force laws along the normal and tangent directions. The red segments represent the sets of admissible values that the contact reactions $\lambda_{N}, \lambda_{T}$ can take while the respective contact velocities are zero. The blue segments represent the admissible constant values when non-zero velocities are present in the respective directions.}
\label{fig:set-valued-force-laws}
\end{figure}
\begin{figure}[t]
\centering
\includegraphics[width=1.0\linewidth]{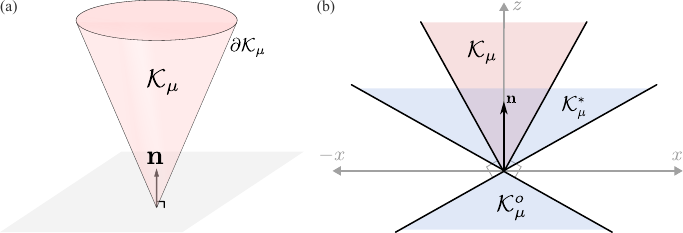}
\caption{The Coulomb friction cone $\mathcal{K}_{\mu}$ (a) and the relationships between the former to its dual cone $\mathcal{K}_{\mu}^{*}$ and polar cone $\mathcal{K}_{\mu}^{o}$ (b). As conjugates of $\mathcal{K}_{\mu}$, the dual and polar cones are those whose elements make obtuse angles with each $\boldsymbol{\lambda} \in \mathcal{K}_{\mu}$, i.e. $\mathbf{x} \in \mathcal{K}_{\mu}^{*} \,,\,\, \mathbf{x}^{T}\,\boldsymbol{\lambda} \geq 0$, and $\mathbf{x} \in \mathcal{K}_{\mu}^{o} \,,\,\, \mathbf{x}^{T}\,\boldsymbol{\lambda} \leq 0$.}
\label{fig:friction-cones}
\end{figure}
\subsubsection*{\textbf{Coulomb Friction}}
\label{sec:models:contacts:coulomb}
Under the assumption of isotropic dry friction, both the static and kinetic frictional reactions are determined by a single friction coefficient $\mu_{k}$. Thus, the two fundamental assumptions of the Coulomb friction model are: (a) static friction equaling $\mu_{k} \, \lambda_{N,k}$ must be overcome in order for motion to occur, and (b) kinetic friction takes the form of a tangential reaction that opposes non-zero contact velocity i.e.
\begin{equation}
\boldsymbol{\lambda}_{T,k} = 
- \mu_{k} \, \lambda_{N,k} \, \frac{\mathbf{v}_{c,T,k}}{\Vert \mathbf{v}_{c,T,k} \Vert}
\,\, , \,\, \Vert \mathbf{v}_{c,T,k} \Vert > 0
\,\,.
\label{eq:contacts:coulomb-sliding-friction}
\end{equation}
These two conditions combined define the set inclusion
\begin{equation} \label{eq:contacts:tangential-friction-disk-inclusion}
\boldsymbol{\lambda}_{T,k} \in \mathcal{D}(\mu_{k}\,\lambda_{N,k})
\,\,,
\end{equation}
where $\mathcal{D}(\mu_{k}\,\lambda_{N,k})$ is a disk of radius $\mu_{k}\,\lambda_{N,k}$. Moreover, the requirement that contact reaction along the normal direction is only repulsive can be written as the set inclusion
\begin{equation} \label{eq:contacts:nonnegative-orthant-normal-inclusion}
\lambda_{N,k} \in \mathbb{R}_{+}
\,\,.
\end{equation}
Fig.~\ref{fig:set-valued-force-laws} depicts these \textit{set-valued force laws} as non-linear functions w.r.t the gap function $g_{c,N,k}$ and contact velocity $\mathbf{v}_{c,k}$. Combining (\ref{eq:contacts:tangential-friction-disk-inclusion}) and (\ref{eq:contacts:nonnegative-orthant-normal-inclusion}) forms the Coulomb friction cone 
\begin{subequations}
\begin{equation}
\boldsymbol{\lambda}_{k} \in \mathcal{K}_{\mu_k} \, ,
\end{equation}
\begin{equation}
\mathcal{K}_{\mu_k} 
:=
\{ \boldsymbol{\lambda} \in \mathbb{R}^{3} 
\, | \, 
\Vert \boldsymbol{\lambda}_{T, k} \Vert_{2} \leq \mu_{k} \, \lambda_{N,k} \,,\,\, \lambda_{N,k} \geq 0 \}
\,.
\end{equation}
\label{eq:contacts:coulomb-friction-cone}
\end{subequations}
The 3D Coulomb friction cone $\mathcal{K}_{\mu_k}$ is a symmetric second-order cone often referred to as a Lorentz cone~\cite{boyd2004convex}, whose aperture is determined by the respective friction coefficient $\mu_{k}$, i.e. $\theta_{k} = 2\,tan^{-1}(\mu_k)$. $\mathcal{K}_{\mu_k}$ is a closed and convex set that bounds the corresponding 3D contact reaction to lie both within and and on its surface boundary $\partial \mathcal{K}_{\mu_k}$.

An important property of Lorentz cones is that they are self dual, i.e. they are equal to their conjugates. This property allows us to also define the \textit{dual cone} of $\mathcal{K}_{\mu}$ as $\mathcal{K}_{\mu_k}^{*} := \mathcal{K}_{\mu_{k}^{-1}}$ and the \textit{polar cone} as $\mathcal{K}_{\mu_k}^{o} := -\mathcal{K}_{\mu_k}^{*} = -\mathcal{K}_{\mu_{k}^{-1}}$. Their geometric relationships are depicted in Fig.~\ref{fig:friction-cones}. As $\mathcal{K}_{\mu}$ is defined in the space of contact reactions, the dual and polar cones are defined in the dual space of the contact velocities.\\

\subsubsection*{\textbf{The Disjunctive Model}}
\label{sec:models:contacts:disjunctive-model}
The Coulomb friction cone model can be combined with the Signorini conditions to define the contact mode $m_k$ in terms of the contact reaction $\boldsymbol{\lambda}_{k}$ and velocity $\mathbf{v}_{c,k}$. This results in the so-called \textit{disjunctive} formulation of the Signorini-Coulomb model:
\begin{equation}
m_{k} =
\begin{cases}
\text{Open} &, \,  v_{c,N,k} > 0  \, \land \, \boldsymbol{\lambda}_{k} = 0 \\
\text{Stick}&, \, \left\Vert \mathbf{v}_{c,T,k} \right\Vert_{2} = 0 \, \land \, \boldsymbol{\lambda}_{k} \in \mathcal{K}_{\mu_k} \\
\text{Slip} &, \, \left\Vert \mathbf{v}_{c,T,k} \right\Vert_{2} > 0 \, \land \, \boldsymbol{\lambda}_{k} \in \partial \mathcal{K}_{\mu_k}
\end{cases}
\,\, .
\label{eq:contacts:disjunctive-signorini-coulomb}
\end{equation}
The first case is trivial, and is equivalent to using $g_{c,N,k} > 0$. The second case corresponds to the friction-cone constraint to when the contact force lies strictly in the interior of the cone. The third corresponds to the case where the contact force lies solely on the boundary $\partial \mathcal{K}_{\mu_k}$ of the cone.\\

\subsubsection*{\textbf{Maximal Dissipation}}
\label{sec:models:contacts:mdp}
The Maximum Dissipation Principle (MDP), most often attributed to J.J. Moreau~\cite{moreau1977application}, asserts that for each active contact $\mathcal{C}_{k}$, the tangential contact reaction must be maximally dissipative, i.e. it should maximize mechanical power loss along the tangent plane. For given normal reaction $\lambda_{N,k} \in \mathbb{R}_{+}$ and tangential velocity $\mathbf{v}_{c,T,k} \in \mathbb{R}^{2}$, the tangential contact reaction according to the MDP is the solution to
\begin{equation}
\begin{array}{rlclcl}
\boldsymbol{\lambda}_{T,k} = 
{\displaystyle
\operatorname*{argmax}_{\mathbf{x}}} 
              & - \mathbf{x}^{T} \, \mathbf{v}_{c,T,k} \\[4pt]
\textrm{s.t.} & \Vert \mathbf{x} \Vert_{2}  \leq \mu_{k} \, \lambda_{N,k}
\end{array} \,\, .
\label{eq:contacts:maximum-dissipation-principle}
\end{equation}
The MDP is a formulation of frictional contact as an optimization problem that generalizes (\ref{eq:contacts:coulomb-sliding-friction}) as it is also applicable to the case of multiple contacts. It is straight-forward to verify that in the single-contact case, (\ref{eq:contacts:coulomb-sliding-friction}) is in fact maximally dissipative. The power product would not obtain its maximum value unless $\boldsymbol{\lambda}_{T,k}$ is colinear and opposite to $\mathbf{v}_{c,T,k}$. Moreover, (\ref{eq:contacts:maximum-dissipation-principle}) is admissible even when $\lambda_{N,k} = 0$ and/or $\mathbf{v}_{c,T,k} = 0$, thus encompassing all contact modes. Furthermore, it motivates the bi-potential method \cite{desaxce1998bipotential} which can be used derive the general NCP. Since (\ref{eq:contacts:coulomb-sliding-friction}) is maximally dissipative, we can form the power-loss product to yield
\begin{equation}
\boldsymbol{\lambda}_{T,k}^{T} \, \mathbf{v}_{c,T,k}
=
- \mu_{k} \, \lambda_{N,k} \, \Vert \mathbf{v}_{c,T,k} \Vert_{2}
\,\,.
\end{equation}
By moving all terms to the left-hand side, and trivially adding the product of normal components, we can form the expression
\begin{equation}
\lambda_{N,k} \, v_{c,N,k} +  \boldsymbol{\lambda}_{T,k}^{T} \, \mathbf{v}_{c,T,k}
+ \mu_{k} \, \lambda_{N,k} \, \Vert \mathbf{v}_{c,T,k} \Vert_{2}
= 0
\,\,,
\label{eq:contacts:coulomb-power-product-result}
\end{equation}
which like the MDP, can be verified to hold in all contact modes. By introducing the \textit{De Saxc\'e correction operator}
\begin{equation}
\boldsymbol{\Gamma}_{\mu_k}(\mathbf{v}_{c,k}) :=
\begin{bmatrix}
0\\
0\\
\mu_{k} \, \Vert \mathbf{v}_{c,T,k} \Vert_{2}
\end{bmatrix}
\,\,,
\label{eq:models:de-saxce-correction}
\end{equation}
and defining the \textit{augmented contact velocity}
\begin{equation}
\hat{\mathbf{v}}_{c,k} 
:= 
\mathbf{v}_{c,k} + \boldsymbol{\Gamma}_{\mu_k}(\mathbf{v}_{c,T,k})
\,\,,
\end{equation}
(\ref{eq:contacts:coulomb-power-product-result}) can be restated as the complementarity condition
\begin{equation}
{\hat{\mathbf{v}}_{c,k}} \in \mathcal{K}_{\mu_{k}}^{*}
\perp 
{\boldsymbol{\lambda}_{k}} \in \mathcal{K}_{\mu_{k}}
\,\,,
\label{eq:contacts:coulomb-signorini-ncp}
\end{equation}
where the complementarity now requires that ${\hat{\mathbf{v}}_{c,k}}$ lie within the dual friction cone $\mathcal{K}_{\mu_{k}}^{*}$. This result is of paramount importance because it demonstrates that (\ref{eq:contacts:coulomb-signorini-ncp}) must hold in order for ${\boldsymbol{\lambda}_{k}}$ to be maximally dissipative, as well as to ensure that the resulting contact velocity lie solely on the tangential plane when the contact is closed. The main difficulty with this formulation is that it makes the complementarity condition \textit{non-associative}, since it is not stated directly in the contact velocities and reactions.\\

\subsubsection*{\textbf{Newtonian Impacts}}
\label{sec:models:contacts:impacts}
To model restitution effects, \textit{Newton's model of impacts} can be incorporated in the velocity-level Signorini conditions (\ref{eq:contacts:signorini-conditions-velocity}). Considering the impact event to occur instantaneously at time $t_{i}$, the twists of bodies $A_{k}$ and $B_{k}$ experience a discontinuity that is denoted by their pre/post-event values $\mathbf{u}_{A,k}^{-/+}$ and $\mathbf{u}_{B,k}^{-/+}$, respectively. According to Newton's model, the pre/post-event contact velocities $\mathbf{v}_{c,N,k}^{-} := \mathbf{v}_{c,N,k}(\mathbf{u}_{A,k}^{-},\mathbf{u}_{B,k}^{-})$ and $\mathbf{v}_{c,N,k}^{+} := \mathbf{v}_{c,N,k}(\mathbf{u}_{A,k}^{+},\mathbf{u}_{B,k}^{+})$ in the direction of the contact normal are coupled by the \textit{restitution equations}
\begin{equation}
\mathbf{v}_{c,N,k}^{+} = - \epsilon_{k} \,\, \mathbf{v}_{c,N,k}^{-}
\,\, ,
\label{eq:contacts:newton-impact-model}
\end{equation}
with $\epsilon_{k}$ being the coefficient of restitution. In effect, the restitution equations (\ref{eq:contacts:newton-impact-model}) serve to augment the Signorini conditions by biasing the contact velocity along the normal direction. Thus the impact-augmented Newton-Signorini model becomes
\begin{equation}
0 \leq {\mathbf{v}_{c,N,k}^{+}} + \epsilon_{k} \,\, \mathbf{v}_{c,N,k}^{-}
\,\, \perp \,\, 
{\boldsymbol{\lambda}_{k}} \in \mathcal{K}_{\mu_{k}}
\,\, .
\label{eq:newton-ncp}
\end{equation}
Lastly, combining (\ref{eq:newton-ncp}) with (\ref{eq:contacts:coulomb-signorini-ncp}), we can form the impact-augmented Newton-Coulomb-Signorini NCP in the form of
\begin{equation}
{\hat{\mathbf{v}}_{c,k}} + \epsilon_{k} \,\, \mathbf{v}_{c,N,k}^{-} \in \mathcal{K}_{\mu_{k}}^{*} 
\,\, \perp \,\, 
{\boldsymbol{\lambda}_{k}} \in \mathcal{K}_{\mu_{k}}
\,\, .
\label{eq:coulomb-newton-signorini-ncp}
\end{equation}

\subsection{Constrained Rigid-Body Dynamics}
\label{sec:constrained-rigid-body-dynamics}
\noindent
With the elements described in the previous sections in hand, we can now derive Equations-of-Motion (EoM) of the CRBD formulation in maximal-coordinates. We begin by formulating, for each body $\mathcal{B}_{i}$, the \textit{Newton-Euler} equations
\begin{equation}
\mathbf{M}_{i} \, \dot{\mathbf{u}}_{i} 
+ 
\begin{bmatrix}
    - m_i \, \mathbf{g} \\
    \left[ \boldsymbol{\omega}_{i} \right] _{\times} \, \mathbf{I}_{i}  \, \boldsymbol{\omega}_{i} 
\end{bmatrix}
= \mathbf{w}_{total, i}
\,\, ,
\label{eq:body-eom}
\end{equation}
where $\mathbf{g}$ is the gravity acceleration vector, and $\mathbf{M}_{i} \in \mathbb{R}^{6 \times 6}$ is the body-wise generalized mass matrix, which is defined as
\begin{equation}
\mathbf{M}_{i} = 
\begin{bmatrix}
    m_i \, \mathbb{I}_{3} & \mathbf{0} \\
    \mathbf{0} & \mathbf{R}_{i} \, _{i}\mathbf{I}_{i} \, \mathbf{R}_{i}^{T}
\end{bmatrix}
\, ,
\end{equation}
and $\mathbf{I}_{i}$ is the body's MoI matrix expressed in world coordinates, which can be expressed using the body's orientation and the local MoI using the \textit{parallel axis theorem} as
\begin{equation}
    \mathbf{I}_{i} = \mathbf{R}_{i} \, _{i}\mathbf{I}_{i} \, \mathbf{R}_{i}^{T} \,\, .
\end{equation}
To further simplify the expression for (\ref{eq:body-eom}), we can denote the non-linear terms of gravity and Coriolis wrenches as 
\begin{equation}
\mathbf{w}_{gc,i} :=
\begin{bmatrix}
    m_i \, \mathbf{g} \\
    - \left[ \boldsymbol{\omega}_{i} \right] _{\times} \, \mathbf{I}_{i}  \, \boldsymbol{\omega}_{i} 
\end{bmatrix}
\,\, ,
\label{eq:nonlinear-eom-wrench}
\end{equation}
and arrive at a concise expression of the body-wise EoM
\begin{equation}
\mathbf{M}_{i} \, \dot{\mathbf{u}}_{i} = \mathbf{w}_{gc,i} + \mathbf{w}_{total, i}
\,\, .
\label{eq:body-eom-spatial}
\end{equation}
Stacking (\ref{eq:body-eom-spatial}) of each body, we can form the system EoM
\begin{equation}
\mathbf{M} \, \dot{\mathbf{u}} = \mathbf{w}_{gc} + \mathbf{w}_{total}
\,\, ,
\label{eq:system-eom-spatial}
\end{equation}
where $\mathbf{M}$ is the block-diagonal generalized mass matrix first introduced in Sec.~\ref{sec:constrained-rigidbody-mechanics}, and $\mathbf{w}_{gc}$ and $\mathbf{w}_{total}$ are the stacked gravity-Coriolis and total wrenches, respectively. 

In order to introduce constraints explicitly, we must decompose $\mathbf{w}_{total}$ into the individual terms which account for all forces applied to the bodies and that are not part of $\mathbf{w}_{gc}$:
\begin{equation}
\mathbf{w}_{total} = \mathbf{w}_{e} + \mathbf{w}_{a} + \mathbf{w}_{J} + \mathbf{w}_{L} + \mathbf{w}_{C}\\
\label{eq:system-total-wrench-components}
\end{equation}
We have also included the external wrenches $\mathbf{w}_{e}$ to account for purely external factors such as external disturbances and forces we may wish to apply ad-hoc. Next we will express $\mathbf{w}_{a}$, $\mathbf{w}_{J}$, $\mathbf{w}_{L}$ and $\mathbf{w}_{C}$ as functions of the respective constraint reactions and generalized actuation forces. To this end, we can apply the screw transforms (\ref{eq:models:screw-transforms}) to (\ref{eq:models:joint-wrench-in-world}) and (\ref{eq:models:limit-wrench-in-world}) combined with (\ref{eq:models:joint-body-wrenches-actuated}), as well as apply  (\ref{eq:models:contacts:contact-reaction-in-world}) to (\ref{eq:contact-wrench-transforms}). Essentially, doing so renders the Jacobian matrices $\mathbf{J}_{a}$, $\mathbf{J}_{J}$, $\mathbf{J}_{L}$, and $\mathbf{J}_{C}$, corresponding to actuated joint DoFs, joints, limits, and contacts, respectively. 

Proceeding element-wise over each $\mathcal{J}_{j}$, $\mathcal{L}_{l}$ , and $\mathcal{C}_{k}$ present in the system, we can define the matrix-blocks of the corresponding Jacobian for each associated body. For each joint $\mathcal{J}_{j}$ attached to body $\mathcal{B}_{i}$, the corresponding matrix-block of $\mathbf{J}_{J}$ is
\begin{equation}
\mathbf{J}_{J,i,j}^{T} = 
\begin{bmatrix}
    \mathbb{I}_{3} & \mathbf{0}_{3} \\
    \left[ \mathbf{r}_{i} - \mathbf{r}_{j} \right]_{\times}  & \mathbb{I}_{3} \, ,
\end{bmatrix} \,
\bar{\mathbf{R}}_{j} \, \bar{\mathbf{X}}_{j} \, \mathbf{S}_{c, j}
\,\, ,
\label{eq:body-joint-passive-Jacobian}
\end{equation}
and if the joint is actuated, the matrix-block of $\mathbf{J}_{a}$ is
\begin{equation}
\mathbf{J}_{a,i,j}^{T} = 
\begin{bmatrix}
    \mathbb{I}_{3} & \mathbf{0}_{3} \\
    \left[ \mathbf{r}_{i} - \mathbf{r}_{j} \right]_{\times}  & \mathbb{I}_{3} \, ,
\end{bmatrix} \,
\bar{\mathbf{R}}_{j} \, \bar{\mathbf{X}}_{j} \, \mathbf{S}_{a, j}
\,\, .
\label{eq:body-joint-actuator-Jacobian}
\end{equation}
 For each limit $\mathcal{L}_{l}$ the matrix-block of $\mathbf{J}_{L}$ is
\begin{equation}
\mathbf{J}_{L,i,j,l}^{T} = 
\begin{bmatrix}
    \mathbb{I}_{3} & \mathbf{0}_{3} \\
    \left[ \mathbf{r}_{i} - \mathbf{r}_{j} \right]_{\times}  & \mathbb{I}_{3} \, .
\end{bmatrix} \,
\bar{\mathbf{R}}_{j} \, \bar{\mathbf{X}}_{j} \, \mathbf{s}_{l}
\,\, ,
\label{eq:joint-limit-Jacobian}
\end{equation}
and for each contact $\mathcal{C}_{k}$  the matrix-block of $\mathbf{J}_{C}$ is
\begin{equation}
\mathbf{J}_{c,i,k}^{T} = 
\begin{bmatrix}
    \mathbb{I}_{3} \\
    \left[ \mathbf{r}_{i} - \mathbf{r}_{k} \right]_{\times} 
\end{bmatrix} \,
\mathbf{R}_{k}
\,\, .
\label{eq:body-contact-Jacobian}
\end{equation}
We will also denote the total constraint Jacobian as
\begin{equation}
\mathbf{J} \, := \, 
\begin{bmatrix}
    \mathbf{J}_{J}^{T} & \mathbf{J}_{L}^{T} & \mathbf{J}_{C}^{T} 
\end{bmatrix}^{T}
\,\, ,
\label{eq:total-constraint-Jacobian}
\end{equation}

Next we will define stacked vectors over all joint DoF actuation forces and constraint reaction Lagrange multipliers corresponding to each constraint group, in the form of
\begin{align}
\boldsymbol{\tau}_{a} :=
\begin{bmatrix}
\boldsymbol{\tau}_{a, 1}\\
\vdots\\
\boldsymbol{\tau}_{a, n_j}
\end{bmatrix}
\,\,,\,\,\,\,
\boldsymbol{\lambda}_{J} :=
\begin{bmatrix}
\boldsymbol{\lambda}_{j, 1}\\
\vdots\\
\boldsymbol{\lambda}_{j, n_j}
\end{bmatrix}
\,\,,\,\,\,\, \nonumber \\[4pt]
\boldsymbol{\lambda}_{L} :=
\begin{bmatrix}
\boldsymbol{\lambda}_{l, 1}\\
\vdots\\
\boldsymbol{\lambda}_{l, n_l}
\end{bmatrix}
\,\,,\,\,\,\,
\boldsymbol{\lambda}_{C} :=
\begin{bmatrix}
\boldsymbol{\lambda}_{c, 1}\\
\vdots\\
\boldsymbol{\lambda}_{c, n_c}
\end{bmatrix}
\,\,,\,\,
\label{eq:system-total-constraint-generalized-forces}
\end{align}
as well as the vector of all constraint reaction multipliers
\begin{equation}
\boldsymbol{\lambda} \, := \,
\begin{bmatrix}
    \boldsymbol{\lambda}_{J}^{T} &
    \boldsymbol{\lambda}_{L}^{T} &
    \boldsymbol{\lambda}_{C}^{T}
\end{bmatrix}^{T}
\,\, .
\label{eq:total-constraint-multipliers}
\end{equation}

Having defined the Jacobian matrices as well as the DoF actuation and constraint reaction vectors, we can now concisely express the total wrenches applied to the system of bodies as 
\begin{flalign}
\mathbf{w}_{total} &= 
\mathbf{w}_{e}
+ \mathbf{J}_{a}^{T} \, \boldsymbol{\tau}_{a} 
+ \mathbf{J}_{J}^{T} \, \boldsymbol{\lambda}_{J} 
+ \mathbf{J}_{L}^{T} \, \boldsymbol{\lambda}_{L}
+ \mathbf{J}_{C}^{T} \, \boldsymbol{\lambda}_{C}\\
&=
\mathbf{w}_{e}
+ \mathbf{J}^{T} \, \boldsymbol{\lambda}
\,\,.
\label{eq:total-constraint-forces}
\end{flalign}
Moreover, by re-arranging some terms, we may define the vector of non-linear generalized force terms
\begin{equation}
\mathbf{h} \, := \, \mathbf{w}_{gc} + \mathbf{w}_{e} + \mathbf{J}_{a}^{T} \, \boldsymbol{\tau}_{a}
\,\, ,
\label{eq:total-nonlinear-forces}
\end{equation}
and thus arrive at the acceleration-level EoM
\begin{equation}
\mathbf{M} \, \dot{\mathbf{u}} = \mathbf{h} + \mathbf{J}^{T} \, \boldsymbol{\lambda}
\,\, .
\label{eq:models:smooth-eom}
\end{equation}
In (\ref{eq:models:smooth-eom}) we recognize the smooth (i.e. continuous) version of the system dynamics. Using the derivation described in Sec.~\ref{sec:constrained-rigidbody-mechanics:impulsive}, it is now quite straightforward to derive the time-stepping version of the system EoM and render
\begin{equation}
\mathbf{M} \, (\mathbf{u}^{+} - \mathbf{u}^{-}) 
= \Delta{t} \, \mathbf{h} + \mathbf{J}^{T} \, \boldsymbol{\lambda}
\,\, .
\label{eq:models:time-stepping-eom}
\end{equation}

\section{Constructing The Dual Problem}
\label{sec:construction}
\noindent
This section describes the construction of the dual FD NCP (\ref{eq:formulation:dual-forward-dynamics-ncp}) using the individual modeling elements described in Sec.~\ref{sec:models}. First, we will detail how the dual NCP can be constructed to incorporate all the elements described in Sec.~\ref{sec:models:contacts} for modeling contacts with restitutive impacts and friction. Second, we will describe two augmentations that will enable us realize constraint stabilization and constraint softening.

\subsection{Origins}
\label{sec:construction:origins}
\noindent
The dual FD NCP transcribed as the NSOCP (\ref{eq:formulation:dual-forward-dynamics-nsocp}) was first introduced by Cadoux~\cite{cadoux2009}. To the best of our knowledge, this is the first instance of employing the \textit{bi-potential method} of De Saxc\'e~\cite{desaxce1998bipotential} to incorporate the non-linear term in the optimization objective. This augmentation is crucial to the NSOCP representing a complete model of contact dynamics, but as the De Saxc\'e operator $\boldsymbol{\Gamma}_{\mu}(\cdot)$ is highly non-linear, it makes the problem NP-hard. The former was first introduced by De Saxc\'e et al in~\cite{desaxce1991new} and subsequently elaborated upon in~\cite{desaxce1998bipotential}. The latter provided a complete theoretical treatment that both justified its validity as well as established its connection to the concept of \textit{superpotentials} of J.J. Moreau~\cite{moreau1977application}, that have long been a key part of the non-smooth dynamics foundations. We recommend the seminal text of Glocker et al~\cite{glocker2001setvalued} for more insight on this topic. The NSOCP as we know it today, as well as the methods to approximate it, were introduced by Acary et al in~\cite{acary2011formulation} where it was derived from the primal FD problem.

\subsection{Problem Definition}
\label{sec:construction:definition}
\noindent
We begin by describing the construction of the three quantities defining the dual NCP, namely, the Delassus matrix $\mathbf{D}$, the free-velocity $\mathbf{v}_{f}$ and the feasible set $\mathcal{K}$ of admissible constraint reactions. However, regarding the Delassus matrix in particular, we do not actually need to state anything further. Having defined the generalized mass matrix $\mathbf{M}$ and the constraint Jacobian $\mathbf{J}$ in Sec.~\ref{sec:models}, we may simply apply (\ref{eq:delassus-matrix}) to compute it directly. Thus, at this stage we can focus solely on $\mathbf{v}_{f}$ and $\mathcal{K}$ as these require some work to define.

The principle mechanism we must employ is projection onto the constraint-space. Exploiting the duality induced by the constraint Jacobian $\mathbf{J}$, we can express the pre/post-event constraint velocities using the relations
\begin{equation}
\mathbf{v}^{-/+} = \mathbf{J} \, \mathbf{u}^{-/+}
\,\,\,\,.
\label{eq:prepost-constraint-velocities}
\end{equation}
In addition, we will re-arrange the time-stepping EoM (\ref{eq:models:time-stepping-eom}) to express the post-event generalized velocity $\mathbf{u}^{+}$ as a function of the constraint reactions $\boldsymbol{\lambda}$ as
\begin{equation}
\mathbf{u}^{+}(\boldsymbol{\lambda} ) = 
\mathbf{u}^{-} 
+ \Delta{t} \, \mathbf{M}^{-1} \, \mathbf{h} 
+ \mathbf{M}^{-1} \, \mathbf{J}^{T} \, \boldsymbol{\lambda} 
\,\,\,\,.
\label{eq:prepost-system-velocities-eom}
\end{equation}
Combining (\ref{eq:prepost-constraint-velocities}) and (\ref{eq:prepost-system-velocities-eom}) yields the vector of post-event constraint-space velocity (\ref{eq:post-event-constraint-velocity}) introduced in Sec.~\ref{sec:formulation:dual-problem}: 
\begin{equation*}
\mathbf{v}^{+}(\boldsymbol{\lambda}) 
= 
\mathbf{D} \, \boldsymbol{\lambda} + \mathbf{v}_{f}
\,\,\,\,.
\end{equation*}

In Sec.~\ref{sec:models:contacts:impacts} we described the Coulomb-Signorini-Newton model (\ref{eq:coulomb-newton-signorini-ncp}) which asserts complementarity conditions on the impact-augmented and De Saxc\'e-corrected constraint-space velocity of each contact $\mathcal{C}_{k}$. We can generalize this statement to include all constraints by forming the NCP
\begin{equation}
\mathcal{K}^{*} \ni 
\mathbf{v}^{+}(\boldsymbol{\lambda}) + \boldsymbol{\Gamma}(\mathbf{v}^{+}(\boldsymbol{\lambda}))
+ \mathbf{E}\,\mathbf{v}^{-}
\perp
\boldsymbol{\lambda} \in \mathcal{K}
\,\,\,\,,
\label{eq:construction:coulomb-signorini-newton-ncp}
\end{equation}
where, and with slight abuse of notation, we can denote the system-level De Saxc\'e operator using its contact-wise definition (\ref{eq:models:de-saxce-correction}) to only act along the contact constraint dimensions
\begin{subequations}
\begin{equation}
\boldsymbol{\Gamma}(\mathbf{v}^{+}(\boldsymbol{\lambda})) := 
\begin{bmatrix}
\,\mathbf{0}_{n_{jd}}\\
\mathbf{0}_{n_{l}}\\
\boldsymbol{\Gamma}_{\boldsymbol{\mu}}(\mathbf{v}_{C}^{+}(\boldsymbol{\lambda}))
\end{bmatrix}
\,\,\,\,,
\end{equation}
\begin{equation}
\boldsymbol{\Gamma}_{\boldsymbol{\mu}}(\mathbf{v}_{C}^{+}(\boldsymbol{\lambda})) :=
\begin{bmatrix}
\boldsymbol{\Gamma}_{\mu_{1}}(\mathbf{v}_{c,1}^{+}(\boldsymbol{\lambda}))\\
\vdots\\
\boldsymbol{\Gamma}_{\mu_{n_c}}(\mathbf{v}_{c,n_c}^{+}(\boldsymbol{\lambda}))\\
\end{bmatrix}
\,\,\,\,.
\end{equation}
\label{eq:construction:system-desaxce-operator}
\end{subequations}
Moreover, $\mathbf{E} \in \mathbb{R}^{n_d \times n_d}$ is the \textit{restitution matrix}, constructed using the coefficients of restitution of all contacts 
\begin{subequations}
\begin{equation}
\mathbf{E} :=
\begin{bmatrix}
0 & 0 & 0\\
0 & 0 & 0\\
0 & 0 & \mathbf{E}_{C}
\end{bmatrix}
\end{equation}
\begin{equation}
\mathbf{E}_{C} :=
\begin{bmatrix}
\begin{bmatrix}
0 & 0 & 0\\
0 & 0 & 0\\
0 & 0 & \epsilon_{c,1}\\
\end{bmatrix} & \dots & 0\\
\vdots & \ddots & \vdots\\
0 & \dots &
\begin{bmatrix}
0 & 0 & 0\\
0 & 0 & 0\\
0 & 0 & \epsilon_{c,n_c}\\
\end{bmatrix}\\
\end{bmatrix}
\,\,\,\,.
\end{equation}
\label{eq:construction:system-restitution-matrix}
\end{subequations}
As it turns out, as the restitution relations are only considered along the contact normal directions\footnote{It is technically feasible to include restitution effects along the contact tangential directions, as is described in~\cite{studer2009numerics,erleben2017} for deriving the PROX/SORprox algorithm. However, we are not aware of any materials whose physical interactions would justify the construction of such a model.}, the impact augmentation term can be absorbed into the free velocity $\mathbf{v}_{f}$ as a component of the bias velocity, i.e. $\mathbf{v}^{*} = \mathbf{E}\,\mathbf{v}^{-}$, which admits the redefinition $\mathbf{v}_{f} \leftarrow \mathbf{v}_{f} + \mathbf{E}\,\mathbf{v}^{-}$. Effectively, this redefinition reduces (\ref{eq:construction:coulomb-signorini-newton-ncp}) to the standard definition of the NCP (\ref{eq:formulation:dual-forward-dynamics-ncp}). We provide a brief proof of this property in Appendix.~\ref{sec:apndx:desaxce-properties:desaxce-invariance}. Thus the free-velocity is simply constructed using
\begin{equation}
\mathbf{v}_{f} = 
\mathbf{J} \, \left( \mathbf{u}^{-} + \Delta{t} \,\mathbf{M}^{-1} \mathbf{h}\right) + \mathbf{E}\,\mathbf{v}^{-}
\label{eq:construction:impact-augmented-free-velocity}
\end{equation}
In practice, however, can avoid constructing $\mathbf{E}$ and performing all the unnecessary matrix-multiplications with zeros by simply appending $\epsilon_{c,k} \, \text{v}_{k}^{-}$ to the row of $\mathbf{v}_{f}$ corresponding to the normal direction of each respective contact $\mathcal{C}_{k}$. Thus the only actual computational cost lies in computing $\mathbf{v}^{-} = \mathbf{J} \, \mathbf{u}^{-}$, which we do so anyway in (\ref{eq:construction:impact-augmented-free-velocity}).

Next, we will define the feasible set $\mathcal{K}$ of admissible constraint reactions. Luckily this is rather straight-forward to do, as it works out to simply be the \textit{Cartesian product} over the element-wise feasibility sets, which is often denoted as 
\begin{equation}
\mathcal{K} = \prod_{i}^{n_j + n_l + n_c} \, \mathcal{K}_{i}
\,\,.
\label{eq:construction:total-feasible-set-product}
\end{equation}
A crucial property of this composition is that, assuming each $\mathcal{K}_{i}$ is a proper cone, then so too is the total feasible set $\mathcal{K}$~\cite{boyd2004convex} . For each joint, limit and contact, we must invoke the appropriate set-valued force law to define the corresponding component set $\mathcal{K}_{i}$ of the aforedescribed product. Thus, for each joint $\mathcal{J}_{j}$, $\mathcal{K}_{j} = \mathbb{R}^{m_j}$, i.e. the constraint reactions are unconstrained and therefore span the entire constraint-space of the joint. For each limit $\mathcal{L}_{i}$, $\mathcal{K}_{l} = \mathbb{R}_{+}$, i.e. the force is constrained to the positive half-space, a.k.a. the nonnegative orthant, of the corresponding joint DoF. Lastly, for each contact $\mathcal{C}_{k}$, $\mathcal{K}_{k} = \mathcal{K}_{\mu_k}$, i.e. the set is the Coulomb friction cone $\mathcal{K}_{\mu_k}$ whose aperture is determined by the contact-specific friction coefficient $\mu_k$. This results in the composite cone
\begin{equation}
\mathcal{K} := \mathbb{R}^{n_{jd}} \times \mathbb{R}^{n_l} \times \mathcal{K}_{\boldsymbol{\mu}}
\,\,,
\label{eq:construction:total-feasible-set}
\end{equation}
where $\boldsymbol{\mu} := [\mu_{1}, \dots, \mu_{n_c}]^{T}$ denotes the vector of all friction coefficients, and $\mathcal{K}_{\boldsymbol{\mu}}$ is the composite cone formed over the Coulomb friction cones of all active contacts.

Lastly, we will form the \textit{total configuration-level constraint function} of the system. To construct it, we can use the definitions of the configuration-level implicit functions of all elements $\mathcal{J}_{j}$, $\mathcal{L}_{l}$  and $\mathcal{C}_{k}$, using (\ref{eq:models:joints:configuration-level-implicit-constraints}), (\ref{eq:models:limits:implicit-configuration-constraint}) and (\ref{eq:contacts:gap-constraint}), respectively. Stacking all individual elements of each constraint group, we first form the respective group-wise constraint functions
\begin{equation}
\mathbf{f}_{J}(\mathbf{s}) =
\begin{bmatrix}
\mathbf{f}_{1}(\mathbf{s})\\
\vdots\\
\mathbf{f}_{n_{j}}(\mathbf{s})
\end{bmatrix}
,
\mathbf{f}_{L}(\mathbf{s}) =
\begin{bmatrix}
g_{l,1}(\mathbf{s})\\
\vdots\\
g_{l,n_{l}}(\mathbf{s})
\end{bmatrix}
,
\mathbf{f}_{C}(\mathbf{s}) =
\begin{bmatrix}
g_{n,1}(\mathbf{s})\\
\vdots\\
g_{n, n_{c}}(\mathbf{s})
\end{bmatrix}
\,.
\label{eq:construction:total-element-constraint-fuctions}
\end{equation}
The total configuration-level implicit constraint function results from simply combining the components in (\ref{eq:construction:total-element-constraint-fuctions}) to form
\begin{equation}
\mathbf{f}(\mathbf{s}) =
\begin{bmatrix}
\mathbf{f}_{J}(\mathbf{s})\\
\mathbf{f}_{L}(\mathbf{s})\\
\mathbf{f}_{C}(\mathbf{s})
\end{bmatrix}
= 0
\,\,.
\label{eq:construction:total-configuration-constraint-fuction}
\end{equation}
Although not directly part of the definition of dual FD NCP (\ref{eq:formulation:dual-forward-dynamics-ncp}), we will use it to realize the augmentations in the sections that follow, as well as later in the definition of performance metrics in Sec.~\ref{sec:benchmarking:metrics:accuracy}.

\subsection{Constraint Stabilization}
\label{sec:construction:stabilization}
\noindent
Without employing a numerical integration scheme that can, at each discrete time-step, update the state of the system in a way that satisfies all constraints exactly, it is inevitable that numerical drift will lead to violations of the configuration-level constraints (\ref{eq:models:joints:configuration-level-implicit-constraints}), (\ref{eq:models:limits:configuration-level-implicit-constraints}) and (\ref{eq:contacts:signorini-conditions-position}). Such is the case when using simpler time-stepping integrators such as semi-implicit Euler (\ref{eq:formulation:time-stepping-euler}) or Moreau's midpoint scheme~\cite{moreau1988unilateral}. In addition, performing CD at discrete time-steps unavoidably leads to some interpenetration of the bodies, which, is directly dependent on the step size $\Delta{t}$ and the potential velocities of the bodies. 

These simulation artifacts are also exacerbated by the imprecision that may be exhibited by dual solvers. Essentially, we can understand why by realizing that the solutions to the NCP/CCP they render are computed with the goal of reducing the constraint \textit{velocity} to zero along the bilateral and unilateral constraint dimensions (i.e. not considering frictional components). This means that dual solvers have no direct impact on the constraints at configuration-level, and any under/over estimation of the constraint reactions will lead to further numerical drift, compounding that of the integration scheme.

However, configuration-level constraint violation can be alleviated by employing so-called \textit{constraint stabilization} techniques, such as that of Baumgarter~\cite{baumgarte1972stabilization}. In this work we followed an approach to realizing Baumgarter-like constraint stabilization based on the work of Preclik et al~\cite{preclik2018mdp} as well as the recent technical guide~\cite{siggraph2022contact}. The end result will be an augmentation of the problem that follows the same procedure described previously in Sec.~\ref{sec:construction:definition} for modeling frictional contacts with restitutive impacts. However, for illustrative purposes, it will be easier to explain the approach by first considering the simpler case where only bilateral joint constraints are present in the system.

The first step involves using (\ref{eq:prepost-constraint-velocities}) and (\ref{eq:prepost-system-velocities-eom}) to express the velocity-level bilateral constraints as a function of the constraint reactions $\boldsymbol{\lambda}$ via the definition of the post-event constraint velocity. Doing so yields the constraints 
\begin{equation}
\mathbf{v}^{+}(\boldsymbol{\lambda}) = \mathbf{J} \, \mathbf{u}^{+}(\boldsymbol{\lambda}) = 0
\label{eq:construction:joints-only-post-event-velocity}
\end{equation}
We will consider an augmentation of (\ref{eq:construction:joints-only-post-event-velocity}) that introduces control feedback terms that are functions of the instantaneous constraint violation. The latter is represented by the vector of constraint residuals computed from the configuration-level constraint functions (\ref{eq:construction:total-element-constraint-fuctions}) of the joints, in the form of
\begin{equation}
\mathbf{r}_{J} = \mathbf{f}_{J}(\mathbf{s})
\label{eq:construction:joint-constraint-residual-definition}
\,\,.
\end{equation}
Thus the \textit{stabilized} bilateral constraints can be stated as
\begin{equation}
\mathbf{v}^{+}(\boldsymbol{\lambda}) + \epsilon \, \boldsymbol{\lambda} = - \frac{v}{\Delta{t}} \, \mathbf{r}_{J}
\,\,\,\,,
\label{eq:construction:general-baumgarte-stabilization}
\end{equation}
where $\epsilon$ and $v$ are often referred to as the Constraint Force Mixing (CFM) constant and the Error Reduction Parameter (ERP), respectively. To the best of our knowledge, these names are attributed to the ODE physics engine~\cite{smith2008ode}. We refer interested readers to~\cite{siggraph2022contact, mujoco2024docs} for a more in depth explanation and interpretation of these parameters. In effect, they serve to respectively soften and dampen the constraint-space dynamics of the system. As our objective is to introduce such a constraint-stabilization scheme to the dual problem without compromising its structure as an NCP, we will necessarily omit the CFM term $\epsilon \, \boldsymbol{\lambda}$ for now. We do so because CFM introduces compliance to the system which effectively softens the constraints, and is the topic of Sec.~\ref{sec:construction:softening}. 

The next step, therefore, is to introduce ERP-type constraint stabilization to the NCP. Similar to how incorporating restitutive impacts augments the complementarity conditions (\ref{eq:contacts:coulomb-signorini-ncp}) to yield (\ref{eq:coulomb-newton-signorini-ncp}), the constraint stabilization of (\ref{eq:construction:general-baumgarte-stabilization}) can be introduced via an additional bias in the form of
\begin{equation}
\mathcal{K}^{*} \ni 
\hat{\mathbf{v}}(\boldsymbol{\lambda}) + \mathbf{v}_{B}
\perp
\boldsymbol{\lambda} \in \mathcal{K}
\,\,\,\,,
\label{eq:construction:stabilized-coulomb-signorini-newton-ncp}
\end{equation}
where $\mathbf{v}_{B}$ is the \textit{constraint stabilization bias velocity} that we need to construct. Moreover, as stabilization acts exclusively on the unilateral components of limit and contact constraints (i.e. contact normal directions), we can once more exploit the invariance of the De Saxc\'e operator w.r.t the bias $\mathbf{v}_{B}$ and include it as an additional component of the total bias to the free-velocity, i.e. $\mathbf{v}^{*} = \mathbf{E}\,\mathbf{v}^{-} + \mathbf{v}_{B}$. This ensures that the NCP retains the structure of the un-augmented problem.

Now we can proceed to construct $\mathbf{v}_{B}$. First of all, we will specify distinct ERPs $\alpha_{j}, \beta_{l}, \gamma_{k} \in \mathbb{R}_{+}$ for each joint $\mathcal{J}_{j}$, limit $\mathcal{L}_{l}$ and contact $\mathcal{C}_{k}$ constraint subset, respectively. For each joint $\mathcal{J}_{j}$,  the bilateral constraint residual and stabilization bias are
\begin{subequations}
\begin{equation}
\mathbf{r}_{J,j} = \textbf{f}_{J,j}(\mathbf{s})
\label{eq:construction:joint-constraint-residual}
\end{equation}
\begin{equation}
\mathbf{v}_{B,J,j} 
= \frac{\alpha_{j}}{\Delta{t}} \, \mathbf{r}_{J,j}
\,\,\,\, .
\label{eq:construction:joint-constraint-correction-bias}
\end{equation}
\label{eq:construction:joint-constraint-stabilization}
\end{subequations}
For each limit $\mathcal{L}_{l}$, as the constraints are unilateral, we will need to additionally constrain the residuals to act in the orthant that satisfies the complementarity conditions. Thus, the corresponding residual and bias velocity in this case are
\begin{subequations}
\begin{equation}
r_{l} = g_{l}(q_{l})
\label{eq:construction:limit-constraint-residual}
\end{equation}
\begin{equation}
\text{v}_{B,L,l} = \frac{\beta}{\Delta{t}} \, \min(0 \,,\,\, r_{l})
\,\,\,\, .
\label{eq:construction:limit-constraint-correction-bias}
\end{equation}
\label{eq:construction:limit-constraint-stabilization}
\end{subequations}
For the unilateral contact constraints of each active contact $\mathcal{C}_{k}$, we do something similar as for limits. Recalling our conventions for contacts described in Sec.\ref{sec:models:contacts}, body penetration occurs when $d_{c,k} \leq 0$. However, some CD methods and software can render potential collisions, i.e. contacts with positive penetration, in order to allow for preemptive actions. For this reason, Preclik et al~\cite{preclik2018mdp} proposed a double-sided method which ensures that the contact reaction remains zero if the contact would not close within the time-step. Thus, for each $\mathcal{C}_{k}$, this approach takes the form of
\begin{subequations}
\begin{equation}
r_{k} = d_{c,k} + \delta_{c}
\label{eq:construction:contact-constraint-residual}
\end{equation}
\begin{equation}
\text{v}_{B,c,k} 
= 
\gamma_{k} \, \min\left( 0 \,,\,\, \frac{r_{c,k}}{\Delta{t}} \right) 
+ \max\left( 0 \,,\,\, \frac{r_{c,k}}{\Delta{t}} \right) 
\,\, ,
\label{eq:construction:contact-constraint-correction-bias}
\end{equation}
\label{eq:construction:contact-constraint-stabilization}
\end{subequations}
where the constant $\delta_{c} \in \mathbb{R}$ is the so-called \textit{contact penetration margin}. The purpose of the penetration margin is to either serve as a tolerance of permissible contact penetration ($\delta_{c} \geq 0$), or to additionally inflate the constraint boundary ($\delta_{c} < 0$).

The constant scaling factors $\alpha_{j}/\Delta{t}$, $\beta_{l}/\Delta{t}$, and $\gamma_{j}/\Delta{t}$ thus bias the constraint reactions to push the bodies in directions that drive the corresponding residuals to zero, e.g. $\mathbf{r}_{J,j} \rightarrow 0$. Setting $\alpha_{j},\beta_{l},\gamma_{j} = 1$, amounts to correcting the residual in a single step. However, one must be careful with using this value, because it can lead to excessive oscillations, as the problem can become very stiff, since the correction behaves equivalently to a linear spring. Some physics engines like NVIDIA PhysX, apply an ad-hoc capping of the this bias, that is referred to as the \textit{maximum depenetration velocity}. In our current implementation we have not yet considered such additional effects, but these may be considered in future work.

Now we can construct the bias velocity vectors for each constraint type by simply stacking the individual contributions of each constraint element using (\ref{eq:construction:joint-constraint-stabilization}), (\ref{eq:construction:limit-constraint-stabilization}) and (\ref{eq:construction:contact-constraint-stabilization}), to form
\begin{equation}
\mathbf{v}_{B,J} 
= 
\begin{bmatrix}
    \text{v}_{B,J,1}\\
    ...\\
    \text{v}_{B,J,n_j} 
\end{bmatrix}
,\,
\mathbf{v}_{B,L} 
= 
\begin{bmatrix}
    \text{v}_{B,L,1}\\
    ...\\
    \text{v}_{B,L,n_l} 
\end{bmatrix}
,\,
\mathbf{v}_{B,C} 
= 
\begin{bmatrix}
    \text{v}_{B,c,1}\\
    ...\\
    \text{v}_{B,c,n_c} 
\end{bmatrix}
,
\label{eq:construction:stabilization-bias-velocities}
\end{equation}
and then combine them to form the total bias velocity
\begin{equation}
\mathbf{v}_{B} 
= 
\begin{bmatrix}
\mathbf{v}_{B,J}^{T} &
\mathbf{v}_{B,L}^{T} &
\mathbf{v}_{B,C}^{T} 
\end{bmatrix}^{T}
\,\,.
\label{eq:construction:total-stabilization-bias-velocity}
\end{equation}
Finally, the constraint-stabilized free-velocity is computed as
\begin{equation}
\mathbf{v}_{f} = 
\mathbf{J} \, \left( \mathbf{u}^{-} + \Delta{t} \,\mathbf{M}^{-1} \mathbf{h}\right) 
+ \mathbf{E}\,\mathbf{v}^{-}
+ \mathbf{v}_{B}
\,\,.
\label{eq:construction:stabilized-impact-augmented-free-velocity}
\end{equation}

\subsection{Constraint Softening}
\label{sec:construction:softening}
\noindent
The dual FD problem poses significant numerical difficulties. First of all, in either NCP (\ref{eq:formulation:dual-forward-dynamics-ncp}) or NSOCP (\ref{eq:formulation:dual-forward-dynamics-nsocp}) forms, the complementarity constraint to be enforced is non-convex and non-smooth, and the De Saxc\'e operator is highly non-linear. Second, even if were to consider a relaxation of the in the form of the CCP (\ref{eq:formulation:dual-forward-dynamics-ccp}) or the SOCP (\ref{eq:formulation:dual-forward-dynamics-socp}), the problem would still not be strictly convex as the Delassus matrix is in general only positive semi-definite i.e. it is rarely positive definite and often ill-conditioned. To address these challenges, we will consider another type of relaxation that regularizes the problem by softening the constraints. 

Specifically, we follow the approach introduced by Todorov et al~\cite{todorov2014analytical} for MuJoCo~\cite{mujoco2024github} that introduces a diagonal regularizer $\mathbf{R} \in \mathbb{S}^{n_d}_{++}$ to the Delassus matrix as well as an additional regulation bias velocity $\mathbf{v}_{r} \in \mathbb{R}^{n_d}$ to the free-velocity vector. However, as the particulars of how MuJoCo realizes this method has evolved significantly since the initial publication of~\cite{todorov2014analytical}, we must clarify that we our implementation is mostly based on that described in the official documentation~\cite{mujoco2024docs} of MuJoCo. Thus, in this section we will describe how MuJoCo's constraint softening scheme can be applied to both the NCP and CCP formulations of the FD problem. We will provide a brief derivation to clarify how this is possible and we refer readers to~\cite{mujoco2024docs} for further details.

The constraint softening scheme essentially serves to augment the NSOCP (\ref{eq:formulation:dual-forward-dynamics-nsocp}) introduced in Sec.~\ref{sec:formulation:dual-problem} to render the Regularized NSOCP (RNSOCP) that takes the form of
\begin{flalign}
& \text{RNSOCP}(\mathbf{D} \,,\,\, \mathbf{R} \,,\,\, \mathbf{v}_{f} \,,\,\, \mathbf{v}_{r} \,,\,\, \mathcal{K}): & \nonumber\\
& \,\, \textbf{Find} \,\, \boldsymbol{\lambda} =
\displaystyle
\operatorname*{argmin}_{\mathbf{x} \in \mathcal{K}} \,\,
\frac{1}{2} \,\mathbf{x}^{T} \, \left( \mathbf{D} + \mathbf{R} \right) \, \mathbf{x}\nonumber\\
& \quad\quad\quad\quad\quad\quad\quad\quad\quad 
+ \mathbf{x}^{T} \, \left( \mathbf{v}_{f} + \mathbf{v}_{r} + \boldsymbol{\Gamma}(\mathbf{v}^{+}(\boldsymbol{\lambda}))\right)
\,.
\label{eq:softening:dual-forward-dynamics-convex-rnsocp}
\end{flalign}
Equivalently, we can also state the CCP-type version as the Regularized Second-Order Cone Program (RSOCP)
\begin{flalign}
& \text{RSOCP}(\mathbf{D} \,,\,\, \mathbf{R} \,,\,\, \mathbf{v}_{f} \,,\,\, \mathbf{v}_{r} \,,\,\, \mathcal{K}): & \nonumber\\
& \,\, \textbf{Find} \,\, \boldsymbol{\lambda} =
\displaystyle
\operatorname*{argmin}_{\mathbf{x} \in \mathcal{K}} \,\,
\frac{1}{2} \,\mathbf{x}^{T} \, \left( \mathbf{D} + \mathbf{R} \right) \, \mathbf{x}
+ \mathbf{x}^{T} \, \left( \mathbf{v}_{f} + \mathbf{v}_{r} \right)
\label{eq:softening:dual-forward-dynamics-convex-rsocp}
\end{flalign}
which, is \textit{strictly convex} since the augmented Delassus matrix $\mathbf{D} + \mathbf{R}$ is symmetric and positive definite (PD). Thus, the combined effects of the CCP relaxation and constraint softening yield a dual FD problem that can be solved much more efficiently than the original NCP problem. However, one limitation of this approach is that it does not resolve the issue of ill-conditioning of $\mathbf{D}$, and can in fact lead to instabilities. This phenomena will be demonstrated later in the experiments presented Sec.~\ref{sec:experiments:animatronics}, but we will outline a potential reason for this at the end of this section.

To derive and construct the constraint softening terms we will once again start by assuming only bilateral joint constraints are present in the system. Fundamentally, employing the diagonal regularizer $\mathbf{R}$ amounts to introducing mechanical compliance to each constraint dimension, which is akin to introducing multiple spring-mass-like elements to the system. We can define the potential energy of this generalized "spring" element, as a function of the generalized state using the definition of the equality constraints functions (\ref{eq:models:joints:velocity-level-implicit-constraints}), to form
\begin{equation}
U_{C}(\mathbf{s}, \mathbf{u}^{+}) 
= \frac{1}{2} \, \left\Vert \, \dot{\mathbf{f}}_{J}(\mathbf{s}, \mathbf{u}^{+}) \, \right\Vert_{\mathbf{K}}^{2}
\,\, ,
\label{eq:softening:compliance-potential-def}
\end{equation}
where $\mathbf{K} \in \mathbb{S}_{++}^{n_d}$ is a PD stiffness matrix. Note how we have used the velocity-level constraints as opposed to the configuration-level variants (\ref{eq:models:joints:configuration-level-implicit-constraints}). This is possible by virtue of index reduction, and is necessary since we require the potential and the resulting force to be functions of the post-event generalized velocities $\mathbf{u}^{+}$. Taking the partial derivative w.r.t the former yields the generalized force of the spring as
\begin{equation}
\boldsymbol{\tau}_{C} 
:= - \frac{\partial U_{C}(\mathbf{s}, \mathbf{u}^{+})}{\partial \mathbf{u}^{+}} 
= - \mathbf{J}_{J}^{T} \, \mathbf{K} \, \dot{\mathbf{f}}_{J}(\mathbf{s}, \mathbf{u}^{+})
\,\, .
\label{eq:softening:compliance-potential-generalized-forces}
\end{equation}
Observing the multiplicand of the Jacobian matrix $\mathbf{J}_{J}$, we can equate this quantity to the force of a Hookean linear spring model in constraint-space, where the residuals play the role of spring deflections. These spring forces are thus exactly the constraint reactions, and thus we can state the equivalence as
\begin{equation}
\boldsymbol{\lambda} = - \mathbf{K} \, \dot{\mathbf{f}}_{J}(\mathbf{s}, \mathbf{u}^{+})
\,\, \Leftrightarrow \,\,
\dot{\mathbf{f}}_{J}(\mathbf{s}, \mathbf{u}^{+}) = - \mathbf{R} \, \boldsymbol{\lambda}
\,\,\,\,,
\label{eq:softening:compliance-lambda-from-constraint-map}
\end{equation}
where now we define the regularizer as $\mathbf{R} := \mathbf{K}^{-1}$ which makes it exactly equivalent to a mechanical compliance. Combining (\ref{eq:softening:compliance-lambda-from-constraint-map}) with the time-stepping EoM (\ref{eq:models:time-stepping-eom}) to form
\begin{equation}
\begin{bmatrix}
    \mathbf{M} & \mathbf{J}_{J}^{T} \\
    \mathbf{J}_{J} & -\mathbf{R}
\end{bmatrix}
\, 
\begin{bmatrix}
     \mathbf{u}^{+} \\
     - \boldsymbol{\lambda}
\end{bmatrix}
=
\begin{bmatrix}
    \Delta{t} \, \mathbf{h} + \mathbf{M}\,\mathbf{u}^{-} \\
    -\mathbf{v}^{*}
\end{bmatrix}
\,\, .
\label{eq:softening:compliant-dynamics-linear-system}
\end{equation}
Note how we have included the bias velocity $\mathbf{v}^{*}$ in the RHS of (\ref{eq:softening:compliant-dynamics-linear-system}). This is to account for the fact that any augmentations considered besides constraint softening would necessarily be included in the definition of the velocity-level constraints, i.e. $\dot{\mathbf{f}}_{J}(\mathbf{s}, \mathbf{u}^{+}) := \mathbf{J}_{J} \, \mathbf{u}^{+} - \mathbf{v}^{*}$. The system (\ref{eq:softening:compliant-dynamics-linear-system}) demonstrates that the inclusion of the compliance in the joints can render the system solvable with direct methods, otherwise its absence would result in a zero matrix in-place of $\mathbf{R}$ and thus potentially admit multiple or no solution at all. Moreover, (\ref{eq:softening:compliant-dynamics-linear-system}) can also result from expressing the KKT conditions of (\ref{eq:softening:dual-forward-dynamics-convex-rsocp}), which in the absence of unilateral constraints is identical to (\ref{eq:softening:dual-forward-dynamics-convex-rnsocp}). 

Taking an optimization perspective, we can express the Lagrangian of the original primal FD problem as
\begin{equation}
\mathcal{L}(\mathbf{u}^{+}, \boldsymbol{\lambda}) 
:= \frac{1}{2} \Vert \mathbf{u}^{+} - \mathbf{u}_{f} \Vert_{\mathbf{M}} + \boldsymbol{\lambda}^{T} \left( \mathbf{J}_{J} \mathbf{u}^{+} + \mathbf{v}^{*} \right)
\,\,,
\label{eq:softening:primal-lagrangian}
\end{equation}
and observe the resulting saddle-point problem
\begin{equation}
\min_{\mathbf{u}^{+}} \max_{\boldsymbol{\lambda}} \, 
\mathcal{L}(\mathbf{u}^{+}, \boldsymbol{\lambda}) 
- \frac{1}{2} \left\Vert \, \boldsymbol{\lambda} \, \right\Vert_{\mathbf{R}}^{2}
\,\,\,\,.
\label{eq:softening:saddle-point-problem-with-tikhonov-regularization}
\end{equation}
The additional term in the Lagrangian constitutes what is referred to in the optimization literature as \textit{Tikhonov regularization} \cite{tikhonov1995numerical}. Applying such a regularization is therefore guaranteed to yield a unique minimum-norm solution, which is a significant benefit of constraint softening, since it stabilizes the resulting contact reactions w.r.t successive time-steps. 

However, this regularization can result in significant underestimation of the constraint forces, leading to intolerable constraint violation. Physically, this means that joints break, bodies drift away from one another, and contacting bodies exhibit significant interpenetration. Todorov et al \cite{todorov2014analytical} proposed a solution to this predicament: construct  $\mathbf{v}_{r}$ as a special form the constraint stabilization described previously in Sec.~\ref{sec:construction:stabilization}. In a nutshell, it causes the constraint-space dynamics to behave as a multi-dimensional, decoupled, and critically-damped spring-mass system. Moreover, by carefully crafting a non-linear modulation of the constraint residual, the effective compliance, damping and stiffness can be adapted automatically according to the magnitude of the constraint violation. Proceeding element-wise using index $j$ for each scalar constraint equation, the MuJoCo approach can be outlined as follows:
\begin{enumerate}
\item Define shaping functions that modulate the residuals $r_{j}$.
\item Use the shaping functions to compute impedance $d_{j}$. 
\item Use each impedance $d_{j}$ to compute corresponding compliance $c_{j}$, damping $b_{j}$ and stiffness $k_{i}$ coefficients.
\item Construct $\mathbf{R}$ and $\mathbf{a}_{r}$ using the triplets $c_{j}, b_{j}, k_{i} \, \forall j$. 
\end{enumerate}

At the core of this approach lies a family of shaping functions, which we'll generically denote as $f_{s}(x)$, where $f_{s} : \mathbb{R} \rightarrow \mathbb{R}$. The shaping functions are reflected sigmoids (i.e. symmetric w.r.t the domain) that modulate the residuals $r_{j}$. Parameterized by the midpoint value $m_{j}$ and power factor $p_{j}$, they take the form of the linear and non-linear variants
\begin{subequations}
\begin{equation}
\text{f}_{s,L}(x) = x
\end{equation}
\begin{equation}
\text{f}_{s,NL}(x) = 
\begin{cases}
\frac{x^{p_{j}}} {m_{j}^{1-p_{j}}} & , x \leq m_{j} \\
1 - \frac{(1 - x)^{p_{j}}}{(1 - m_{j})^{1-p_{j}}} & , x > m_{j}
\end{cases}
\,\, .
\end{equation}
\label{eq:softening:mujoco-sigmoids}
\end{subequations}
\begin{figure}[!hb]
\centering
\includegraphics[width=1.0\linewidth]{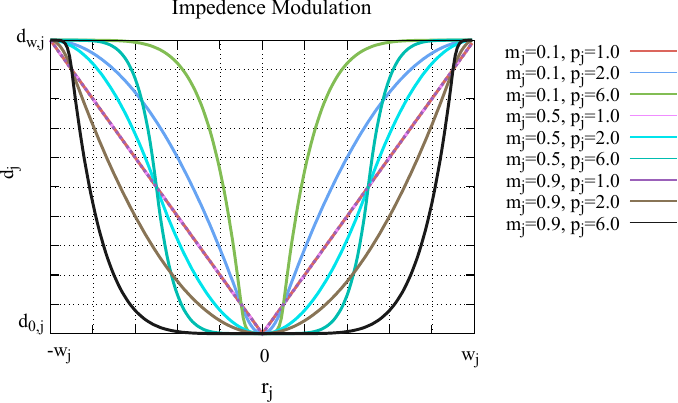}
\caption{Examples of impedance shaping using the sigmoidal shaping functions $f_{s}(x)$, for various settings of the midpoint value $m_{j}$ and power factor $p_{j}$ parameters. The 2D plots are agnostic to the absolute impedance values, and are shown only relative to the minimum and maximum impedance defined by $d_{0,j}$ (dmin) and $d_{0,j}$ (dmax) respectively. The X-axis is clipped to the range defined by the width parameter $w_{j}$.}
\label{fig:impedence-shaping}
\end{figure}
We can choose between either depending on whether we select $p_{j} = 1$ to use the linear, or $p_{j} > 1$ for the nonlinear. Having chosen a variant, the constraint impedance is then computed using the piece-wise continuous pseudo-function
\begin{subequations}
\begin{equation}
x_{j} = \frac{ | r_{j} - m_{j} | } {w_{j}}
\end{equation}
\begin{equation}
d_{j} :=
\begin{cases}
\frac{1}{2} \, (d_{0,j} + d_{w,j}) & ,  d_{0,j} = d_{w,j} \lor w_{j} \approx 0 \\
d_{w,j} & , x_{j} \geq 1 \\
d_{0,j} & , x_{j} \leq 0 \\
d_{0,j} + f_{s}(x_{j}) \, (d_{w,j} - d_{0,j} )  & , \text{otherwise}
\end{cases}
\,\, ,
\end{equation}
\label{eq:softening:mujoco-impedence}
\end{subequations}
where $x_j$ is a shorthand for the absolute of the scaled and offset residual $r_j$, $w_{j}$ is the width parameter that defines the interval $[-w_{j}, w_{j}]$ within which the impedance is shaped by $r_j$, $d_{0,j} = d(0)$ is the minimum impedance when $x_{j} = 0$, and $d_{w,j} = d(w_{j})$ is the maximum impedance when $r_j = w_{j}$. Note that this operation effectively bounds the residual-to-impedance mapping to the ranges $r_j \in [-w_{j}, w_{j}]$ and $d_{j} \in [d_{0,j}, d_{w,j}]$. Moreover, as we require that both impedance limits $d_{0,j}$ and $d_{w,j}$ (with $d_{0,j} \leq d_{w,j}$) are selected within the $[0, 1]$ range, the computed impedance also lies in the range $d_{j} \in [0, 1]$. Fig.\ref{fig:impedence-shaping} demonstrates examples of the residual-to-impedance mappings for various shaping function parameters. Therefore, given impedance $d_{j}$ computed as just described, the compliance, damping and stiffness coefficients are respectively 
\begin{equation}
c_{j}  = \frac{1 - d_{j}}{d_{j}}
\,,\,\,
b_{j} = \frac{2}{d_{w,j} \, T_{j}} 
\,,\,\,
k_{j} = \frac{d_{j}}{d_{w,j}^{2} \, T_{j}^{2} \, \beta_{j}^{2}}
\label{eq:softening:mujoco-cbk-terms}
\end{equation}
where $T_{j}$ and $\beta_{j}$ are the time-constant and the damping ratio, respectively, of the effective mass-spring-damper system of the $j$-th constraint. The time-constant is essentially the inverse of the natural frequency of the system multiplied by the damping ratio. Finally, the coefficients of the compliance matrix and reference acceleration vector are respectively computed as
\begin{equation}
R_{jj} = c_{j} \, D_{jj}
\,\,,\,\,\,\,
v_{r,j} =  \Delta{t} \,  \left( b_{j} \, v_{j}^{-} + k_{j} \, r_{j} \right)
\label{eq:softening:mujoco-R-vr-elementwise}
\end{equation}
where $D_{jj}$ is the $j$-th diagonal coefficient of the Delassus matrix $\mathbf{D}$, and $\text{v}_{j}$ is the $j$-th coefficient of constraint-space velocity vector $\mathbf{v}$. This formulation therefore proceeds element-wise and allows us to select distinct values of the parameter set $\{ m, p, w, d_{0}, d_{w}, T, \beta \}_{j}$ for each constraint $j$. This affords us great flexibility in selecting those that achieve the best results possible, but with the disadvantage of not being automated.

The final step involves constructing the regulation velocity $\mathbf{v}_{r}$. Although in practice we actually add each term directly to each dimension of $\mathbf{v}_{f}$, for illustrative purposes, we can define this more clearly by expressing the diagonal matrices
\begin{flalign}
\mathbf{K}_{r} &= \operatorname{diag}([k_{1}, \dots, \ k_{n_d}]^{T}) \\[4pt]
\mathbf{B}_{r} &= \operatorname{diag}([b_{1}, \dots, \ b_{n_d}]^{T}) \\[4pt]
\mathbf{R} &= \operatorname{diag}([c_{1} \, D_{11}, \dots, \ c_{n_d} \, D_{n_d \, n_d}]^{T})
\,\,\,\,.
\label{eq:softening:mujoco-KBR}
\end{flalign}
The regulation velocity $\mathbf{v}_{r}$ is thus computed as
\begin{equation}
\mathbf{v}_{r} = \Delta{t} \,  \left( \mathbf{B}_{r} \, \mathbf{v}_{j}^{-} + \mathbf{K}_{r} \, \mathbf{r} \right)
\,\,\,\,.
\label{eq:softening:mujoco-vr}
\end{equation}

Having outlined what MuJoCo-type constraint-softening terms, we can briefly analyze what they actually represent, and why they may lead to instabilities. As the regularized Delassus matrix is positive definite due to the presence of $\mathbf{R}$, it is thus also invertible, admitting the closed-form solution
\begin{equation}
\boldsymbol{\lambda}_{0} := 
\left( \, \mathbf{D} + \mathbf{R} \, \right)^{-1} \, ( \, \mathbf{v}_{f} + \mathbf{v}_{r} \, )
\,\,.
\label{eq:softening:compliant-dual-unconstrained-solution}
\end{equation}
Using (\ref{eq:softening:compliant-dual-unconstrained-solution}) renders the constraint-space velocity
\begin{equation}
\mathbf{v}_{0}^{+} = 
\mathbf{D} \, \left(\mathbf{D} + \mathbf{R} \right)^{-1} \, \mathbf{v}_{f}
+ \mathbf{R} \, \left(\mathbf{D} + \mathbf{R} \right)^{-1} \, \mathbf{v}_{r}
\,\, ,
\label{eq:softening:compliant-constrained-velocity}
\end{equation}
which, reveals the element-wise constraint-space dynamics
\begin{equation}
v_{0,j}^{+} = d_{j} \, ( b_{j} \, v_{j}^{-} + k_{j} \, r_{j}) + (1 - d_{j}) \, v_{f,j}
\,\,\,\,.
\label{eq:softening:compliant-constraint-dynamics}
\end{equation}
This demonstrates that the system will behave as multiple decoupled spring-mass-damper elements. However, this can lead to significant oscillations if the parameters that determine the non-linear impedance are not chosen carefully. Therefore, although the regularization and regulation of the problem is somewhat automated, it still requires the user to choose the aforedescribed parameters appropriately.

One final clarification that is necessary to make before concluding this part, is that the MuJoCo-like constraint softening method is mostly, but not exactly, compatible with the NSOCP formulation (\ref{eq:softening:dual-forward-dynamics-convex-rnsocp}). As we described previously in Sec.~\ref{sec:construction:definition} and Sec.~\ref{sec:construction:softening}, an augmentation that biases the free-velocity $\mathbf{v}_{f}$ via $\textbf{v}^{*}$  must satisfy the invariance property of the De Saxc\'e operator. However, as $\mathbf{v}_{r}$ includes non-zero components in the frictional dimensions due to the damping terms, it violates this requirement. By virtue of approximation, we can proceed to apply it in the same manner as all other augmentations.

\section{Dual Solvers}
\label{sec:solvers}

Given a complete definition, and construction, of the dual FD NCP (\ref{eq:formulation:dual-forward-dynamics-ncp}), we now address how it can be solved numerically. This section therefore describes the set of first-order dual FD algorithms that fulfill the template of Alg.\ref{alg:generic-first-order-dual-solver}. For brevity, we henceforth refer to them simply as \textit{dual solvers}. Following the taxonomies described in~\cite{acary2008numerical,acary2018comparisons,siggraph2022contact,lidec2024models}, we will summarize how they are related as well as their distinguishing factors. Of course, one may note that their respective derivations can be based on many seemingly different perspectives of the dual FD problem, i.e. \textit{complementarity systems}, \textit{projected dynamics}, \textit{differential inclusions} and \textit{constrained optimization}. However, the underlying theory of \textit{convex analysis} can be used to establish their equivalence as demonstrated by Brogliato~\cite{brogliato2006equivalence}.

As Moreau first established in his seminal work~\cite{moreau1988unilateral}, a unifying view of these first-order methods can be attained via the \textit{augmented Lagrange method} (ALM)~\cite{nocedal1999numerical}, and \textit{proximal operators}~\cite{parikh2014proximal}. We will first define the theoretical framework that justifies the template of Alg.~\ref{alg:generic-first-order-dual-solver} and then proceed to describe the set of computational building-blocks that lead to the construction of concrete dual solvers. Effectively, the solvers resulting from this perspective are those that can be described as variants of the celebrated PGS algorithm, some which can be seen as generalizations and improvements thereof. In addition, we also consider a version of a state-of-the-art solver based on the primal-dual Alternating Direction Method of Multipliers (ADMM) algorithm~\cite{boyd2011admm}, that was recently proposed by Carpentier et al in~\cite{carpentier2024unified}. This ADMM method is also derived using ALM and proximal operators and thus also fulfills the template of Alg.~\ref{alg:generic-first-order-dual-solver}, further broadening their applicability as the underlying foundation. The complete list of solvers we have evaluated is presented in Tab.~\ref{tab:dual-solver-definitions}, and includes relevant implementation details.

\subsection{Preliminaries}
\label{sec:solvers:preliminiaries}
\noindent
We begin with a brief review of elements from \textit{convex analysis} and \textit{proximal optimization} that will be required in the continuation. We refer the reader to ~\cite{rockafellar2009variational, parikh2014proximal} for details.\\

The \textit{indicator function} of a closed and convex set $\mathcal{C} \subset \mathbb{R}^{n}$ evaluated at some point $\mathbf{x} \in \mathbb{R}^{n}$ is defined as
\begin{equation}
\Psi_{\mathcal{C}}(\mathbf{x}) 
:=
\begin{cases}
0 &,\,\, \mathbf{x} \in \mathcal{C} \\
\infty &,\,\, \mathbf{x} \not\in \mathcal{C} \\
\end{cases}
\,\,,
\label{eq:solvers:indicator-function}
\end{equation}
and its conjugate is the so-called \textit{support function}
\begin{equation}
\mathcal{S}_{\mathcal{C}}(\mathbf{x})
= 
\Psi_{\mathcal{C}}^{*}(\mathbf{x}) 
:=
\sup_{\mathbf{y}} \,\{\mathbf{y}^{T} \, \mathbf{x} \,:\, \mathbf{y} \in \mathcal{C} \}
\,\,.
\label{eq:solvers:support-function}
\end{equation}
It is important to note at this point, that as $\mathcal{C}$ is convex, then so too are $\Psi_{\mathcal{C}}(\mathbf{x})$ and $\mathcal{S}_{\mathcal{C}}(\mathbf{x})$. In general, given a convex function $f : \mathbb{R}^{n} \rightarrow \mathbb{R}$, then its \textit{subdifferential} $\partial \, f (\mathbf{x})$ is defined as
\begin{equation}
\partial \, f(\mathbf{x}) 
:=
\{\mathbf{y} \,:\, \mathbf{y}^{T} \left(\mathbf{z} - \mathbf{x} \right) \leq \left( f(\mathbf{z}) - f(\mathbf{x}) \right) \,,\,\, \forall \mathbf{z} \} 
\,\,.
\label{eq:solvers:subdifferential-convex-function}
\end{equation}
Thus given that the indicator and support functions are convex, we can apply (\ref{eq:solvers:subdifferential-convex-function}) to define their respective subdifferentials as
\begin{equation}
\begin{array}{l}
\partial \Psi_{\mathcal{C}}(\mathbf{x}) 
:=
\{\mathbf{y} \,:\, \mathbf{y}^{T} \left(\mathbf{z} - \mathbf{x} \right) \leq 0 \,,\,\, \forall \mathbf{z} \in \mathcal{C} \}
\end{array}
\,\,,
\label{eq:solvers:subdifferential-indicator-function}
\end{equation}
\begin{equation}
\begin{array}{l}
\partial \mathcal{S}_{\mathcal{C}}(\mathbf{x}) 
:=
\{\mathbf{y} \,:\, \mathbf{y}^{T} \, \mathbf{z} \leq 0 \,,\,\, \forall \mathbf{z} \in \partial \, \Psi_{\mathcal{C}}(\mathbf{x})  \}
\end{array}
\,\,.
\label{eq:solvers:subdifferential-support-function}
\end{equation}
We now observe that, the aforedefined subdifferentials correspond to the normal and tangent cones of $\mathcal{C}$ at point $\mathbf{x} \in \mathcal{C}$, respectively, i.e. 
$\mathcal{N}_{\mathcal{C}}(\mathbf{x}) = \partial \, \Psi_{\mathcal{C}}(\mathbf{x})$ and $\mathcal{T}_{\mathcal{C}}(\mathbf{x}) = \partial \, \mathcal{S}_{\mathcal{C}}(\mathbf{x})$. This property establishes the duality (i.e. conjugacy) between the two sets and their elements, in the form of
\begin{equation}
\mathbf{y} \in \mathcal{N}_{\mathcal{C}}(\mathbf{x}) 
\Leftrightarrow
\mathbf{x} \in \mathcal{T}_{\mathcal{C}}(\mathbf{y})
\Leftrightarrow
\mathbf{x} \in \mathcal{C} \,,\,\,
\mathbf{x}^{T}\mathbf{y} = \mathcal{S}_{\mathcal{C}}(\mathbf{y})
\,\,.
\label{eq:solvers:indicator-support-subdifferential-duality}
\end{equation}
Moreover, when the set under consideration is a closed and convex cone $\mathcal{K}$, then in addition, the support function is identical to the indicator function of the dual cone $\mathcal{K}^{*}$, i.e. $\mathcal{S}_{\mathcal{K}}(\mathbf{x}) = \Psi_{\mathcal{K}^{*}}(\mathbf{x})$. This property, together with the duality defined in (\ref{eq:solvers:indicator-support-subdifferential-duality}) results in the complementarity property
\begin{equation}
\mathbf{y} \in \mathcal{N}_{\mathcal{K}}(\mathbf{x}) 
\Leftrightarrow
\mathbf{x} \in \mathcal{N}_{\mathcal{K}^{*}}(\mathbf{y})
\Leftrightarrow
\mathcal{K} \ni \mathbf{x} \perp \mathbf{y} \in \mathcal{K}^{*}
\,\,.
\label{eq:solvers:cone-indicator-complementarity}
\end{equation}

Next, let us introduce the \textit{proximal operator}, which, defines a generalized notion of projection. For any weighted metric norm $\Vert \mathbf{x} \Vert_{\mathbf{A}} := \sqrt{\mathbf{x}^{T} \, \mathbf{A} \, \mathbf{x}}$, where $\mathbf{A} \in \mathbb{S}_{++}^{n}$, a proximal-point to a convex function $f(\mathbf{x})$ at some $\mathbf{x} \in \mathbb{R}^{n}$ can be defined as
\begin{equation}
\displaystyle \displaystyle
\text{prox}_{\mathbf{f}}^{\mathbf{A}}(\mathbf{x}) := 
\operatorname*{argmin}_{\mathbf{y}} \, \mathbf{f}(\mathbf{y}) 
+ \frac{1}{2} \, \Vert \mathbf{x} - \mathbf{y} \Vert_{\mathbf{A}}^{2}
\,\, .
\label{eq:solvers:proximal-projection-operator-to-function}
\end{equation}
The resulting $\mathbf{z} = \text{prox}_{\mathbf{f}}^{\mathbf{A}}(\mathbf{x})$ is often called the \textit{proximal-point} to the convex function $f$ at $\mathbf{x}$, and $\mathbf{z} \in \textbf{dom}\,f$. In addition, we can define a mechanism to quantify the distance of the proximal projection via the \textit{proximal vector distance function}
\begin{equation}
\text{dist}_{\mathbf{f}}^{\mathbf{A}}(\mathbf{x}) := \mathbf{x} - \text{prox}_{\mathbf{f}}^{\mathbf{A}}(\mathbf{x})
\,\, .
\label{eq:solvers:proximal-distance-operator-to-function}
\end{equation}
Furthermore, proximal operators are also applicable to define projections onto convex sets. Applying (\ref{eq:solvers:proximal-projection-operator-to-function}) to the indicator function (\ref{eq:solvers:indicator-function}) of a convex set $\mathcal{C}$, the proximal projection and distance operators take the form of
\begin{subequations}
\begin{equation}
\displaystyle \displaystyle
\text{prox}_{\mathcal{C}}^{\mathbf{A}}(\mathbf{x}) := 
\operatorname*{argmin}_{\mathbf{y} \in \mathcal{C}} \, \frac{1}{2} \, \Vert \mathbf{x} - \mathbf{y} \Vert_{\mathbf{A}}^{2} 
\label{eq:solvers:proximal-projection-operator-to-set}
\end{equation}
\begin{equation}
\text{dist}_{\mathcal{C}}^{\mathbf{A}}(\mathbf{x}) := \mathbf{x} - \text{prox}_{\mathcal{C}}^{\mathbf{A}}(\mathbf{x})
\,\,.
\label{eq:solvers:proximal-distance-operator-to-set}
\end{equation}
\label{eq:solvers:proximal-operators-to-set}
\end{subequations}
Lastly, the subdifferential of the proximal operator also yields an inclusion to the normal cone $\mathcal{N}_{\mathcal{C}}(\mathbf{x})$. Expanding (\ref{eq:solvers:proximal-projection-operator-to-set}) using the definition of the indicator function $\Psi_{\mathcal{C}}(\mathbf{x})$ as
\begin{equation}
\mathbf{x}^{*} = \text{prox}_{\mathcal{C}}^{\mathbf{A}}(\mathbf{x}) 
= 
\displaystyle \operatorname*{argmin}_{\mathbf{y}} 
\, 
\underbrace{
\Psi_{\mathcal{C}}(\mathbf{y}) + \frac{1}{2} \, \Vert \mathbf{x} - \mathbf{y} \Vert_{\mathbf{A}}^{2}}_{f(\mathbf{y})}
\,\,,
\label{eq:solvers:proximal-indicator-function-definition}
\end{equation}
and evaluating the subdifferential at the proximal-point yields
\begin{equation}
\partial\,f(\mathbf{x}^{*})
=
\partial \, \Psi_{\mathcal{C}}(\mathbf{x}^{*})
- 
\mathbf{A} \left( \mathbf{x} - \mathbf{x}^{*} \right)
\ni 0
\,\,.
\label{eq:solvers:proximal-indicator-function-inclusion}
\end{equation}
Thus we can see that the proximal operator renders a set inclusion to the normal cone of $\mathcal{C}$ at point $\mathbf{x}$ in the form of
\begin{equation}
\mathbf{x} \in \left( \mathcal{N}_{\mathcal{C}}(\mathbf{x}^{*}) + \mathbf{A}^{-1} \mathbf{x}^{*} \right)
\,\,.
\label{eq:solvers:proximal-normal-cone-inclusion}
\end{equation}

\subsection{The Proximal Perspective}
\label{sec:solvers:proximal-point-derivation}
\noindent
To enunciate how and why projective first-order methods follow the template in Alg.~\ref{alg:generic-first-order-dual-solver}, we provide a brief derivation that highlights the underling common structure. This derivation, based on the augmented Lagrange method and proximal operators, has its roots in the early seminal work of J.J Moreau in~\cite{moreau1988unilateral} and has been elaborated upon extensively by others as well. The derivation provided here is principally based on the work of Studer et al in \cite{studer2007solving,studer2008augmented} and Erleben in~\cite{erleben2017}.

As first introduced in Sec.~\ref{sec:problem-formulation}, in order to tackle the NCP (\ref{eq:formulation:dual-forward-dynamics-ncp}), the principle tool will be the optimization perspective via the NSOCP (\ref{eq:formulation:dual-forward-dynamics-nsocp}). Due to the nonlinear and non-differentiable De Saxc\'e operator, it is very difficult to solve directly. One method that has shown great success is that proposed by Cadoux~\cite{cadoux2009} and further developed by Acary et al in~\cite{acary2011formulation}, both taking inspiration from the original \textit{predictor-corrector} approach of De Saxc\'e in ~\cite{desaxce1998bipotential}. This method attains a convexification of (\ref{eq:formulation:dual-forward-dynamics-nsocp}) via fixing the De Saxc\'e term to the image of its argument and proceeding by successive approximation of this value. As we will show, it renders the problem convex and differentiable, fitting well with the derivation of projective first-order methods. This convexified NCP takes the form of
\begin{flalign}
\label{eq:solvers:convex-ncp-socp}
& \text{SOCP}(\mathbf{D}, \mathbf{v}_{f}, \mathcal{K}): & \nonumber\\[4pt]
& \quad \textbf{Given} \,\, \mathbf{s} = \boldsymbol{\Gamma}(\mathbf{v}^{+}(\boldsymbol{\lambda})) \\[4pt]
& \quad \textbf{Find} \,\, \boldsymbol{\lambda} =
\displaystyle
\operatorname*{argmin}_{\mathbf{x} \in \mathcal{K}} \,\,
\frac{1}{2} \,\mathbf{x}^{T} \, \mathbf{D} \, \mathbf{x}
+ \mathbf{x}^{T} 
\left(
\mathbf{v}_{f} + \mathbf{s}
\right) \nonumber
\end{flalign}
To simplify notation, we denote (\ref{eq:solvers:convex-ncp-socp}) with the shorthand
\begin{equation}
\begin{array}{rlclcl}
\displaystyle \displaystyle 
\min_{\mathbf{x}} \,
    & h(\mathbf{x}) \\[4pt]
\textrm{s.t.} & \mathbf{f}(\mathbf{x}) \in \mathcal{K}
\end{array}
\,\, .
\label{eq:solvers:constrained-problem}
\end{equation}
The first step involved rewriting the inclusion constraint as
\begin{equation}
\begin{array}{rlclcl}
\displaystyle \displaystyle 
\min_{\mathbf{x},\mathbf{y}} \,
    & h(\mathbf{x}) \\[4pt]
\textrm{s.t.} & \mathbf{A} \, (\mathbf{f}(\mathbf{x}) - \mathbf{y}) = 0 \\[4pt]
              & \mathbf{y} \in \mathcal{K}
\end{array}
\,\, ,
\label{eq:solvers:constrained-problem-slacked}
\end{equation}
where $\mathbf{y} \in \mathbb{R}^{n_d}$ is a vector of \textit{slack variables} and $\mathbf{A} \in \mathbb{S}_{++}^{n_d}$ is a matrix that will serve as a weighted metric norm. Since $\mathbf{A}$ is PD, the equality constraint can only be satisfied when $\mathbf{f}(\mathbf{x}) - \mathbf{y} = 0$. The augmented Lagrangian of (\ref{eq:solvers:constrained-problem-slacked}) is 
\begin{subequations}
\begin{equation}
\mathcal{L}_{\rho,\mathbf{A}}^{A}(\mathbf{x}, \mathbf{y}, \mathbf{z}) := 
h(\mathbf{x}) 
+ \mathbf{z}^{T} \, \mathbf{A} \, \left( \mathbf{f}(\mathbf{x}) - \mathbf{y} \right)
+ \frac{\rho}{2} \, \Vert \mathbf{f}(\mathbf{x}) - \mathbf{y} \Vert_{A}^{2}
\label{eq:solvers:constrained-problem-augmented-lagrangian-def}
\end{equation}
\begin{equation}
= 
h(\mathbf{x}) 
- \frac{1}{2\rho} \, \mathbf{z}^{T} \, \mathbf{A} \, \mathbf{z}
+ \frac{\rho}{2} \, \Vert \rho^{-1} \, \mathbf{z} + \mathbf{f}(\mathbf{x}) - \mathbf{y} \Vert_{A}^{2}
\label{eq:solvers:constrained-problem-augmented-lagrangian-1}
\end{equation}
\begin{equation}
= 
h(\mathbf{x}) 
- \frac{\rho^{-1}}{2} \, \Vert  \mathbf{z} \Vert_{A}^{2} 
+ \frac{\rho}{2} \, \Vert \rho^{-1} \, \mathbf{z} + \mathbf{f}(\mathbf{x}) - \mathbf{y} \Vert_{A}^{2}
\,\,,
\label{eq:solvers:constrained-problem-augmented-lagrangian-2}
\end{equation}
\label{eq:solvers:constrained-problem-augmented-lagrangian}
\end{subequations}
where $\mathbf{z} \in \mathbb{R}^{m}$ is the vector of Lagrange multipliers, $\rho \in \mathbb{R}$ is the additional penalty parameter, and $\Vert \mathbf{x} \Vert_{A}^{2} = \mathbf{x}^{T} \, \mathbf{A}\, \mathbf{x}$ denotes taking the norm in the metric space defined by $\mathbf{A}$. Next, we eliminate the slack variables $\mathbf{y}$ from (\ref{eq:solvers:constrained-problem-slacked}) by observing that $\mathcal{L}_{\rho,\mathbf{A}}^{A}(\mathbf{x}, \mathbf{y}, \mathbf{z})$ is minimized w.r.t $\mathbf{y} \in \mathcal{K}$ at the point
\begin{equation}
\displaystyle \displaystyle 
\mathbf{y}^{*} 
= \displaystyle\operatorname*{argmin}_{\mathbf{y} \in \mathcal{K}} \, \Vert \rho^{-1} \, \mathbf{z} + \mathbf{f}(\mathbf{x}) - \mathbf{y} \Vert_{A}^{2}
= \text{prox}_{\mathcal{K}}^{\mathbf{A}}(\rho^{-1} \, \mathbf{z} + \mathbf{f}(\mathbf{x}))
\,\,,
\label{eq:saddle-point-problem-minimizing-slack-variable}
\end{equation}
yielding the $\mathbf{y}$-minimized augmented Lagrangian
\begin{equation}
\displaystyle \displaystyle 
\mathcal{L}_{\rho,\mathbf{A}}^{A}(\mathbf{x}, \mathbf{y}^{*}, \mathbf{z}) = 
h(\mathbf{x}) 
- \frac{\rho^{-1}}{2} \Vert \mathbf{z} \Vert_{A}^{2} 
+ \frac{\rho}{2} \Vert \text{dist}_{\mathcal{K}}^{\mathbf{A}}(\rho^{-1}\mathbf{z} + \mathbf{f}(\mathbf{x})) \Vert_{A}^{2}
\,\,.
\label{eq:solvers:constrained-problem-augmented-lagrangian-reduced}
\end{equation}
The elimination of $\mathbf{y}$ renders the reduced saddle-point problem
\begin{equation}
\displaystyle \displaystyle 
\min_{\mathbf{x}} \, \max_{\mathbf{z}} \, \mathcal{L}_{\rho,\mathbf{A}}^{A}(\mathbf{x}, \mathbf{y}^{*}, \mathbf{z})
\,\, .
\label{eq:solvers:saddle-point-problem}
\end{equation}
Forming the KKT optimality conditions 
\begin{equation}
\nabla_{\mathbf{x}} \, \mathcal{L}_{\rho,\mathbf{A}}^{}(\mathbf{x}^{*}, \mathbf{y}^{*}, \mathbf{z}^{*}) = 0
\,\,,\,\,
\nabla_{\mathbf{z}} \, \mathcal{L}_{\rho,\mathbf{A}}^{}(\mathbf{x}^{*}, \mathbf{y}^{*}, \mathbf{z}^{*}) = 0
\,\,,
\label{eq:solvers:saddle-point-problem-optimality-conditions-def}
\end{equation}
for (\ref{eq:solvers:constrained-problem-augmented-lagrangian-reduced}), results in the system of equations
\begin{subequations}
\begin{equation}
\nabla_{\mathbf{x}}\,h(\mathbf{x}) 
+ \rho \, \mathbf{A} \, \text{dist}_{\mathcal{K}}^{\mathbf{A}}(\rho^{-1}\mathbf{z} + \mathbf{f}(\mathbf{x})) \, \nabla_{\mathbf{x}}\,\mathbf{f}(\mathbf{x})
= 0
\end{equation}
\begin{equation}
- \rho^{-1} \, \mathbf{A} \, \mathbf{z}
+ \mathbf{A} \, \text{dist}_{\mathcal{K}}^{\mathbf{A}}(\rho^{-1}\mathbf{z} + \mathbf{f}(\mathbf{x}))
= 0
\,\,,
\end{equation}  
\label{eq:solvers:saddle-point-problem-optimality-conditions}
\end{subequations}
and with a bit of refactoring we can arrive at
\begin{subequations}
\begin{equation}
\nabla_{\mathbf{x}}\,h(\mathbf{x}) + \mathbf{A}\,\mathbf{z}\, \nabla_{\mathbf{x}}\,\mathbf{f}(\mathbf{x}) = 0
\end{equation}
\begin{equation}
\mathbf{f}(\mathbf{x}) = \text{prox}_{\mathcal{K}}^{\mathbf{A}}(\mathbf{f}(\mathbf{x}) + \rho^{-1} \, \mathbf{z})
\,\,.
\end{equation}
\label{eq:solvers:saddle-point-problem-optimality-equations}
\end{subequations}

Now we can plug in the definition of the SOCP (\ref{eq:solvers:convex-ncp-socp}) to (\ref{eq:solvers:saddle-point-problem-optimality-equations}), and observe that the Lagrange multipliers of the dual problem correspond to $\mathbf{z} := -\hat{\mathbf{v}}$, i.e. negative augmented post-event constraint velocities. The resulting system is
\begin{subequations}
\begin{equation}
\hat{\mathbf{v}} = \mathbf{A} \, \left( \mathbf{D} \, \boldsymbol{\lambda} + \mathbf{v}_{f} + \mathbf{s} \right) 
\end{equation}
\begin{equation}
\boldsymbol{\lambda} = \text{prox}_{\mathcal{K}}^{\mathbf{A}}(\boldsymbol{\lambda} - \rho^{-1} \, \hat{\mathbf{v}})
\,\,,
\end{equation}
\label{eq:solvers:dual-problem-saddle-point-problem-optimality-equations}
\end{subequations}
which, when combined form the \textit{proximal-point equation}
\begin{equation}
\boldsymbol{\lambda} = \text{prox}_{\mathcal{K}}^{\mathbf{A}}(\boldsymbol{\lambda} - \rho^{-1} \, \mathbf{A}^{-1} \, (\mathbf{D} \, \boldsymbol{\lambda} + \mathbf{v}_{f} + \mathbf{s}))
\,\,.
\label{eq:solvers:dual-problem-proximal-point-equation}
\end{equation}
A crucial property of the proximal operator is that its solutions are also \textit{fixed-points} of the projection operation, i.e. 
\begin{equation}
\boldsymbol{\lambda}^{*} = \text{prox}_{\mathcal{K}}^{\mathbf{A}}(\boldsymbol{\lambda}^{*})
\,\,.
\label{eq:solvers:dual-problem-proximal-point-equation-fixed-point-def}
\end{equation}
Therefore, interpreting (\ref{eq:solvers:dual-problem-proximal-point-equation-fixed-point-def}) as a recursive procedure, and starting from some initial point $\boldsymbol{\lambda}^{0}$, successively applying the proximal operator should converge to one of its fixed-points in a finite number of steps. This is the so-called \textit{proximal-point algorithm} and can be conceptualized as
\begin{equation}
\boldsymbol{\lambda}^{*} \approx \boldsymbol{\lambda}^{n} = \text{prox}_{\mathcal{K}}^{\mathbf{A}} \circ ... \circ \text{prox}_{\mathcal{K}}^{\mathbf{A}}(\boldsymbol{\lambda}^{0})
\,\,.
\label{eq:dual-problem-proximal-point-equation-fixed-point-recursion}
\end{equation}

The final step involves realizing the $\text{prox}_{\mathcal{K}}^{\mathbf{A}}(\cdot)$ operator. The specifics of how this is done is exactly that which distinguishes most algorithms from one another. To recapitulate, in order to construct a proximal operator we need to choose:
\begin{itemize}
\item an appropriate metric norm $\mathbf{A}$
\item an appropriate representation of $\rho$
\item a method to realize the projection
\item an iterative procedure
\end{itemize}
Moreover, another crucial choice is whether to realize the projection as one \textit{global} operation accounting for the whole problem at once, or break it up into chunks that can be solved \textit{locally}, depending of course on the structure of $\mathbf{A}$. The latter is what is referred to in the optimization literature as \textit{block-coordinate descent}~\cite{wright2015coordinate}. We will leave the global approach for Sec.~\ref{sec:solvers:admm} where it is used in the ADMM-based solver. For now, we will focus on the block-wise methods that includes PGS and its variants. If $\mathbf{A}$ is block-diagonal, then we can express the block-wise operation
\begin{equation}
\begin{array}{c}
\displaystyle \mathbf{z}_{j}^{i} = \left (\sum_{m=1}^{n_{j} + n_{l} + n_{c}} \mathbf{D}_{mj} \, \boldsymbol{\lambda}_{m}^{i} \right) + \mathbf{v}_{f,j} + \mathbf{s}_{j} \\[8pt]
\boldsymbol{\lambda}_{j}^{i+1} = \text{prox}_{\mathcal{K}_{j}}^{\mathbf{A}_{jj}} \left( \boldsymbol{\lambda}_{j}^{i} - \frac{1}{\rho} \, \mathbf{A}_{jj}^{-1} \, \mathbf{z}_{j}^{i} \right)
\end{array}
\,\,,
\label{eq:block-wise-proximal-operator}
\end{equation}
and, if $\mathbf{A}$ is purely diagonal, the coordinate-wise version is
\begin{equation}
\begin{array}{c}
z_{j}^{i} = \sum_{m=1}^{n_{d}} D_{mj} \, \lambda_{m}^{i} + v_{f,j} + s_{j}\\[6pt]
\lambda_{j}^{i+1} = \text{prox}_{\mathcal{K}_{j}}^{a_{j}} \left(\lambda_{j}^{i} - \rho^{-1}\,a_{j}^{-1} \, z_{j}^{i} \right)
\end{array}
\,\,.
\label{eq:constraint-wise-proximal-operator}
\end{equation}
We should now recognize that (\ref{eq:block-wise-proximal-operator}) and (\ref{eq:constraint-wise-proximal-operator}) correspond to a single Jacobi iteration, followed by a projection onto the feasible set of each coordinate. We can therefore alternatively use a Gauss-Seidel (GS) iteration~\cite{strang2006linear}, in the form of
\begin{equation}
\begin{array}{c}
\displaystyle
\mathbf{z}_{j}^{i} := 
\sum_{m=1}^{m \leq j} \mathbf{D}_{mj} \boldsymbol{\lambda}_{m}^{i+1} 
+ \sum_{m=j+1}^{n_{j} + n_{l} + n_{c}} \mathbf{D}_{mj} \boldsymbol{\lambda}_{m}^{i}
+ \mathbf{v}_{f,j} + \mathbf{s}_{j}\\[14pt]
\boldsymbol{\lambda}_{j}^{i+1} = \text{prox}_{\mathcal{K}_{j}}(\, \boldsymbol{\lambda}_{j}^{i} - \rho^{-1}\,\mathbf{A}_{jj}^{-1} 
\, \mathbf{z}_{j}^{i} \,)
\end{array}
,
\label{eq:constraint-wise-proximal-operator-gauss-seidel-iteration}
\end{equation}
\begin{equation}
\begin{array}{c}
\displaystyle
z_{j}^{i} := \sum_{m=1}^{m \leq j} D_{mj} \lambda_{m}^{i+1} + \sum_{m=j+1}^{n_{d}} D_{mj} \lambda_{m}^{i} + v_{f,j} + s_{j}\\[14pt]
\lambda_{j}^{i+1} = \text{prox}_{\mathcal{K}_{j}}(\, \lambda_{j}^{i} -\rho^{-1}\,A_{jj}^{-1} \, z_{j}^{i} \,)
\end{array}
\,\,.
\label{eq:element-wise-proximal-operator-gauss-seidel-iteration}
\end{equation}
We can also use a Successive-Over-Relaxation (SOR) update
\begin{equation}
\boldsymbol{\lambda}_{j}^{i+1} =
(1 - \omega) \boldsymbol{\lambda}_{j}^{i} 
+ \omega \boldsymbol{\lambda}_{j}^{i+1}
\,\,,
\label{eq:solvers-sor}
\end{equation}
after every block-wise or coordinate-wise proximal projection. $\omega \in [0,2]$ is the so-called relaxation factor that linearly interpolates between the previous and current iterates. This additional operation presents some interesting theoretical and practical properties, and is often used to accelerate convergence. However, it can potentially lead to divergence if not selected carefully, and choosing a value that generalizes across all problems can prove difficult. Choosing $\omega = 1$ corresponds to the GS iteration and thus this additional mechanism is a very versatile and easy extension to apply. 

Essentially, all of the aforementioned \textit{local} approaches treat each sub-problem as though all other constraint reactions are fixed constants. Moreover, the structure of the proximal-point algorithm is akin to \textit{projected gradient descent}. As indicated by (\ref{eq:solvers:saddle-point-problem-optimality-equations}), the velocity term corresponds to an estimate of the gradient, $\mathbf{A}$ serves as an approximate Hessian and $\rho$ can be understood as a step-size. Lastly, the projection onto the feasible set via the proximal operator exclusively serves to satisfy the set-valued and inequality constraints. We have thus laid the foundation which unifies the set for tackling the dual FD NCP. It remains now to choose how to realize the projection and select $\rho$ and $\mathbf{A}$, and will be covered next.

\subsection{Projectors \& Local Solvers}
\label{sec:solvers:projectors}
\noindent
This section describes the set of projection operators that are used to realize the proximal operator (\ref{eq:solvers:proximal-distance-operator-to-set}). As the latter can be considered more of a theoretical construct as opposed to a literal operation, we will denote an \textit{implementation} thereof as a projection operator, i.e. \textit{projector}, $\mathcal{P}^{\mathcal{K}}(\cdot)$ introduced in the template of Alg.\ref{alg:generic-first-order-dual-solver}. Many of them are derived analytically in closed-form, while some require solving an inner optimization problem. Recalling the construction (\ref{eq:construction:total-feasible-set-product}) used for the composite set $\mathcal{K}$ of admissible constraint reactions, we can equivalently construct a composite projector as 
\begin{equation}
\mathcal{P}^{\mathcal{K}}(\mathbf{x}) :=
\begin{bmatrix}
\mathcal{P}^{\mathcal{K}_{1}}(\mathbf{x}_{1}) \\
\vdots\\
\mathcal{P}^{\mathcal{K}_{n_i}}(\mathbf{x}_{n_i})
\end{bmatrix}
\,\,,
\end{equation}
where $n_i$ is the number of inequality constraint sets. Note that, in general, $n_i \neq n_l + n_c$, as the constraint sets can differ depending on the contact model e.g. decoupled or spatial.

\paragraph*{\textbf{Euclidean Projectors}}
The first set consists of geometric projectors admitting analytical expressions. Specifically, they realize Euclidean projections onto their respective cones, as they are derived assume the Euclidean metric norm. These include the:
\begin{itemize}
\item projector to the nonnegative orthant $\mathbb{R}_{+}$
\begin{equation}
\mathcal{P}^{\mathbb{R}_{+}}(x) := \max(0, x)
\,\,.
\label{eq:solvers:projectors:positive-reals}
\end{equation}
\item projector to a 2D disk $\mathcal{D}(r) \subset \mathbb{R}^{2}$ of radius $r \in \mathbb{R}_{+}$:
\begin{equation}
\mathcal{P}^{\mathcal{D}(r)}(\mathbf{x}) :=
\begin{cases}
\mathbf{x} &, \Vert \mathbf{x} \Vert_{2} \leq r \\
r \, \frac{\mathbf{x}}{\Vert \mathbf{x} \Vert_{2}} &, \Vert \mathbf{x} \Vert_{2} > r \\
\end{cases}
\,\,
\label{eq:solvers:disc-projection}
\end{equation}
\item projector to a generic 3D Lorentz cone $\mathcal{K}_{3} \subset \mathbb{R}^{3}$:
\begin{flalign}
\mathcal{P}^{\mathcal{K}_{3}}(\mathbf{x}) :=
\begin{cases}
0 &, \Vert \mathbf{x}_{xy} \Vert_{2} \leq -\text{x}_{z} \\
\mathbf{x} &, \Vert \mathbf{x}_{xy} \Vert_{2} \leq \text{x}_{z} \\
\displaystyle \frac{1}{2} \left(1 + \frac{x_z}{\Vert \mathbf{x}_{xy} \Vert_{2}} \right) 
\begin{bmatrix}
\mathbf{x}_{xy}\\
\Vert \mathbf{x}_{xy} \Vert_{2}
\end{bmatrix}
&, \Vert \mathbf{x}_{xy} \Vert_{2} > |\text{x}_{z}| \\
\end{cases}
\,\,
\label{eq:solvers:lorentz-cone-projection}
\end{flalign}
\item projector to a Coulomb friction cone $\mathcal{K}_{\mu} \subset \mathbb{R}^{3}$
\begin{flalign}
\mathcal{P}^{\mathcal{K}_{\mu}}(\mathbf{x}) :=
\begin{cases}
0 &, \Vert \mathbf{x}_{t} \Vert_{2} \leq -\mu^{-1} \, \text{x}_{n} \\
\mathbf{x} &, \Vert \mathbf{x}_{t} \Vert_{2} \leq \mu \, \text{x}_{n} \\
\displaystyle \frac{\mu \, \Vert \mathbf{x}_{t} \Vert_{2} + \text{x}_{n}}{\mu^{2} + 1}
\begin{bmatrix}
\frac{\mu \, \mathbf{x}_{t}}{\Vert \mathbf{x}_{t} \Vert_{2}} \\
1
\end{bmatrix}
&, \Vert \mathbf{x}_{t} \Vert_{2} > \mu \, |\text{x}_{n}| \\
\end{cases}
\,\,,
\label{eq:solvers:coulomb-cone-projection}
\end{flalign}
\item projector to the dual Coulomb friction cone $\mathcal{K}_{\mu}^{*} := \mathcal{K}_{\frac{1}{\mu}}$
\begin{flalign}
\mathcal{P}^{\mathcal{K}_{\mu}^{*}}(\mathbf{x}) :=
\begin{cases}
0 &, \Vert \mathbf{x}_{t} \Vert_{2} \leq -\mu \, \text{x}_{n} \\
\mathbf{x} &, \Vert \mathbf{x}_{t} \Vert_{2} \leq \mu^{-1} \, \text{x}_{n} \\
\displaystyle \frac{\Vert \mathbf{x}_{t} \Vert_{2} + \mu \, \text{x}_{n}}{\mu^{2} + 1}
\begin{bmatrix}
\frac{\mathbf{x}_{t}}{\Vert \mathbf{x}_{t} \Vert_{2}} \\
\mu
\end{bmatrix}
&, \Vert \mathbf{x}_{t} \Vert_{2} > \mu^{-1} \, |\text{x}_{n}| \\
\end{cases}
\,\,.
\label{eq:solvers:dual-coulomb-cone-projection}
\end{flalign}
\end{itemize}
Note that the Coulomb cone and its dual are essentially specializations of the 3D Lorentz cone (\ref{eq:solvers:lorentz-cone-projection}). Moreover, although omitted from the aforedescribed list, $\mathbb{R}^{n}$ also admits a projector in the form of the identity function $\mathcal{P}^{\mathbb{R}^{n}}(\mathbf{x}) = \mathbf{x}$.\\

\paragraph*{\textbf{Local Solvers}} The second set consists of projectors that are additionally embedded with a contact model, in contrast to the aforedescribed primitives which are purely geometric. They can explicitly incorporate the Signorini-Coulomb contact model, resulting in projections to the Coulomb friction cone that satisfy the complementarity conditions. They are therefore solvers of the single-contact dual problem and are used in constraint-wise iteration schemes.

The first such projectors we consider are $\mathcal{P}_{CCP}(\cdot)$ and $\mathcal{P}_{NCP}(\cdot)$ that are respectively defined in Alg.~\ref{alg:ccp-local-projection} and Alg.~\ref{alg:ncp-local-projection}. These are very simple approaches to realizing the proximal projection that involve nothing more than a single-step predictor-corrector operation. Specifically, $\mathcal{P}_{CCP}(\cdot)$ can be interpreted as assuming $\mathbf{A} = \mathbb{I}_{n_d}$ and $\rho$ set to the inverse of the average diagonal of the local Delassus matrix (i.e. respective diagonal block). A subtlety of this approach is that if an estimate of the local De Saxc\'e correction is omitted, then the dual variable is effectively the contact velocity $\mathbf{v}_{j}^{+}$, which leads to non-zero solutions along the contact normal. This is because it only enforces complementarity between primal and dual variables, without explicit consideration of the reaction and velocity along the contact normal. 
\begin{algorithm}[!b]
\setstretch{1.5}
\caption{CCP-Type Local Solver $\mathcal{P}_{CCP}(\cdot)$}
\label{alg:ccp-local-projection}
\begin{algorithmic}[1]
\Require $\boldsymbol{\lambda}_{j}$, $\mathbf{D}_{jj}$, $\mathbf{v}_{f,j}$, $\mathbf{s}_{j}$, $\mu$
\textcolor{gray}{\Comment{local problem definition}}
\State $\mathbf{z}_{j} \gets \mathbf{D}_{jj} \, \boldsymbol{\lambda}_{j} + \mathbf{v}_{f,j} + \mathbf{s}_{j}$
\textcolor{gray}{\Comment{compute local velocity}}
\State $\boldsymbol{\lambda}_{j} \gets \boldsymbol{\lambda}_{j} - \displaystyle\frac{3}{\textbf{tr}(\mathbf{D}_{j})}\,\mathbf{z}_{j}$
\State $\boldsymbol{\lambda}_{j} \gets \mathcal{P}^{\mathcal{K}_{\mu}}(\boldsymbol{\lambda}_{j})$
\textcolor{gray}{\Comment{Coulomb cone projector (\ref{eq:solvers:coulomb-cone-projection})}}\\
\Return $\boldsymbol{\lambda}_{j}$
\end{algorithmic}
\end{algorithm}
\begin{algorithm}[!b]
\setstretch{1.7}
\caption{NCP-Type Local Projector $\mathcal{P}_{NCP}(\cdot)$}
\label{alg:ncp-local-projection}
\begin{algorithmic}[1]
\Require $\boldsymbol{\lambda}_{j}$, $\mathbf{D}_{jj}$, $\mathbf{v}_{f,j}$, $\mathbf{s}_{j}$, $\mu$
\textcolor{gray}{\Comment{local problem definition}}
\State $\mathbf{z}_{j} \gets \mathbf{D}_{jj} \, \boldsymbol{\lambda}_{j} + \mathbf{v}_{f,j} + \mathbf{s}_{j}$
\textcolor{gray}{\Comment{compute local velocity}}
\State $\lambda_{j,N} \gets \lambda_{j,N} - \displaystyle\frac{1}{D_{j,NN}}\,z_{j,N}$
\State $\lambda_{j,N} \gets \mathcal{P}^{\mathbb{R}_{+}}(\lambda_{j,N})$
\textcolor{gray}{\Comment{orthant projector (\ref{eq:solvers:projectors:positive-reals})}}
\State $\boldsymbol{\lambda}_{j,T} \gets \boldsymbol{\lambda}_{j,T} - \displaystyle\frac{1}{\min(D_{j,TT,x}, D_{j,TT,y})}\,\mathbf{z}_{j,T}$
\State $\boldsymbol{\lambda}_{j,T} \gets \mathcal{P}^{\mathcal{D}(\mu\,\lambda_{j,N})}(\boldsymbol{\lambda}_{j,T})$
\textcolor{gray}{\Comment{disk projector (\ref{eq:solvers:disc-projection})}}\\
\Return $\boldsymbol{\lambda}_{j}$
\end{algorithmic}
\end{algorithm}

Conversely, $\mathcal{P}_{NCP}(\cdot)$ addresses this issue by decomposing the problem along the normal and tangent directions as two separate problems. This ensures zero contact velocity along the contact normal by forcing the normal contact reaction to be non-negative regardless of the presence of a De Saxc\'e correction term. For this projector, again $\mathbf{A} = \mathbb{I}_{n_d}$ but $\rho$ is different for the normal and tangential directions. For the normal direction, it is taken directly as the inverse of the respective diagonal component of $\mathbf{D}$, while for both tangential components, the smallest of the two diagonal components is used. However, this projector can lead to contact reactions that are not maximally dissipative since the decoupling cannot account for the couplings induced by other constraints. 

Both $\mathcal{P}_{CCP}(\cdot)$ and $\mathcal{P}_{NCP}(\cdot)$ can be considered as \textit{naive} projectors. A more holistic approach would be to consider solving each local problem exactly, accounting for all other couplings brought by the local dual variable $\mathbf{z}_{j}^{i}$. This approach has been described in works such as that of Studer in~\cite{studer2008augmented}, Bonnefon et al in~\cite{bonnefon2011quartic}, and Todorov in~\cite{todorov2014analytical}. These methods solve a single-contact problem by taking the geometric perspective described in Sec.~\ref{sec:models:contacts} and treating the open, stick and sliding cases explicitly. For the first two cases, the projection is trivially zero and the unconstrained solution of the linear system (i.e. initially assuming the dual variable is zero). For the sliding case, the solution is found by considering the conic section formed by the Coulomb cone and plane of maximum compression. Finding the correct point amounts to finding a point on a 2D ellipse which can be rendered as a solution to a Quadratic Cone Program (QCP). In the approaches of~\cite{studer2008augmented,bonnefon2011quartic}, the QCP was used to derive a semi-analytical solution in the form of a root-finding problem of a \textit{quartic polynomial}. In this work we will denote such QCP-based local solvers as the Nonlinear-Block projector $\mathcal{P}_{NB}(\cdot)$, defined in Alg.~\ref{alg:nb-local-projection}. Such a projector can used to realize a form of dual solver some authors refer to as the Nonlinear-Block Gauss Seidel method~\cite{ortega2000iterative, preclik2014models,todorov2014analytical}. However, the derivation of the QCP can vary depending on certain modeling assumptions.
\begin{algorithm}[!b]
\setstretch{1.1}
\caption{Nonlinear-Block Local Solver $\mathcal{P}_{NB}(\cdot)$}
\label{alg:nb-local-projection}
\begin{algorithmic}[1]
\Require $\boldsymbol{\lambda}_{j}$, $\mathbf{D}_{jj}$, $\mathbf{v}_{f,j}$, $\mathbf{s}_{j}$, $\mu$
\textcolor{gray}{\Comment{local problem definition}}
\State $ \boldsymbol{\lambda}_{j}^{0} \gets -\mathbf{D}_{jj}^{-1} \, \left( \mathbf{v}_{f,j} + \mathbf{s}_{j} \right)$
\textcolor{gray}{\Comment{unconstrained solution}}
\If{$\text{v}_{f,n} < 0$}
    \If{$\mu = 0$}
        \State $\boldsymbol{\lambda}_{j} = [0 ,\, 0 ,\, -D_{nn}^{-1}\,v_{f,n}]^{T}$ \textcolor{gray}{\Comment{contact is frictionless}}
    \Else
        \If{$\boldsymbol{\lambda}_{0} \in \mathcal{S}\cap\mathcal{K}_{\mu}$}
            \State $\boldsymbol{\lambda}_{j} = \boldsymbol{\lambda}_{0}$ \textcolor{gray}{\Comment{contact is sticking}}
        \Else
            \If{$\text{v}_{f,n} = 0$}
                \If{$r_{d/\mu} > 1$}
                    \State $\boldsymbol{\lambda}_{j}= 0$ \textcolor{gray}{\Comment{degenerate elliptic}}
                \Else
                    \State $\boldsymbol{\lambda}_{i}$ from (\ref{eq:degenerate-hyperbolic-conic-section}) \textcolor{gray}{\Comment{degenerate para/hyperbolic}}
                \EndIf
            \Else
                \State $\boldsymbol{\lambda}_{j}$ from (\ref{eq:polar-angle-quartic-polynomial}) \textcolor{gray}{\Comment{contact is slipping}}
            \EndIf
        \EndIf
    \EndIf
\Else
    \State $\boldsymbol{\lambda}_{j} = 0$ \textcolor{gray}{\Comment{contact is open}}
\EndIf\\
\Return $\boldsymbol{\lambda}_{j}$
\end{algorithmic}
\end{algorithm}

In this work we have realized the variants proposed by Preclik et al in~\cite{preclik2018mdp} and by Hwangbo et al in~\cite{hwangbo2018percontact}, respectively denoted as $\mathcal{P}_{MDP}(\cdot)$ and $\mathcal{P}_{BS}(\cdot)$. Both of these projectors are based on the derivation originally described by Preclik et al in~\cite{preclik2018mdp}, but differ in how the QCP is solved. In the case of the former, the single-contact problem is transcribed as a single-variable QCP that is solved as a quartic polynomial, similar to~\cite{studer2008augmented,bonnefon2011quartic,todorov2014analytical}. It differs to the aforementioned, however, in that the QCP is derived using a model that employs an alternate form of the disjunctive Signorini-Coulomb model (\ref{eq:contacts:disjunctive-signorini-coulomb}), in conjunction with a objective function derived from the MDP (\ref{eq:contacts:maximum-dissipation-principle}). This approach was later adapted by Hwangbo et al in~\cite{hwangbo2018percontact}, who revised it by solving the QCP using a bisection-search (BS) method, thusly employing it in the RaiSim simulator. A recent analysis of the former by Lidec et al in~\cite{lidec2024reconciling}, identified that the QCP formulation was actually violating the MDP when strong couplings between tangential and normal reactions occur, and proposed an amendment based on approximating the De Saxc\'e correction term (\ref{eq:models:de-saxce-correction}).

We will briefly summarize the definition of the projector through the lens of proximal operators, while a complete derivations can be found in Appendix.~\ref{sec:apndx:single-contact}. Readers are also referred to~\cite{preclik2018mdp, hwangbo2018percontact, lidec2024reconciling} for further details. In our context, the local-contact QCP derived by Preclik that also incorporates the De Saxc\'e correction, takes the form of the problem
\begin{equation}
\begin{array}{rlclcl}
\textbf{Find} \quad
\boldsymbol{\lambda}_{j}
=
\displaystyle
\operatorname*{argmin}_{\mathbf{x}} \,
    & \frac{1}{2} \, \mathbf{x}^{T} \, \mathbf{D}_{jj} \, \mathbf{x}
    + \mathbf{x}^{T} \, \left( \mathbf{v}_{f,j} + \mathbf{s}_{j} \right)\\[4pt]
\textrm{s.t.} & \mathbf{x} \in \mathcal{K}_{\mu_j} \\[4pt]
              & \mathbf{D}_{jj} \, \mathbf{x} + \mathbf{v}_{f,j} \geq 0 \\[4pt]
              & \mathbf{D}_{jj} \, \mathbf{x} \geq 0
\end{array}
\,\,.
\label{eq:solvers:projectors::single-contact-mdp-problem}
\end{equation}
The additional inequality constraints correspond to the Signorini condition on the post-event contact velocity, and to the requirement that the contact reaction does not induce a velocity bias along the plane normal. Although the former is typically understood to be a condition of optimality, including it explicitly admits its geometric reinterpretation as finding a point on a conic section. Keeping to the broader view, we will express (\ref{eq:solvers:projectors::single-contact-mdp-problem}) as a proximal operator. Denoting the inequality constraints as inclusions to the sets
\begin{subequations}
\begin{equation}
\mathbf{n} := [0,\, 0,\, 1]^{T}
\,\,,
\label{eq:solvers:projectors::single-contact-local-normal-vector}
\end{equation}
\begin{equation}
\mathcal{V}_{\mathbf{n}} := \{ {\mathbf{x}} \,:\, \mathbf{n}^{T} \, \left( \mathbf{D}_{jj} \, \mathbf{x} + \mathbf{v}_{f,j} \right) \geq 0 \}
\,\,,
\label{eq:solvers:projectors::single-contact-normal-velocity-set}
\end{equation}
\begin{equation}
\Lambda_{\mathbf{n}} := \{ {\mathbf{x}} \,:\, \mathbf{n}^{T} \, \mathbf{D}_{jj} \, \mathbf{x} \geq 0 \}
\label{eq:solvers:projectors::single-contact-normal-reaction-set}
\end{equation}
\end{subequations}
we can recast (\ref{eq:solvers:projectors::single-contact-mdp-problem}) as a proximal projection in the form of
\begin{subequations}
\begin{equation}
\displaystyle
\text{prox}_{\mathcal{V}_{\mathbf{n}}\cap\mathcal{K}_{\mu}}^{\mathbf{D}}(\mathbf{x}) := 
\operatorname*{argmin}_{\mathbf{y}} \, 
f_{\text{MDP}}(\mathbf{y}, \mathbf{x})
\end{equation}
\begin{flalign}
f_{\text{MDP}}(\mathbf{x}, \mathbf{y}) := &\,\,
\Psi_{\mathcal{K}_{\mu_j}}(\mathbf{y})
+ \Psi_{\mathcal{V}_{\mathbf{n}}}(\mathbf{y}) 
\nonumber\\
&
+ \Psi_{\Lambda_{\mathbf{n}}}(\mathbf{y}) 
+ \frac{1}{2} \, \Vert \mathbf{x} - \mathbf{y} \Vert_{\mathbf{D}_{jj}}^{2}
\,\,.
\label{eq:solvers:projectors::single-contact-mdp-proximal-operator}
\end{flalign}
\end{subequations}
Note that this equivalence holds only if $\mathbf{x} := -\mathbf{D}_{jj}\,(\mathbf{v}_{f,j} + \mathbf{s}_{j})$. Plugging in the unconstrained solution to (\ref{eq:solvers:projectors::single-contact-mdp-proximal-operator}), yields (\ref{eq:solvers:projectors::single-contact-mdp-problem}). This projector therefore corresponds to a proximal operator with the single-contact Delassus matrix as a weighted metric norm, i.e. $\mathbf{A} = \mathbf{D}_{jj}$ and $\rho = 1$. The realization of (\ref{eq:solvers:projectors::single-contact-mdp-proximal-operator}) as a projector applicable to per-constraint iteration is thus
\begin{flalign}
\mathcal{P}_{\text{MDP}}(\mathbf{x}) :=
\begin{cases}
0 &, \bar{\text{v}}_{f,n} < 0 \\
\mathbf{x} &, \bar{\text{v}}_{f,n} \geq 0 ,\, \mathbf{x} \in \mathcal{K}_{\mu} \\
\text{QCP}(\mathbf{x}, \bar{\mathbf{D}}, \bar{\mathbf{v}}_{f}, \mathcal{K}_{\mu}) &, \bar{\text{v}}_{f,n} \geq 0 ,\, \mathbf{x} \not\in \mathcal{K}_{\mu}
\end{cases}
\,\,.
\label{eq:solvers:signorini-coulomb-mdp-projector}
\end{flalign}

\subsection{Convergence}
\label{sec:solvers:convergence}
\noindent
This section describes how we may realize the termination criteria functions $f_{stop}(\cdot)$ of Alg.~\ref{alg:generic-first-order-dual-solver}. Although such operations do not directly affect the convergence trajectories of the corresponding algorithms, they are critical to the quality of the resulting solutions and useful for assessing solver performance. Thus, special care must be taken in selecting which quantities should be evaluated in order to terminate a given solver, as the criteria must reflect the effect of the algorithm.

Given the transcription of the problem as the NSOCP (\ref{eq:formulation:dual-forward-dynamics-nsocp}), the most relevant measures to consider, would be the \textit{primal, dual and complementarity residuals}, defined respectively as
\begin{equation}
\textbf{r}_{p}(\boldsymbol{\lambda}) :=
\boldsymbol{\lambda} - \mathcal{P}^{\mathcal{K}}(\boldsymbol{\lambda})
\label{eq:solvers:convergence:primal-residual}
\end{equation}
\begin{equation}
\textbf{r}_{d}(\mathbf{v}) :=
\mathbf{v} - \mathcal{P}^{\mathcal{K}^{*}}(\mathbf{v})
\,\,,
\label{eq:solvers:convergence:dual-residual}
\end{equation}
\begin{equation}
\displaystyle
\textbf{r}_{cp}(\boldsymbol{\lambda}, \mathbf{v}) :=
[\, \boldsymbol{\lambda}_{i}^{T} \, \mathbf{v}_{i} \,]_{n_l+n_c}
\,\,,
\label{eq:solvers:convergence:complementarity-residual}
\end{equation}
where $[x_{i}]_{n} := [x_1, \cdots, x_1]^{T}$. Note how the primal and dual residuals are conveniently expressed using the projectors described in Sec.~\ref{sec:solvers:projectors}. This property reflects the fact that the primal and dual variables of the problem should satisfy the normal-cone inclusions (\ref{eq:solvers:cone-indicator-complementarity}). Thus, a measure of the error corresponds to the proximal distance operator (\ref{eq:solvers:proximal-distance-operator-to-function}).

There is a crucial subtlety regarding the formulation of the  dual and complementarity residuals compared to the primal. While, the primal residual (\ref{eq:solvers:convergence:primal-residual}) is a universal measure that is applicable in all cases, (\ref{eq:solvers:convergence:dual-residual}) and (\ref{eq:solvers:convergence:complementarity-residual}) must be evaluated according to whether the algorithm is solving a NCP or CCP formulation of the dual problem. For the NCP, they would be evaluated using the augmented constraint velocity $\hat{\mathbf{v}}^{+}$ that includes the De Saxc\'e correction (\ref{eq:augmented-constraint-velocity}), while for the CCP, it would be the post-event constraint velocity $\mathbf{v}^{+}$ (\ref{eq:post-event-constraint-velocity}).

Moreover, we can also consider the \textit{iterate residual}
\begin{equation}
\textbf{r}_{iter}(\boldsymbol{\lambda}^{i}, \boldsymbol{\lambda}^{i-1}) :=
\boldsymbol{\lambda}^{i} - \boldsymbol{\lambda}^{i-1}
\,\,,
\label{eq:solvers:convergence:iterate-residual}
\end{equation}
i.e. the relative difference between successive iterates $\boldsymbol{\lambda}^{i}$ and $\boldsymbol{\lambda}^{i+1}$ of the solution. This alternative measure can be used in several ways. Firstly, some approaches described in the literature employ (\ref{eq:solvers:convergence:iterate-residual}) as a computationally cheaper method of determining convergence~\cite{acary2018comparisons}, when the evaluation of (\ref{eq:solvers:convergence:primal-residual}-\ref{eq:solvers:convergence:complementarity-residual}) may prove costly. Second, it can be used in conjunction with problem residuals as an auxiliary criteria for detecting solver stagnation, and thus invoke early stopping of the algorithm.

Some simulators like MuJoCo~\cite{todorov2012mujoco}, option for using alternative convergence criteria based on the objective function. Specifically, in the case of MuJoCo's GPGS solver, the quadratic objective (\ref{eq:formulation:quadratic-objective}) is evaluated over successive solution iterates in order to compute a \textit{relative improvement} of the objective function, that is scaled by the total system inertia:
\begin{equation}
r_{f} := I_{total}^{-1} \, |\, f(\boldsymbol{\lambda}^{i}) - f(\boldsymbol{\lambda}^{i-1}) \,|
\,\,,
\label{eq:solvers:convergence:early-stopping-criteria}
\end{equation}
where $I_{total} > 0$ is the total diagonal inertial, i.e. the sum of diagonal terms of the generalized mass matrix $\mathbf{M}$. Essentially, this heuristic aims at identifying plateaus in the optimization in order to perform early stopping, thus preventing the algorithm from iterating without significant effect. The scaling using the total diagonal inertia, can be understood to improve conditioning w.r.t large mass ratios present in the system.

In addition to the aforedefined metrics, another useful element from VI theory is the \textit{natural map function}
\begin{equation}
\mathbf{F}_{\text{VI}}^{\text{nat}}(\mathbf{x}) :=
\mathbf{x}
-
\mathcal{P}^{\mathcal{K}}
\left( \mathbf{x} - \mathbf{F}(\mathbf{x}) \right)
\,\,,
\label{eq:solvers:convergence:vi-natural-map}
\end{equation}
which can be stated generically for any NCP of the form 
\begin{equation}
\mathcal{K}^{*} \ni \mathbf{F}(\mathbf{x})
\perp
\mathbf{x} \in \mathcal{K}^{*}
\,\,.
\label{eq:solvers:general-ncp}
\end{equation}

\begin{figure}[!t]
\centering
\includegraphics[width=0.8\linewidth]{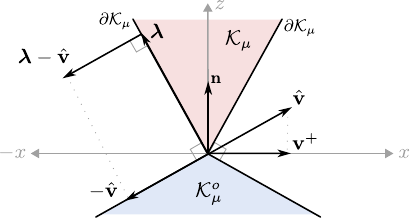}
\caption{A depiction of the geometric interpretation of the natural-map function. The VI functional is defined to be the augmented constraint velocity $\hat{\mathbf{v}} := \mathbf{F}(\boldsymbol{\lambda})$. If $\boldsymbol{\lambda} \in \partial \mathcal{K}_{\mu}$ and $-\mathbf{F}(\boldsymbol{\lambda}) \in \partial \mathcal{K}^{o}$ with $\boldsymbol{\lambda} \perp \mathbf{F}(\boldsymbol{\lambda})$, then the vector $\boldsymbol{\lambda} - \mathbf{F}(\boldsymbol{\lambda})$ originating at $\boldsymbol{\lambda}$, is exactly perpendicular to the surface of $\mathcal{K}_{\mu}$. This means that its projection onto the cone should yield $\mathbf{x}$, i.e. $\mathbf{F}_{\text{VI}}^{\text{nat}}(\boldsymbol{\lambda}) = 0$.}
\label{fig:natural-map-geometry}
\end{figure}
This functional stems from a rather intuitive geometric interpretation of (\ref{eq:solvers:general-ncp}) that is depicted in Fig.\ref{fig:natural-map-geometry}. In~\cite{eaves1971,facchinei2003finite,acary2008numerical}, it is proven that $\mathbf{F}_{\text{VI}}^{\text{nat}}(\mathbf{x}^{*}) = 0$ for some vector $\mathbf{x}^{*} \in \mathcal{K}$, is a necessary and sufficient condition for $\mathbf{x}^{*}$ to be a solution to the NCP (\ref{eq:solvers:general-ncp}). The utility of this metric is often overlooked in the robotics and graphics fields, but as Acary et al demonstrate in~\cite{acary2018comparisons}, it can serve as a reliable metric for qualifying if the output of solvers actually correspond to solutions of the NCP/CCP. In effect, it subsumes the primal, dual and complementarity residuals, summarizing them into a single expression, albeit one that is significantly more expensive to compute. In the case of the dual FD problems (\ref{eq:formulation:dual-forward-dynamics-ncp}) and (\ref{eq:formulation:dual-forward-dynamics-ccp}), the functional $\mathbf{F}(\cdot)$ corresponds to $\hat{\mathbf{v}}^{+}$ and $\mathbf{v}^{+}$, respectively. Thus, the natural-map functions are for each problem are
\begin{equation}
\mathbf{F}_{\text{NCP}}^{\text{nat}}(\boldsymbol{\lambda}) :=
\boldsymbol{\lambda}
-
\mathcal{P}^{\mathcal{K}}
\left(
\boldsymbol{\lambda} - \mathbf{D}\,\boldsymbol{\lambda}  + \mathbf{v}_{f} + 
\boldsymbol{\Gamma}
\left(
\mathbf{v}^{+}(\boldsymbol{\lambda}) 
\right)
\right)
\,\,,
\label{eq:solvers:convergence:dual-fd-ncp-natural-map}
\end{equation}
\begin{equation}
\mathbf{F}_{\text{CCP}}^{\text{nat}}(\boldsymbol{\lambda}) :=
\boldsymbol{\lambda}
-
\mathcal{P}^{\mathcal{K}}
\left(
\boldsymbol{\lambda} - \mathbf{D}\,\boldsymbol{\lambda}  + \mathbf{v}_{f} 
\right)
\,\,.
\label{eq:solvers:convergence:dual-fd-ccp-natural-map}
\end{equation}

Lastly, we must clarify how the aforedescribed vector-valued residuals and metrics may be used in practice. The two most prominent approaches employed in the majority of numerical methods are the averaged $L_2$ and $L_\infty$ norms. The former computes $n_{d}^{-1} \Vert \cdot \Vert_{2}$, and can serve as a measure of the average error. The latter computes $\Vert \cdot \Vert_{\infty}$ and effectively renders the worst-case error. In this work, however, employ the $L_\infty$ norm as it scales better w.r.t problem dimensions. 

\subsection{Projective Splitting Methods}
\label{sec:solvers:projective-splitting-methods}
\noindent
With all the aforedescribed building-blocks in hand, we can now define the first set of concrete algorithms we will use to solve the FD dual NCP. We will refer to these as \textit{Projective Splitting Methods}, in accordance with the taxonomy described in~\cite{acary2018comparisons}. Generally speaking, these algorithms all are block-wise versions of the SOR method described in Sec.~\ref{sec:solvers:proximal-point-derivation}. SOR is the most general of the so-called fixed-point iteration methods based on matrix splitting\footnote{SOR, Gauss-Seidel and the Jacobi method are so-called \textit{splitting methods}. These are a family of the more broader class of fixed-point iteration methods that also includes \textit{Krylov subspace methods} such as the linear Conjugate Gradient (CG) algorithm, and the Minimum Residual (MINRES) method.}. This categorization is also justified by the observation that in the case of only bilateral constraints being present, the solvers are exactly equivalent to an indirect method for solving the linear system formed by the EoM (\ref{eq:models:time-stepping-eom}) and the velocity-level constraints (\ref{eq:models:joints:velocity-level-implicit-constraints}). 

As a matter of convention, however, we can think of these as variants of the classic Projected Gauss-Seidel (PGS)~\cite{jean1992, jourdan1998gaussseidel,anitescu2004fixed}, which is ubiquitous in the realm of physical simulation. Indeed it is present in most, if not all, of the popular physics engines such MuJoCo\cite{mujoco2024github}, PhysX \cite{nvidia2024physx}, Bullet \cite{coumans2022bullet3}, RaiSim~\cite{hwangbo2018percontact} and ODE \cite{smith2008ode}. However, the variants that are most commonly distinguished from "vanilla" PGS are PSOR~\cite{erleben2007}, GPGS~\cite{todorov2012mujoco,todorov2014analytical}, NBGS~\cite{preclik2014models}. Moreover, a generalization of all of these algorithms, and arguably the most theoretically founded, is the PROX algorithm~\cite{studer2007solving, erleben2017} (a.k.a. SORProx). In addition to being the closest to the exact structure of the block-wise proximal-point algorithm (\ref{eq:block-wise-proximal-operator}), it selects $\mathbf{A}$ and $\rho$ exactly according to the conditioning of the problem and performs an automatic adaptation of these parameters over successive iterations. However, purely due to time considerations, the authors regretfully have not been able to include the PROX algorithm in this evaluation, but we hope to do so in the future.

All projective splitting-based solvers are realized for the multi-constrained multi-body FD problem case using the scheme defined in Alg.~\ref{alg:solvers:projective-splitting}. PGS-CCP by Tasora et al~\cite{tasora2011matrix} is a blocked variant of PGS that solves the CCP (\ref{eq:formulation:dual-forward-dynamics-ccp}) formulation of the dual problem using Euclidean projection, i.e. $\mathbf{A} = \mathbb{I}_{n_d}$ and an ALM penalty equal to the average of the diagonal terms of each local Delassus matrix. PGS-NCP is a version of the classic PGS adapted to 3D Coulomb friction cones that uses decoupled normal-tangent projections to enforce the NCP constraints. NBGS by Preclik et al~\cite{preclik2018mdp} that uses an blocked version of PGS that solves each local per-contact problem exactly using a semi-analytical solution based on a decision tree and a quartic polynomial. RAISIM by Hwangbo et al~\cite{hwangbo2018percontact} is fundamentally identical to NBGS except that it uses the $\mathcal{P}_{BS}(\cdot)$ projector based on the bisection-search method. RAISIM-DS by Lidec et al~\cite{lidec2024reconciling} is an enhanced version of RAISIM that also incorporates successive approximations of the De Saxc\'e correction in the local solver of per-contact constraints. RAISIM-DS-ES, an Early Stopping (ES) version of RAISIM-DS that replaces the constraint-based termination criteria with the relative objective function improvement (\ref{eq:solvers:convergence:early-stopping-criteria}).

\subsection{Alternating Direction Method of Multipliers}
\label{sec:solvers:admm}
\noindent
In the context of mechanics, ADMM has only recently began to be explored as an alternative to first-order methods such as PGS and its variants. Tasora et al in~\cite{tasora2021admm}, applied ADMM to the relaxed CCP (\ref{eq:formulation:dual-forward-dynamics-ccp}) of the dual FD problem transcribed as the convex SOCP (\ref{eq:formulation:dual-forward-dynamics-socp}). This work demonstrated impressive versatility in being able to simulate both constrained rigid-body systems such as a robotic manipulator as well as stacks of flexible bodies such as deformable bars. Similarly, others have also applied ADMM to the primal FD problem with equally impressive demonstrations of its capabilities. Daviet in~\cite{daviet2020simple, daviet2023interactive} used it to solve a softened version of the problem, rendering exceptionally realistic simulations of compliant hair fibers. Very recently, Lee et al in~\cite{lee2025variations} used it to solve an NCP-type formulation for simulating articulated robots, demonstrating competitive accuracy-speed tradeoffs compared to a state-of-the-art Newton-based method.

The second set of solvers we have evaluated are based on the ADMM algorithm of Boyd et al~\cite{boyd2011admm}. Specifically, we have realized two ADMM-based methods to solve the dual FD problem that respectively tackle the CCP (\ref{eq:formulation:dual-forward-dynamics-ccp}) and NCP (\ref{eq:formulation:dual-forward-dynamics-ncp}). The first, here referred to as ADMM-CCP, is based on the algorithm proposed by Lidec et al in~\cite{lidec2024models}, that similarly to that of Tasora et al~\cite{tasora2021admm}, constitutes a direct application of ADMM to the SOCP (\ref{eq:formulation:dual-forward-dynamics-socp}). The second and more advanced solver, here referred to as ADMM-NCP, is based on the proximal formulation proposed by Carpentier et al in~\cite{carpentier2024unified} that combines ADMM with proximal algorithms in a form that can solve the NCP via the convex SOCP (\ref{eq:solvers:convex-ncp-socp}). In both cases, we have adapted the original algorithms to the maximal-coordinate CRBD formulation in order to also incorporate bilateral joint and unilateral limit constraints.

As it turns out however, we found that both ADMM-CCP and ADMM-NCP are so similar, that they can share a common implementation. They only difference between them, in fact, is that ADMM-NCP includes two additional terms: (a) the estimate of the De Saxc\'e correction (\ref{eq:construction:system-desaxce-operator}), and (b) an additional proximal regularization objective. Thus, to bring these solvers into the context of this work, we provide a brief outline of the proximal formulation of ADMM-NCP, describe the computations that are required to realize it, and clarify how it is relaxed to form ADMM-CCP. For further details regarding ADMM-NCP and ADMM in general, we respectively refer the interested reader to~\cite{boyd2011admm} and~\cite{carpentier2024unified}. We highly recommend these original works, as well as~\cite{tasora2021admm}, as these proved very insightful in better understanding how and why ADMM can be applied to such problems.\\

\subsubsection*{\textbf{Proximal ADMM}}
\label{sec:solvers:admm:outline}
\noindent
As yet another first-order method adhering to the template of Alg.~\ref{alg:generic-first-order-dual-solver}, ADMM-NCP shares a near-identical derivation to splitting-based methods. Thus our starting point to deriving it is a direct sequel to Sec.~\ref{sec:solvers:proximal-point-derivation}. The first step is recognizing that ADMM is in fact a global solver that accounts for all constraint simultaneously, as opposed to PGS and its variants that use the per-constraint iteration scheme in conjunction with a local single-contact solver. 

Thus, ADMM-NCP effectively computes a direct approximation of the global proximal-point projection
\begin{equation*}
\boldsymbol{\lambda} = \text{prox}_{\mathcal{K}}^{\mathbf{A}}(\boldsymbol{\lambda} - \rho^{-1} \, \mathbf{A}^{-1} \, (\mathbf{D} \, \boldsymbol{\lambda} + \mathbf{v}_{f} + \mathbf{s}))
\,\,.
\label{eq:solvers:dual-problem-proximal-point-equation-recap}
\end{equation*}
Rewriting the objective function of the convex SOCP (\ref{eq:solvers:convex-ncp-socp}) as 
\begin{equation}
h(\mathbf{x},\mathbf{s}) := 
\frac{1}{2}\,\mathbf{x}^{T}\,\mathbf{D}\,\mathbf{x}
+ \mathbf{x}^{T} (\mathbf{v}_{f} + \mathbf{s})
\,\,,
\label{eq:solvers:admm-objective}
\end{equation}
and setting the metric tensor $\mathbf{A}=\mathbb{I}_{n_d}$ to employ the Euclidean norm, we can define the augmented Lagrangian
\begin{equation}
\begin{array}{ll}
\displaystyle
\mathcal{L}_{\rho,\eta}^{A}(\mathbf{s},\,\mathbf{x},\,\mathbf{y},\,\mathbf{z}) :=& 
h(\mathbf{x},\,\mathbf{s}) 
+ \Psi_{\mathcal{K}}(\mathbf{y}) \\[12pt]
& + \frac{\eta}{2} \Vert \mathbf{x} - \mathbf{x}^{-} \Vert_{2}^{2} - \frac{1}{2\rho} \Vert \mathbf{z} \Vert_{2}^{2}\\[12pt]
& + \frac{\rho}{2} \Vert \mathbf{x} - \mathbf{y} - \rho^{-1}\mathbf{z} \Vert_{2}^{2}
\end{array}
\,\,,
\label{eq:solvers:admm-augmented-lagrangian}
\end{equation}
where $\eta$ is an additional proximal parameter, and $\mathbf{x}^{-}$ is the previous estimate of the primal variable. Typically, the additional proximal parameter is set to $\eta=10^{-6}$. Now note how(\ref{eq:solvers:admm-augmented-lagrangian}) is otherwise identical to (\ref{eq:solvers:constrained-problem-augmented-lagrangian-2}), except for the additional proximal-point term $\frac{\eta}{2} \Vert \mathbf{x} - \mathbf{x}^{-} \Vert_{2}^{2}$. However simple it may seem, this single term was one of the important insights of~\cite{carpentier2024unified}, as it provides a special form of regularization to the Delassus matrix that renders the Hessian of (\ref{eq:solvers:admm-objective}) strictly convex while totally avoiding any biasing effects in the solution. Another way to interpret the effect of this term is that the Lagrangian (\ref{eq:solvers:admm-augmented-lagrangian}) implies a double proximal projection, one to the feasible set and another to the previous iterate.

A crucial shortcut that ADMM takes from the typical ALM derivation described in Sec.~\ref{sec:solvers:proximal-point-derivation}, is that it defines the alternating directions that update the primal variables $\mathbf{x}$ and $\mathbf{y}$. Thus each iteration of ADMM-NCP decomposes the problem into a cascade the of sub-problems 
\begin{subequations}
\begin{equation}
\mathbf{s}^{i} = \boldsymbol{\Gamma}(\mathbf{v}^{+}(\mathbf{x}^{i-1}))\\[2pt]
\,\,,
\label{eq:solvers:admm-desaxce-estimate}
\end{equation}
\begin{equation}
\mathbf{x}^{i} = \displaystyle \operatorname*{argmin}_{\mathbf{x}} \,\mathcal{L}_{\rho,\eta}^{A}(\mathbf{s}^{i},\,\mathbf{x},\,\mathbf{y}^{i-1},\,\mathbf{z}^{i-1})
\,\,,
\label{eq:solvers:admm-x-minimization}
\end{equation}
\begin{equation}
\mathbf{y}^{i} = \displaystyle \operatorname*{argmin}_{\mathbf{y}} \,\mathcal{L}_{\rho,\eta}^{A}(\mathbf{s}^{i},\,\mathbf{x}^{i},\,\mathbf{y},\,\mathbf{z}^{i-1})
\,\,,
\label{eq:solvers:admm-y-minimization}
\end{equation}
\begin{equation}
\mathbf{z}^{i} = \mathbf{z}^{i-1} - \rho\,(\mathbf{x}^{i} - \mathbf{y}^{i})
\,\,.
\label{eq:solvers:admm-z-dual-update}
\end{equation}
\label{eq:solvers:admm-cascade}
\end{subequations}

The first step (\ref{eq:solvers:admm-desaxce-estimate}) performs the convex relaxation of the NSOCP (\ref{eq:formulation:dual-forward-dynamics-nsocp}) via the iterate estimates of the De Saxc\'e correction. The second step (\ref{eq:solvers:admm-x-minimization}) corresponds to an unconstrained QP that renders the unconstrained solution $\mathbf{x}^{i}$ directly as 
\begin{equation}
\mathbf{x}^{i} = - \mathbf{D}_{\rho,\eta}^{-1} 
\, \left( \mathbf{v}_{f} + \mathbf{s}^{i} - \eta\,\mathbf{x}^{k} - \rho\,\mathbf{y}_{k} - \mathbf{z}_{k} \right)
\,\,,
\label{eq:solvers:admm-unconstrained-solution}
\end{equation}
where $\mathbf{D}_{\rho,\eta} := \mathbf{D} + (\eta + \rho)\,\mathbb{I}_{n_d}$ denotes Delassus matrix regularized by the proximal point parameter and the ALM penalty. The third sub-problem in (\ref{eq:solvers:admm-y-minimization}) corresponds to the proximal projection onto the feasible set, that is realized using the Euclidean projector onto the cone $\mathcal{K}$, i.e.
\begin{equation}
\mathbf{y}^{i} = \operatorname{prox}_{\mathcal{K}}^{\rho}\left( \mathbf{x}^{i} - \rho^{-1} \mathbf{z}^{i-1} \right) 
= \mathcal{P}^{\mathcal{K}}\left( \mathbf{x}^{i} - \rho^{-1} \mathbf{z}^{i-1} \right) 
\,\,.
\label{eq:solvers:admm-proximal-operator}
\end{equation}
The last sub-problem (\ref{eq:solvers:admm-z-dual-update}) corresponds to the update of the dual variable $\mathbf{z}$ of ADMM, and is akin to performing a simple gradient descent step with the penalty $\rho$ acting as a step-size, and the \textit{consensus error} between the primal variables acting as the gradient. Lastly, to extract ADMM-CCP from the aforedescribed derivation, we can simply neglect (\ref{eq:solvers:admm-desaxce-estimate}), i.e. $\mathbf{s}^{i} = 0 ,\,\forall i$ and omit $\frac{\eta}{2} \Vert \mathbf{x} - \mathbf{x}^{-} \Vert_{2}^{2}$ from (\ref{eq:solvers:admm-augmented-lagrangian}).\\

\subsubsection*{\textbf{Convergence Criteria}}
\label{sec:solvers:admm:convergence}
The typical definition of ADMM involves a convergence criteria based solely on the primal-dual residuals $r_{p}$ and $r_{d}$, as defined in (\ref{eq:solvers:convergence:primal-residual}) and (\ref{eq:solvers:convergence:dual-residual}), respectively. However, as ADMM-NCP is solving an NCP, it is crucial that (\ref{eq:solvers:convergence:complementarity-residual}) must also be evaluated in order to ensure that solutions satisfy the physical constraints of the problem. However, compared to the other splitting-based solvers, ADMM-NCP does not need to evaluate the costly computations described in Sec.~\ref{sec:solvers:admm:convergence}, which would require several matrix-vector multiplications to evaluate. Instead, for ADMM-NCP and ADMM-CCP, we can evaluate $r_{p}$ and $r_{d}$ directly using the iterate values of the primal-dual variables $\mathbf{x}^{i},\,\mathbf{y}^{i},\,\mathbf{z}^{i}$ as
\begin{subequations}
\begin{equation}
\mathbf{r}_{p}^{i} := \mathbf{x}^{i} - \mathbf{y}^{i}
\label{eq:solvers:admm-primal-residual}
\end{equation}
\begin{equation}
\mathbf{r}_{d}^{i} := \eta\,(\mathbf{x}^{i} - \mathbf{x}^{i-1}) + \rho\,(\mathbf{y}^{i} - \mathbf{y}^{i-1})
\label{eq:solvers:admm-dual-residual}
\end{equation}
\begin{equation}
\mathbf{r}_{cp}^{i} := [\, {\mathbf{x}_{j}^{i}\,}^{T} \mathbf{z}_{j}^{i} \,]_{n_l+n_c}
\label{eq:solvers:admm-compl-residual}
\end{equation}\\
\end{subequations}

\subsubsection*{\textbf{Penalty Adaptation}}
\label{sec:solvers:admm:penalty-adaptation}
The final component of ADMM-NCP involves the determination of the ALM penalty parameter $\rho$. Similarly to the description of splitting-based methods and how each variant elects to choose $\rho$, there several options exist for ADMM as well. In particular, we have realized three: 
\begin{enumerate}
\item A fixed-penalty (FP), where it is set to a user-specified constant, i.e. $\rho = \rho^{0} > 0$, which by default, we have found $\rho^{0} = 1$ to work well across all problems.\\ 
\item A \textit{Linear Adaptation} (LA) scheme in the form of
\begin{equation}
\rho^{i} = R_{LA}(\rho^{i-1}) :=
\begin{cases}
\rho^{i-1}\,\tau_{inc} &,\, r_{p}^{i}/r_{d}^{i} \geq \alpha \\ 
\rho^{i-1}\,\tau_{dec} &,\, r_{p}^{i}/r_{d}^{i} \leq \alpha^{-1} \\ 
\rho^{i-1}\, &,\, \alpha^{-1} < r_{p}^{i}/r_{d}^{i} < \alpha
\end{cases}
\,\,,
\label{eq:solvers:admm:linear-penalty-adaptation}
\end{equation}
where $r_{p}^{i}$ and $r_{d}^{i}$ are respectively the primal and dual residual of the current iteration, $\alpha > 1$ is the a threshold on the ratio of primal-dual residuals, and $\tau_{inc},\tau_{dec} > 0$ are the linear increment/decrement factors. The latter essentially modulate the penalty parameter only when the ratio of primal-dual residuals exceeds the threshold $\alpha$, attempting to maintain them both relatively close to each other. Typical values for these parameters are $\alpha=10.0$ and $\tau_{inc},\tau_{dec}=1.5$.\\
\item Based on an analysis by Nishihara et al in~\cite{nishihara2015} on the of the convergence properties of ADMM, Carpentier et al proposed a \textit{Spectral Adaptation} (SA) scheme that both initializes and adapts the penalty $\rho$ based on the spectral properties of the Delassus matrix $\mathbf{D}$, i.e. the conditioning of the problem. Denoting the largest and smallest eigenvalues of $\mathbf{D}$ respectively as $L := \lambda_{max}(\mathbf{D})$ and $m := \lambda_{min}(\mathbf{D})$, and subsequently the condition number as $\kappa = \kappa(D) =: {L}/{m}$, the SA scheme is
\begin{subequations}
\begin{equation}
\rho^{0} = \sqrt{L_{0} \, m_{0}} \, \kappa_{0}^{\tau_{0}}
\end{equation}
\begin{equation}
\rho^{i} = R_{SA}(\rho^{i-1}) :=
\begin{cases}
\rho^{i-1}\,\kappa_{0}^{\tau} &,\, r_{p}^{i}/r_{d}^{i} \geq \alpha \\ 
\rho^{i-1}\,\kappa_{0}^{-\tau} &,\, r_{p}^{i}/r_{d}^{i} \leq \alpha^{-1} \\ 
\rho^{i-1}\, &,\, \alpha^{-1} < r_{p}^{i}/r_{d}^{i} < \alpha
\end{cases}
\end{equation}
\label{eq:solvers:admm:spectral-penalty-adaptation}
\end{subequations}
where $\tau_{0},\tau > 0$ are the initialization and adaptation factors. Typical values for SA are $\alpha=10.0,\,\tau_{0}=0.2,\, \tau=0.05$. A final remark regarding the values of $L,m$ in the SA scheme, is that it is not actually necessary to compute $m := \lambda_{min}(\mathbf{D})$, since $\eta$ ensures that it will be its worst-case value. Thus we can always set $m = \eta$. Moreover, the value of $L$ need not be exact, meaning that we can employ computationally cheap estimates that can be rendered using techniques such as the Power-Iteration method~\cite{mises1929praktische}. Sec.~\ref{sec:benchmarking:ill-conditioning} provides more details on the spectral properties $L,m,\kappa$ and how they are determined by the properties of the system.
\end{enumerate}

\begin{algorithm}[!ht]
\caption{Projective Splitting-based Solver}
\begin{algorithmic}[1]
\Require $N_{max}$, $\omega^{0}$, $\omega_{min}$, $\gamma$, $\epsilon_{p}$, $\epsilon_{d}$, $\epsilon_{cp}$
\textcolor{gray}{\Comment{solver parameters}}
\Require $\mathbf{D},\,\,\mathbf{v}_{f},\,\,\mathcal{K},\,\,\mathcal{P}(\cdot),\,\,f_{stop}(\cdot)$ 
\textcolor{gray}{\Comment{problem definition}}
\Require $\boldsymbol{\lambda}^{0},\,\,\mathbf{v}^{+ 0}$
\textcolor{gray}{\Comment{optional warmstart}}
\Statex
\textcolor{gray}{\Comment{initialization}}
\State $\omega \gets \omega^{0}$
%
\For{$i = 1$ to $N_{max}$} \textcolor{gray}{\Comment{solver iteration}}
\Statex
\For{$j = 1$ to $n_j$} \textcolor{gray}{\Comment{joint iteration}}
\State $\bar{\mathbf{v}}_{f,j} \gets \mathbf{v}_{f,j} - \sum_{m < j} \mathbf{D}_{jm} \, \boldsymbol{\lambda}_{m}^{i} - \sum_{m > j} \mathbf{D}_{jm} \, \boldsymbol{\lambda}_{m}^{i-1}$
\State $\boldsymbol{\lambda}_{j,0}^{i} \gets -\mathbf{D}_{jj}^{-1} \, \bar{\mathbf{v}}_{f,j}$
\State $\boldsymbol{\lambda}_{j}^{i} \gets (1-\omega) \, \boldsymbol{\lambda}_{j}^{i-1} \, + \, \omega \, \boldsymbol{\lambda}_{j,0}^{i}$
\EndFor
\Statex
\For{$l = 1$ to $n_l$} \textcolor{gray}{\Comment{limit iteration}}
\State $\bar{v}_{f,l} \gets v_{f,l} - \sum_{m < l} \mathbf{D}_{lm} \, \boldsymbol{\lambda}_{m}^{i} - \sum_{m > j} \mathbf{D}_{lm} \, \boldsymbol{\lambda}_{m}^{i-1}$
\State $\lambda_{l,0}^{i} \gets -D_{ll}^{-1} \, \bar{v}_{f,l}$
\State $\lambda_{l}^{i} \gets \mathcal{P}^{\mathbb{R}_{+}}(\lambda_{l,0}^{i})$
\State $\lambda_{l}^{i} \gets (1-\omega) \, \lambda_{l}^{i-1} \, + \, \omega \, \lambda_{l}^{i}$
\EndFor
\Statex
\For{$k = 1$ to $n_c$} \textcolor{gray}{\Comment{contact iteration}}
\State $\bar{\mathbf{v}}_{f,k} \gets \mathbf{v}_{f,k} - \sum_{m < k} \mathbf{D}_{km} \, \boldsymbol{\lambda}_{m}^{i} - \sum_{m > k} \mathbf{D}_{km} \, \boldsymbol{\lambda}_{m}^{i-1}$
\State $\mathbf{s}_{i} \gets \boldsymbol{\Gamma}(\mathbf{D}_{kk}\,\boldsymbol{\lambda}_{k}^{i} + \bar{\mathbf{v}}_{f,k})$
\State $\boldsymbol{\lambda}_{k}^{i} \gets \mathcal{P}(\boldsymbol{\lambda}_{k}^{i}, \mathbf{D}_{kk}, \bar{\mathbf{v}}_{f,k}, \mathbf{s}_{i}, \mu_{k})$
\State $\boldsymbol{\lambda}_{k}^{i} \gets (1-\omega) \, \boldsymbol{\lambda}_{k}^{i-1} \, + \, \omega \, \boldsymbol{\lambda}_{k}^{i}$

\EndFor
\Statex
\State $\mathbf{v}^{+} \gets \mathbf{D} \, \boldsymbol{\lambda}^{i} + \mathbf{v}_{f}$ \textcolor{gray}{\Comment{update solution}}
\State $\boldsymbol{\lambda} \gets \boldsymbol{\lambda}^{i}$
\Statex
\If{$\,\,
f_{stop}(\boldsymbol{\lambda}^{i}, \mathbf{v}^{i}, \epsilon_{p}, \epsilon_{d}, \epsilon_{cp}) = \text{true}
\,\,$}
\State \textbf{break}
\EndIf 
%
\State $\omega \gets \min(\omega_{min}, \gamma \, \omega)$ \textcolor{gray}{\Comment{relaxation decay}}
\Statex
\EndFor
\Statex
\Statex
\Return $\boldsymbol{\lambda} \,,\,\, \mathbf{v}^{+}$
\end{algorithmic}
\label{alg:solvers:projective-splitting}
\end{algorithm}
\begin{algorithm}[!ht]
\caption{Proximal ADMM-based Solver}
\setstretch{1.17}
\begin{algorithmic}[1]
\Require $N_{max}$, $\rho^{0}$, $\eta$, $\omega$ , $\epsilon_{p}$, $\epsilon_{d}$, $\epsilon_{cp}$ 
\textcolor{gray}{\Comment{solver parameters}}
\Require $\mathbf{D},\,\,\mathbf{v}_{f},\,\,\mathcal{K},\,\,\mathcal{P}^{\mathcal{K}}(\cdot)$ 
\textcolor{gray}{\Comment{problem definition}}
\Require $\boldsymbol{\lambda}^{0},\,\,\mathbf{v}^{+ 0}$
\textcolor{gray}{\Comment{optional warmstart}}
\Statex
\State $\mathbf{x}^{0} \gets \mathbf{y}^{0} \gets \boldsymbol{\lambda}^{0}$
\textcolor{gray}{\Comment{initialization}}
\State $\mathbf{z}^{0} \gets \mathbf{v}^{+ 0}$
\Statex
\For{$i = 1$ to $N_{max}$} \textcolor{gray}{\Comment{solver iterations}}
\Statex
%
\State $\mathbf{s}^{i} \gets \boldsymbol{\Gamma}(\mathbf{z}^{i-1})$ \textcolor{gray}{\Comment{De Saxc\'e} estimate}
\Statex
\State $\mathbf{v}^{i} \gets \mathbf{v}_{f} + \mathbf{s}^{i} - \eta\,\mathbf{x}^{i-1} - \rho^{i}\,\mathbf{y}^{i-1} - \mathbf{z}^{i-1}$ 
\textcolor{gray}{\Comment{offset update}}
\State $\mathbf{x}^{i} \gets - \left(\mathbf{D} + (\eta + \rho^{i})\,\mathbb{I}_{n_d}\right)^{-1} \, \mathbf{v}^{i}$
\textcolor{gray}{\Comment{primal update}}
\State $\mathbf{x}^{i} \gets (1-\omega) \, \mathbf{y}^{i-1} + \omega \, \mathbf{x}^{i}$ \textcolor{gray}{\Comment{over-relaxation}}
\State $\mathbf{y}^{i} \gets \mathcal{P}^{\mathcal{K}}\left( \mathbf{x}^{i} - \frac{1}{\rho^{i}} \, \mathbf{z}^{i-1} \right)$ \textcolor{gray}{\Comment{slack update}}
\State $\mathbf{z}^{i} \gets \mathbf{z}^{i-1} - \rho^{i} \left( \mathbf{x}^{i} - \mathbf{y}^{i} \right)$ 
\textcolor{gray}{\Comment{dual update}}
\Statex
\State $\boldsymbol{\lambda} \gets \mathbf{y}^{i}$ \textcolor{gray}{\Comment{update solution}}
\State $\mathbf{v}^{+} \gets \mathbf{z}^{i} - \boldsymbol{\Gamma}(\mathbf{z}^{i})$
\Statex
\State $r_{p}^{i} \gets \Vert \mathbf{x}^{i} - \mathbf{y}^{i} \Vert_{\infty}$ \textcolor{gray}{\Comment{residuals}}
\State $r_{d}^{i} \gets \Vert \eta \, (\mathbf{x}^{i} - \mathbf{x}^{i-1}) + \rho^{i} \, (\mathbf{y}^{i} - \mathbf{y}^{i-1}) \Vert_{\infty}$
\State $r_{cp}^{i} \gets \Vert \boldsymbol{r}_{cp}(\mathbf{x}^{i} \,,\,\, \mathbf{z}^{i}) \Vert_{\infty}$ \textcolor{gray}{\Comment{from (\ref{eq:solvers:convergence:complementarity-residual})}}
\Statex
\If{$
(r_{p}^{i} < \epsilon_{p}) \land
(r_{d}^{i} < \epsilon_{d}) \land
(r_{cp}^{i} < \epsilon_{cp})
$}
\State \textbf{break}
\EndIf 
\Statex
\State $\rho^{i} \gets R\left( \rho^{i} \,,\,\, r_{p}^{i} \,,\,\, r_{d}^{i} \right)$ \textcolor{gray}{\Comment{from (\ref{eq:solvers:admm:linear-penalty-adaptation}) or (\ref{eq:solvers:admm:spectral-penalty-adaptation})}}
\Statex
\EndFor
\Statex
\Statex
\Return $\boldsymbol{\lambda}\,,\,\,\mathbf{v}^{+}$
\end{algorithmic}
\label{alg:admm}
\end{algorithm}
\begin{table*}[!ht]
\centering
\caption{\textit{Dual Solvers}: Definition of each solver considered in this evaluation.}
\begin{tabular}{|l|l|l|l|c|c|l|}
\hline
\multicolumn{2}{|l|}{\textbf{Solver ID}} & \textbf{Scheme} & \textbf{Projector} & \textbf{De Saxc\'e} & \textbf{Termination} & \textbf{Notes} \\
\hline
\hline
\multicolumn{2}{|l|}{PGS-CCP}           & SOR (Alg.~\ref{alg:solvers:projective-splitting}) & CCP (Alg.~\ref{alg:ccp-local-projection}) & yes & (\ref{eq:solvers:convergence:primal-residual})(\ref{eq:solvers:convergence:dual-residual})(\ref{eq:solvers:convergence:complementarity-residual}) & \\
\hline
\multicolumn{2}{|l|}{PGS-NCP}           & SOR (Alg.~\ref{alg:solvers:projective-splitting}) & NCP (Alg.~\ref{alg:ncp-local-projection}) & no & (\ref{eq:solvers:convergence:primal-residual})(\ref{eq:solvers:convergence:dual-residual})(\ref{eq:solvers:convergence:complementarity-residual}) & \\
\hline
\multicolumn{2}{|l|}{NBGS}              & SOR (Alg.~\ref{alg:solvers:projective-splitting}) & NB (Alg.~\ref{alg:nb-local-projection}) + MDP & no & (\ref{eq:solvers:convergence:primal-residual})(\ref{eq:solvers:convergence:dual-residual})(\ref{eq:solvers:convergence:complementarity-residual}) & \\
\hline
\multicolumn{2}{|l|}{RAISIM}            & SOR (Alg.~\ref{alg:solvers:projective-splitting}) & NB (Alg.~\ref{alg:nb-local-projection}) + BS & no & (\ref{eq:solvers:convergence:primal-residual})(\ref{eq:solvers:convergence:dual-residual})(\ref{eq:solvers:convergence:complementarity-residual}) & \\
\hline
\multicolumn{2}{|l|}{RAISIM-DS}         & SOR (Alg.~\ref{alg:solvers:projective-splitting}) & NB (Alg.~\ref{alg:nb-local-projection}) + BS & yes & (\ref{eq:solvers:convergence:primal-residual})(\ref{eq:solvers:convergence:dual-residual})(\ref{eq:solvers:convergence:complementarity-residual}) & \\
\hline
\multicolumn{2}{|l|}{RAISIM-DS-ES}      & SOR (Alg.~\ref{alg:solvers:projective-splitting}) & NB (Alg.~\ref{alg:nb-local-projection}) + BS & yes & (\ref{eq:solvers:convergence:early-stopping-criteria}) & \\
\hline
\multicolumn{2}{|l|}{ADMM-CCP}          & ADMM (Alg.~\ref{alg:admm}) & Euclidean (\ref{eq:solvers:coulomb-cone-projection}) & no & (\ref{eq:solvers:convergence:primal-residual})(\ref{eq:solvers:convergence:dual-residual})(\ref{eq:solvers:convergence:complementarity-residual}) & fixed penalty \\
\hline
\multirow{3}{*}{ADMM-NCP-*}    & FP     & ADMM (Alg.~\ref{alg:admm}) & Euclidean (\ref{eq:solvers:coulomb-cone-projection}) & yes & (\ref{eq:solvers:convergence:primal-residual})(\ref{eq:solvers:convergence:dual-residual})(\ref{eq:solvers:convergence:complementarity-residual}) & fixed penalty \\
                               & LA    & ADMM (Alg.~\ref{alg:admm}) & Euclidean (\ref{eq:solvers:coulomb-cone-projection}) & yes & (\ref{eq:solvers:convergence:primal-residual})(\ref{eq:solvers:convergence:dual-residual})(\ref{eq:solvers:convergence:complementarity-residual}) & linear adapt. (\ref{eq:solvers:admm:linear-penalty-adaptation}) \\
                               & SA    & ADMM (Alg.~\ref{alg:admm}) & Euclidean (\ref{eq:solvers:coulomb-cone-projection}) & yes & (\ref{eq:solvers:convergence:primal-residual})(\ref{eq:solvers:convergence:dual-residual})(\ref{eq:solvers:convergence:complementarity-residual}) & spectral adapt. (\ref{eq:solvers:admm:spectral-penalty-adaptation}) \\
\hline
\end{tabular}
\label{tab:dual-solver-definitions}
\end{table*}

\section{Benchmarking Framework}
\label{sec:benchmarking}
\noindent
In this section we proposed a general-purpose framework for evaluating FD solvers. Fundamentally, our approach involves three core elements. Firstly, we have compiled an extensive suite of problems involving constrained rigid multi-body systems ranging from primitive toy problems to full-scale complex mechanical assemblies such as robotic systems and \textit{Audio-Animatronics}\textregistered\,\, figures. Secondly, we define a scheme to categorize dual problems based on the relative constraint dimensionality and conditioning of each, which enables a systematic characterization of solver performance based on the problem difficulty. Third, the we prescribe a set of quantitative performance metrics based on first-principles that are indicative of solver accuracy, robustness and speed. The presented approach is greatly inspired by the past work of Acary et al in~\cite{acary2014fclib, acary2018comparisons}, and is in effect, a direct extension of the former to cases of complex mechanical assemblies such as robotic systems. Within the scope of this work, we utilized it to only compare the specific set of algorithms described in Sec.~\ref{sec:solvers}, as will be presented in Sec.~\ref{sec:experiments}. 

\subsection{Benchmark Suite}
\label{sec:benchmarking:suite}
\noindent
The suite of simulation problems, depicted in Fig.~\ref{fig:benchmark-suite}, consists of twelve constrained rigid multi-body systems, spanning a wide range of ill-conditioned and numerically challenging scenarios. The toy problems principally serve to establish performance baselines of each solver on the distinct types of ill-conditioning emerging in the more complex systems, and thus facilitate systematic analyses of their qualitative behaviors. The full-scale systems are used to asses how algorithm performance scales with problem dimensionality and conditioning directly in cases of immediate practical relevance. It is thus organized into three categories of increasing system complexity and difficulty, namely:
\begin{enumerate}
\item \textit{Primitives}: a set of toy problems, each exhibiting a specific case of ill-conditioning to be tackled in isolation.
\item \textit{Robotics}: a set of problems involving robotic systems that are of moderate to high difficulty. They are all articulated systems without kinematic loops, and collisions involve only geometric primitives or minimal-size convex meshes,  as is typical in robotics applications.
\item \textit{Audio-Animatronics}\textregistered: a set of complex mechanical systems, each exhibiting multiple kinematic loop-closures.  While most are affixed w.r.t the world via a unary fixed-joints one is also free-floating. For these systems collisions are evaluated solely using the raw true mesh-based geometry of the mechanical assembly. The geometry is also multi-layered, as the internal mechanisms are enclosed by exterior \textit{shells}. Thus potential collisions include combinations between the world, shells and internal components.
\end{enumerate}

\subsection{Characterizing Ill-conditioning}
\label{sec:benchmarking:ill-conditioning}
\noindent
It is important to first define what exactly ill-conditioning is and how it arises in the context of mechanics. Generally speaking, given a function or operator defined by the mapping $f : \mathbb{R}^{n} \rightarrow \mathbb{R}^{m}$, we say that the former is \textit{ill-conditioned}, if its so-called \textit{condition number} $\kappa_{f} \in (0, \infty)$ exhibits large values, i.e. many orders of magnitude larger than unity. The precise definition of $\kappa_{f}$ depends on the nature of $f$ and the choice of metric norm used to measure distances in the domain and image of the mapping, but effectively, it is always some measure of the sensitivity of the input-output mapping $\mathbf{y} = f(\mathbf{x})$. In our case we will employ definitions pertaining to \textit{convex functions} and \textit{matrices}, which, are closely related through the theory of \textit{monotone operators}~\cite{ryu2016primer}.

Specifically, if the operator is a convex function $\mathbf{f} : \mathbb{R}^{n} \rightarrow \bar{\mathbb{R}}$, where $\bar{\mathbb{R}} := \mathbb{R} \cup \{ \pm \infty \}$, the condition number of $y = f(\mathbf{x})$ is defined as $\kappa_{f} = L/m$, where $m > 0$ is the \textit{convexity parameter} and $L > 0$ is the \textit{Lipschitz constant} of $f$. See~\cite{ryu2016primer} for more details. Intuitively, $m$ and $L$ are respectively equivalent to the smallest and largest singular values of the Hessian $\nabla^{2} f(\mathbf{x})$ if $f$ were twice continuously-differentiable. This equivalence naturally leads to second definition regarding matrices: in the case of affine functions of the form $f(x) = \mathbf{A}\,\mathbf{x} + \mathbf{b}$, the equivalent quantity is the \textit{matrix condition number} $\kappa_{\mathbf{A}} := \sigma_{max}(\mathbf{A}) / \sigma_{min}(\mathbf{A})$, i.e. the ratio of its largest and smallest singular values $\sigma_{max}(\mathbf{A}) := L$, $\sigma_{min}(\mathbf{A}) := m$ of matrix $\mathbf{A} \in \mathbb{R}^{n \times m}$. This also means that the Lipschitz constant corresponds to the spectral radius of $\mathbf{A}$, i.e. $L = \rho(\mathbf{A})$. Moreover, for symmetric matrices $\mathbf{A} \in \mathbb{R}^{n\times n}$, i.e. $\mathbf{A} = \mathbf{A}^{T}$, $\kappa_{\mathbf{A}} := \lambda_{max}(\mathbf{A}) / \lambda_{min}(\mathbf{A})$ where now the condition number can be defined using its largest and smallest eigenvalues.

With these definitions in hand, let us now consider how ill-conditioning arises within the mechanical context of rigid-body dynamics. The two key elements at play here are the Jacobian matrix $\mathbf{J}$ and the generalized mass matrix $\mathbf{M}$. Therefore, when either of the two (or both) exhibit large condition numbers, then the overall FD problem is said to be ill-conditioned, and can occur in cases of:
\begin{itemize}
\item \textit{Hyperstaticity}: this corresponds to rank-deficiency in the Jacobian matrix $\mathbf{J}$ due to constraint coupling. It can occur in the following ways: (a) intrinsically due to kinematic loops occurring in the joint morphology, (b) extrinsically due to multiple contacts acting on the same body or on the system overall, and (c) when the system approaches a kinematic singularity. Rank-deficiency therefore amounts to $\mathbf{J}$ exhibiting one or more zero-valued eigenvalues, i.e. $\sigma_{min}(\mathbf{A}) = 0$, resulting in $\kappa \rightarrow \infty$.
\item \textit{Inertial disparity}: this is when the system's mass ratio $r_{m} = {m_{max}}/{m_{min}} \gg 1$, i.e. that of the largest and smallest mass present in the system is significantly larger than unity. How much larger than unity is of course relative. Practically speaking, and in terms of orders-of-magnitude, a forward dynamics problem can be ill-conditioned if $r_{m} \approx 10^3$. Although $\mathbf{M}$ is necessarily positive-definite, its inversion with a high $r_{m}$ multiplied by the effective lever-arms of $\mathbf{J}$ can result in exceptionally large $\kappa_{\mathbf{D}}$, the condition number of the Delassus matrix.
\end{itemize}

Therefore, the combined effects of hyperstaticity and inertial disparity can prove detrimental to both the rate of convergence of dual solvers as well as the stability of the resulting constraint reactions. The former leads to excessive iterations being required to solve the problem, while the latter causes erratic changes to the forces rendered between successive iterations without substantial changes to the system's state.

\begin{figure*}[!t]
\centering
\includegraphics[width=1.0\linewidth]{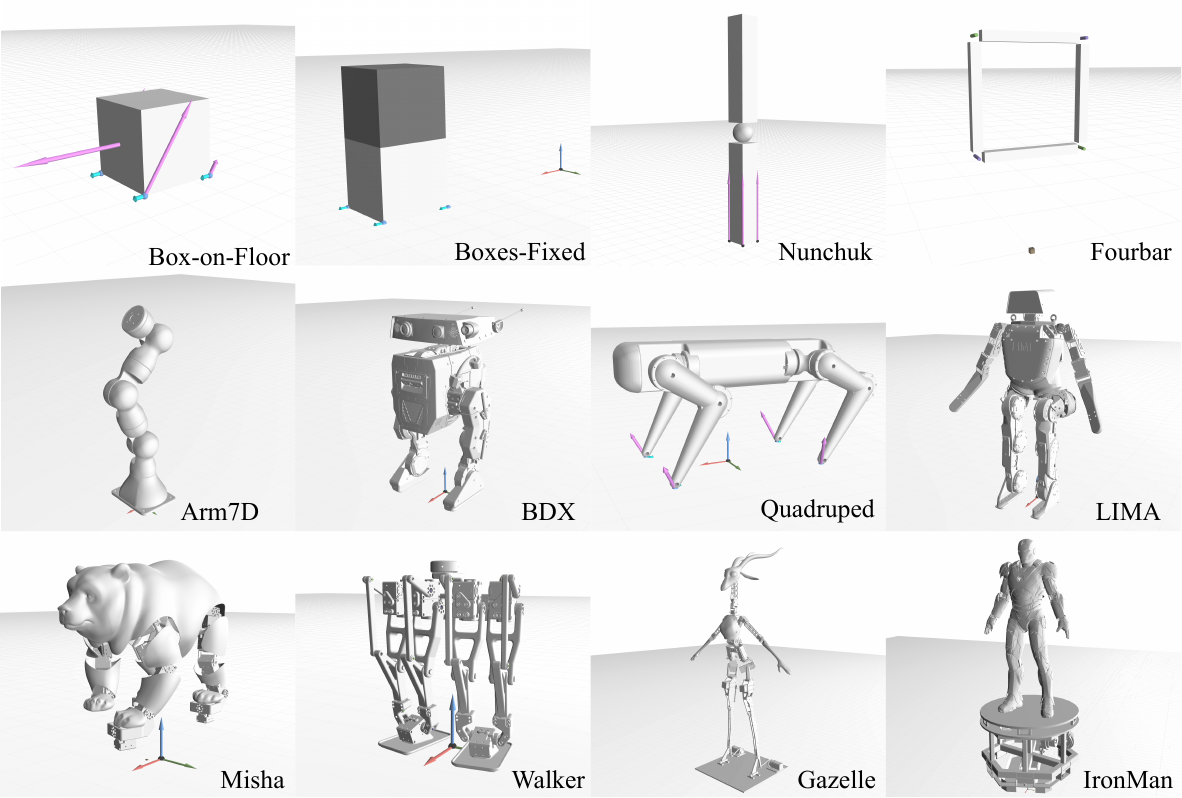}
\caption{\textit{Benchmark Suite}: a set of twelve constrained rigid multi-body systems spanning three categories. The top, middle and bottom rows respectively depict the \textit{Primitive}, \textit{Robotics}, and \textit{Audio-Animatronics}\textregistered\,\, problem categories, and the system complexity is increasing from left to right. The first category contains a set of toy problems that are often used to establish a baseline of behavior in specific cases of ill-conditioning. The second category involving robotic systems serves to evaluate solver performance on well established problems in encountered in robotics applications. These problems do not exhibit any inherent ill-conditioning, but the presence of joint limits and contacts, in addition to the use of meshes as the collision geometry, significantly increase the difficulty of simulating them. The third category involves mostly fixed-bases systems with the exception of Walker which is floating-based. All of these problems exhibit multiple intrinsic kinematic loops with a relatively large number of bodies, and present the biggest challenge, as the systems are inherently ill-conditioned due to the combined effects of: mesh collisions, inherent kinematic loops, passive (i.e. unactuated) joints, presence of unary joints, significantly larger problem dimensions, and large disparity in body masses.}
\label{fig:benchmark-suite}
\end{figure*}

\begin{table*}[!t]
\label{tab:benchmark-suite-problem-dimensions}
\centering
\caption{\textit{Benchmark Suite}: \text{The intrinsic dimensions of each system, i.e. without active limit and contact constraints.}}
\setlength\extrarowheight{1pt}
\begin{tabular}{|lcccccccc|}
\hline
\textbf{System} & \textbf{Base} & \textbf{Bodies} & \textbf{Joints} & \textbf{Total DoFs} & \textbf{Passive DoFs} & \textbf{Actuated DoFs} & \textbf{Constraints} & \textbf{Mass Ratio}\\
\hline \hline
Box-on-Plane    & Free & 1        & 0        & 6      & 0             & 0              & 0      &  1.00    \\
\hline
Boxes-Fixed     & Free & 2        & 1        & 6      & 0             & 0              & 6      &  1000.00    \\
\hline
Nunchuk         & Free & 3        & 2        & 12     & 6             & 0              & 6      &   1.00   \\
\hline
Fourbar (fixed) & Fixed & 4        & 5        & 1      & 0             & 1              & 26     &   1.00   \\
\hline
Fourbar (free)  & Free & 4        & 4        & 7      & 5             & 1              & 20     &   1.00   \\
\hline
Arm7D           & Fixed & 8       & 7        & 7      & 0             & 7              & 41     &   19.54   \\
\hline
Quadruped       & Free & 13       & 12       & 18     & 6             & 12             & 60     &   29.07  \\
\hline
BDX             & Free & 15       & 14       & 14     & 6             & 14             & 70     &   42.49  \\
\hline
LIMA            & Free & 21       & 20       & 26     & 6             & 20             & 100    &   321.39 \\
\hline
Misha           & Fixed & 22       & 24       & 23     & 0             & 23             & 121    &   90.67  \\
\hline
Walker          & Free & 31       & 36       & 42     & 18            & 24             & 180    &   99.55  \\
\hline
Gazelle         & Fixed & 38       & 45       & 60     & 38            & 22             & 214    &   1012.54 \\
\hline
IronMan         & Fixed & 54       & 67       & 100    & 80            & 20             & 304    &   25544.44 \\
\hline
\end{tabular}
\end{table*}

\subsection{Sample Categories}
\label{sec:benchmarking:sample-categories}
\noindent
Given the aforedescribed definition of ill-conditioning, we aim to systematically characterize and categorize time-wise samples extracted from each problem in the suite. In this context, a \textit{sample} is defined as a single dual problem in the form of the NCP (\ref{eq:formulation:dual-forward-dynamics-ncp}) or CCP (\ref{eq:formulation:dual-forward-dynamics-ccp}). Each sample can therefore be solved by a set of solver instances varying in terms of the algorithm or hyper-parameters. This aspect of our work has been inspired by the tremendously thorough technical reports of Acary et al on the FCLib~\cite{acary2014fclib} and numerical solver comparisons in~\cite{acary2018comparisons}, as well as that of Ledic et al in their comparisons of contact models in~\cite{lidec2024models}. 

Employing a similar categorization as in~\cite{acary2018comparisons}, we characterize each sample problem based on so-called \textit{constraint densities} $d_{j} = n_{jd} / n_{bd}$, $d_{c} = 3 \, n_{c} / 6 \, n_{b}$, $d_{jlc} = n_{d} / 6 \, n_{b}$ corresponding to that of the joints, contacts and total constraint sets, respectively, where $n_{jd} := \sum_{j=1}^{n_j} m_j$ denotes the total number of constraint dimensions of the joints. These densities express the ratio of number of constraints to body DoFs, and whenever they are greater than unity, express the \textit{potential} for rank-deficiency of $\mathbf{J}$. For this reason we also measure the actual matrix rank $r_{\mathbf{J}} = \text{rank}(\mathbf{J})$ of the system Jacobian. Note that we have omitted the isolated constraint density of joint limits as it, by definition, cannot exceed the joint constraint density $d_{j}$, and so does not really contribute much information on its own. It is, however, included in the total constraint density measurement $d_{jlc}$. Thus, using these densities and marix rank, we categorize each sample as either:
\begin{itemize}
\item Independent Joints: $n_c = 0$, $n_j \geq 1$, and $r_{\mathbf{J}} = n_{jd}$.
\item Redundant Joints: $n_c = 0$, $n_j \geq 1$, $r_{\mathbf{J}} < n_{jd}$, and $d_{j} < 1$
\item Dense Joints: $n_c = 0$, $n_j \geq 1$, $r_{\mathbf{J}} < n_{jd}$, and $d_{j} \geq 1$
\item Single Contact: $n_c = 1$
\item Sparse Contacts: $n_c \in [2 \,,\,\, 2\,n_b]$ (i.e. $d_c < 1$)
\item Dense Contacts: $n_c > 2\,n_b$ and $d_{c} > 1$
\item Dense Constraints: $n_c \geq 1$, $n_j \geq 1$ and $d_{jlc} > 1$
\end{itemize}
The first case of \textit{Independent Joints} is used to establish a baseline of nominal performance for samples that can admit a full-row rank $\mathbf{J}$. Note also the special case of \textit{Single Contact} $n_c = 1$. Although seemingly trivial, for problems with large mass ratios and/or large number of joints, the single-contact case can inform us about the capacity of a particular solver to correctly propagate and distribute constraint reactions throughout the system.

\subsection{Performance Metrics}
\label{sec:benchmarking:metrics}
\noindent
Selecting appropriate performance metrics is crucial to evaluating dual solvers w.r.t their physical plausibility, accuracy, speed, and robustness, both in relative and absolute terms. Moreover, our objective must be to state such comparisons as fairly as possible, but also in a manner that is representative of the actual needs of relevant applications. Overall, our approach for evaluating dual solvers consists of a comparative framework that combines and extends the work of Acary et al~\cite{acary2018comparisons} and Lidec et al~\cite{lidec2024models, carpentier2024unified}. Specifically, it involves two aspects: a) a set of sample-wise quantitative performance metrics, and b) a schema for representing aggregate performance using statistical measures computed over multiple samples.\\

\subsubsection*{\textbf{Accuracy}} 
\label{sec:benchmarking:metrics:accuracy}
We will first define metrics for physical plausibility and accuracy. In this regard, comparing dual solvers in systematic manner is a rather non-trivial task. As the algorithms can differ significantly in terms of the employed numerical scheme and the contact model, this precludes the use of their respective convergence criteria as comparative measures. Therefore, we must resort to using metrics which can generalize to all cases, e.g. a CCP or NCP-based contact model etc. In particular, we propose to use:
\begin{enumerate}
\item the \textit{Penetration Residual}
\begin{equation}
r_{\text{pen}}
=: 
\Vert\, \mathbf{f}(\mathbf{s}) \,\Vert_{\infty}
\,\,,
\label{eq:benchmark:penetration-residual}
\end{equation}
where $\mathbf{f} : \mathbb{S}^{n_b} \rightarrow \mathbb{R}^{n_{sd}}$ is the total configuration-level constraint function defined in (\ref{eq:construction:total-configuration-constraint-fuction}). This metric effectively indicates the worst-case configuration-level constraint violation occurring as a byproduct of the computed constraint reactions.
\item the \textit{NCP Dual Residual}
\begin{equation}
r_{\text{dual}}
=: 
\Vert \textbf{r}_{\text{d}}(\mathbf{v}^{+} + \boldsymbol{\Gamma}(\mathbf{v}^{+})) \Vert_{\infty}
\,\,.
\end{equation}
This quantity is useful because it indicates if non-zero velocities are present along directions that violate the constraint. For joints this simply amounts to the total constraint-space velocities, which should ideally be zero at all times. For limits and contacts, it corresponds to non-zero velocities along the unilateral constraint dimension, e.g. the contact normal. A proof of this property is provided in Appendix.~\ref{sec:apndx:desaxce-properties}.
\item the \textit{NCP Complementarity Residual}
\begin{equation}
r_{\text{ncp}}
=: 
\Vert \textbf{r}_{\text{cp}}(\boldsymbol{\lambda} \,,\,\, \mathbf{v}^{+} + \boldsymbol{\Gamma}(\mathbf{v}^{+})) \Vert_{\infty}
\,\,,
\end{equation}
which serves as a measure of adherence to the MDP. Essentially, non-zero values indicate a misalignment between constraint reactions and corresponding contact velocities along the tangential planes, i.e. the deviation from maximal dissipativity, as well as erroneous power distributions along contact normals.
\item  the \textit{NCP Natural-Map Residual} 
\begin{equation}
r_{\text{nat}}
=: 
\Vert \mathbf{F}_{\text{NCP}}^{\text{nat}}(\boldsymbol{\lambda},\mathbf{v}^{+}) \Vert_{\infty}
\,\,,
\end{equation}
 which indicates how near the computed contact reaction is to a solution of the NCP. As described in Sec.~\ref{sec:solvers:convergence}, it combines the effects of the primal, dual and complementarity residuals into a single value, providing a convenient short-hand for both solver convergence and solution validity.\\
\end{enumerate}

\subsubsection*{\textbf{Speed}}
\label{sec:benchmarking:metrics:speed}
Next, we will define a set of metrics for assessing the speed and convergence behavior of each algorithm. These will serve as proxies to estimating the effective computational complexity exhibited by each solver on the set of problems:
\begin{enumerate}
\item the \textit{Number of Iterations} $i_{stop} \leq N_{max}$, which reflects a solver's convergence to a solution. $N_{max} \in \mathbb{N}\cup\{\infty\}$ is the maximum iterations permitted for a particular simulation, and can represent the computational budget afforded to a dual solver in terms of the number of operations executed on each problem.
\item the \textit{Total Solve Time} $t_{solve} \leq t_{max}$, which represents the total computational cost incurred by a solver on a particular problem. $t_{max} \in \mathbb{R_{+}}\cup\{\infty\}$ is the maximum permissible time-duration, and equivalently to $N_{max}$, can represent the computational budget in terms of the absolute wall-clock time.
\item the \textit{Mean Iteration Time} $t_{iter}$, which approximates the computational cost of the per-iteration operations of a particular algorithm. Given the similar overall structure of all dual solvers, $t_{iter}$ is indicative of the computational cost incurred by the specifics of each algorithms, such as the numerical scheme, the employed contact model and the residuals used as termination criteria.\\
\end{enumerate}

\subsubsection*{\textbf{Robustness}}
\label{sec:benchmarking:metrics:robustness}
The last aspect of solver performance to consider is robustness. In this regard, we will asses how a solver behaves in the presence of ill-conditioning. Specifically, for each sample problem, our goal is to examine if a solver can: (a) converge at all, (b) converge within a specified computational budget, (c) produce stable motions despite not converging within the afforded budget, and (d) diverge by resulting in \texttt{NaN}, \texttt{Inf} or arbitrarily large values. To this end, two methods often employed in benchmarks for numerical optimization methods is to evaluate whether $t_{solve} \geq t_{max}$ or $i_{stop} \geq N_{max}$. If either case holds on a given problem, we can consider the solver to have failed to solve it. In this work we will exclusively employ the latter check w.r.t $i_{stop} \geq N_{max}$, as our evaluation focuses mainly on the numerical properties of each algorithm, as opposed to their absolute time complexity.\\

\subsubsection*{\textbf{Profiles}}
\label{sec:benchmarking:metrics:profiles}
Lastly, we employ the methodology of \textit{performance profiles} by Dolan et al~\cite{dolan2002} in order to statistically evaluate the relative performance of each solver, w.r.t the aforedefined metrics. Such an approach has nowadays become standard practice for benchmarking numerical optimization algorithms, and is slowly gaining favor in application domains such as physical simulation as well~\cite{acary2018comparisons, carpentier2024unified}.It enables us to aggregate performance metrics collected over individual or multiple simulations, and render cumulative probability distributions that highlight the relative strengths and weaknesses of each solver/algorithm in a systematic way.

In this setting, we assume a set $\mathcal{P}$ consisting of $n_{dp} := |\mathcal{P}|$ dual problems, as well as a set $\mathcal{S}$ of $n_{ds} := |\mathcal{S}|$ dual solvers. For each index pair $p,s \in \mathcal{P} \times \mathcal{S}$, we execute solver $s$ on problem $p$ and collect the set of metrics $\mathcal{M}_{p,s} := \{ r_{pen} \,,\,\, r_{d} \,,\,\, r_{ncp} \,,\,\, r_{nat} \,,\,\, i_{stop} \,,\,\, t_{solve} \,,\,\, t_{iter} \}$ defined previously. It is important to note that all performance metrics in $\mathcal{M}_{p,s}$ admit an increasing partial ordering, i.e. lowest is best, and it holds that $m_{p,s} \geq 0 \,\,,\,\,\, \forall m_{p,s} \in \mathcal{M}_{p,s}$. For every metric sample $m_{p,s}$ we compute a corresponding relative \textit{performance ratio} 
\begin{equation}
r_{p,s} := \frac{m_{p,s}}{\max{\{ m_{p,s} : s \in \mathcal{S} \}} + \epsilon} \,,\,\, r_{p,s} \in [1 \,,\,\, r_M]
\,\,,
\label{eq:performance-ratio}
\end{equation}
where $m_{p}^{*} = \max{\{ m_{p,s} : s \in \mathcal{S} \}}$ is the \textit{baseline} metric for problem $p$, i.e. the performance metric of the best performing solver for the specific problem. $\epsilon$ is an infinitesimal constant added to the denominator to protect against divisions by zero, and $r_M$ is a parameter that we can set ad-hoc to represent the ratio corresponding to failure to converge, e.g. $r_M = \infty$. 

A \textit{performance profile} of each solver $s$ w.r.t a specific metric is defined as the cumulative distribution function 
\begin{equation}
\rho_{s}(\tau) := \frac{1}{n_p} \, \operatorname{size} \left\{ p \in \mathcal{P} : r_{p,s} \leq \tau \right\}
\,\,,
\label{eq:performance-profile}
\end{equation}
over values $\tau$ of the performance ratio. Essentially, each value $\rho_{s}(\tau)$ expresses the probability that solver $s \in \mathcal{S}$ is withing a factor $\tau$ of the best performing solver over all problems in the set $\mathcal{S}$. Two important limit cases are to compare the values $\rho_{s}(1)$ and $\rho_{s}(r_M)$ of each solver. The former represents the number of problems in which $s$ exhibited the best performance, and the latter indicates the number of problems that $s$ solved successfully (i.e. converged in).

However, there lies a crucial subtlety in computing performance ratios for samples that can contain values smaller than the machine precision $\epsilon(d)$\footnote{For 64-bit double-precision floating-point, $\epsilon(d) = 2^{-52} \approx 2.2 \cdot 10^{-16}$.}, such as when $m_{p,s}$ corresponds to one of the aforedescribed residuals. Using (\ref{eq:performance-ratio}) directly in such cases may result in divisions with zero or arbitrarily minuscule values that can skew or completely distort the generation of each performance profile using (\ref{eq:performance-profile}). To overcome this problem, we utilize the augmentation proposed by Dingle \& Higham~\cite{dingle2023reducing}, which replaces each sample as 
\begin{equation}
\bar{m}_{p,s} = 
\begin{cases}
m_{p,s} &,\,\, m_{p,s} > m_{min}\\
m_{p,s} \, \frac{m_{max} - m_{min}}{m_{min}} + m_{min} &,\,\, m_{p,s} \leq m_{min} 
\end{cases}
\label{eq:performance-ration-augmentation}
\end{equation}
where $m_{max} = 0.5 \cdot \epsilon(d)$ and $m_{min} = 10^{-2} \cdot \epsilon(d)$ are the upper and lower limit metric values. Using (\ref{eq:performance-ration-augmentation}) essentially re-normalizes values approaching $\epsilon(d)$ while preserving the strict ordering of samples and thus the monotonicity of each performance profile (\ref{eq:performance-profile}).

\section{Experiments}
\label{sec:experiments}

This section presents a series of experiments designed to evaluate the algorithms described in Sec.\ref{sec:solvers} using the benchmark suite defined in Sec.~\ref{sec:benchmarking}. The primitive problems will be used to analyze the behavior of the dual solvers and determine \textit{if}, \textit{when} and \textit{how} they can exhibit characteristics required for physically plausible simulations. The full-scale problems will be used to extract performance statistics and quantify the relative differences between solvers at a larger scale.

\subsection{Setup}
\label{sec:experiments:setup}
\noindent
The software facilitating our evaluation is implemented entirely in C++17 and employs the Eigen3~\cite{eigen2010} library as a front-end for all BLAS operations. One aspect that is crucial to any numerical method is whether the problem data is represented using \textit{sparse} and/or \textit{dense} matrices and vectors. In our current implementation we have only used dense representations as our focus lies mainly in the evaluation of the numerical properties of the relevant algorithms. We do, however, plan to extend our implementation to sparse algebra in future work that will include other runtime performance enhancements as well. All results presented herein were executed on two distinct hardware set-ups. The first is a high-end desktop PC with an Intel Core i9-14900K 32-core CPU and 128GB of RAM, running Ubuntu 20.04. The code in this setup uses Intel MKL as the BLAS back-end to Eigen3, and is compiled with \texttt{GCC 10} (\texttt{Ubuntu 10.5.0-1ubuntu1~20.04}). The second setup is a 2023 MacBook Pro with an Apple M3 Max CPU and 64GB RAM running macOS Sonoma 14.6, and uses with apple-specific \texttt{clang 15} (\texttt{arm64-apple-darwin-23.6.0}). Having both set-ups allows us to evaluate the potential disparity induced by hardware and compiler-specific quirks in the numerical output of all solvers.

\subsection{A Note on Collision Detection}
\label{sec:experiments:collision-detection}
\noindent
In all of the presented experiments, the contact configuration is \textit{never assumed}, and in all cases, is generated based on the configuration of the system at the begging of each time-step. Regarding CD with geometric primitives and meshes, we use the library provided by the ODE~\cite{smith2008ode}. In addition, we found it necessary to perform a post-processing of the set of generated contact points. Firstly, we established a robust grouping of geometries in order to control broad-phase collision checks between neighboring bodies, i.e. those connected by joints, and to be able to toggle internal collisions within the rigid-body system. Second, in the near-phase check, we cull redundant points that lie within a minimum distance of each other. This was particularly crucial to handling collisions involving mesh-based geometries, where potentially hundreds of points can be generated for each body without culling. Lastly, in order to deal with artifacts such as \textit{contact chattering}, we added a mechanism to set a configurable collision detection margin that facilitates a penetration tolerance between bodies. Unlike other CD libraries such as FCL~\cite{pan2012fcl} and Coal~\cite{coal2024hithub} (a.k.a. HPP-FCL), ODE does not inherently support collision margins.

\subsection{Primitive Problems}
\label{sec:experiments:primitives}
\noindent
The first phase of experiments presents a fine-grained analysis of the set of solvers defined in Sec.\ref{sec:solvers} on the subset of primitive problems. The objective of this phase is to examine how each solver behaves in the presence of certain types of ill-conditioning, and observe the physical artifacts that can manifest. All primary experiments involve independent simulations executed with each solver, for the purposes of observing the distinct motions they generate. These will be referred to as \textit{independent runs}. However, in cases where performance profiles will be used to quantify relative performance between solvers, we necessarily need to compare them on the same data, i.e. identical dual problems. In such cases we use ADMM-NCP as a reference solver to generate a sequence of problems that are then re-solved by all others. We refer to these as \textit{curated runs}. All solvers are by default configured for \textit{high-precision}; $N_{max} = 10^4$ and $\epsilon_{abs} = 10^{-12}$ used for all residuals. In certain cases, which will be stated explicitly, we will use \textit{high-throughput} settings with $N_{max} = 10^3$ and $\epsilon_{abs} = 10^{-6}$. Unless otherwise specified, all contact interactions will involve coefficients of friction $\mu_{c} = 0.7$ and restitution $\epsilon_{c} = 0$, and default values for constraint softening are
$cs_{j} = \{d_0 = 0.9,\,d_w=0.95,\,w=0.001,\,m=0.5,\,p=2.0,\,T=0.02,\,\beta=1.0\}_{j}$\\ 

\begin{figure}[!b]
\centering
\includegraphics[width=1.0\linewidth]{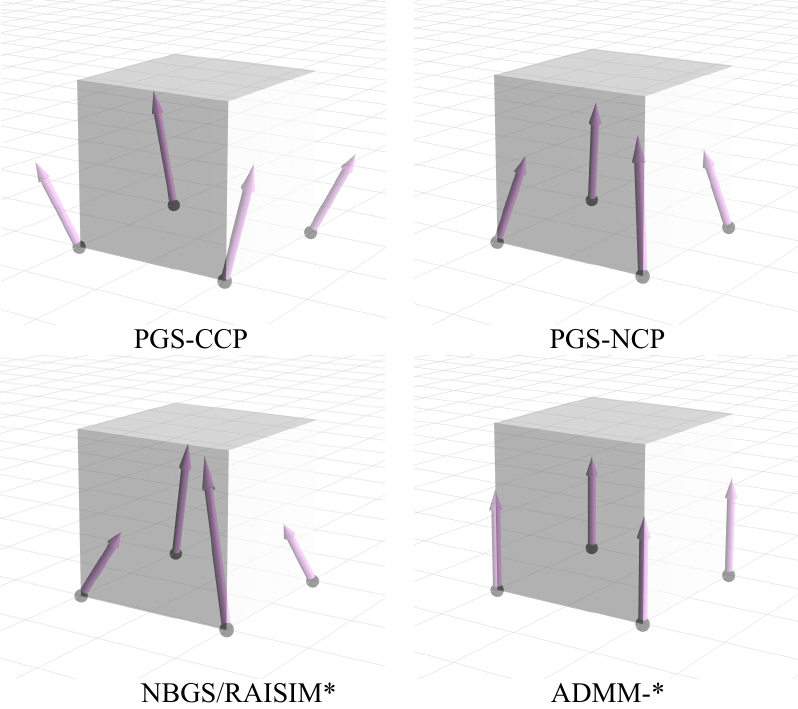}
\caption{The different force distributions exhibited by each solver on the hard-contact problem. Only the ADMM solvers yield a minimum-norm solution that evenly distributes contact forces without exhibiting internal forces, while all splitting-based solvers result in uneven and internal forces. This is expected however, since the per-contact iteration scheme is known to suffer from a lack of symmetry~\cite{siggraph2022contact}. This artifact is due to the order in which the constraints are accessed, which, inevitably induces bias in the local solutions, especially when using a cold-start. Conversely, ADMM can account for all couplings simultaneously and eliminates internal forces.}
\label{fig:box-on-plane-force-distributions}
\end{figure}
\subsubsection*{\textbf{Box-on-Plane}}
\label{sec:experiments:primitives:Box-on-Plane}
This first experiment involves a box placed on an a flat plane. A similar setup has been demonstrated in other works, such as those of Drumwright et al~\cite{drumwright2011evaluation} and Lidec et al~\cite{lidec2024models}. The box is a single cuboid body with mass $m = \unit[1]{kg}$, sides of dimension $d = \unit[20]{cm}$. The objective here is to observe: (a) the distribution of reaction forces among the $n_c \leq 4$ contacts, (b) stick-slip transitions, and (c) the overall motion of the body. The box is initialized resting undisturbed on the plane for the first two seconds. An external force is applied thereafter about the box's CoM along the global X-axis, with a linearly increasing magnitude of $\unit[0]{Nm}$ to $f_{max} = \unit[2\,\mu_c \, m \, g]{Nm}$ over a duration of $\Delta{t}_{ext} = \unit[6]{s}$. The force is removed at $t = \unit[8]{s}$ and the box is allowed to decelerate due to friction until it comes to rest, where it is once again is left undisturbed. The entire procedure lasts a total of $t_f = \unit[10]{s}$. The experiment is executed over two distinct run. In the first run, the dual problem is configured for hard contacts, and in the second, it is configured with constraint softening enabled.
\begin{figure*}[!ht]
\centering
\includegraphics[width=0.95\linewidth]{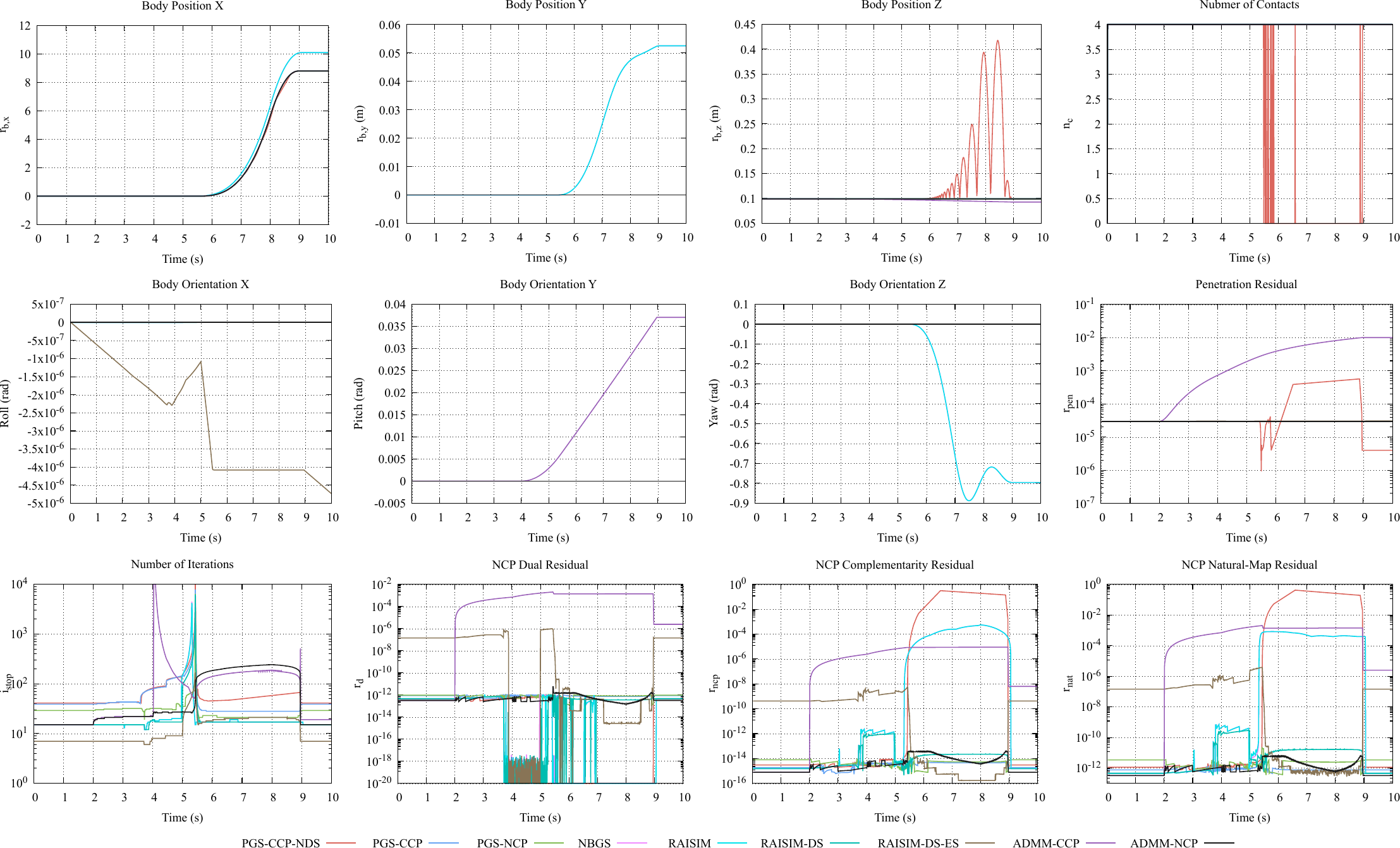}
\caption{\textit{Box-on-Plane}: Time-series plots of the first run on the hard-contact problem. The PGS-CCP, PGS-NCP, RAISIM-DS and ADMM-NCP solvers yield similar motions, while the others exhibit significant translational and rotational deviations along the Y and Z-axes. When sliding, the PGS-CCP-NDS results in positive normal velocities thus leading to jumping, as can be seen in the right middle plot (red).}
\label{fig:box-on-plane-hard}
\end{figure*}
\begin{figure*}[!ht]
\centering
\includegraphics[width=0.95\linewidth]{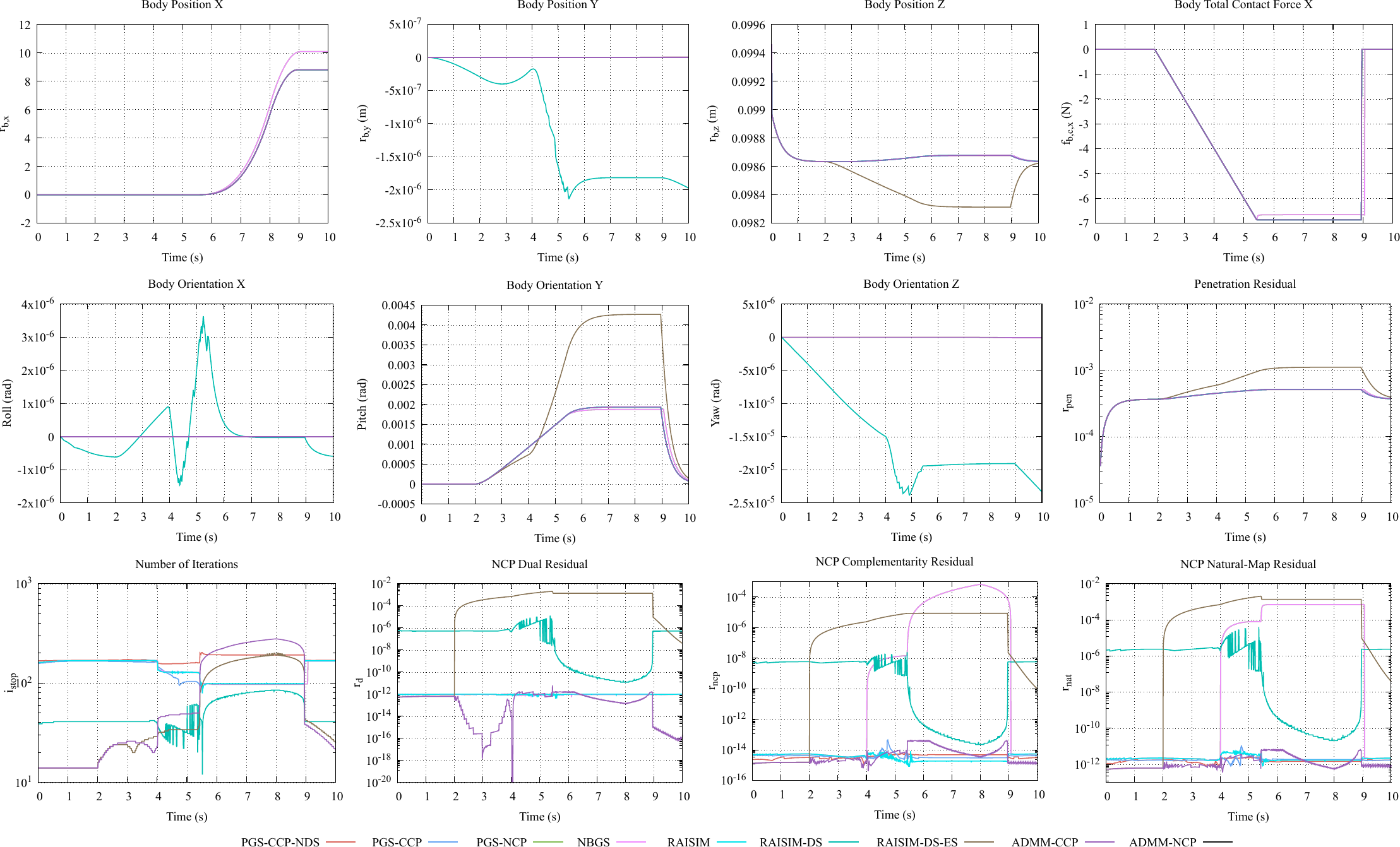}
\caption{\textit{Box-on-Plane}: Time-series plots of the second run with MuJoCo-like constraint softening enabled. All solvers render minimum-norm contact forces and the motions exhibit significantly less deviations due to the improved symmetry of the former. NBGS and RAISIM still overshoot the trajectory of the other solvers due to the forces not being maximally dissipative.}
\label{fig:box-on-plane-soft}
\end{figure*}
\begin{figure}[!t]
\centering
\includegraphics[width=1.0\linewidth]{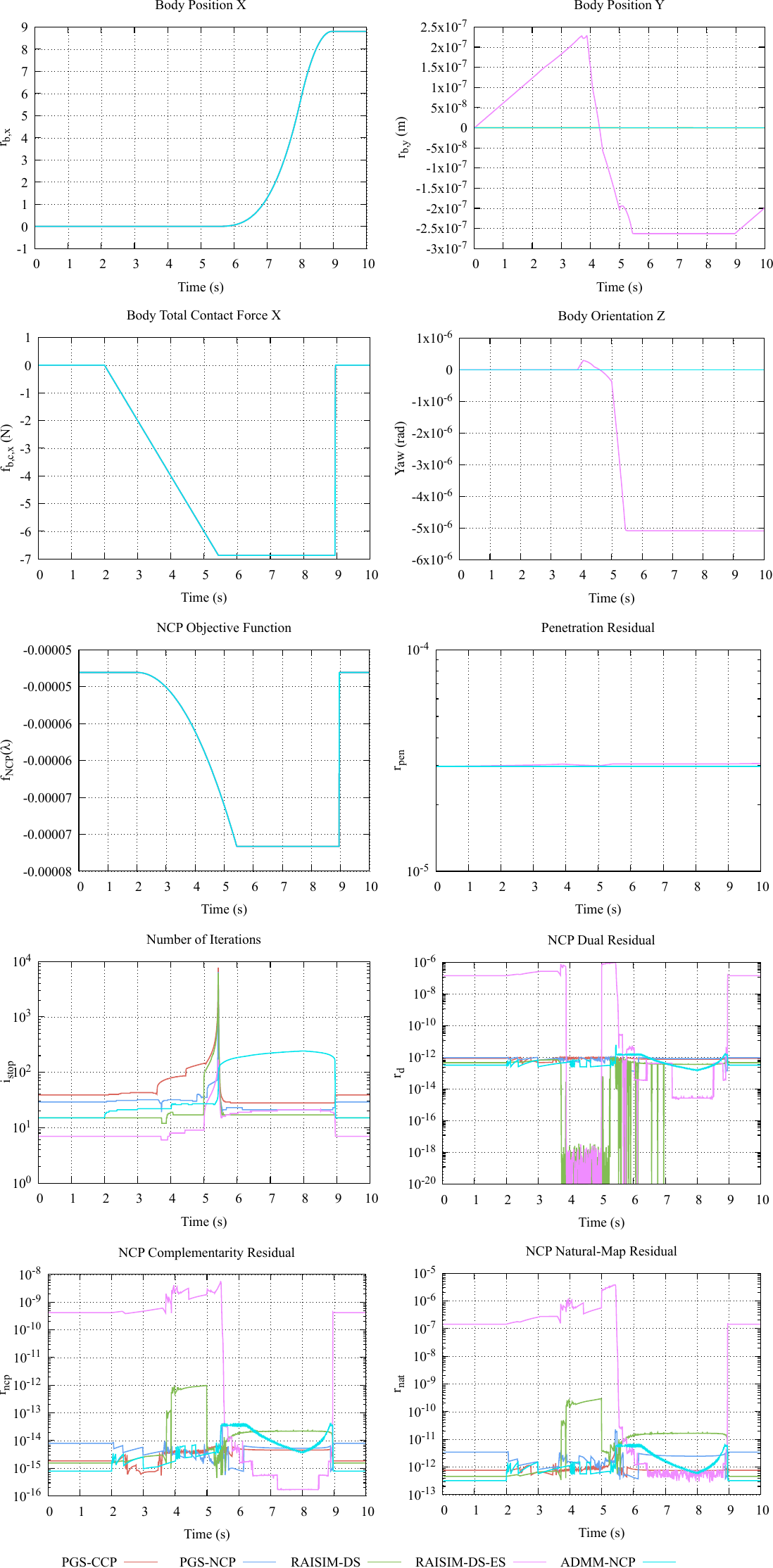}
\caption{\textit{Box-on-Plane}: A comparison focusing on the NCP-type solvers. While the overall motions, forces and optimization objectives generated by all solvers can seem very similar, a detailed view of the NCP dual, NCP complementarity and NCP natural-map residuals can exemplify the subtle differences in their solutions.}
\label{fig:box-on-plane-hard-ncp}
\end{figure}

The first run highlights the differences in the generated forces and motions. Fig.~\ref{fig:box-on-plane-force-distributions} depicts the contact force distributions rendered by each solver while the box is unforced, and Fig.~\ref{fig:box-on-plane-hard} presents time-series data over the interval $[0, t_f]$. For this run only, we also include PGS-CCP-NDS, a version of PGS-CCP without the De Saxc\'e correction. The results indicate that on this particular problem, the motion of the box is nearly identical for all NCP-type solvers, i.e. PGS-CCP, PGS-NCP, RAISIM-DS, and ADMM-NCP, while all others exhibit significant discrepancies. To complement Fig.~\ref{fig:box-on-plane-hard}, we also provide an isolated comparison of the NCP-type solvers in Fig.~\ref{fig:box-on-plane-hard-ncp}. ADMM-CCP is similar to ADMM-NCP, but the absence of the De Saxc\'e correction leads significant underestimation of the normal contact forces when the external force is applied, leading to a gradual sinking of the box as it slides. Although the force distributions of PGS-NCP and RAISIM-DS are uneven while the contacts are sticking, when sliding occurs they do not incur any translation along the Y-axis or rotation about the Z-axis. Contrary to this, the NBGS and RAISIM solvers do not render maximally dissipative frictional forces, as the respective motion trajectories overshoot those of the NCP-type solvers. Moreover, they result in translational and rotational deviations along the Y-axis and Z-axis, respectively. 

The PGS-CCP-NDS solver results in contact velocities with positive vertical components that lead to a jumping-like behavior. This artifact is a very clear demonstration of the effects of the CCP model, where duality is being enforced between contact reactions and velocities directly. Interestingly though, despite the jumping, the forces seem to be maximally dissipative, as the overall motion along the X and Y axes remains close to that yielded by NCP-type solvers. Comparatively, the inclusion of the De Saxc\'e correction in PGS-CCP 
proves indeed capable of eliminating the jumping artifacts, and results in the same overall behavior as the other NCP-type solvers. Also noteworthy is the fact that the forces rendered by PGS-CCP at rest point outwards as opposed to all other splitting-based solvers which point inwards\footnote{This result is quite counter-intuitive; at rest, the only term influencing the contact reactions is the gravitational force acting upon the box's CoM, which, via the lever-arms defined between the CoM and the contact positions defines the global free-velocity $\mathbf{v}_f$. From the perspective of each local problem, the gravity-induced wrench would cause outward sliding in the absence of all other contact forces. This means that when using a cold-start, in the first iteration we would expect the forces to point inwards opposing the sliding velocity, which would then be corrected by the increasing influence of the other contacts over multiple solver iterations. This is exactly the behavior observed for PGS-NCP, NBGS and RAISIM solvers, but not for PGS-CCP.}.

The second run demonstrates the effects of the regularization induced by constraint softening, and is summarized in Fig.~\ref{fig:box-on-plane-soft}. While the box is unforced, the softening leads to virtually the same minimum-norm force solution as that of the ADMM solvers. Similarly to the first run, the NBGS and RAISIM solvers fail to yield maximally dissipative frictional forces and thus overshoot the NCP-type motion trajectories. PGS-CCP-NDS once again leads to jumping but now the motion becomes even more erratic, causing rolling and even larger deviations along the Y-axis. For this reason we have excluded it from the comparisons presented Fig.~\ref{fig:box-on-plane-soft}. Interestingly, the NCP-type exhibit nearly identical trajectories as when configured for hard-contacts. This outcome indicates that constraint softening is indeed capable of producing a very near approximation of the NCP solution, given an appropriate selection of the parameters described in Sec.~\ref{sec:construction:softening}.

In addition, we executed an auxiliary curated run to directly highlight the relative performance between solvers on this problem. Performance profiles resulting from this are presented in Fig.~\ref{fig:box-on-plane-perfprof}, and exemplify several subtleties in comparing solvers. First, they demonstrate that all performance metrics from Sec.~\ref{sec:benchmarking:metrics} provide distinct and equally important views of constraint satisfaction. Case in point, the NCP dual and natural-map residuals are equally as important metrics as NCP complementarity. Second, the performance profiles can directly quantify how similar the NCP-type solvers actually are. Fig.~\ref{fig:box-on-plane-perfprof} shows that PGS-CCP, PGS-NCP, RAISIM-DS and ADMM-NCP differ by less than an order of magnitude to each other on all accuracy residuals, when compared head-to-head on the same problems. These differences are not so apparent in the time-series data of the independent runs.\\
\begin{figure}[!t]
\centering
\includegraphics[width=1.0\linewidth]{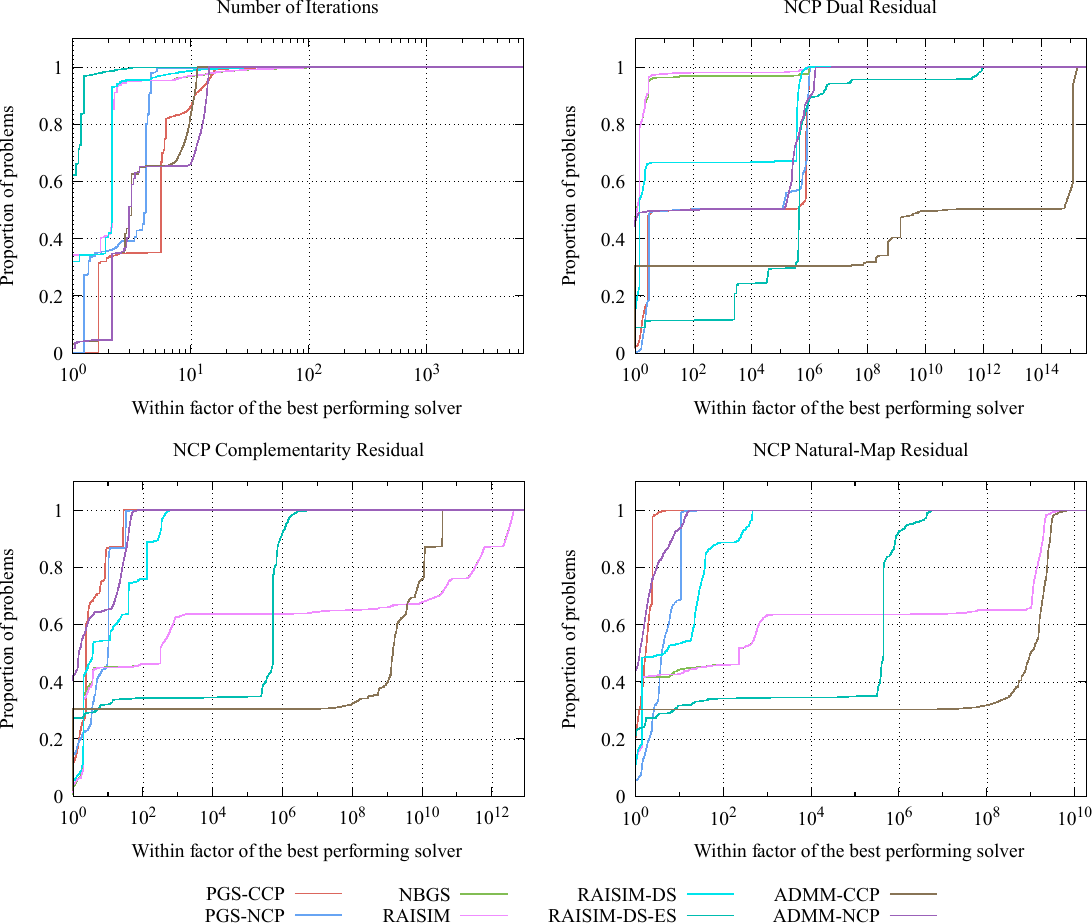}
\caption{\textit{Boxes-on-Plane}: Performance profiles showing the relative performance between solvers. While ADMM-NCP demonstrates the best behavior and performance overall, it is still noteworthy that PGS-CCP, PGS-NCP and RAISIM-DS come very close in terms of the accuracy metrics $r_{ncp}$ and $r_{nat}$ (bottom row). Interestingly, NBGS and RAISIM exhibit better performance overall in terms of $r_{dual}$ (top right), which is consistent with the fact that they demonstrated the least penetration error. However, their significantly reduced performance in terms of $r_{ncp}$ and $r_{nat}$ demonstrates that all accuracy metrics must be taken in to account simultaneously in order to summarize the overall accuracy of a solver.}
\label{fig:box-on-plane-perfprof}
\end{figure}

\subsubsection*{\textbf{Boxes-Fixed}}
\label{sec:experiments:primitives:boxes-fixed}
The second experiment involves a constrained system with a large mass-ratio being dropped and then dragged on a flat plane. Specifically, the system consists of two boxes rigidly attached via a fixed joint, where the bottom box has mass $m_b = \unit[0.1]{kg}$ and the top $m_t = \unit[10^3]{kg}$ resulting in a mass-ratio of $r_{m} = 10^4$. Both boxes have sides of dimension $d = \unit[20]{cm}$. A similar setup has been described by Erleben in~\cite{erleben2004stable} and Lidec in~\cite{lidec2024models}, but they stacked the two unequally weighted boxes atop one another as opposed to using a fixed joint. We chose the jointed configuration as our focus mainly lies in the interaction between joints and contacts. Note that collision detection was necessarily disabled between the two bodies since they were jointed. The objective here is to observe the combined effects of inertial-disparity and rank-deficiency on solver convergence. We executed three runs with all solvers set to high-precision: (a) start at rest and apply a linearly increasing force, (b) drop the system from an initial height of $z_{b} = \unit[0.5]{m}$ of the bottom box, and (c) combine both scenarios (a) and (b) and enable constraint stabilization with $\alpha_{j},\gamma_{k} = 0.1$. Note that ADMM-CCP was excluded from this experiment as it often diverged during our initial testing. 
\begin{figure}[!t]
\centering
\includegraphics[width=1.0\linewidth]{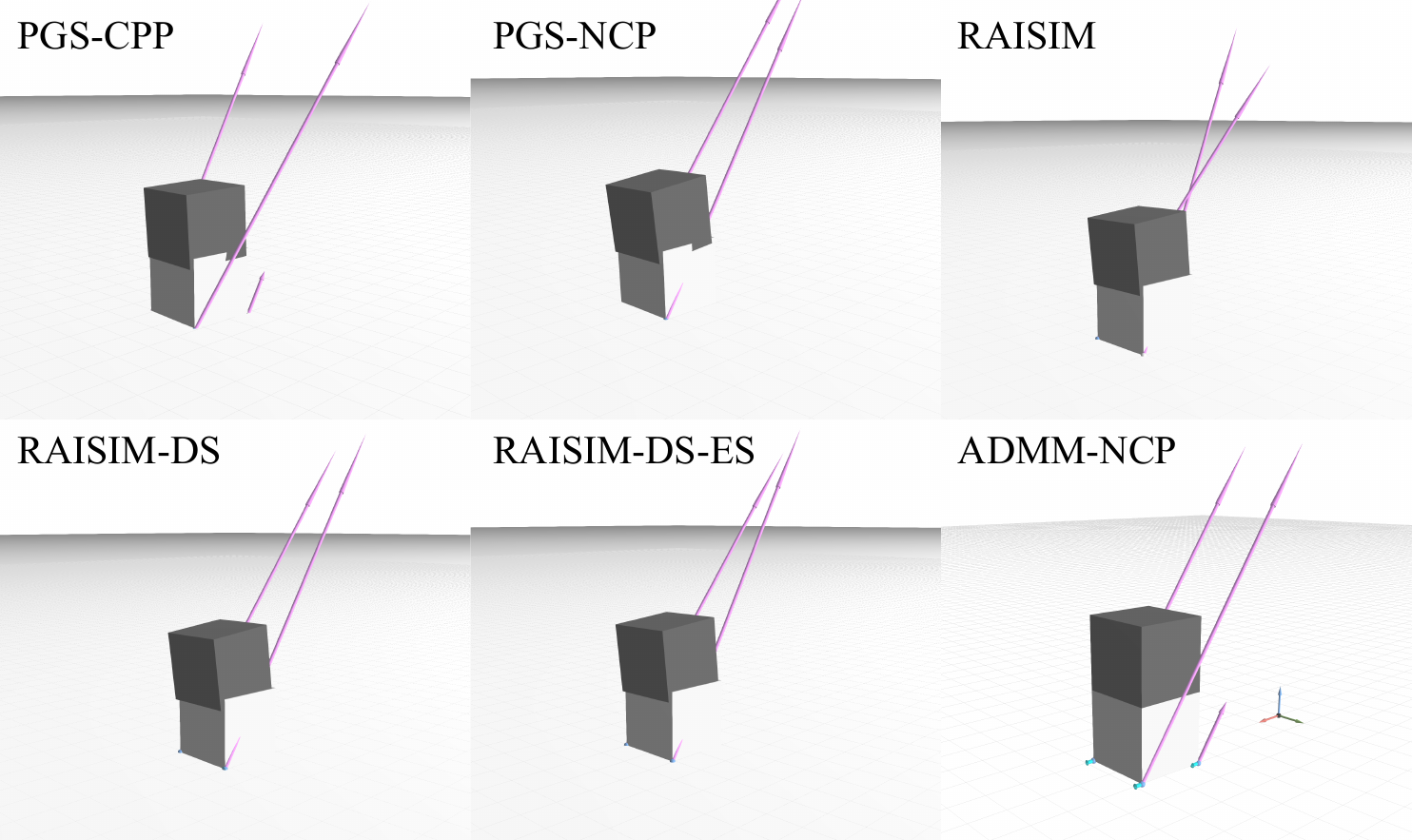}
\caption{\textit{Boxes-Fixed}: Examples of the constraint violation accumulated by each solver as well as the forces rendered. The translational deviation of the box causes a shifting of the overall CoM of the system that leads to the rear contacts opening while sliding.}
\label{fig:boxes-fixed-visualized}
\end{figure}

In the first run, the setup is similar to that described in Sec.~\ref{sec:experiments:primitives:Box-on-Plane}, in that the external force applied to the bottom box is ramped-up from $\unit[0]{Nm}$ to $f_{max} = \unit[1.01\,\mu_c \, (m_b + m_t) \, g]{Nm}$\footnote{The $1.01$ factor in the computation of the maximum external force is used to induce slipping by only marginally exceeding the stiction force.} over a duration of approximately $\Delta{t}_{ramp} = \unit[2]{s}$, and held at the maximum value for another $\Delta{t}_{max} = \unit[6]{s}$. Once the external force is removed, the system is left to decelerate and rest over a final period of $\Delta{t}_{free} = \unit[2]{s}$, with and thus total duration of the experiment is $t_f = \unit[12]{s}$. The time-series plots of this run is shown in Fig.~\ref{fig:boxes-fixed-dragged}. Importantly, we observed that all solvers failed to converge within $N_{max}$ iterations. However, despite the lack of convergence, ADMM-NCP was still able to exhibit the expected behavior. All others where able to initially render correct constraint reactions, but with the application of the external force, they underestimated the joint reactions leading to the top box drifting away from the nominal pose relative to the bottom box. A screen capture showing this effect is depicted in Fig.~\ref{fig:boxes-fixed-visualized}. Thus, one important observation from this run is that, for practical applications, solvers can still output usable constraint reactions despite a lack of convergence. This is a somewhat a trait of robustness, since non-convergence does not mean that forces are infeasible. However, this also demonstrates how choosing an appropriate value for $N_{max}$ forces us to make a compromise between accuracy and speed.
\begin{figure*}[!t]
\centering
\includegraphics[width=1.0\linewidth]{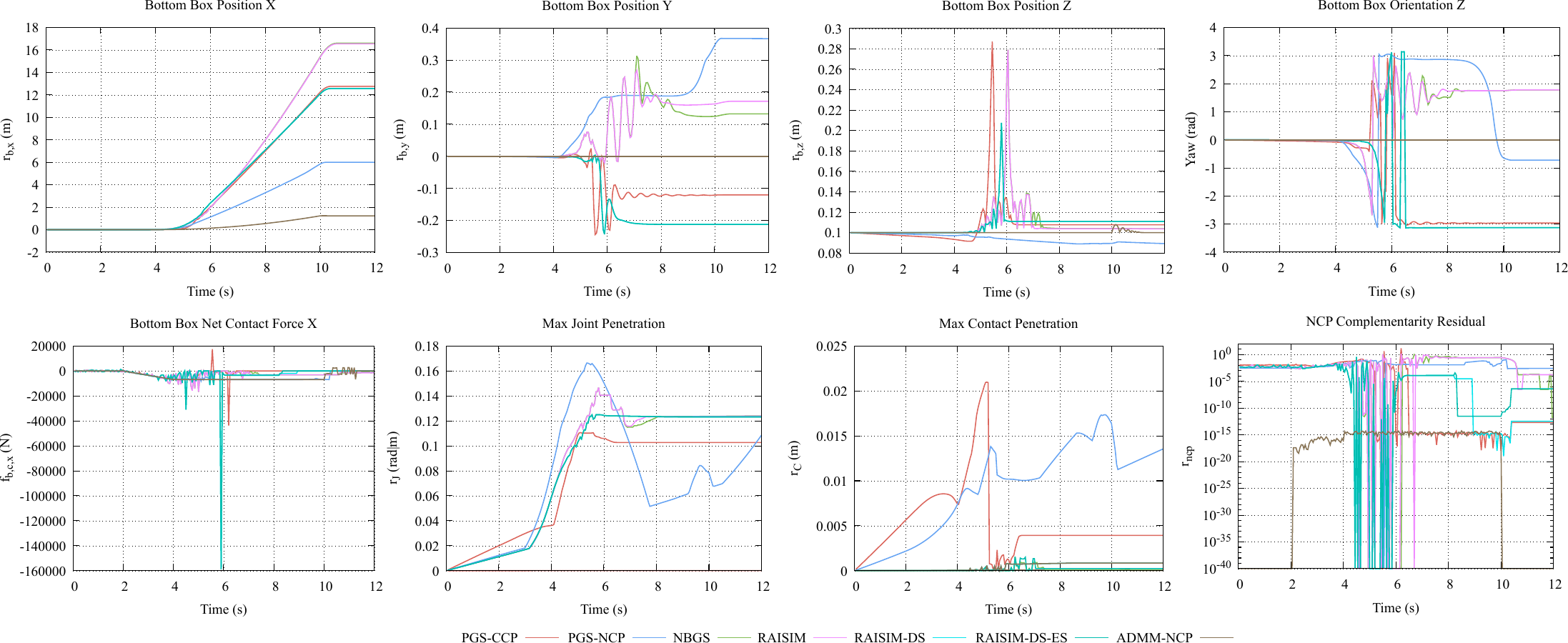}
\caption{\textit{Boxes-Fixed}: Time-series plots of the motions, forces and performance metrics for the dragging run. The top box is of mass $m_t=\unit[10^{3}]{kg}$ while the bottom box $m_b=\unit[0.1]{kg}$ with a system mass-ratio of $r_m=10^{4}$. Despite not converging, ADMM-NCP was still able to realize the expected behavior with minimal constraint violation. All other solvers (PGS-*, NBGS and RAISIM-*) grossly underestimated the constraint reaction forces leading to a collapsing of the joint and body interpenetration.}
\label{fig:boxes-fixed-dragged}
\end{figure*}
\begin{figure*}[!t]
\centering
\includegraphics[width=1.0\linewidth]{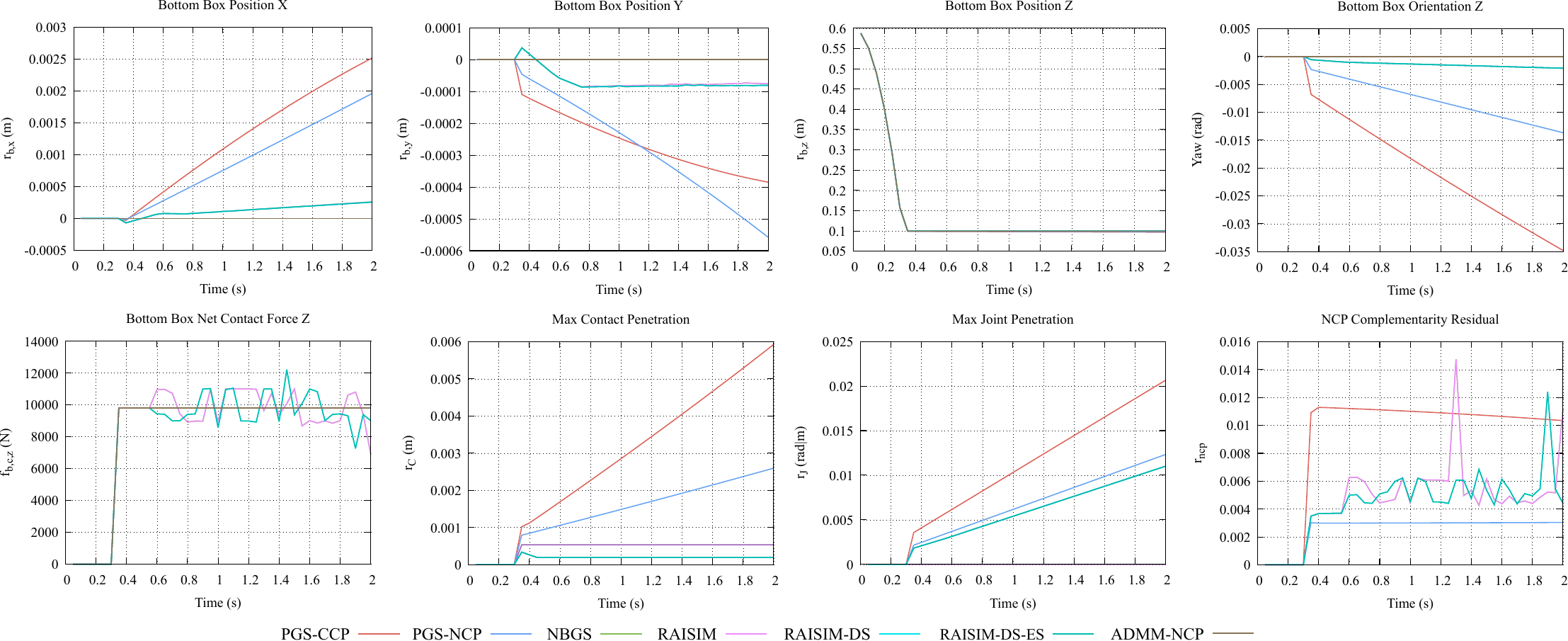}
\caption{\textit{Boxes-Fixed}: Time-series plots of the dropping run showing the effects of the impact on the constraint violation (center-bottom) and overall motion (top row). These effects are also apparent in the NCP complementarity residual (bottom-right) were the solvers with large constraint violation also exhibit large values thereof. While NBGS and RAISIM-* variants exhibit relatively good performance w.r.t velocity constraint satisfaction, the contact forces along the normal (bottom-left) exhibit very large fluctuations w.r.t the correct solution attained by PGS-* and ADMM-NCP.}
\label{fig:boxes-fixed-dropped}
\end{figure*}
\indent In the second run, the system is dropped from an initial height to evaluate the capacity of the solvers to correctly propagate impact forces through the joints. The corresponding time-series plots are shown in Fig.~\ref{fig:boxes-fixed-dropped}. Despite the large mass-ratio, constraint violation was initially minimal after the impact for all solvers. The splitting-based solvers, however, accumulated significant joint constraint violation over time. Of particular note is that NBGS and the variants of RAISIM, all were able to keep contact penetration values even smaller than those of ADMM-NCP. This demonstrates the efficacy of using a holistic local contact model and solver such as (\ref{eq:solvers:projectors::single-contact-mdp-problem}), as opposed to those of PGS-CCP and PGS-NCP.

The third run that introduces constraint stabilization is depicted Fig.~\ref{fig:boxes-fixed-full-stbl}. The stabilization brought a significant reduction in constraint violation for all solvers. As observed previously, the application of the external force causes NBGS and vanilla RAISIM to insufficiently propagate the forces to the joint when sliding begins, causing the motion to differ w.r.t. the other solvers. However, this is somewhat of a false negative, as the deviation is due to stack of boxes tipping over during the dragging phase. See Fig.~\ref{fig:boxes-fixed-full-stbl} for a more detailed description of this discrepancy. Overall, we observed that constraint stabilization is capable of reducing constraint violation while retaining the properties of the NCP\footnote{This in part, can be considered experimental verification of the invariance property of the De Saxc\'e operator, as mentioned in Sec.\ref{sec:construction:stabilization} and proven analytically in Appendix.~\ref{sec:apndx:desaxce-properties:desaxce-invariance}}.
\begin{figure*}[!t]
\centering
\includegraphics[width=1.0\linewidth]{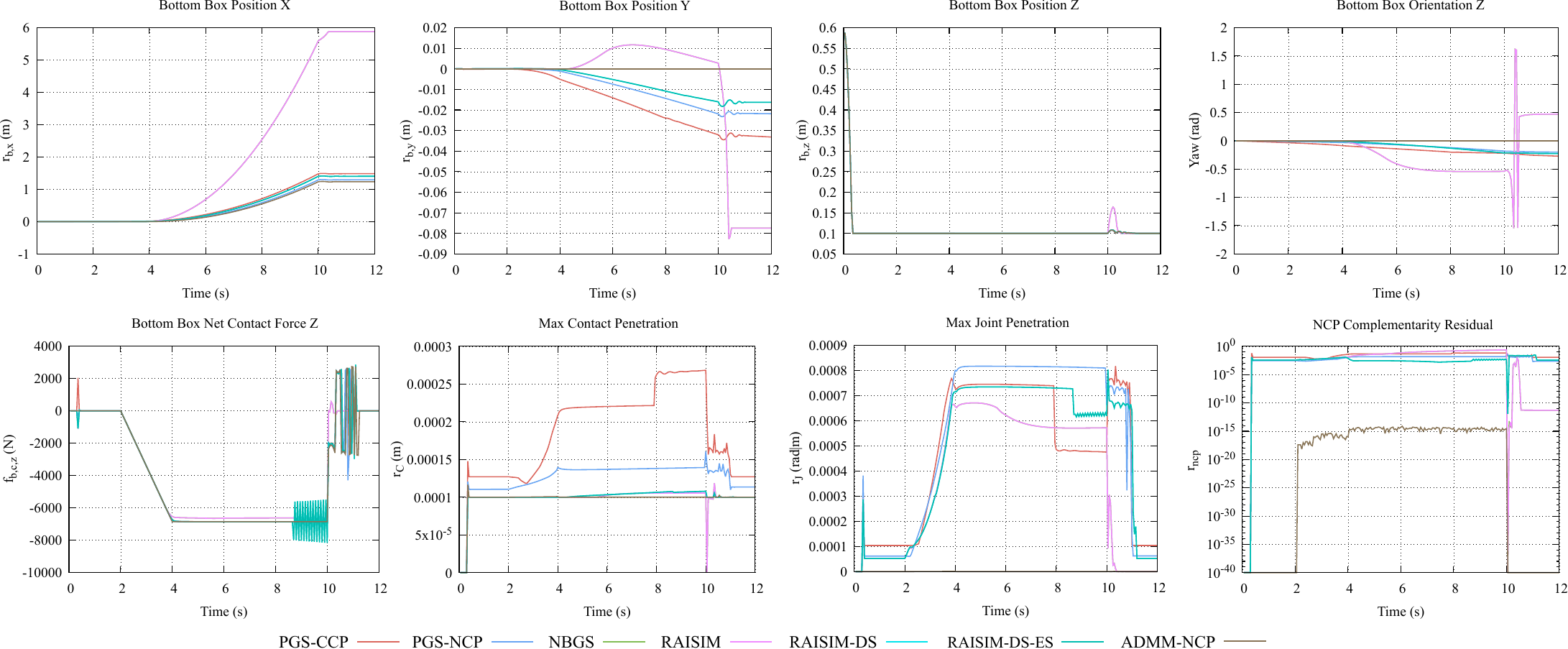}
\caption{\textit{Boxes-Fixed}: Time-series plots of the combined drop and dragging run that also enables constraint stabilization. The augmentation to the problem leads to significant improvement w.r.t constraint violation (center-bottom) for all solvers. In the case of the NBGS (green) and vanilla solvers RAISIM (magenta) the large deviation from the correct solution (brown, ADMM-NCP) is somewhat circumstantial. The initial impact causes a dislodging of the joint, which leads to a shifting of the total center of mass. This leads to a tip-over occurred, which induced a short flight-phase wherein the external force dominated the reduced number of contacts, therefore dragging the system excessively.}
\label{fig:boxes-fixed-full-stbl}
\end{figure*}
\begin{figure}[!t]
\centering
\includegraphics[width=1.0\linewidth]{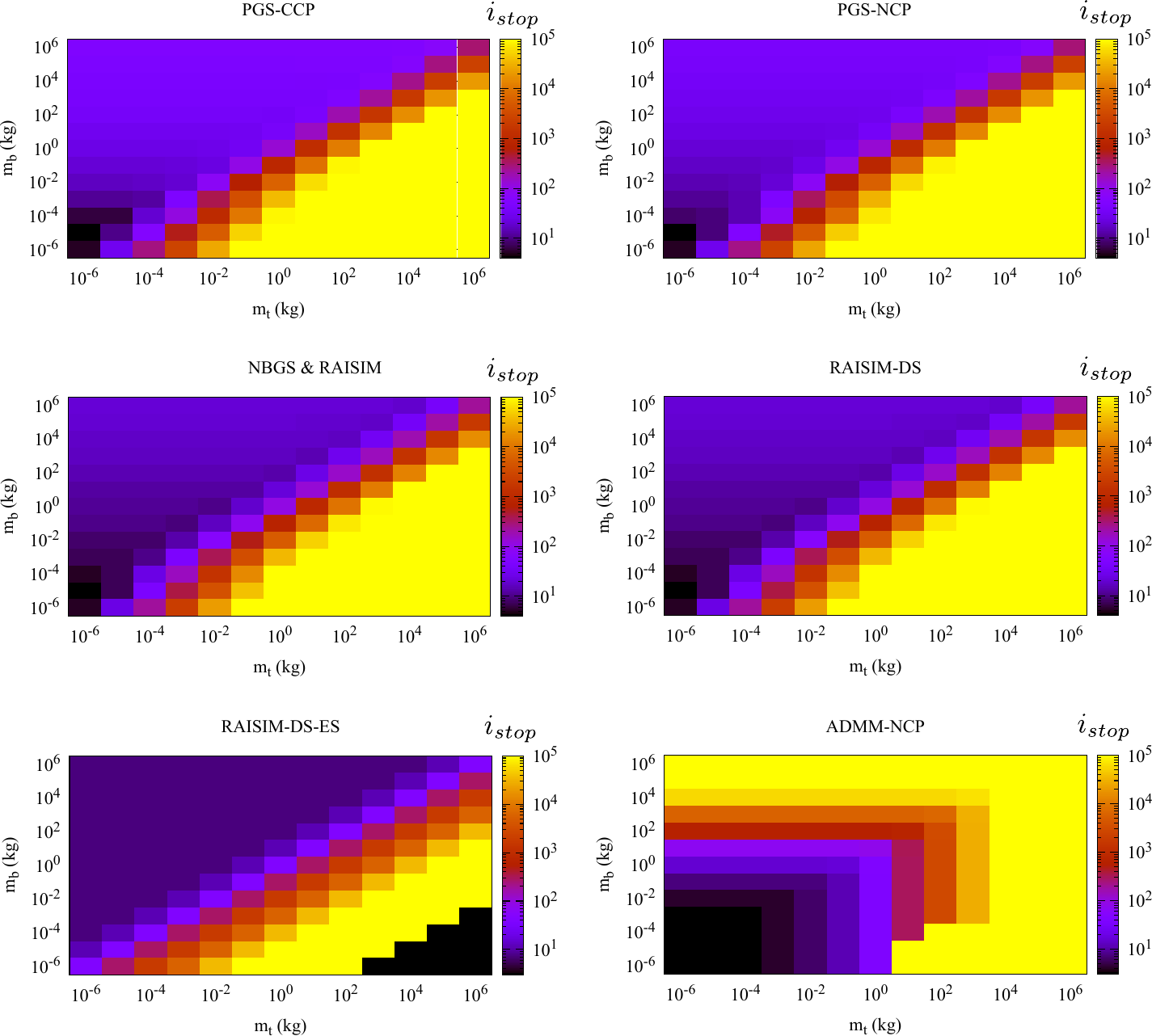}
\caption{\textit{Boxes-Fixed}: A heatmap-like rendering of solver iterations $i_{stop}$ as functions of the masses $m_b$ and $m_t$ of the bottom and top bodies, respectively. All splitting methods perform better when the masses are equal (left-to-right diagonal elements) as well as when the top body is lighter (upper-left diagonals). Note also the varying scales and respective color gradients of each solver. We were unable to use a common scale as the gradients became incomprehensible. Most revealing is the experimental verification of the fundamental  differences between ADMM and splitting w.r.t inertial disparity (i.e. mass-rations). The latter all suffer when the relative mass ratio is larger than $\gtrapprox 1000$ and the bottom body is heavier. This shows how splitting-based solvers struggle to propagate contact reactions through the joints to the top body. Conversely, ADMM-NCP is more robust to large mass ratios, and in most cases manages to converge within $N_{max}=10^{4}$ iterations. However it seems to be more sensitive to the absolute scale of the heaviest body, which, directly determines the spectral radius $L = \rho(\mathbf{D})$ of the Delassus matrix.}
\label{fig:boxes-fixed-heatmap-rest-iters}
\end{figure}
\begin{figure}[!t]
\centering
\includegraphics[width=1.0\linewidth]{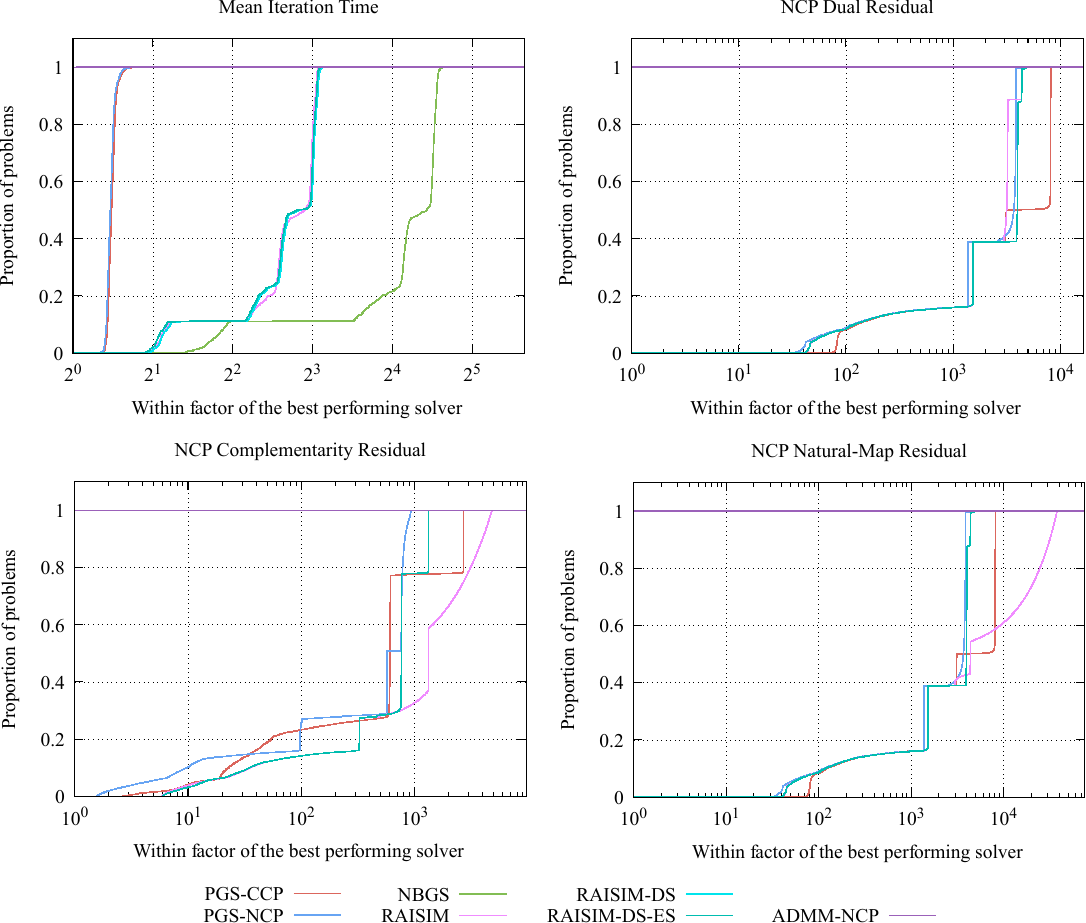}
\caption{\textit{Boxes-Fixed}: performance profiles showing the relative performance between solvers. In this scenario ADMM-NCP is the outright winner, as indicated by being the best solver on 100\% of the samples across all metrics. It was the only solver to ever be able to converge on some samples, and dominated all other solvers by more than two orders of magnitude on all accuracy metrics on the majority of samples.}
\label{fig:box-fixed-perfprof}
\end{figure}
\begin{figure*}[!t]
\centering
\includegraphics[width=1.0\linewidth]{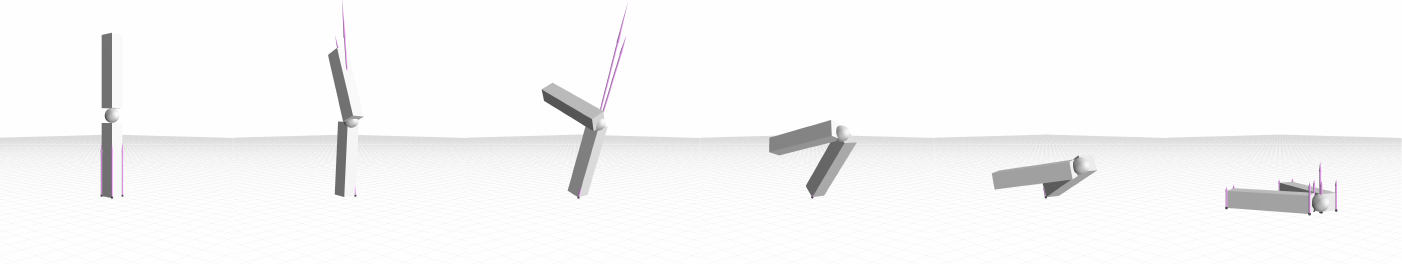}
\caption{\textit{Nunchaku}: A rendering that showcases the motions of the system in the corresponding set of experiments. The system is initialized to drop from a fixed initial height. After impacting the plane, all four corners of the bottom box are in contact, and the system is able to remain vertical (left-most frame) for several seconds until numerical drift from the integrator eventually leads to the top box collapsing into the intermediate sphere and onto the bottom box (left middle frame). As collisions are disabled between the boxes and the spheres, only the boxes between them can render internal contacts. The collapse indeed leads to this case, and the top and bottom boxes remain in internal contact until the entire system contacts the plane (right-most frame). This particular rendering was generated using the ADMM-NCP solver, which can generate minimum-norm contact reactions without internal couplings (left/right-most frames).}
\label{fig:nunchaku-motion}
\end{figure*}

In addition, motivated by the results in the three primary runs, we also evaluated the effects of varying the mass-ratio. To more thoroughly analyze how mass-ratios impact convergence, we performed a grid-search over combinations of masses in the range $m_{bt},m_{bt} \in [10^{-6}, 10^{6}]$, generating a set of single-step problems: one at rest and another while the bottom box was pushed along the X-axis with $f_{max}$. Solving both sets with $N_{max}=10^5$ to allow for convergence to occur reveals very interesting properties that distinguishes the splitting-based solvers from ADMM-NCP. The full extent of this analysis can be found in Fig.~\ref{fig:boxes-fixed-heatmap-rest} and Fig.~\ref{fig:boxes-fixed-heatmap-push} in Appendix.~\ref{sec:apndx:additional-figures}, but an exemplary extract is provided in Fig.~\ref{fig:boxes-fixed-heatmap-rest-iters}. These figures present heat-map-like renderings of the solver iterations and accuracy metrics as functions of the two body masses. Particularly revealing is how they show that ADMM-NCP is more sensitive to the absolute numerical scale of the actual masses, as opposed to their ratio. More specifically, these renderings, and especially those depicting $i_{stop}$ in Fig.~\ref{fig:boxes-fixed-heatmap-rest-iters}, indicate that the convergence of ADMM-NCP actually depends more on the spectral radius $\rho(\mathbf{D})$ than on the condition number $\kappa_{\mathbf{D}}$. This behavior seems to be inline with the analysis presented by Nishihara et al in~\cite{nishihara2015} regarding the convergence properties of ADMM. This also motivates the use of \textit{preconditioning strategies}, as was proposed by Tasora~\cite{tasora2021admm}, which we intend to investigate further in the future.

Lastly, we executed a curated run using the initial mass distributions in order to render the performance profiles depicted in Fib.~\ref{fig:box-fixed-perfprof}. On this particular problem ADMM-NCP clearly dominated, as it was the only one able to converge at all on certain samples, and exhibited the best accuracy metrics on all samples. Most importantly, w.r.t the accuracy metrics, all splitting-based trailed by a factor of at least one, and up to three, orders of magnitude, essentially corresponding to significant digits of accuracy. This degradation of accuracy is therefore consistent with the observed underestimation of the computed constraint reactions.\\
\begin{figure*}[!t]
\centering
\includegraphics[width=1.0\linewidth]{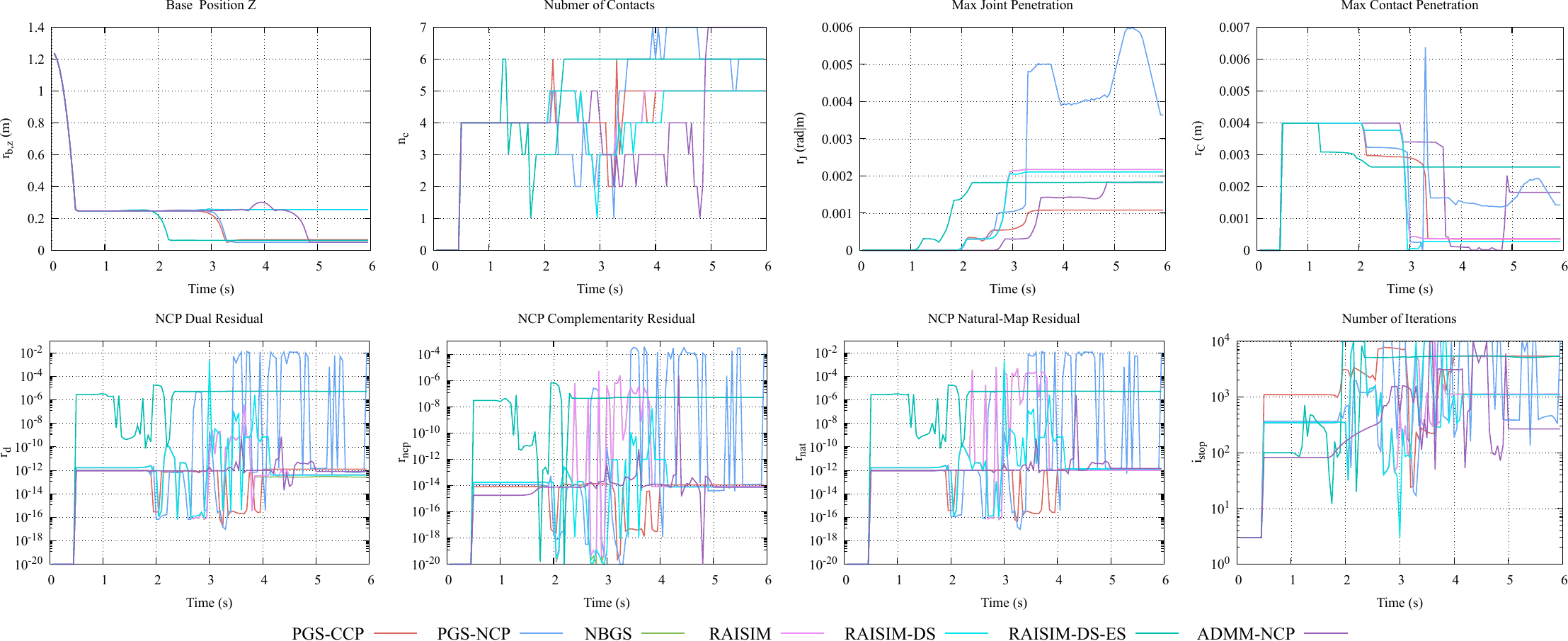}
\caption{\textit{Nunchaku}: Summary of the first run, solving the hard-contact problem without augmentations. All solvers exhibit non-zero constraint penetration, i.e. configuration-level constraint residuals. Considering the absence of constraint stabilization, their performance is rather good on such a small-scale problem.}
\label{fig:nunchaku-hard}
\end{figure*}
\begin{figure}[!t]
\centering
\includegraphics[width=1.0\linewidth]{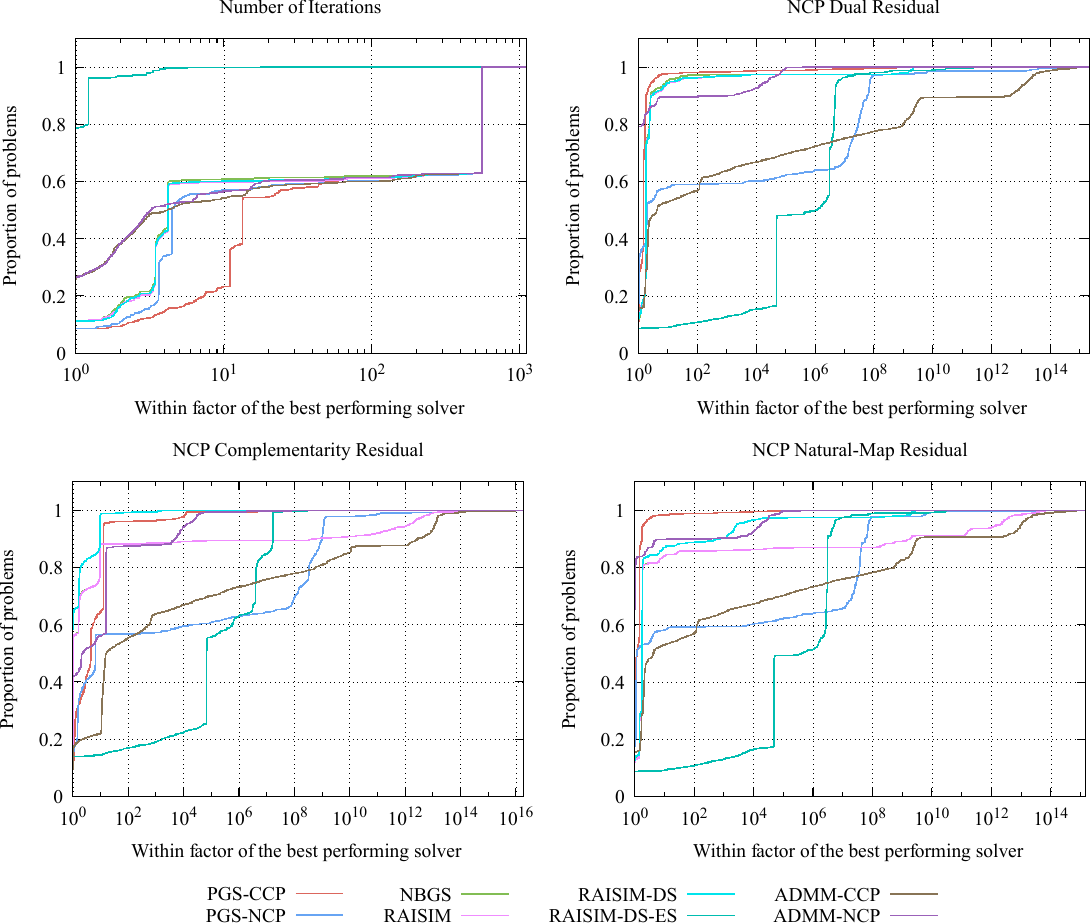}
\caption{\textit{Nunchaku}: Performance profiles showing the relative performance between solvers on this problem. While PGS-CCP is the least performing in terms of number of iterations (top left), it demonstrates accuracy comparable to ADMM-NCP. RAISIM and RAISIM-DS also seem to perform best w.r.t NCP complementarity (bottom left). The least accurate, but fastest to converge is the early-stopping RAISIM-DS-ES.}
\label{fig:boxes-nunchaku-perfprof}
\end{figure}
\subsubsection*{\textbf{Nunchaku}}
\label{sec:experiments:primitives:nunchaku}
The third primitive experiment involves a system we have fondly named \textit{Nunchaku}, consisting of two boxes and a sphere connected by two spherical joints. Each box is of dimensions $\unit[0.1]{m} \times \unit[0.1]{m} \times \unit[0.5]{m}$, the sphere is of radius $r = \unit[0.05]{m}$, and all bodies are of equal mass $m = \unit[1]{kg}$. The objective of this experiment is to observe the effects of hyperstaticity when multiple bodies come into contact with the environment (i.e. external collisions) as well as with each other (i.e. internal collisions). In addition, we also wanted to observe the efficacy of using constraint softening and stabilization. In this scenario, we do not apply any external force on the system. Instead, we vertically align all bodies, suspend the system at a fixed height above the ground plane and let it drop. The collapsing motion of the constrained system leads to a large number of internal and external contacts, thus resulting in a multitude of constraint loops, as exemplified in Fig.~\ref{fig:nunchaku-motion}. This experiment is executed in the form of an ablation of constraint softening and stabilization over four runs: (a) as a hard-constrained problem, (b) with only constraint stabilization enabled, (c) with only constraint softening enabled (d) both constraint softening and stabilization enabled.

In the first run all solvers behaved reasonably, with minimally perceptible constraint violation. The time-series data, depicted in Fig.~\ref{fig:nunchaku-hard}, indicate that most solvers are indeed capable of keeping constraint violation to within a few millimeters for both joints and contacts, with PGS-NCP and and ADMM-CCP performing the worst. Observing the trends of the number iterations required for each solve, it becomes clear that the aforementioned solvers reached the max iterations exactly at times of increased violation. Thus, similarly to the effects of a high mass-ratios observed in the Boxes-Fixed, the rate of convergence on Nunchaku is impacted by the hyperstaticity due to significant rank-deficiency in the contact Jacobian $\mathbf{J}_{C}$.

The second run introducing constraint stabilization is depicted in Fig.~\ref{fig:nunchaku-hard-stbl}, and demonstrates a drastic reduction in constraint violation across all solvers, even in cases where the solvers fail to converge before the maximum iterations. The third run is depicted in Fig.~\ref{fig:nunchaku-soft} and shows that MuJoCo-like constraint softening scheme brings a significant decrease in the number of iterations and the overall constraint violation. In addition, it bears no discernible impact on the complementarity and natural-map residuals, indicating that the constraint softening mostly retains physical plausibility and accuracy. 

The fourth run that includes both constraint softening and stabilization is shown in Fig.~\ref{fig:nunchaku-soft-stbl}. Combining the two augmentations significantly increases overall performance for all solvers, both in terms of constraint violation and the number of iterations. One can argue that if the parameters of constraint softening are set appropriately for the problem then additional constraint stabilization is redundant. Our goal, however, is to evaluate the individual effects, as well as the interplay, of these augmentations when using recommended default parameters. Although an extensive tuning of these parameters can indeed be beneficial, it is often quite laborious since, to the best of our knowledge, no automatic procedure exists at present time.

An final curated run on the hard-contact problem is used to generate the performance profiles depicted in Fig.~\ref{fig:boxes-nunchaku-perfprof}, and summarize the comparison of all solvers applied to this system. Interestingly, these show that on this system RAISIM and RAISIM-DS seem to outperform even ADMM-NCP w.r.t the complementarity residual $r_{ncp}$ accuracy metric, even if only by a small margin. Although this result is surprising, it can be understood by considering that the vast majority of samples include sticking or opening contacts, whereas one of the most prominent advantages of ADMM-NCP over the other solvers lies in cases of frictional sliding. Moreover, PGS-CCP also demonstrates good accuracy performance comparable to that of ADMM-NCP, but is the slowest of all in terms of convergence. RAISIM-DS-ES was the fastest solver overall on this problem, but also the least accurate, trailing all others by multiple orders of magnitude w.r.t the accuracy metrics.\\

\subsubsection*{\textbf{Fourbar}}
\label{sec:experiments:primitives:fourbar}
\noindent
The fourth and final primitive experiment involves a classic planar four-bar linkage, which we will simply refer to as \textit{Fourbar} for brevity. Such systems form the basis of many types of mechanisms that transfer motion and power. Indeed they are ubiquitous in mechanical assemblies, ranging from Ackermann steering mechanisms~\cite{mitchell2006analysis} in automobiles to autonomous excavators~\cite{egli2022general}. The Fourbar problem we employ in this work is categorized as a planar parallelogram-type with an RRRR joint arrangement, where each 'R' indicates a 1D revolute (a.k.a. hinge) joint. This system provides a versatile reference problem for evaluating all components of physics simulators such the dual solvers we evaluate in this work, as well as numerical integration schemes~\cite{maloisel2025versatile}.
\begin{figure*}[!t]
\centering
\includegraphics[width=1.0\linewidth]{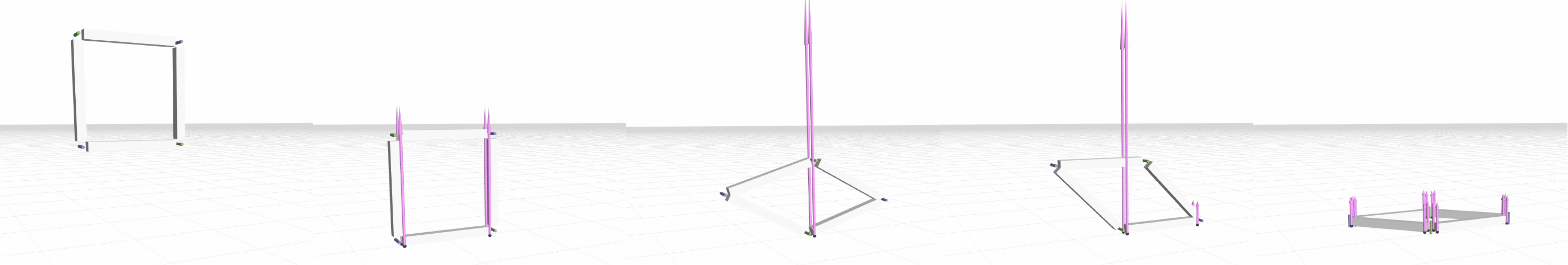}
\caption{\textit{Fourbar}: Renderings of the motion of the system and active contact forces along the different phases of the experiment. The four-bar is dropped from an initial height from which it begins to shift onto one side due to gradually accumulating numerical drift due to the semi-implicit Euler integrator. From there it rolls over to the right lateral link, where it is pushed by an external force so that it can collapse. The external disturbance serves to push the system into a configuration where all bodies are contacting the ground, each at 4 co-planar points. In many cases this occurs while the joint limits are still active. This in this final phase, a maximal number of active constraints are present, in addition to the intrinsic kinematic loop of the system.}
\label{fig:fourbar-drop}
\end{figure*}
\begin{figure}[!ht]
\centering
\includegraphics[width=1.0\linewidth]{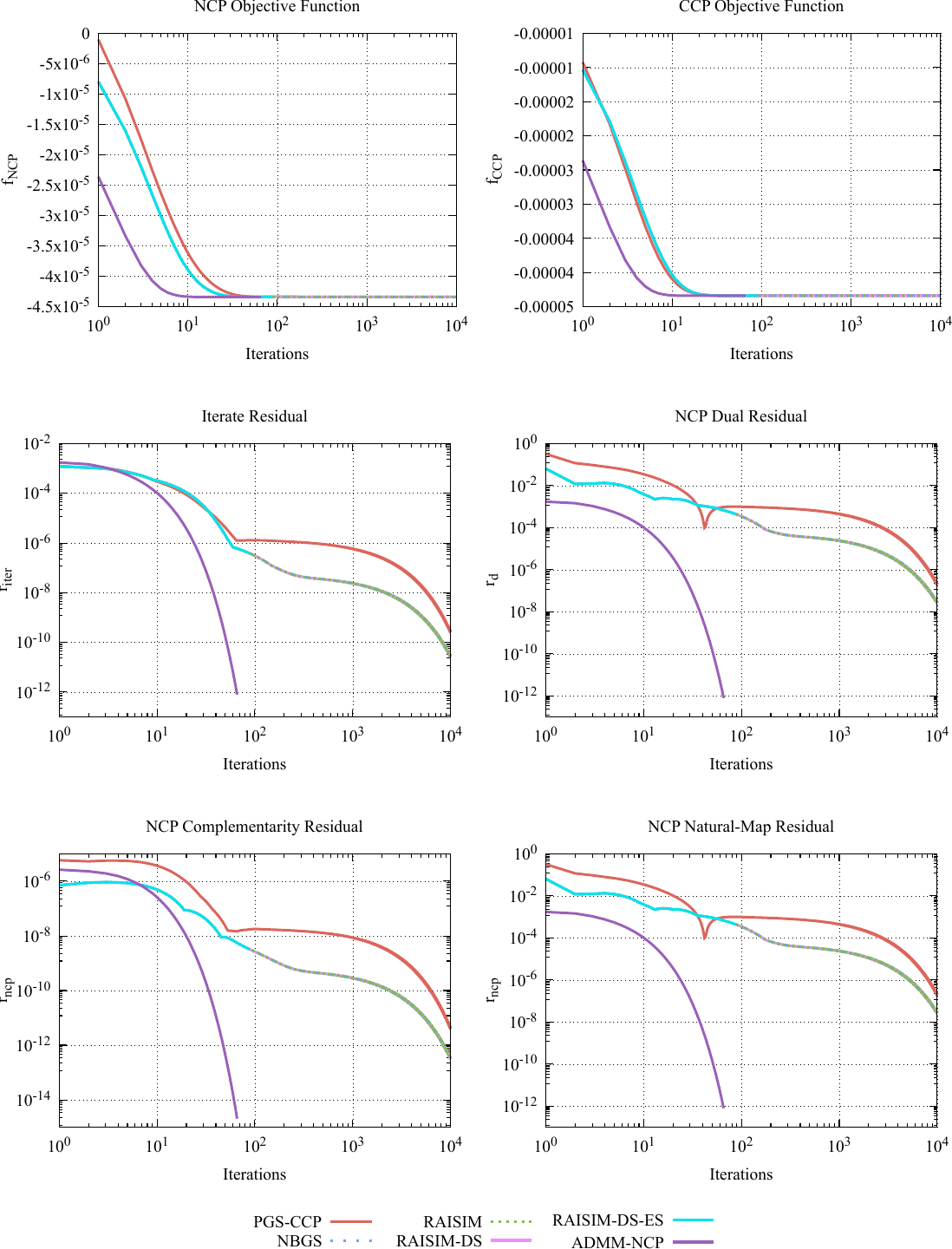}
\caption{\textit{Fourbar}: Convergence profiles of the most prominent solvers, with both X and Y axes are rendered in $log10$ scale. The NCP and CCP objectives (top row) are juxtaposed with the accuracy performance metrics to highlight the crucial differences between them. RAISIM-DS-ES relies on the former to determine convergence, and thus traverses the same trajectory as NBGS, RAISIM and RAISIM-DS, but stops much sooner than these. Meanwhile, PGS-CCP presents the slowest convergence of all, with a notable \textit{notch} in the NCP dual and natural-map residuals at around iteration $i=70$. ADMM-NCP presents essentially a linear convergence rate that is approximately two orders of magnitude larger than the splitting-based solvers.}
\label{fig:fourbar-convergence}
\end{figure}
\begin{figure}[!ht]
\centering
\includegraphics[width=1.0\linewidth]{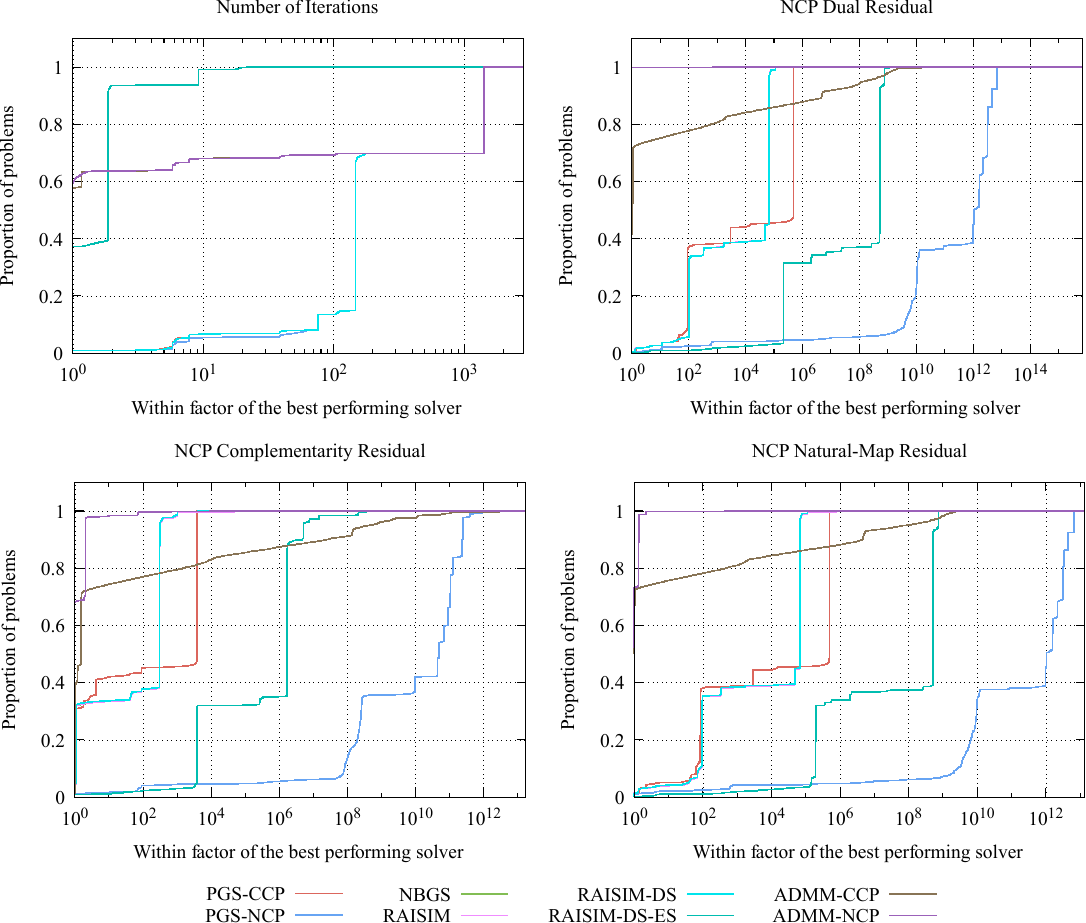}
\caption{\textit{Fourbar}: Performance profiles of the full set of solvers. In contrast to the other primitive problems where the differences between solvers were less pronounced, in this case ADMM-NCP proves to be the best solver overall by a wide margin. Even w.r.t the number of iterations $i_{stop}$, ADMM-NCP was best on approximately $\%60$ of the problems, and $\%70$ w.r.t the accuracy metrics, while being within a factor of 10 close to the best solver on $\%100$ of the problems.}
\label{fig:fourbar-perfprof}
\end{figure}

All four bodies are of dimensions $\unit[0.01]{m} \times \unit[0.01]{m} \times \unit[0.1]{m}$ and mass $m = \unit[1]{kg}$. Of the four revolute joints, two are actuated and two are passive with the actuated joints opposing each other w.r.t one of the diagonals of the parallelogram. All four joints define their rotation w.r.t the rest state when the Fourbar is a square and are subject to limits at $q_{j} \in [-\frac{\pi}{4}, \frac{\pi}{4}]$. The bottom link is the designated \textit{base} body and is initialized at a height $z_{b} = \unit[10]{cm}$ above a plane aligned to the world origin. The system is first dropped in order to impact the plane, from where it eventually begins to roll to one side due to numerical drift from state integration. Due to the imposed joint limits, it locks onto a rhombus-like state and is left to stabilize until a force amounting to the total mass of the system is applied to whichever body is the highest at time $t=\unit[10]{s}$. The push causes the Fourbar to fall onto the plane where all four bodies contact the ground. The objective of this experiment is to therefore evaluate the full set of solvers in scenarios where all constraints types are present, i.e. joints, limits and contacts, and while they all act simultaneously in the presence of an intrinsic kinematic loop. As done for the previous primitive systems, the experiments are executed over multiple runs to perform an ablation on constraint stabilization and softening. Moreover, in order to also evaluate the effects of inertial disparity, we execute an additional run where the masses of the lateral (left-right) bodies are set to $m_L = \unit[1]{kg}$ and $m_H = \unit[1000]{kg}$. This final runs thus evaluates the solvers in a scenario where all types of ill-conditioning are present.

\begin{figure*}[!t]
\centering
\includegraphics[width=1.0\linewidth]{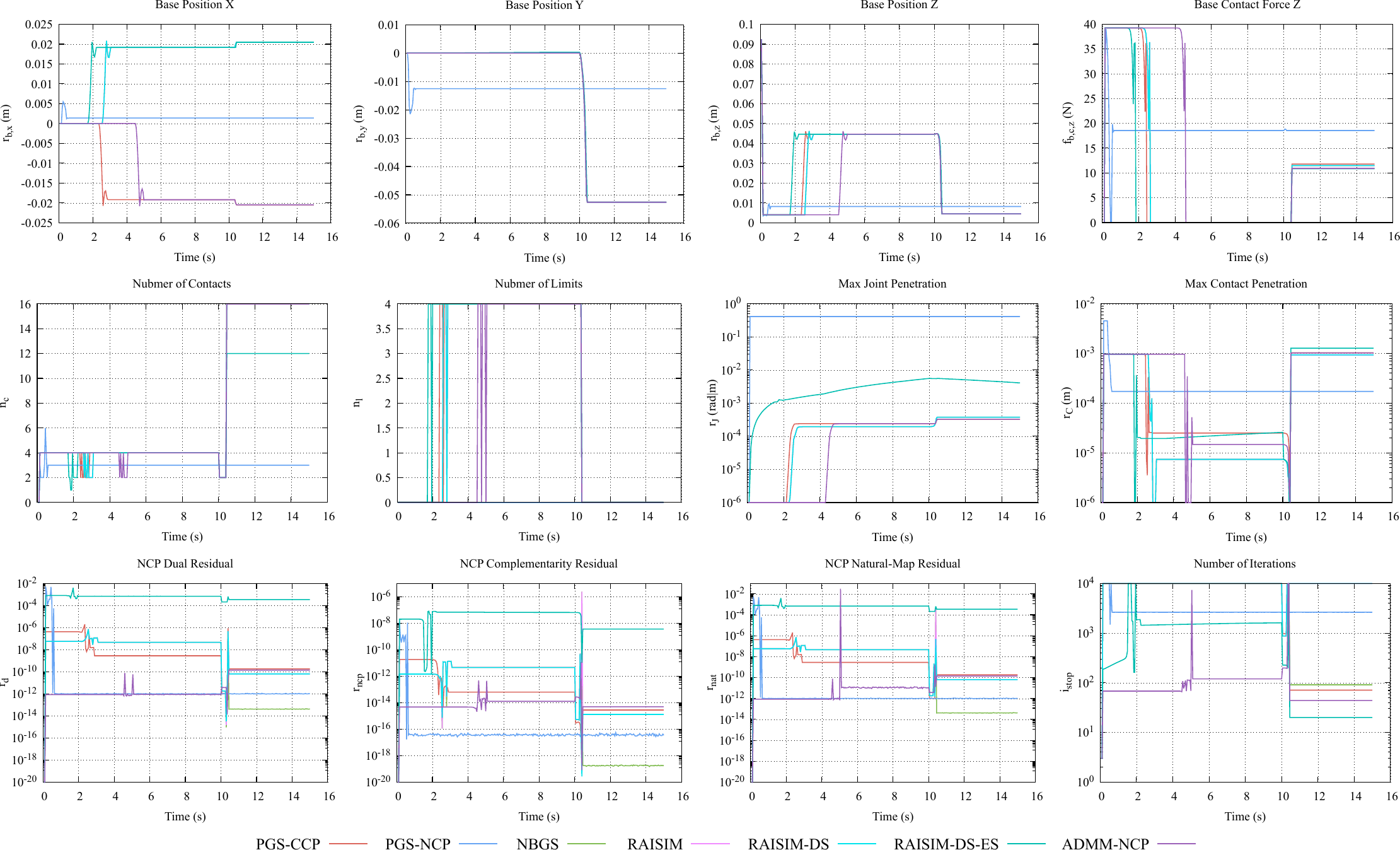}
\caption{\textit{Fourbar}: Summary of the first run on the un-augmented hard-contact problem. The top row depicts the absolute position of the designated base link, i.e. the lowest of the two horizontal bodies. The second row depicts the net contact forces acting on the base along the respective coordinates. The system is initialized at a fixed height (top-right) from where it falls until impacting the ground at about $t \simeq \unit[0.2]{s}$. The system shifts to one of two sides along the global X-axis, depending on the numerical drift introduced by semi-implicit integrator, which depends on the precision of the forces computed by each solver. For some solvers, the system begins shifting left, while for others right. ADMM-NCP being the most precise, accumulates drift the least, thus shifting the latest of all, while RAISIM-DS-ES shifts the earliest due to the imprecision induced by the early-stopping termination criteria. The latter also exhibits the most amount of configuration-level constraint violation for joints (lower-middle-right in cyan). At the end of the shifting motion, the limits of all four joints activate resulting in the Fourbar rolling over to one side, and at $t=\unit[10]{s}$ the external force acting on the top-most body along -Y-axis pushes the system onto the ground where all bodies are simultaneously contact it. This results in $n_c=16$ total simultaneous contacts, thus making the system significantly over-constrained. The plot showing the contact force along the Y-axis (top-middle-center) shows how the early-stopping results in significant internal forces due to the hyperstaticity of the system. Conversely, ADMM-NCP presents the least (virtually zero) coupling.}
\label{fig:fourbar-hard}
\end{figure*}
The first run is depicted in Fig.~\ref{fig:fourbar-hard}. ADMM-NCP and RAISIM-DS-ES were the only solvers capable of converging within the maximum number of iterations in the majority of the time-steps. To elaborate on this result, convergence profiles are provided in Fig.~\ref{fig:fourbar-convergence}, corresponding to a time-step where the Fourbar has impacted the plane. Moreover, NBGS, RAISIM and RAISIM-DS performed on par with ADMM-NCP w.r.t. constraint violation of both joints and contacts. Meanwhile, the early-stopping RAISIM-DS-ES exhibits very high joint constraint violation compared to the former. Although the early-stopping termination criteria presents a significant aid in avoiding excessive iterations of the solver, it results in significant underestimation of the constraint reactions. 

The second run with constraint stabilization is shown in Fig.~\ref{fig:fourbar-hard-stbl} and demonstrates a drastic reduction in constraint violation for all solvers. ADMM-NCP and RAISIM-DS-ES were mostly unaffected in terms of convergence, except when all bodies contacted the ground simultaneously. We hypothesize that the additional bias of the stabilization term causes an effect similar to the wedging example described by Erleben in~\cite{erleben2004stable} and Horak et al in~\cite{horak2019similarities}. Essentially when multiple constraints are acting along approximately co-linear and opposing directions, the solver ends up in a ping-pong effect where the contribution of each constraint cancels that of another, causing it to make slow or no progress towards a solution. In this case the opposing biases are those of the contact stabilization and the constraint dimensions of the joints. As convergence is conditioned on the required tolerance, in this case set to $\epsilon_{abs} = 10^{-12}$, the fact that ADMM-NCP still renders correct forces is an indication that the ping-pong effect most likely manifests at smaller scales of the relevant residuals. 

The third run, executed with constraint softening applied to all constraints, is depicted in Fig.~\ref{fig:fourbar-soft}. This augmentation results in all solvers being able to converge well within $N_{max}$. ADMM-NCP and RAISIM-DS-ES exhibited the most significant improvement, with the latter being able to do so within $i_{stop} < 100$. However, this acceleration came at a significant cost of increased constraint violation for all solvers, which was close to an order of magnitude larger compared to the previous runs. Comparatively, the fourth run depicted in Fig.~\ref{fig:fourbar-soft-stbl} shows that combining both constraint stabilization and softening yields the best results overall.

The fifth run, that introduces a large mass-ratio to the system and with constraint stabilization, is depicted in Fig.~\ref{fig:fourbar-hard-stbl-lmr}. In this scenario, all solvers fail to converge with $N_{max}=10^{4}$. However, most solvers exhibit a relatively low degree of constraint violation, except the splitting-based methods when initial impact with the plane occurs. For these, all the performance metrics pertaining to physical accuracy are significantly degraded w.r.t. previous runs. This is consistent with the Boxes-Fixed experiment where all splitting-based solvers were incapable of correctly propagating forces and satisfying both configuration-level and dynamics constraints.

Lastly, we must remark that initial testing caused us to exclude PGS-NCP and ADMM-CCP from this experiment as they resulted in erroneous behaviors. The former caused significant rotation of the base link upon first impact from the drop, even when constraint stabilization was enabled. The latter diverged moments after all bodies contacted the ground after the push at $t=\unit[10]{s}$. These motions distorted the time-series plots in the aforementioned figures to such a degree that we opted to  remove them from the comparison.

\subsection{Full-Scale Systems}
\label{sec:experiments:animatronics}
The second set of experiments evaluates the dual solvers on a set of full-scale systems. With these, we demonstrate the capabilities and limitations of each solver on problems that are representative of practical applications for controlling robotic systems and \textit{Audio-Animatronics}\textregistered\,\, figures.\\

\begin{figure}[!t]
\centering
\includegraphics[width=1.0\linewidth]{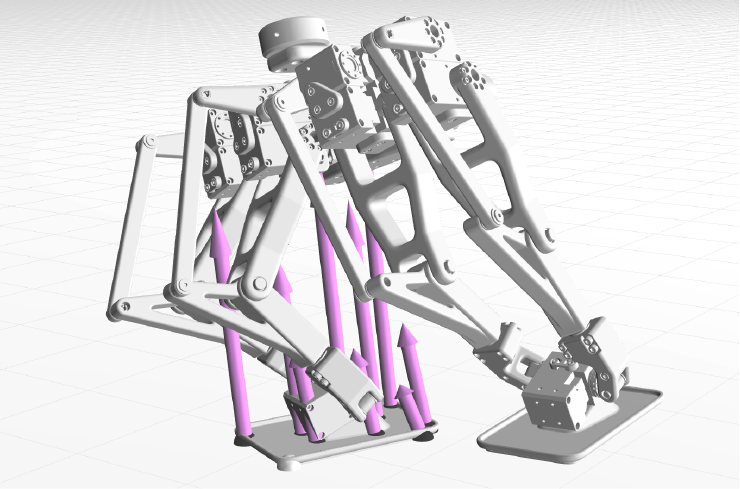}
\caption{\textit{Walker}: The forces distributed over the non-convex collision geometry of the system's planar feet rendered by the ADMM solver. Black and white spheres indicated contacts that are sticking and opening, respectively. The collection of contacts on the heel of the left foot is consistent with the CoM of the system. Importantly, ADMM is able to generate minimum-norm contact reactions that exhibit minimal to no internal cross-couplings, despite the large number of co-planar contacts being generated in this case.}
\label{fig:walker-force-dist}
\end{figure}
\subsubsection*{\textbf{Walker}}
\label{sec:experiments:animatronics:walker}
The first full-scale system is the bipedal Walker~\cite{gim2018design}. Although we categorized it as belonging to the \textit{Audio-Animatronics}\textregistered\,\, group, it is in fact a hybrid between the former and a robotic system. It is a multi-limb floating-based system with planar feet, intended for evaluating bipedal locomotion of highly over-and-under-constrained systems. Indeed, each leg exhibits three intrinsic kinematic loops, forming seven overall when both feet are in contact with the ground. As mentioned in Sec.~\ref{sec:benchmarking:suite}, collision detection is preformed using the raw mechanical meshes imported from the mechanical designs of CAD, i.e. we have not applied any post-processing to convexify or simplify them. This system, however, does not impose joint limits. Locomotion is realized naively in open-loop by using joint-space PID control to track reference trajectories of duration $t_f \approx \unit[23]{s}$. These references are retargeted from an animation generated for a simpler joint morphology, as detailed in~\cite{Schumacher2021}. The objectives of this experiment, corresponding one-to-one to each run, evaluate solvers in the cases of:
\begin{enumerate}
\item the un-augmented system free-floating without gravity, in order to establish a baseline for performance
\item the un-augmented system while it is locomoting on the ground with gravity and contacts enabled
\item when including constraint stabilization
\item when including constraint softening only for contacts\footnote{Initial testing, conducted to set up these experiments, lead us to conclude that it is not feasible to apply constraint stabilization to joints on this system. The induced under-determination of the forces resulted in almost instantaneous dislocation of the joints and all bodies flying apart.}
\item simulation stability when using high-throughput settings for all solvers, with $N_{max}=500$ and $\epsilon_{abs} = 10^{-4}$
\end{enumerate}
However, we have had to exclude the ADMM-CCP solver from all but the first run, as it once again diverged in certain time-steps, causing significant distortion of all time-series plots. Regardless, since the experiment includes the primary solvers of interest, the comparisons can proceed without being significantly affected by this absence. 

\begin{figure}[!t]
\centering
\includegraphics[width=1.0\linewidth]{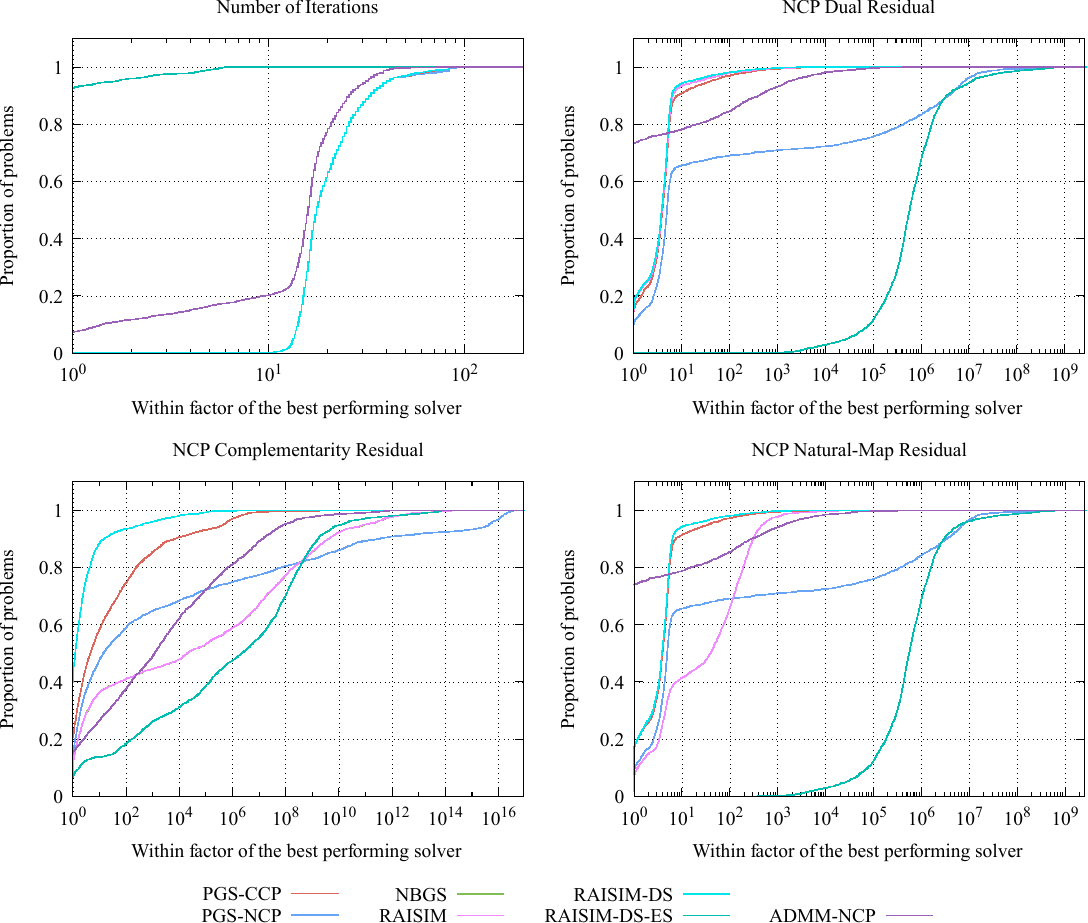}
\caption{\textit{Walker}: Performance profiles on the set of solvers. ADMM-CCP has been excluded is it often diverged on this problem resulting in invalid solutions. On this particular problem RAISIM-DS-ES is by a wide margin the fastest solver (top left), while ADMM-NCP is the most accurate overall. However, RAISIM-DS and PGS-CCP prove to be better w.r.t the accuracy metric of the complementarity residual $r_{ncp}$ (bottom left), and within a factor of 10 close to ADMM-NCP w.r.t the NCP dual and natural-map residuals (right column). This result indicates that RAISIM-DS and PGS-CCP, can be effective in rendering realistic simulations w.r.t the accuracy metrics.}
\label{fig:walker-perfprof}
\end{figure}
\begin{figure*}[!t]
\centering
\includegraphics[width=0.97\linewidth]{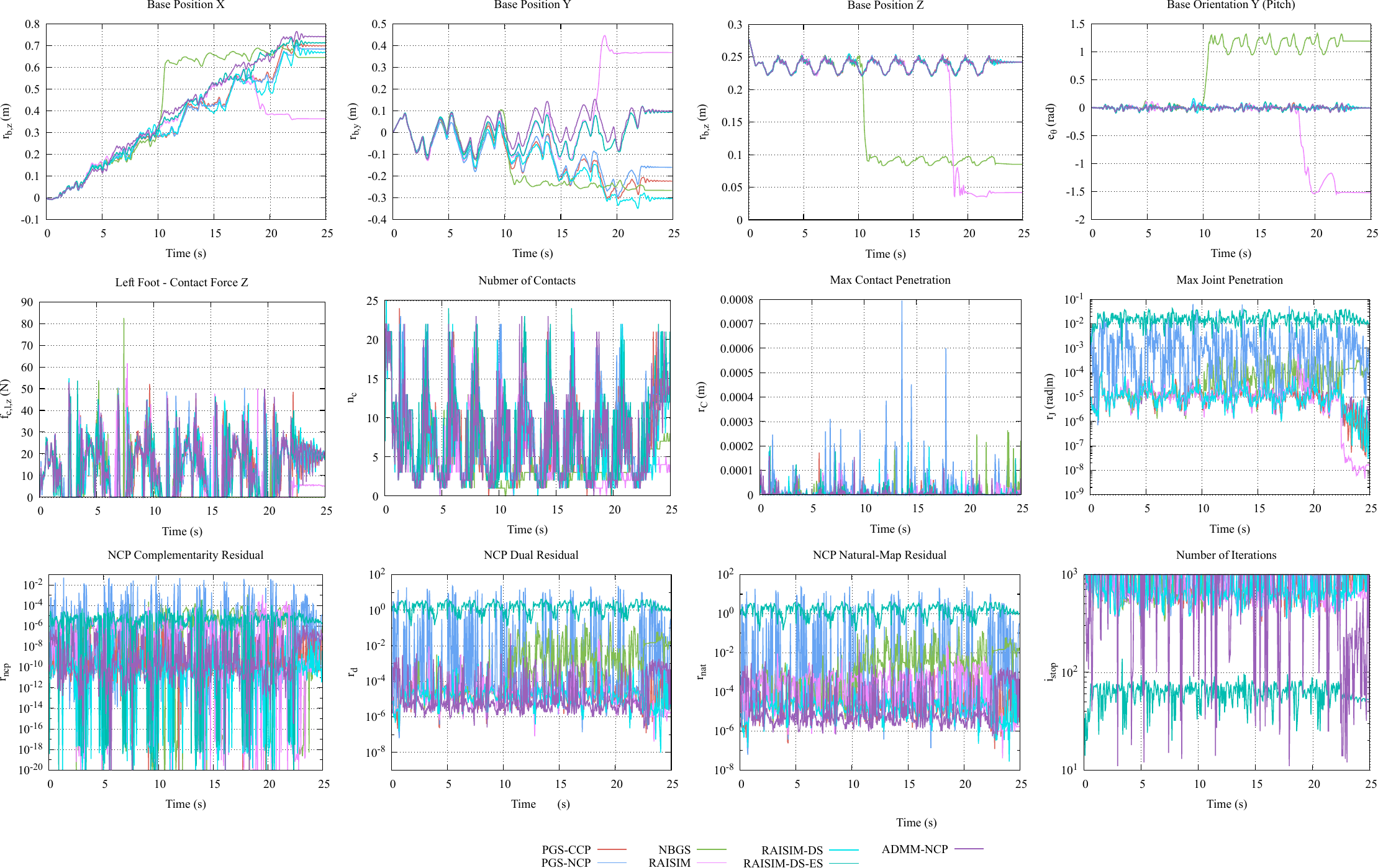}
\caption{\textit{Walker}: Time-series plots of the third run where stabilization is applied to all constraint sets. ADMM-NCP exhibits the best overall performance on this problem with significantly better performance w.r.t the accuracy metrics of the NCP dual, complementarity and natural-map residuals. RAISIM-DS and PGS-CCP perform similarly well, and the motions they generate remain close to that of the former. Meanwhile, RAISIM-DS-ES, although reasonably performance, exhibits excessive constraint violation close to unity, due to the early-stopping termination criteria. PGS-NCP, NBGS and vanilla RAISIM are the least performing solvers.}
\label{fig:walker-stbl}
\end{figure*}
\begin{figure*}[!t]
\includegraphics[width=1.0\linewidth]{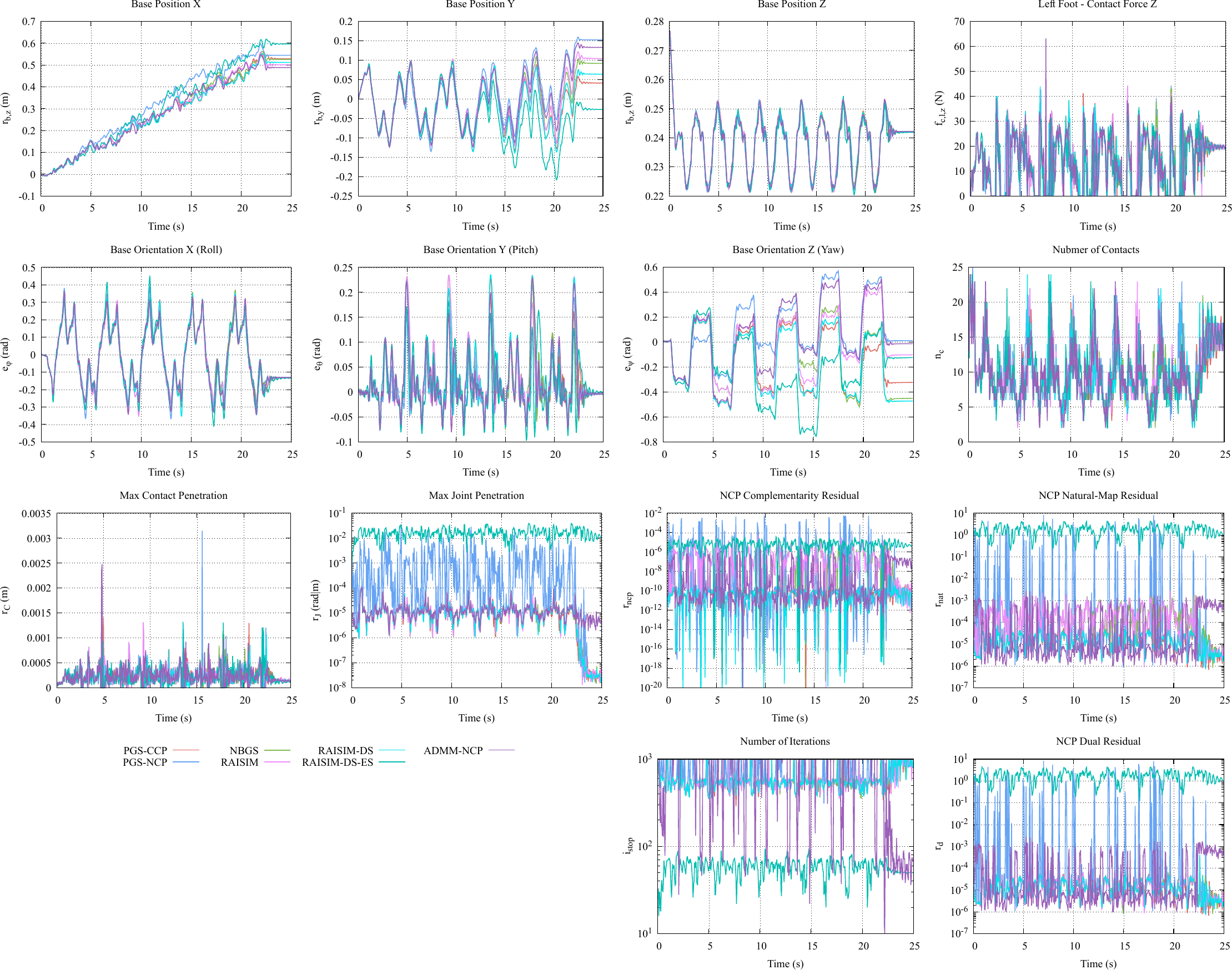}
\raisebox{1.8cm}[0pt][0pt]{%
\hspace{0.0cm}%
\parbox{8.5cm}{\caption{\textit{Walker}: Summary of the third run that combines constraint stabilization with constraint softening for only the contact constraints. The softening significantly improves the convergence rate of all splitting-based solvers, while the constraint stabilization ensures that configuration-level constraint violation is kept at minuscule values. All solvers except for PGS-NCP and RAISIM-DS-ES were able to maintain minimal constraint violation for the joints (middle-center-left). ADMM-NCP and RAISIM-DS performed the best overall on this run.}}}
\label{fig:walker-softc}
\end{figure*}

The first run is depicted in Fig.~\ref{fig:walker-floating}. Since in this case essentially only the bilateral joint constraints are active, the only relevant metric is that of the dual residual $r_{d}$. In this case, this metric represents the residual velocities in constraint-space, which should ideally be zero at all times. As we do not include any constraint stabilization, all solvers accumulate significant configuration-level constraint violation in the form of the joint penetration residual. RAISIM-DS-ES exhibits more than an order of magnitude larger compared to other solvers. ADMM-NCP performs the best in this case, with the PGS and RAISIM variants performing comparatively similar.

The second run introduces contacts and gravity, and is summarized in Fig.~\ref{fig:walker-hard}. In this case RAISIM-DS-ES and PGS-NCP lead to the biped falling over within the first few steps, while all other solvers manage to render simulations were locomotion proceeds until the end of the reference trajectory. It becomes apparent that as these also exhibit the largest amount of violation in the configuration-level joint constraints, and thus the rapid accumulation of error at the ankle joints leads to the tip-over. This scenario is of course not representative of actual applications, but illustrates the rate and scale at which the solvers deviate in the presence of multiple contacts that exacerbate the ill-conditioning of the problem.

The third run with constraint stabilization enabled is presented in Fig.~\ref{fig:walker-stbl}. In this case, most solvers and including RAISIM-DS-ES and PGS-NCP, can render simulations that keep constraint violation to reasonably small values. In this case, NBGS and vanilla RAISIM lead to tip-overs, but we consider this a false negative as it is a mere artifact of their respective simulations. In fact, by setting the collision margin to $\delta_{c} = 10^{-5}$ we were able to make all solvers reach the end of the animation without falling. The most important outcome of this run is that most solvers were sufficiently aided by the constraint stabilization in maintaining violation low. The degree by which this was accomplished, renders simulations whose output is usable in actual applications.

The fourth run is executed with constraint stabilization on all constraints and constraint softening applied only for contacts. Fig.~\ref{fig:walker-softc} depicts a summary of this run. The combination of both augmentations proves effective in rendering simulations that look very similar across all solvers. The softening provides a significant boost to the convergence of all splitting-based solvers but ADMM-NCP seems mostly unaffected. Moreover, for all but PGS-NCP and RAISIM-DS-ES, constraint violation is mostly the same as the first run, demonstrating that both augmentations can be combined effectively on such problems, similar to what was observed for Nunchaku and Fourbar.

The fifth run, depicted in Fig.~\ref{fig:walker-fast}, shows the effects of the high-throughput settings. In this case RAISIM-DS-ES was unable to maintain constraint violation, both kinematic and dynamic, to sufficient levels and very rapidly leads to a tip-over. All other solvers seemed mostly unaffected by this relaxation. ADMM-NCP, in particular, performed nearly identically to the second and third runs, thus exhibiting exceptional robustness w.r.t $N_{max}$ and $\epsilon_{abs}$.

A final curated run on the constraint-stabilized problem was used to generate the performance profiles presented in Fig.~\ref{fig:walker-perfprof}. From these we can conclude that although ADMM-NCP is once again the best solver overall, RAISIM-DS-ES, RAISIM-DS and PGS-CPP, present surprisingly comparable performance, in certain aspects. Specifically, RAISIM-DS-ES proved best overall in terms of number of iterations, RAISIM-DS proved best w.r.t. the NCP complementarity accuracy metric, and PGS-CPP proved to always come within a factor of 10 of the best solver w.r.t the NCP dual and natural-map metrics. Thus, although ADMM-NCP is once again the overall winner, the results indicate that using the problem augmentations with appropriate setting of parameters can prove effective in reducing the performance gap between the former and splitting-based solvers. 

Lastly, we must remark on the influence of the PID controller in terms of simulation stability. We observed that despite the overall robustness of most solvers, setting the derivative gain of the PID controller to large values lead to a rapid explosion of the bodies, even when using ADMM-NCP. This made it exceptionally challenging to tune the controller gains, but most importantly, it is indicative that the present formulation is sensitive to actuator dynamics. Many physics engines alleviate this problem by introducing an implicit PD controller directly into the formulation of the forward dynamics problem~\cite{tan2011stable, yin2020linear}. However, although this effect is well known, from our investigations, it seems to not be well covered in the known literature, and we intend to investigate this aspect further in future work.\\

\begin{figure}[!t]
\centering
\includegraphics[width=1.0\linewidth]{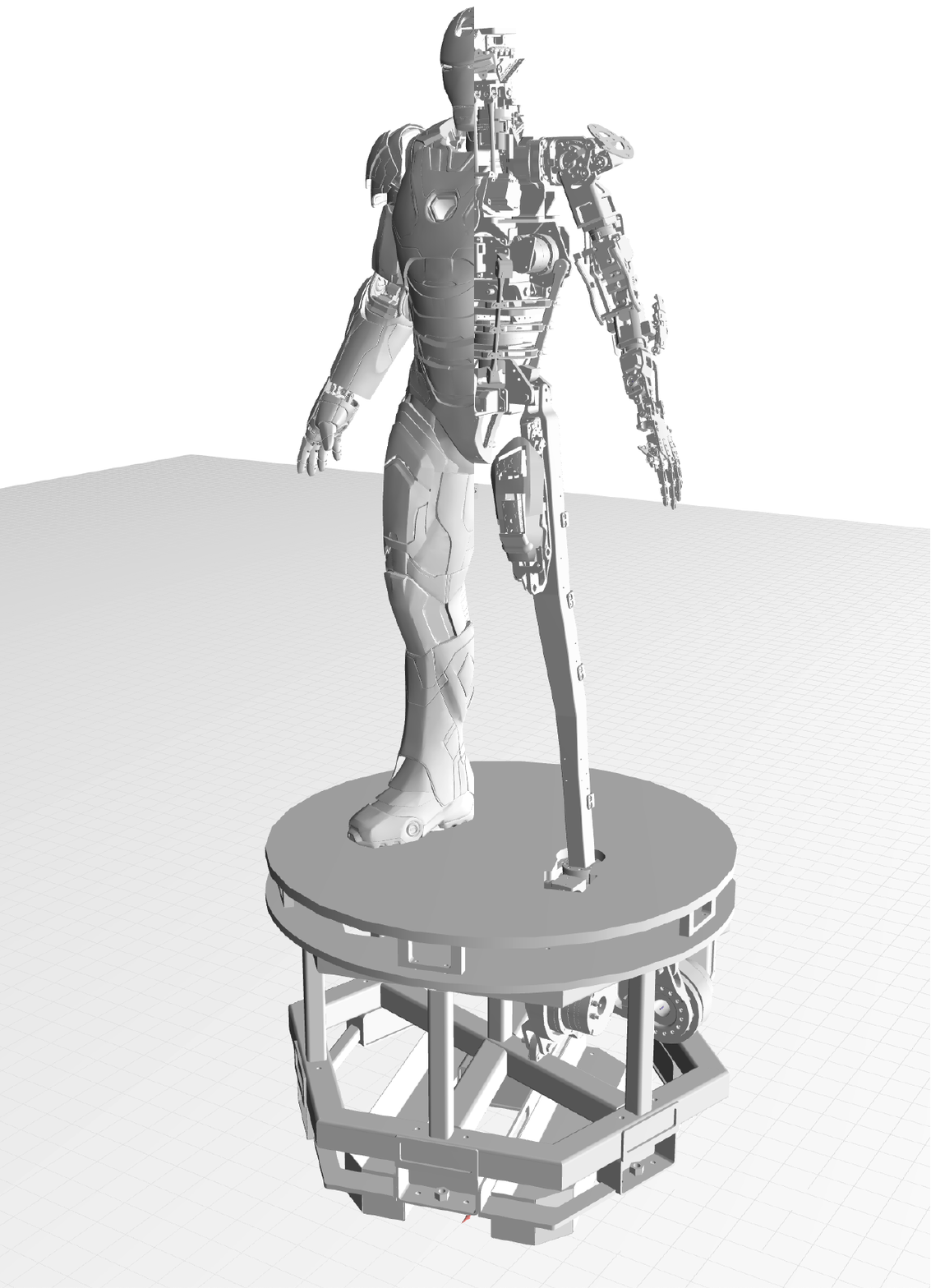}
\caption{\textit{IronMan}\,\,\textcopyright\,\,\textit{MARVEL}: the \textit{Audio-Animatronics}\textregistered\,\, figure, featured in real-life at Disneyland Paris in France, and shown here with (left) and without (right) the external shell geometries. The internal mechanical assembly reveals an exceptionally complex humanoid system.}
\label{fig:ironman-spliced}
\end{figure}
\begin{figure}[!t]
\centering
\includegraphics[width=1.0\linewidth]{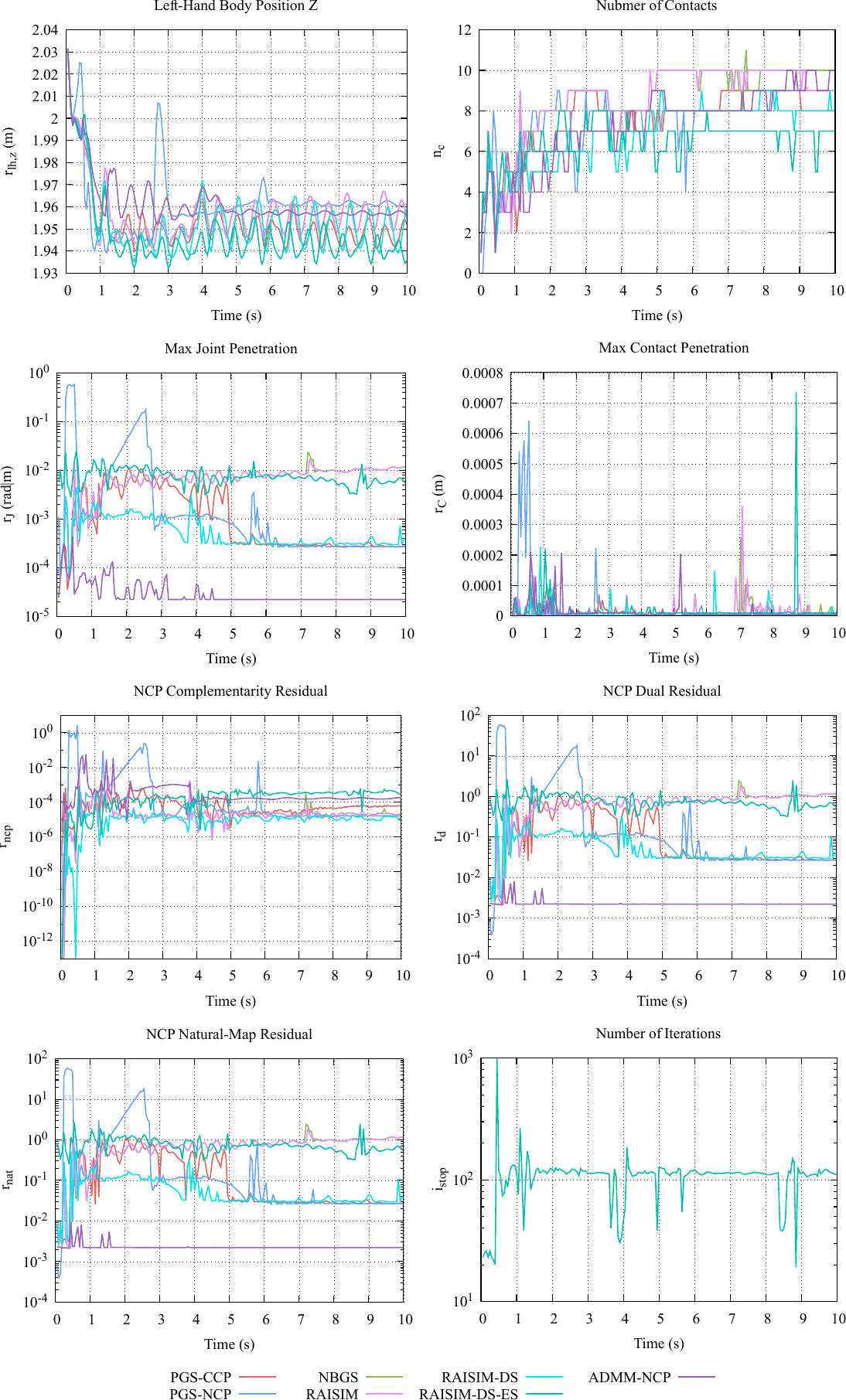}
\caption{\textit{IronMan}: Summary of the first run where the system is unactuated and left to rag-doll. While joint contact penetration remain reasonable low for most solvers, and on the order of a few millimeters, the dynamic constraints represented by the accuracy metrics of the NCP dual, complementarity and natural-map residuals exhibit much larger values compared to other systems in the benchmark suite. In this regard, ADMM-NCP is the best performing one overall, and RAISIM-DS-ES the least.}
\label{fig:ironman-ragdoll}
\end{figure}
\begin{figure}[!t]
\centering
\includegraphics[width=1.0\linewidth]{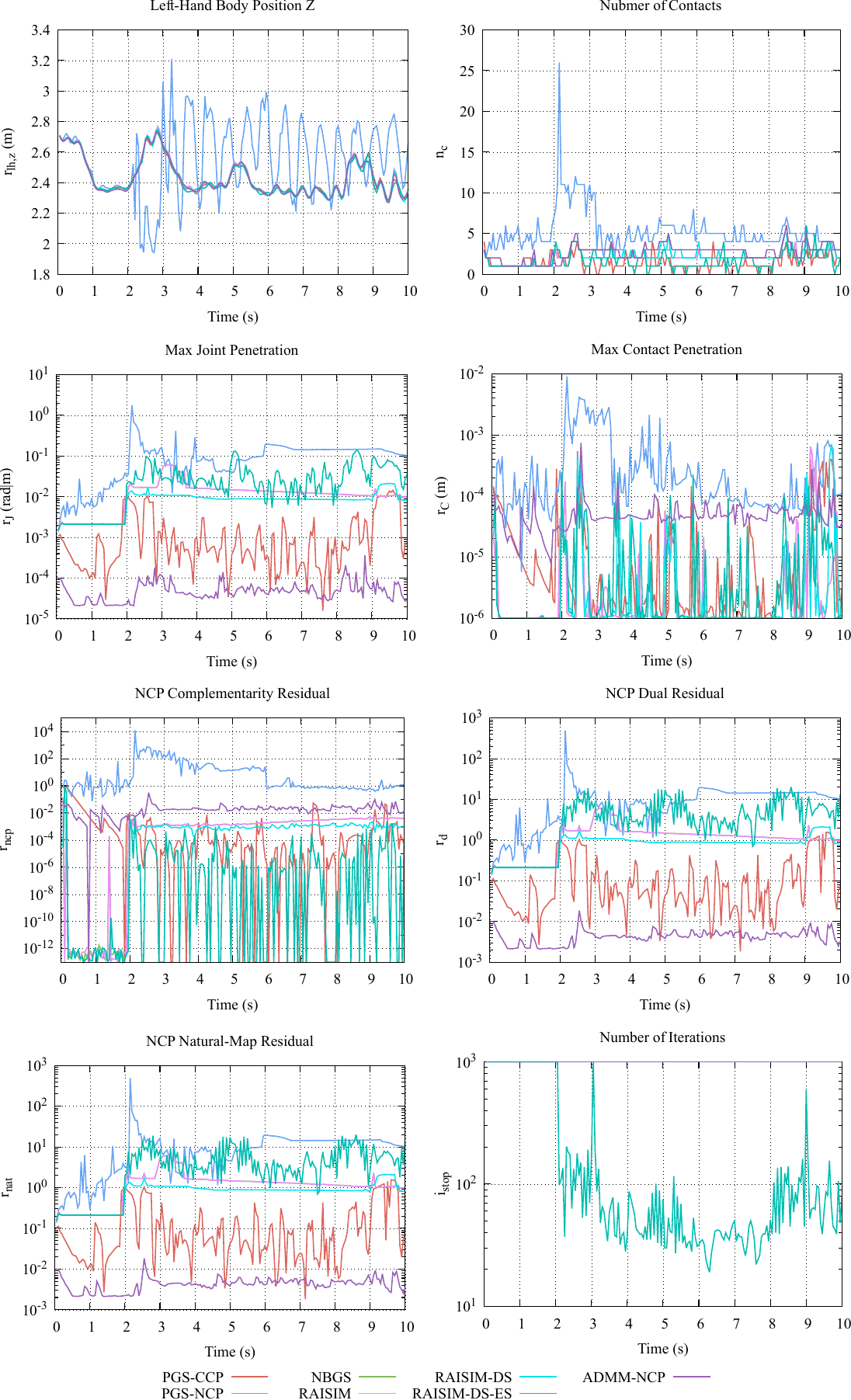}
\caption{\textit{IronMan}: Summary of the second run in which the system is actuated to track a reference trajectories using a naive joint-space PID controller.}
\label{fig:ironman-animated}
\end{figure}
\begin{figure}[!t]
\centering
\includegraphics[width=1.0\linewidth]{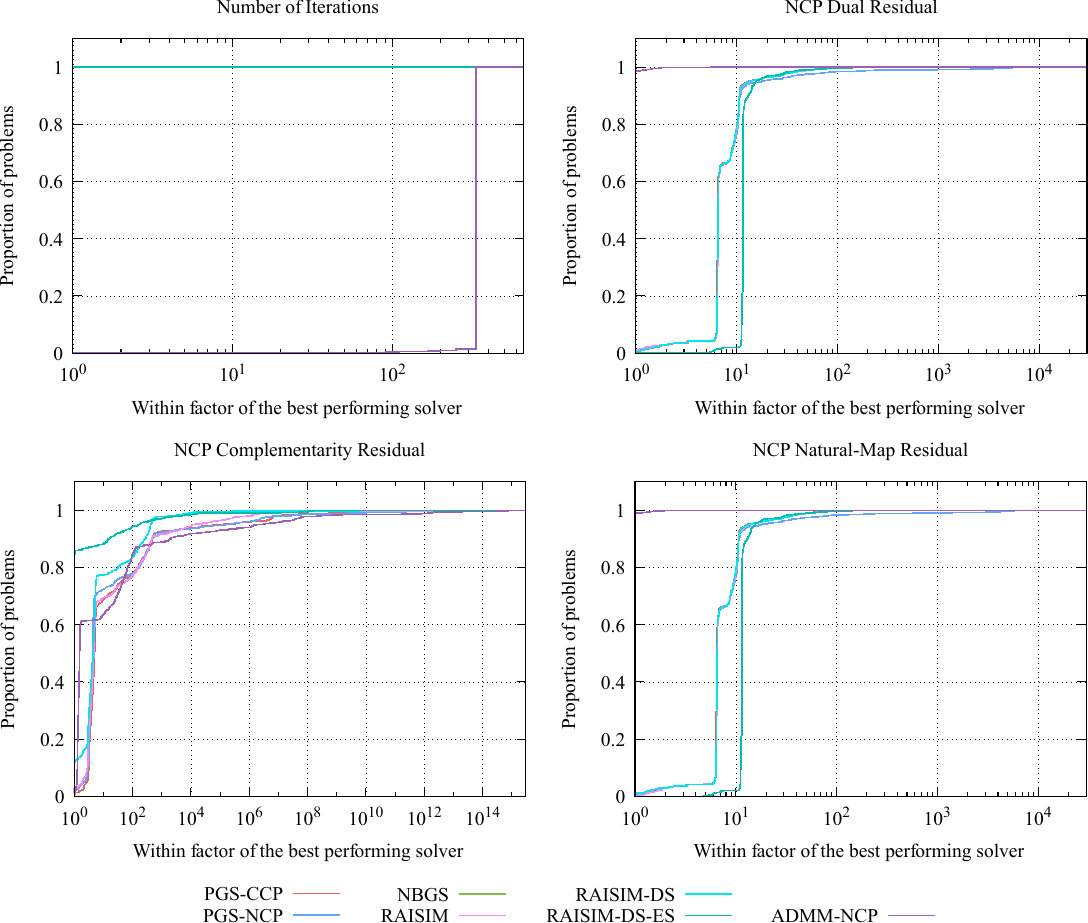}
\caption{\textit{IronMan}: performance profiles}
\label{fig:ironman-perfprof}
\end{figure}
\subsubsection*{\textbf{IronMan}}
\label{sec:experiments:animatronics:ironman}
\noindent
The second full-scale problem is the highly-complex \textit{IronMan}\,\,\textcopyright\,\,\textit{MARVEL} \textit{Audio-Animatronics}\textregistered\,\, figure realized using the state-of-the-art A1000\footnote{\url{https://sites.disney.com/waltdisneyimagineering/a-1000/}} platform and currently being exhibited at Disneyland Paris, France. It is by far the most challenging system in the benchmark suite, and undoubtedly exceeds the mechanical complexity of most humanoid robots by a large factor, with its 54 links and 67 joints. By comparison, the most complicated humanoid figures known to date, such as the first generation Atlas~\cite{kuindersma2016optimization} of Boston Dynamics, consist of 29 links and 28 joints. Fig.~\ref{fig:ironman-spliced} provides a split view that reveals the exceptionally complex mechanical assembly concealed by the exterior shells.  

For this final systems, we will not proceed as exhaustively as before. Our objective here is merely to evaluate how solver performance scales for systems of this size, and to observe potential simulation artifacts. To this end, we executed two runs where: (a) the system is left to \textit{rag-doll}\footnote{This is a reference to the term used computer animations and video games.}, i.e. collapse due to no actuation being active, and (b) use a joint-space PID controller to track reference motions retargeted from an animation. In both runs, all solvers are set to the high-throughput setting with $N_{max}=1000$ and $\epsilon_{abs}=10^{-6}$. 

The first run where the system is allowed to rag-doll is summarized in Fig.~\ref{fig:ironman-ragdoll}. Despite the size of the system, a relatively low number of contacts occur during the motion, due to the fixed-base. Each solver overall performed comparatively similar to the Walker system, with ADMM-NCP exhibiting the best performance on all metrics, despite the lack of convergence. Although RAISIM-DS-ES was the only solver to converge on this problem, it performed the worst, yet the corresponding simulation proved sufficiently stable and with no visible artifacts. Most surprising on this system is that although configuration-level constraint violation is kept relatively low, the dynamic constraints paint a different picture. The dual $r_{d}$ and natural-map $r_{nat}$ residuals approach scales near unity, indicating that some constraints contain significant residual constraint-space velocities. However, given the large inertial disparity of the system, coupled with the lack of solver convergence, these values are most likely due to the joints that couple very massive bodies with relatively small ones, such as the ones in the base of the system. Such a phenomena is justification for considering a pre-conditioning of the problem, and/or employing vector-valued tolerances, so to ensure that residuals are evaluated in a scale-invariant manner. 

The second run including the effects of the joint-space PID controller is depicted in Fig.~\ref{fig:ironman-animated}. In this scenario most solvers performed similarly to the first run, with the exception of PGS-NCP which lead to erratic oscillations propagating throughout the system. Indeed, as shown in Fig.~\ref{fig:ironman-animated}, PGS-NCP resulted in joint constraint violations up to $\unit[1]{m}$ at certain instances. This result, together with all other systems described previously, causes us to conclude that PGS-NCP is the least performing solver of the set we evaluated and least likely to be considered for use in actual applications. 

Finally, an additional curated run on the actuated setup results in the performance profiles depicted in Fig.~\ref{fig:ironman-perfprof}. From these we see that ADMM-NCP performs well only on the NCP dual and natural-map accuracy metrics, while failing to beat RAISIM-DS-ES and RAISIM-DS in terms of the number of iterations and NCP dual residual. The degree of this disparity between accuracy metrics is very surprising, even considering the respective comparison made for Walker in Fig.~\ref{fig:walker-perfprof}. However, the absolute performance presented in Fig.~\ref{fig:ironman-ragdoll} and Fig.~\ref{fig:ironman-animated}, provide context. The relative performance, although seemingly large, is only up to a factor of 100 lower for ADMM-NCP w.r.t residual-based metric on the majority of sample problems. This in combination with the fact that ADMM-NCP and PGS-CCP do not manage to converge, thus makes it difficult for this comparison to be conclusive.

\subsection{Benchmarks}
\label{sec:experiments:benchmarks}
\noindent
The third and final set of experiments we present in this work, involves quantifying the aggregate relative performance of each solver over multiple systems contained in the benchmark suite of Sec.~\ref{sec:benchmarking:suite}. However, realizing this aggregation naively using curated runs over each problem can prove error prone and cumbersome. The most straightforward solution to this predicament is to compile a fixed set of of dual problems that can be stored and retrieved for offline solver execution, in a systematic manner and on-demand. This is the approach pioneered by Acary et al with the development and publication of the FCLIB~\cite{acary2014fclib}, presenting a set of frictional-contact problems for various large-scale mechanical systems, generated using the Siconos~\cite{acary2007siconos} and LMGC90~\cite{dubois2013lmgc90} software platforms. Other similar efforts also exist in the literature, notably that of Lu et al in~\cite{lu2016framework}. However, the problems contained in the aforementioned works cater to systems like granular media and static structures, with a clear absence of those with highly-coupled bilateral kinematic constraints that are most relevant to our application domain. 
\begin{table}[!t]
\centering
\caption{\textit{Benchmarks}: Data contents of each dual problem sample.}
\label{tab:benchmark-sample-contents}
\begin{tabular}{|l|c|l|}
\hline
\textbf{Type}                                                                  & \textbf{Entry}              & \textbf{Description}            \\
\hline
\hline
\multirow{4}{*}{\begin{tabular}[c]{@{}l@{}}Problem \\ Definition\end{tabular}} & $\mathbf{D}$                & Delassus matrix                 \\
                                                                               & $\mathbf{v}_{f}$            & free-velocity vector            \\
                                                                               & $\boldsymbol{\mu}$          & friction coefficients           \\
                                                                               & dt                          & time-step                       \\
\hline
\multirow{7}{*}{\begin{tabular}[c]{@{}l@{}}Problem\\ Dimensions\end{tabular}}  & $n_b$                       & no. of bodies                   \\
                                                                               & $n_j$                       & no. of joints                   \\
                                                                               & $n_l$                       & no. of active limits            \\
                                                                               & $n_c$                       & no. of active contacts          \\
                                                                               & $\{d_{j}\}$                 & joint dimensions                \\
                                                                               & $\{i_{j}\}$                 & joint index offset              \\
                                                                               & $\{i_{l}\}$                 & limit index offset              \\
\hline
\multirow{3}{*}{\begin{tabular}[c]{@{}l@{}}Problem \\ Properties\end{tabular}} & c                           & sample category                 \\
                                                                               & $\textbf{rank}(\mathbf{J})$ & Jacobian rank                   \\
                                                                               & $r_{m}$                     & mass-ratio                      \\
\hline
\end{tabular}
\end{table}
\begin{figure*}[!ht]
\centering
\includegraphics[width=1.0\linewidth]{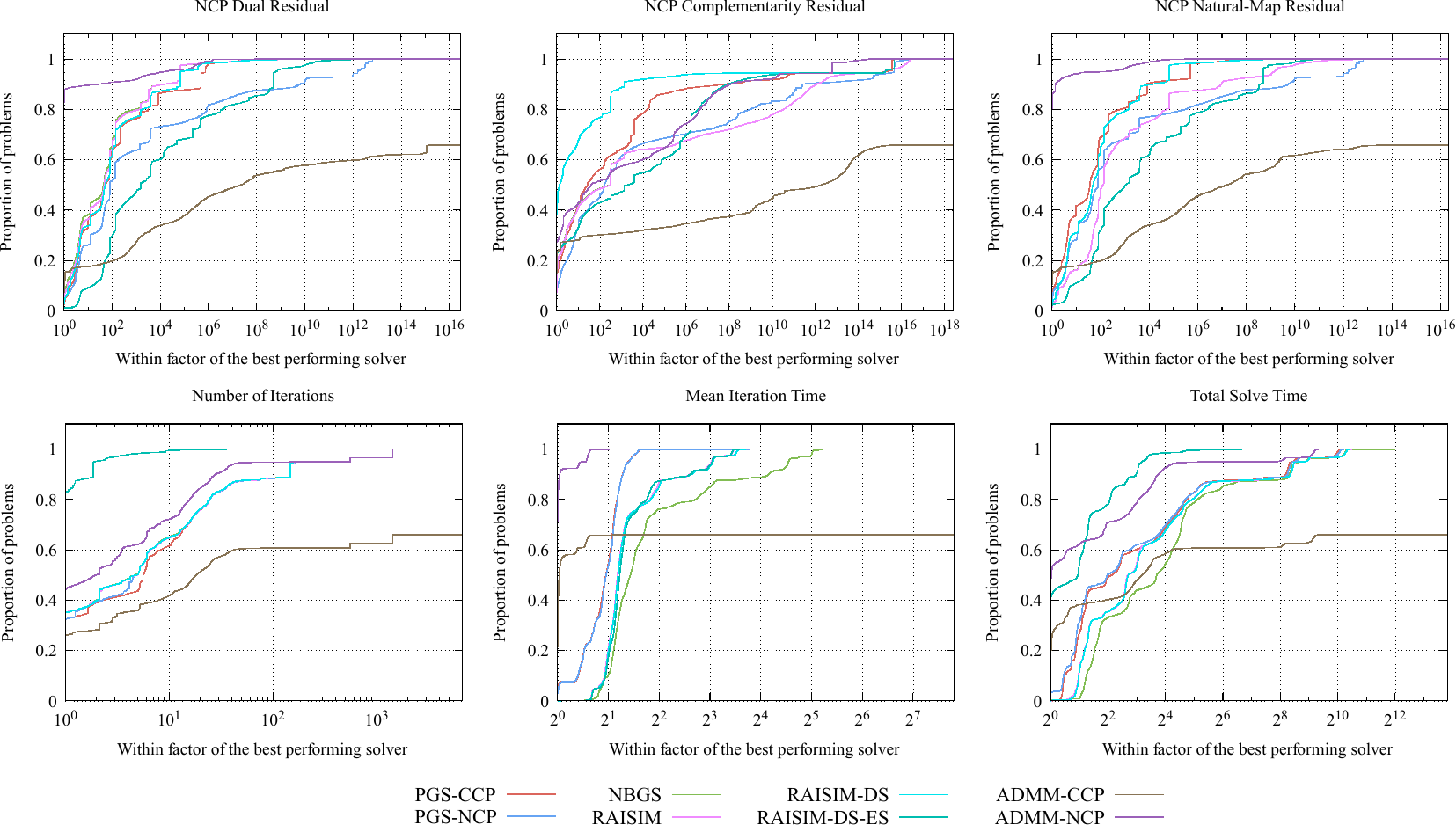}
\caption{\textit{Benchmarks}: performance profiles generated over all problems in the suite.}
\label{fig:perfprof-all}
\end{figure*}

In order to realize such a setup, we executed a single curated run of ADMM-NCP on each of the twelve systems, sampling dual problems at certain frames (i.e. time-steps). The sampling criteria is based on defining sample \textit{buckets} of fixed-size corresponding to each of the categories defined in Sec.~\ref{sec:benchmarking:sample-categories}. This ensures that the sampling process collects a diverse set of samples spanning all cases of ill-conditioning. All captured frames were compiled into a fixed data-set of $N_{p} = 13298$ dual problems stored in the HDF5 file format~\cite{hdf5}, similar to FCLIB. Each problem entry $p$ consists of a problem definition tuple $\{\mathbf{D}, \mathbf{v}_{f}, \boldsymbol{\mu}\}_{p}$ as well as additional meta-data that includes its sample categorization defined in Sec.~\ref{sec:benchmarking:sample-categories}, its ill-conditioning properties in terms of those defined in Sec.~\ref{sec:benchmarking:ill-conditioning}, and individual constraint dimensionalities. Table~\ref{tab:benchmark-sample-contents} outlines the data contents of each dual problem sample. The total relative performance of each solver aggregated over all systems in the suite is presented in Fig.~\ref{fig:perfprof-all}. Moreover, we provide additional aggregations per system categories and sample categories in Appendix.~\ref{sec:apndx:additional-figures}.

\subsection{Discussion}
\label{sec:experiments:discussion}
\noindent
We will conclude this section with a summary of the observations made throughout the course of this evaluation.\\

\subsubsection*{\textbf{Collision Detection}}
\begin{itemize}
\item One aspect of ODE that we found particularly challenging to deal with is that \textit{box-to-plane} collisions yield contacts that are spuriously re-assigned at each time-step-wise query. This resulted in a re-assignment of the index order of each contact constraint, leading to respectively erratic contact reactions for the splitting-based solvers, which are sensitive to this ordering. We thus conclude that when using this type of solver, one must pay close attention to the aforementioned behavior and ensure that a CD pipeline can render the same contact over sequential invocations.\\
\item Employing a collision detection margin is crucial to stabilizing contacts. One challenge of co-planar contacts, such as those present in the Primitive systems but in Walker as well, is that minuscule contact velocities along the respective normal directions will cause separation of the bodies and contact chattering. Defining a collision detection margin, either in the CD pipeline itself, or as a threshold in the Baumgarte-like contact constraint stabilization term, can alleviate this effect. By permitting some marginal body interpenetration, e.g. $\delta=10^{-6}$, contact chattering can be eliminated without significantly impacting the overall results of the simulation.\\
\item When evaluating collisions between bodies with mesh-based geometries an inordinate number of contacts may be rendered, which can be detrimental to simulation performance. A systematic method is thus required that can reduce these to a minimal number per-body pair. However, it is open problem as to how to define an minimal number of contacts that optimally represent the collision. We found that a naive approach using an ad-hoc filtering works very well in practice, but this can still render more contacts than required, as demonstrated on the Walker system. There, when each foot contacted the ground, up to $n_c=20$ contacts occurred on each planar foot, reduced from approximately $n_c=200$ using the minimum-distance rejection sampling-based scheme.\\ 
\end{itemize}

\subsubsection*{\textbf{Splitting-Based Solvers}}
\begin{itemize}
\item Of all splitting-based solver, RAISIM-DS and RAISIM-DS-ES performed the best, with significant differences w.r.t vanilla RAISIM, and NBGS on which they are based. This result indicates that the inclusion of the De Saxc\'e correction significantly improves the quality of the simulation w.r.t all accuracy metrics. Moreover, this result also speaks in favor of preferring holistic local contact models and solvers instead of simpler blocked methods such as PGS-CCP and PGS-NCP.\\
\item RAISIM-DS-ES was the clear overall winner among splitting-based solvers. Although the early-stopping criteria can transiently manifest artifacts due to under-estimation of the contact reactions, using constraint stabilization can provide sufficient correction to render results nearly indistinguishable from the more accurate solvers that ran for more iterations. We must credit the authors of MuJoCo for this particular improvement, as we did not encounter it in the various other works we reviewed. A more in-depth analysis for why it works so well is warranted, however.\\
\item The inclusion of the De Saxc\'e correction in the PGS-CCP was exceptionally easy to incorporate into the original PGS-CCP solver described by Tasora et al in~\cite{tasora2011matrix}, and rendered overall better results than PGS-NCP. This indicates that a decoupling of the normal and tangential dynamics of each local contact model is not as effective as enforcing duality between the augmented contact velocity $\hat{\mathbf{v}}$ and the contact reaction $\boldsymbol{\lambda}$.\\

Furthermore, the DS-corrected PGS-CCP solver seems to perform almost as well as RAISIM-DS w.r.t the dual $r_{d}$ and natural-map $r_{nat}$ accuracy metrics, while being computationally cheaper. This means that PGS-CCP is still a viable candidate for applications where real-time performance requirements may outweigh solver accuracy.\\ 
\item As note by Tasora et al in~\cite{tasora2011matrix} and Sheldon et al in~\cite{siggraph2022contact}, the order in which constraints are updated in splitting-based solvers is crucial for convergence and solution fidelity. We experimented with various combinations for ordering the constraint sets of joints, limits and contacts, and found that the best results were rendered by the order defined in Alg.~\ref{alg:solvers:projective-splitting}, where first bilateral joint constraints are updated, then limits and contact constraints last.\\
\item Sheldon et al in~\cite{siggraph2022contact} also describe additional augmentations to the per-constraint iteration, in the form of dual forward-backward and randomization of the indexing. We experimented with these modifications by respectively augmenting Alg.~\ref{alg:solvers:projective-splitting}. In the former case, we observed no substantive improvement to convergence or difference in the rendered solution. In the latter case, we also randomized the indices of each constraint subset, and did not observe any tangible improvement. In fact, the randomization leads to contact force instabilities as the directions of the corresponding forces along the respective tangential directions changed erratically due to the reordering.\\
\end{itemize}

\subsubsection*{\textbf{ADMM}}
\begin{itemize}
\item ADMM-NCP proved to be the best performing solver overall. Indeed, it was the only one capable of rendering the expected, and most physically plausible, simulations for the primitive systems. Most importantly, it demonstrated the best performance on the highly-coupled systems like Fourbar, Walker and IronMan and w.r.t all accuracy metrics. Moreover, it was the most robust w.r.t using low tolerances and restricting the number to iterations so that simulations could be run in real-time.\\
\item The proximal regularization in ADMM-NCP proved to be exceptionally effective against even severe ill-conditioning of the dual problem, particularly hyperstatic cases where the system Jacobian was significantly rank-deficient. As an added benefit, the regularization also eliminates internal forces to a very large degree, rendering minimum-norm contact reactions, which are crucial aspects of physically plausible simulations in which direct measurements of these might be necessary for control applications.\\
\item ADMM-NCP was also the most robust w.r.t inertial disparity large mass-ratios. Our analysis in Sec.~\ref{sec:experiments:primitives:Box-on-Plane} demonstrated that, in fact, ADMM-NCP is sensitive mainly w.r.t the spectral radius $L = \rho(\mathbf{D})$, of the Delassus matrix, which is also the Lipschitz constant of the NSOCP (\ref{eq:formulation:dual-forward-dynamics-nsocp}) objective function. This result indicates that its convergence properties could be further improved given a pre-conditioning of the problem, such as that proposed by Tasora et al in~\cite{tasora2021admm}. Lastly, we remark that the spectral adaptation strategy of the penalty parameter proposed by Carpentier et al in~\cite{carpentier2024unified} provide an additional mechanism to accelerate convergence, that could potentially be combined with the former.\\
\item Unfortunately, our analysis caused us to conclude that ADMM-CCP is unpredictably unstable. Although performing comparably close to ADMM-NCP on most problems, it diverged in particular instances on several of the applications-relevant systems such as Fourbar and Walker. This led us to exclude it altogether from several of the aforedescribed experiments. We are unsure as to why and when instabilities arise, as the exclusion of the De Saxc\'e correction should not, in principle, impact the solver directly. This warrants further investigation.\\
\end{itemize}

\subsubsection*{\textbf{Solver Comparisons}}
\begin{itemize}
\item Although ADMM-NCP was the best solver overall, RAISIM-DS, RAISIM-DS-ES and PGS-CCP proved to perform comparatively well. Our experiments demonstrate that in terms of raw simulation throughput, RAISIM-DS-ES was the best, and is a viable candidate for lower-fidelity simulations where accuracy is not as important as speed.\\
On many problems RAISIM-DS and PGS-CCP performed better than ADMM-NCP w.r.t the NCP complementarity residual metric $r_{ncp}$. We provide such an example of these differences in Sec.~\ref{sec:experiments:animatronics:walker} were we presented experiments on the Walker system. This result was rather surprising, given the overall better behavior of ADMM-NCP and failure of the others to perform well in cases of severe ill-conditioning. We speculate that this may be due to the lack of samples in the large-scale systems that include sliding contacts. We know from the primitive experiments involving slipping contacts that all solvers exhibit both degraded convergence and small residual velocities along contact normal directions, with ADMM-NCP performing the best in these cases, and RAISIM-DS behaving comparably close.\\
\item Our results are generally in strong agreement with those presented by Lidec et al in~\cite{lidec2024reconciling} and~\cite{lidec2024models}. Moreover, the former work inspired us to consider the introduction of the De Saxc\'e correction to all compatible solvers. We argue that such a trivial modification should definitely be a considered for applications where differentiability is not necessary.\\
\end{itemize}

\subsubsection*{\textbf{Constraint Stabilization \& Softening}}
\begin{itemize}
\item Experiments that included constraint stabilization and softening to the dual problem were in strong agreement with the theoretical result presented in Sec.~\ref{sec:construction:stabilization}, were we should the necessary conditions for an augmentation to retain the properties of the NCP. Our results show that the augmentations are effective means to improve simulation  throughput and fidelity in such a minimal-coordinate formulation.\\
\item When the mass ratios are relatively low, or close to unity, the MuJoCo-based constraint softening proves exceptionally effective in improving convergence and therefore simulation throughput. However, as we saw with the Boxes-Fixed experiment, high mass-ratios make it effectively impossible to use such a relaxation, even with extraneous tuning of the softening parameters.\\
\item Our results indicate that it is both viable and effective to combine classic Baumgarte-like constraint stabilization with the MuJoCo-based constraint softening. Indeed, their combination can bring the benefits of each and thus accelerate simulation throughput while minimizing configuration-level constraint violation.\\
\end{itemize}

\newpage
\section{Conclusion}
\label{sec:conclusion}
\noindent
This work has presented an extensive analysis of constrained rigid multi-body simulation realized using the architecture of separately solving the forward dynamics problem and performing time-integration of the state. We have provided a rigorous literature survey to build as complete as possible of a picture of past and recent recent directions. In addition, we have employed fundamental and advanced principles of mechanics to derive and construct the dual problem of forward dynamics in the form of a NCP that encompasses a majority of the complex mechanical systems that one would need to simulate in the domains of robotics and \textit{Audio-Animatronics}\textregistered\,\,. Through the lens of convex analysis, we have employed the theory of the augmented Lagrange method and proximal operators to depict a unifying perspective of strategies for solving the dual problem of forward dynamics, and have presented a summary of established and current state-of-the-art methods. 

We have contributed a framework for performing comparative analyses of the aforementioned solver strategies, founded on fundamental physical principles and an broad inclusion of systems that are relevant for practical applications. Lastly, we have presented an extensive set of experiments aiming to elucidate both qualitative and quantitative differences between solvers, on both toy and application-level problems, and present our findings to the community. We hope that this work, as the ones that inspired it, can guide others in developing the next generation of physics simulators that can further the fields of robotics, character animation, and computer graphics. 

Our experiments have lead us to conclude that both state-of-the-art methods such as ADMM as well as established approaches based on splitting-methods such as Gauss-Seidel iteration, can be made competitive on all forms of problems, given small modifications to the formulation of the problem, augmentations to stabilize it, and a careful selection of algorithm termination criteria. Most notable is how the inclusion of the so-called De Saxc\'e correction is nearly trivial, yet exceptionally effective in improving simulation accuracy and fidelity, and with negligible added computational cost.

Although we attempted to incorporate as many of the important aspects of realizing forward dynamics solvers for physical simulation, there were aspects which we did not have yet consider. These include introcution of an implicit PD controller, the application of warm-starting, pre-conditioning, and parameter adaptation that can accelerate the numerical convergence of all solvers for whom these are applicable. Moreover, we did not consider the effects of sparse matrix algebra, and intra-operation parallelization, e.g. graph-coloring and parallelization, as means to accelerate the computational time-complexity of our implementations. Lastly, we did not consider the family of approaches tackling the \textit{primal} formulation of the forward dynamics problem, and most notably those employing Newton-based second-order methods such as those employed in Drake~\cite{drake} and MuJoCO. All of the aforementioned, as well as an implementation capable of executing on the GPU, are part of future work.\\

\appendices

\newpage
\section{Properties of the De Saxc\'e Operator}
\label{sec:apndx:desaxce-properties}
\begin{definition}
In the constraint-space of each contact, we can decompose any vector $\mathbf{v} \in \mathbb{R}^{3}$ into its 2D tangential and scalar normal components in the form of
\begin{equation}
\mathbf{v} := 
\begin{bmatrix}
\mathbf{v}_{T} \\[4pt]
\text{v}_{N}
\end{bmatrix}
\,\,.
\label{eq:apndx:any-velocity-decomposition}
\end{equation}
\end{definition}

Thus, we can decompose the post-event contact velocity as
\begin{equation}
\mathbf{v}^{+} := 
\begin{bmatrix}
\mathbf{v}_{T}^{+} \\[4pt]
\text{v}_{N}^{+}
\end{bmatrix}
\,\,,
\label{eq:apndx:post-impact-velocity-decomposition}
\end{equation}
and the augmented contact velocity as
\begin{equation}
\begin{array}{ll} 
\hat{\mathbf{v}} &:=\, \mathbf{v}^{+} + \boldsymbol{\Gamma}_{\mu}(\mathbf{v}^{+}) \\[8pt]
&\,\,=
\begin{bmatrix}
\mathbf{v}_{T}^{+} \\[4pt]
\text{v}_{N}^{+}
\end{bmatrix}
+ 
\begin{bmatrix}
0 \\[4pt]
\mu \, \Vert \mathbf{v}_{T}^{+} \Vert_{2}
\end{bmatrix}
\,\,.\\[8pt]
\end{array}
\label{eq:apndx:augmented-velocity-decomposition}
\end{equation}

\subsection{Velocity Constraint Equivalence}
\label{sec:apndx:desaxce-properties:velocity-constraint-equivalence}
\begin{lemma}
The post-impact velocity $\mathbf{v}^{+}$ lies in the half-space defined by the unit vector $\mathbf{n} \in \mathbb{R}^{3}$ of the contact normal, if and only if, the augmented velocity lies in the dual cone $\mathcal{K}^{*}$, i.e. the following conditions are equivalent: 
\begin{equation}
\text{v}_{N}^{+} \geq 0 \,\,\Leftrightarrow\,\, \hat{\mathbf{v}} \in \mathcal{K}_{\mu}^{*}
\end{equation}
\label{lemma:apndx:velocity-constraint-equivalence}
\end{lemma}
\begin{proof}
Since the dual cone of the Coulomb friction cone $\mathcal{K}_{\mu}$ is $\mathcal{K}^{*} \equiv \mathcal{K}_{\mu^{-1}}$, then any vector $\mathbf{x} \in \mathcal{K}^{*}$ satisfies the relations
\begin{equation}
\Vert \mathbf{x}_{T} \Vert_{2} \leq \frac{1}{\mu} \, \text{x}_{N} \,,\,\, \text{x}_{N} \geq 0
\,\,.
\label{eq:apndx:dual-cone-element-conditions}
\end{equation}
Applying (\ref{eq:apndx:dual-cone-element-conditions}) to the decomposition (\ref{eq:apndx:augmented-velocity-decomposition}) of $\hat{\mathbf{v}}$, we see that
\begin{equation}
\begin{array}{ll}
\hat{\mathbf{v}} \in \mathcal{K}^{*} 
\,\,\Leftrightarrow &
\Vert \mathbf{v}_{T}^{+} \Vert_{2} 
\leq \frac{1}{\mu} \left(\text{v}_{N}^{+} + \mu \, \Vert \mathbf{v}_{T}^{+} \Vert_{2} \right) \\[8pt]
& \,,\,\, \text{v}_{N}^{+} + \mu \, \Vert \mathbf{v}_{T}^{+} \Vert_{2} \geq 0
\end{array}
\,\,.
\end{equation}
The $\mu \, \Vert \mathbf{v}_{T}^{+} \Vert_{2}$ terms cancel out in the RHS of the first inequality, leading to a simplification in the form of
\begin{equation*}
\begin{array}{ll}
\hat{\mathbf{v}} \in \mathcal{K}^{*} 
\,\,\Leftrightarrow &
0 \leq \text{v}_{N}^{+} \\[8pt]
& \,,\,\, \text{v}_{N}^{+} + \mu \, \Vert \mathbf{v}_{T}^{+} \Vert_{2} \geq 0
\,\,.
\end{array}
\end{equation*}
Noting that the second inequality is now trivially satisfied, it becomes redundant, and removing it yields the equivalence
\begin{equation*}
\hat{\mathbf{v}} \in \mathcal{K}_{\mu}^{*} \,\,\Leftrightarrow\,\,  \text{v}_{N}^{+} \geq 0 
\end{equation*}
\end{proof}
\begin{remark}
Interpreting these relations as constraints, the equivalence implies that they must yield the same solution.
\end{remark}

\subsection{Invariance of the De Saxc\'e Operator}
\label{sec:apndx:desaxce-properties:desaxce-invariance}
\begin{lemma}
The De Saxc\'e operator is invariant to differences along the normal direction, i.e. for vectors $\mathbf{v}_{1}, \mathbf{v}_{2} \in \mathbb{R}^{3}$ differing by an offset in the form of 
\begin{equation*}
\mathbf{v}_{2} := \mathbf{v}_{1} + \mathbf{v}_{b} \,\,,\,\,\,
\mathbf{v}_{b} = 
\begin{bmatrix}
0\\
\text{v}_{b,N}\\
\end{bmatrix}
\,,\,\,
\end{equation*}
where $\text{v}_{b,N} \in \mathbb{R}$, then it holds that
\begin{equation}
\boldsymbol{\Gamma}_{\mu}(\mathbf{v}_{1}^{+}) = \boldsymbol{\Gamma}_{\mu}(\mathbf{v}_{2}^{+})
\end{equation}
\label{lemma:apndx:desaxce-invariance}
\end{lemma}
\begin{proof}
Using the decomposition (\ref{eq:apndx:any-velocity-decomposition}), the vectors $\mathbf{v}_{1}, \mathbf{v}_{2}$ are expressed using their tangential and normal components as 
\begin{equation*}
\mathbf{v}_{2} 
=
\begin{bmatrix}
\mathbf{v}_{2,T}\\
\text{v}_{2,N}\\
\end{bmatrix}
=
\begin{bmatrix}
\mathbf{v}_{1,T}\\
\text{v}_{1,N} + \text{v}_{b,N}\\
\end{bmatrix}
\,,\,\,
\end{equation*}
and applying the De Saxc\'e operator to the first, yields
\begin{equation*}
\boldsymbol{\Gamma}_{\mu}(\mathbf{v}_{2})
=
\begin{bmatrix}
0\\
\mu \, \Vert \mathbf{v}_{2,T} \Vert_{2}
\end{bmatrix}
=
\begin{bmatrix}
0\\
\mu \, \Vert \mathbf{v}_{1,T} \Vert_{2}
\end{bmatrix}
=
\boldsymbol{\Gamma}_{\mu}(\mathbf{v}_{1})
\,.
\end{equation*}
\end{proof}
\begin{corollary}
Lemma~\ref{lemma:apndx:desaxce-invariance} allows all biases to the augmented contact velocity $\hat{\mathbf{v}}$  to be absorbed as biases to the free-velocity $\mathbf{v}_{f}$. Expressing the post-impact contact velocity as a function of the constraint reaction $\boldsymbol{\lambda} \in \mathbb{R}^{3}$ in the form of
\begin{equation*}
\mathbf{v}^{+}(\boldsymbol{\lambda}) := \mathbf{D}\,\boldsymbol{\lambda} + \mathbf{v}_{f}
\,\,,
\end{equation*}
and subsequently the augmented contact velocity as
\begin{equation*}
\hat{\mathbf{v}}(\boldsymbol{\lambda}) := \mathbf{D}\,\boldsymbol{\lambda} + \mathbf{v}_{f} 
+ \boldsymbol{\Gamma}(\mathbf{v}^{+}(\boldsymbol{\lambda})) 
\,\,,
\end{equation*}
then any bias velocity included in the definition of the complementarity conditions asserted between $\hat{\mathbf{v}}(\boldsymbol{\lambda})$ and $\boldsymbol{\lambda}$, can be transferred to $\mathbf{v}_{f}$. For example, the introduction of Newton's model of restitutive impacts, incurs the bias
\begin{equation*}
\mathbf{v}_{N}^{-} =
\begin{bmatrix}
0 & 0 & v_{N}^{-}
\end{bmatrix}^{T}
\,\,,
\end{equation*}
which we use to re-state the Signorini-Coulomb-Newton NCP model (\ref{eq:coulomb-newton-signorini-ncp}), here recapitulated as
\begin{equation*}
\mathcal{K}_{\mu}^{*} \ni \hat{\mathbf{v}}( \boldsymbol{\lambda}) + \mathbf{v}_{N}^{-} \perp \boldsymbol{\lambda} \in \mathcal{K}_{\mu}
\,\,.
\end{equation*}
Lemma~\ref{lemma:apndx:desaxce-invariance} essentially allows us to re-state the problem as

\begin{equation*}
\begin{array}{l}
{\mathbf{v}'}^{+}(\boldsymbol{\lambda}) := \mathbf{D}\,\boldsymbol{\lambda} + \mathbf{v}_{f} + \mathbf{v}_{N}^{-} \\[8pt]
\hat{\mathbf{v}}'(\boldsymbol{\lambda}) := \mathbf{D}\,\boldsymbol{\lambda} + \mathbf{v}_{f} 
+ \boldsymbol{\Gamma}({\mathbf{v}'}^{+}(\boldsymbol{\lambda})) \\[8pt]
\mathcal{K}_{\mu}^{*} \ni \hat{\mathbf{v}}'( \boldsymbol{\lambda}) \perp \boldsymbol{\lambda} \in \mathcal{K}_{\mu}
\,\,.
\end{array}
\end{equation*}
\end{corollary}

\section{Derivation of the Dual Forward Dynamics NCP}
\label{sec:apndx:ncp-derivation}

\begin{theorem}
The dual FD problem stated as the NSOCP
\begin{equation}
\label{eq:apndx:dual-forward-dynamics-nsocp}
\begin{array}{ll}
\boldsymbol{\lambda} = \displaystyle \operatorname*{argmin}_{\mathbf{x} \in \mathcal{K}} &  
\frac{1}{2} \mathbf{x}^{T} \,\mathbf{D}\, \mathbf{x}
+ \mathbf{x}^{T}\,\mathbf{v}_{f} 
+ \mathbf{x}^{T}\,\boldsymbol{\Gamma}(\mathbf{v}^{+}(\mathbf{x}))\\[8pt]
\end{array}
\end{equation}
is the Lagrange dual of the primal FD problem
\begin{align}
\label{eq:apndx:primal-forward-dynamics-cqp}
\mathbf{u}^{+} &= \displaystyle \operatorname*{argmin}_{\mathbf{x}}\,
\frac{1}{2} \Vert \mathbf{x} - \mathbf{u}_{f} \Vert_{\mathbf{M}} \\[4pt]
& \text{s.t.}\quad \mathbf{J}_{J}\,\mathbf{x} = 0 \nonumber\\[4pt]
& \quad\quad\,\, \mathbf{J}_{L}\,\mathbf{x} \geq 0 \nonumber\\[4pt]
& \quad\quad\,\, \mathbf{J}_{C,N}\,\mathbf{x} \geq 0 \nonumber
\quad,
\end{align}
when the constraint reactions are additionally required to adhere to the Coulomb friction cone model. $\mathbf{u}_{f} := \mathbf{M}^{-1}\,\left(\mathbf{u}^{-} + \Delta{t}\,\mathbf{h} \right)$, and $\mathbf{J}_{C,N}$ is a sub-matrix of the contact constraint Jacobian matrix $\mathbf{J}_{C}$ constructed from the rows corresponding to the contact normal directions.
\end{theorem}
\begin{remark}
The primal problem (\ref{eq:apndx:primal-forward-dynamics-cqp}) is a re-formulation of the forward dynamics NCP (\ref{eq:formulation:time-stepping-ncp}) as an optimization problem, that is transcribed using the \textit{extended principle of least constraint} and \textit{the duality principle of Gauss}. See Glocker~\cite{glocker2001setvalued} for more details regarding these principles. The resulting primal FD problem is a linearly-constrained QP directly in the post-impact generalized velocities of the system, and is more general than (\ref{eq:apndx:dual-forward-dynamics-nsocp}), as it holds for any model of friction.
\end{remark}
\begin{proof}
Starting from (\ref{eq:apndx:primal-forward-dynamics-cqp}), we first apply Lemma~\ref{lemma:apndx:desaxce-invariance} to each individual unilateral contact constraint to form the equivalence
\begin{equation} \label{eq:apndx:primal-problem-velocity-constraint-equivalence}
\mathbf{J}_{C,N}\,\mathbf{u}^{+} = \mathbf{v}_{C,N}^{+} \geq 0 \,\,\Leftrightarrow\,\, \hat{\mathbf{v}} \in \mathcal{K}_{\boldsymbol{\mu}}^{*}
\,\,\,\,,
\end{equation}
which enables the substitution 
\begin{align} \label{eq:apndx:principle-of-least-constraint-step1}
\mathbf{u}^{+} &= \displaystyle \operatorname*{argmin}_{\mathbf{x}}\,
\frac{1}{2} \Vert \mathbf{x} - \mathbf{u}_{f} \Vert_{\mathbf{M}} \\[4pt]
& \text{s.t.}\quad \mathbf{J}_{J}\,\mathbf{x} = 0 \nonumber\\[4pt]
& \quad\quad\,\, \mathbf{J}_{L}\,\mathbf{x} \geq 0 \nonumber\\[4pt]
& \quad\quad\,\, \mathbf{J}_{C}\,\mathbf{x} + \boldsymbol{\Gamma}(\mathbf{J}_{C}\,\mathbf{x}) \in \mathcal{K}_{\boldsymbol{\mu}}^{*} \nonumber
\,\,\,\,.
\end{align}
\begin{remark}
This step has implicitly introduced the Coulomb friction model via the definition of the augmented velocity.
\end{remark}

The next step involves the technique of introducing \textit{slack variables} to decouple the velocity terms from the inequalities and set-inclusions. This yields the equivalent problem
\begin{align} \label{eq:apndx:principle-of-least-constraint-step2}
\mathbf{u}^{+} &= \displaystyle \operatorname*{argmin}_{\mathbf{x},\mathbf{y},\mathbf{z}}\,
\frac{1}{2} \Vert \mathbf{x} - \mathbf{u}_{f} \Vert_{\mathbf{M}} \\[4pt]
& \text{s.t.}\quad \mathbf{J}_{J}\,\mathbf{x} = 0 \nonumber\\[4pt]
& \quad\quad\,\, \mathbf{y} = \mathbf{J}_{L}\,\mathbf{x} \nonumber\\[4pt]
& \quad\quad\,\, \mathbf{y} \in \mathbb{R}_{+}^{n_l} \nonumber\\[4pt]
& \quad\quad\,\, \mathbf{z} = \mathbf{J}_{C}\,\mathbf{x} + \boldsymbol{\Gamma}(\mathbf{J}_{C}\,\mathbf{x}) \nonumber\\[4pt]
& \quad\quad\,\, \mathbf{z} \in \mathcal{K}_{\boldsymbol{\mu}}^{*} \nonumber
\,\,\,\,,
\end{align}
where $\mathbf{y}\,,\,\,\mathbf{z}$ are the slack variables corresponding to the limit and contact constraints. This decoupling enables us to remove the inequalities and set-inclusions using indicator functions as
\begin{align} \label{eq:apndx:principle-of-least-constraint-step3}
\mathbf{u}^{+} &= \displaystyle \operatorname*{argmin}_{\mathbf{x},\mathbf{y},\mathbf{z}}\,
\frac{1}{2} \Vert \mathbf{x} - \mathbf{u}_{f} \Vert_{\mathbf{M}} 
+ \Psi_{\mathbb{R}_{+}^{n_l}}(\mathbf{y})
+ \Psi_{\mathcal{K}_{\boldsymbol{\mu}}^{*}}(\mathbf{z}) \\[4pt]
& \text{s.t.}\quad \mathbf{J}_{J}\,\mathbf{x} = 0 \nonumber\\[4pt]
& \quad\quad\,\, \mathbf{y} - \mathbf{J}_{L}\,\mathbf{x} = 0 \nonumber\\[4pt]
& \quad\quad\,\, \mathbf{z} - \mathbf{J}_{C}\,\mathbf{x} - \boldsymbol{\Gamma}(\mathbf{J}_{C}\,\mathbf{x}) = 0\nonumber
\quad.
\end{align}
The primal problem (\ref{eq:apndx:principle-of-least-constraint-step3}) now only contains equality constraints, for which defining the vector Lagrange multipliers $\boldsymbol{\lambda} = [\boldsymbol{\lambda}_{J}^{T} \,\, \boldsymbol{\lambda}_{L}^{T} \,\, \boldsymbol{\lambda}_{C}^{T}]^{T}$, we form the Lagrangian functional
\begin{equation} \label{eq:apndx:lagrangian-function}
\begin{array}{ll}
\mathcal{L}(\mathbf{x}, \mathbf{y}, \mathbf{z}, \boldsymbol{\lambda}) := 
& \frac{1}{2} \Vert \mathbf{x} - \mathbf{u}_{f} \Vert_{\mathbf{M}} 
+ \Psi_{\mathbb{R}_{+}^{n_l}}(\mathbf{y})
+ \Psi_{\mathcal{K}_{\boldsymbol{\mu}}^{*}}(\mathbf{z})\\[8pt]
& -\,\, \boldsymbol{\lambda}_{J}^{T} \, \mathbf{J}_{J}\,\mathbf{x} 
- \boldsymbol{\lambda}_{L}^{T} \, \left( \mathbf{J}_{L}\,\mathbf{x} - \mathbf{y} \right) \\[8pt]
& -\,\, \boldsymbol{\lambda}_{C}^{T} \left(\mathbf{J}_{C}\,\mathbf{x} + \boldsymbol{\Gamma}(\mathbf{J}_{C}\,\mathbf{x}) - \mathbf{z} \right)
\quad,
\end{array}
\end{equation}
and thus the post-event generalized velocities $\mathbf{u}^{+}$ are solutions of the \textit{saddle-point problem}
\begin{equation} \label{eq:apndx:lagrangian-saddle-point-problem}
\mathbf{u}^{+} = \operatorname*{argm} \min_{\mathbf{x}, \mathbf{y}, \mathbf{z}} \max_{\boldsymbol{\lambda}} \, 
\mathcal{L}(\mathbf{x}, \mathbf{y}, \mathbf{z}, \boldsymbol{\lambda}) 
\,\,.
\end{equation}

At this point we will begin to introduce what we already know from mechanics in order to start removing variables and terms from (\ref{eq:apndx:lagrangian-function}). First of all, we know that the EoM (\ref{eq:models:time-stepping-eom}) result from the gradient $\nabla_{\mathbf{x}}\mathcal{L}(\mathbf{x}^{*}, \mathbf{y}, \mathbf{z}, \boldsymbol{\lambda}) = 0$. Exploiting this fact, we substitute $\mathbf{x}^{*} = \mathbf{u}^{+}$ in (\ref{eq:apndx:lagrangian-function}), and use (\ref{eq:models:time-stepping-eom}) to expand all quadratic terms. This results in a crucial property of the objective function (\ref{eq:apndx:primal-forward-dynamics-cqp}), which conveniently reduces it to
\begin{equation} \label{eq:apndx:primal-dual-objective-equivalence}
\frac{1}{2} \Vert \mathbf{u}^{+} - \mathbf{u}_{f} \Vert_{\mathbf{M}} = 
\frac{1}{2} \boldsymbol{\lambda}^{T} \mathbf{D} \boldsymbol{\lambda} = 
\frac{1}{2} \Vert \boldsymbol{\lambda} \Vert_{\mathbf{D}}
\,\,\,\,.
\end{equation}
Substituting (\ref{eq:apndx:primal-dual-objective-equivalence}) in (\ref{eq:apndx:lagrangian-function}), the Lagrangian becomes
\begin{equation}
\begin{array}{ll}
\mathcal{L}(\mathbf{u}^{+}, \mathbf{y}, \mathbf{z}, \boldsymbol{\lambda}) &= 
\frac{1}{2} \boldsymbol{\lambda}^{T} \mathbf{D} \boldsymbol{\lambda}
+ \Psi_{\mathbb{R}_{+}^{n_l}}(\mathbf{y})
+ \Psi_{\mathcal{K}_{\boldsymbol{\mu}}^{*}}(\mathbf{z}) \\[8pt]
& -\,\, \boldsymbol{\lambda}_{J}^{T} \, \mathbf{J}_{J}\,\mathbf{u}^{+}
- \boldsymbol{\lambda}_{L}^{T} \, \left( \mathbf{J}_{L}\,\mathbf{u}^{+} - \mathbf{y} \right)\\[8pt]
& -\,\, \boldsymbol{\lambda}_{C}^{T} \left( \mathbf{J}_{C}\,\mathbf{u}^{+} + \boldsymbol{\Gamma}(\mathbf{J}_{C}\,\mathbf{u}^{+}) -\mathbf{z} \right)
\end{array}
\,\,.
\end{equation}

Now comes one of the most difficult and subtle parts of this derivation. By duality of the cones $\mathbb{R}_{+}^{n} \leftrightarrow \mathbb{R}_{+}^{n}$ and $\mathcal{K}_{\boldsymbol{\mu}} \leftrightarrow \mathcal{K}_{\boldsymbol{\mu}}^{*}$, as well as the definition (\ref{eq:solvers:support-function}) of their respective \textit{support functions}, when $\Psi_{\mathcal{K}_{\boldsymbol{\mu}}^{*}}(\mathbf{z}) = 0$ and $\Psi_{\mathcal{K}_{\boldsymbol{\mu}}}(\boldsymbol{\lambda}_{C}) = 0$, then necessarily $\displaystyle \min_{\mathbf{z}}\boldsymbol{\lambda}_{C}^{T}\,\mathbf{z} = 0$. These observations result in the following equivalences between the primal and dual sub-problems for the limit and contact constraints: 
\begin{equation} \label{eq:apndx:inclusion-sub-problems-equivalence}
\begin{array}{cc}
\displaystyle \min_{\mathbf{y}} \Psi_{\mathbb{R}_{+}^{n_l}}(\mathbf{y}) 
\,\,\Leftrightarrow\,\, 
\displaystyle \max_{\boldsymbol{\lambda}_{L}} - \Psi_{\mathbb{R}_{+}^{n_l}}(\boldsymbol{\lambda}_{L})\\[8pt]
\displaystyle \min_{\mathbf{z}} \Psi_{\mathcal{K}_{\boldsymbol{\mu}}^{*}}(\mathbf{z}) + \boldsymbol{\lambda}_{C}^{T}\,\mathbf{z}
\,\,\Leftrightarrow\,\, 
\max_{\boldsymbol{\lambda}_{C}} -\Psi_{\mathcal{K}_{\boldsymbol{\mu}}}(\boldsymbol{\lambda}_{C})
\end{array}
\end{equation}

Thus, when $\mathbf{y}^{*} = \mathbf{v}^{+}_{L}$ and $\mathbf{z}^{*} = \hat{\mathbf{v}}_{C}$, and exploiting (\ref{eq:apndx:inclusion-sub-problems-equivalence}), the problem is effectively equivalent to 
\begin{equation} \label{eq:apndx:principle-of-least-constraint-step4}
\begin{array}{ll}
\mathcal{L}(\mathbf{u}^{+}, \mathbf{v}^{+}_{L}, \hat{\mathbf{v}}_{C}, \boldsymbol{\lambda}) &= 
\frac{1}{2} \boldsymbol{\lambda}^{T} \mathbf{D} \boldsymbol{\lambda}\\[8pt]
& - \Psi_{\mathbb{R}_{+}^{n_l}}(\boldsymbol{\lambda}_{L})
- \Psi_{\mathcal{K}_{\boldsymbol{\mu}}}(\boldsymbol{\lambda}_{C})
\\[8pt]
& -\,\, \boldsymbol{\lambda}_{J}^{T} \, \mathbf{J}_{J}\,\mathbf{u}^{+}
- \boldsymbol{\lambda}_{L}^{T} \, \mathbf{J}_{L}\,\mathbf{u}^{+}\\[8pt]
& -\,\, \boldsymbol{\lambda}_{C}^{T} \, \mathbf{J}_{C}\,\mathbf{u}^{+} - \boldsymbol{\lambda}_{C}^{T}\boldsymbol{\Gamma}(\mathbf{J}_{C}\,\mathbf{u}^{+})
\,\,\,\,.
\end{array}
\end{equation}

We can simplify (\ref{eq:apndx:principle-of-least-constraint-step4}) by grouping all linear terms using the definition of the total Jacobian matrix $\mathbf{J}$ to form 
\begin{equation} \label{eq:apndx:principle-of-least-constraint-step5}
\begin{array}{ll}
\mathcal{L}(\mathbf{u}^{+}, \mathbf{v}^{+}_{L}, \hat{\mathbf{v}}_{C}, \boldsymbol{\lambda}) &= 
- \frac{1}{2} \boldsymbol{\lambda}^{T} \mathbf{D} \boldsymbol{\lambda}\\[8pt]
& - \Psi_{\mathbb{R}_{+}^{n_l}}(\boldsymbol{\lambda}_{L})
- \Psi_{\mathcal{K}_{\boldsymbol{\mu}}}(\boldsymbol{\lambda}_{C})\\[8pt]
&\quad - \boldsymbol{\lambda}^{T} \, \mathbf{J}\,\mathbf{u}^{+}
- \boldsymbol{\lambda}_{C}^{T}\boldsymbol{\Gamma}(\mathbf{v}_{C}^{+}(\boldsymbol{\lambda}))
\,\,\,\,,
\end{array}
\end{equation}
and then use (\ref{eq:prepost-system-velocities-eom}) to substitute $\mathbf{u}^{+}$, resulting in
\begin{equation} \label{eq:apndx:principle-of-least-constraint-step6}
\begin{array}{ll}
\mathcal{L}(\mathbf{u}^{+}, \mathbf{v}^{+}_{L}, \hat{\mathbf{v}}_{C}, \boldsymbol{\lambda}) &= 
\frac{1}{2} \boldsymbol{\lambda}^{T} \mathbf{D} \boldsymbol{\lambda}\\[8pt]
& - \Psi_{\mathbb{R}_{+}^{n_l}}(\boldsymbol{\lambda}_{L})
- \Psi_{\mathcal{K}_{\boldsymbol{\mu}}}(\boldsymbol{\lambda}_{C})\\[8pt]
& - \boldsymbol{\lambda}^{T} \, \mathbf{D}\, \boldsymbol{\lambda} - \boldsymbol{\lambda}^{T}\,\mathbf{v}_{f}\\[8pt]
& - \boldsymbol{\lambda}_{C}^{T}\boldsymbol{\Gamma}(\mathbf{v}_{C}^{+}(\boldsymbol{\lambda}))
\,\,\,\,,
\end{array}
\end{equation}
which is now only a function of the Lagrange multipliers $\boldsymbol{\lambda}$. Before proceeding with the last step, simplifying the quadratic terms, brings us to the dual maximization problem
\begin{equation} \label{eq:apndx:principle-of-least-constraint-step7}
\begin{array}{ll}
\displaystyle \max_{\boldsymbol{\lambda}} &  
-\frac{1}{2} \boldsymbol{\lambda}^{T} \mathbf{D} \boldsymbol{\lambda}
- \boldsymbol{\lambda}^{T}\,\mathbf{v}_{f} 
- \boldsymbol{\lambda}_{C}^{T}\boldsymbol{\Gamma}(\mathbf{v}_{C}^{+}(\boldsymbol{\lambda}))\\[8pt]
&\quad
- \Psi_{\mathbb{R}_{+}^{n_l}}(\boldsymbol{\lambda}_{L})
- \Psi_{\mathcal{K}_{\boldsymbol{\mu}}}(\boldsymbol{\lambda}_{C})
\,\,\,\,.
\end{array}
\end{equation}

The final step involves using the definition of the total cone
\begin{equation*}
\mathcal{K} := \mathbb{R}^{n_{jd}} \times \mathbb{R}_{+}^{n_l} \times \mathcal{K}_{\boldsymbol{\mu}}
\,\,,
\end{equation*}
and inversion of the sign of the objective function to render a minimization problem. These at last yield the Lagrange dual of the primal FD problem, in the form of
\begin{equation} \label{eq:apndx:principle-of-least-constraint-step8}
\begin{array}{ll}
\displaystyle \min_{\boldsymbol{\lambda} \in \mathcal{K}} &  
\frac{1}{2} \boldsymbol{\lambda}^{T} \mathbf{D} \boldsymbol{\lambda}
+ \boldsymbol{\lambda}^{T}\,\mathbf{v}_{f} 
+ \boldsymbol{\lambda}^{T}\boldsymbol{\Gamma}(\mathbf{v}^{+}(\boldsymbol{\lambda}))\\[8pt]
\end{array}
\end{equation}
\end{proof}

\section{Solving of the Single-Contact Problem}
\label{sec:apndx:single-contact}

\subsection{A Semi-Analytical Approach}
\label{sec:apndx:single-contact:analytical-formulation}

This derivation is based on that presented by Preclik et al in~\cite{preclik2014models, preclik2018mdp}, and assumes only a single contact is present. Thus we can assume that  $\mathbf{D} \in \mathbb{S}_{++}^{3}$, $\mathbf{v}_{f} \in \mathbb{R}^{3}$, and $\mathcal{K}_{\mu} \subset \mathbb{R}^{3}$. The dual FD problem can be stated 
\begin{equation}
\begin{array}{rlclcl}
\displaystyle \displaystyle 
\min_{\boldsymbol{\lambda}} \,
    & \frac{1}{2} \, \boldsymbol{\lambda}^{T} \, \mathbf{D} \, \boldsymbol{\lambda}
    + {\boldsymbol{\lambda}}^{T} \, (\mathbf{v}_{f} + \mathbf{s}) \\[4pt]
\textrm{s.t.} & \boldsymbol{\lambda} \in \mathcal{K}_{\mu} \\[4pt]
              & \text{v}_{c,N}^{+}(\boldsymbol{\lambda}) \geq 0 \\[4pt]
              & \lambda_{N}^{T} \, \text{v}_{c,N}^{+}(\boldsymbol{\lambda}) = 0
\end{array}
\,\, ,
\label{eq:single-contact-dual-problem}
\end{equation}
where the Signorini conditions are to be enforced explicitly, as opposed to be assumed to hold only at the solution. Since we are only dealing with a single contact, we can express the problem definition using the coefficients of the Delassus matrix and the free velocity as
\begin{equation}
\mathbf{D} := 
\begin{bmatrix}
D_{tt} & D_{to} & D_{tn} \\
D_{ot} & D_{oo} & D_{on} \\
D_{nt} & D_{no} & D_{nn} 
\end{bmatrix}
\,\, , \,\,
\mathbf{v}_{f} := 
\begin{bmatrix}
\text{v}_{f,t} \\
\text{v}_{f,o} \\
\text{v}_{f,n} 
\end{bmatrix}
\end{equation}
The total, tangential and normal contact reactions expressed in local coordinates are respectively
\begin{equation}
\boldsymbol{\lambda} :=
\begin{bmatrix}
\lambda_{t} \\
\lambda_{o} \\
\lambda_{n} 
\end{bmatrix}
\, , \,\,
\boldsymbol{\lambda}_{T} :=
\begin{bmatrix}
\lambda_{t} \\
\lambda_{o} \\
0
\end{bmatrix}
\, , \,\,
\boldsymbol{\lambda}_{N} :=
\begin{bmatrix}
0 \\
0 \\
\lambda_{n}
\end{bmatrix}
\,\, .
\end{equation}

Next we will project the Delassus matrix and the free velocity onto the local axis of the contact normal $\mathbf{n}$, here denoted as $\boldsymbol{e}_{n} = [0 \, 0 \, 1]^{T}$, to form
\begin{equation}
\mathbf{d}_{n} := \boldsymbol{e}_{n}^{T} \, \mathbf{D}
=
\begin{bmatrix}
D_{tn} \\
D_{on} \\
D_{nn} 
\end{bmatrix}
\,\,\,\,
\text{v}_{f,n} = \boldsymbol{e}_{n}^{T} \, \mathbf{v}_{f}
\,\,,
\label{eq:plane-of-maximum-compression-parameters}
\end{equation}
to simplify the expression of the post-event contact velocity
\begin{equation}
\mathbf{v}_{c,N}^{+} = 
\boldsymbol{e}_{n}^{T} \, (\mathbf{D} \, \boldsymbol{\lambda} + \mathbf{v}_{f})
= \mathbf{d}_{n}^{T} \, \boldsymbol{\lambda} + \text{v}_{f,n} \geq 0
\,\, .
\label{eq:post-event-velocity-normal-projection}
\end{equation}
Thus, the single-contact dual problem is rewritten as 
\begin{equation}
\begin{array}{rlclcl}
\displaystyle \displaystyle 
\min_{\boldsymbol{\lambda}} \,
    & \frac{1}{2} \, \boldsymbol{\lambda}^{T} \, \mathbf{D} \, \boldsymbol{\lambda}
    + \boldsymbol{\lambda}^{T} \, \mathbf{v}_{f} \\[4pt]
\textrm{s.t.} & \lambda_{n} \geq 0 \\[4pt]
              & \mathbf{d}_{n}^{T} \, \boldsymbol{\lambda} \geq 0 \\[4pt]
              & \mathbf{d}_{n}^{T} \, \boldsymbol{\lambda} + \text{v}_{f,n} \geq 0 \\[4pt]
              & \Vert \boldsymbol{\lambda}_{T} \Vert_{2} \leq \mu \, \lambda_{n} \quad .
\end{array}
\label{eq:reduced-single-contact-dual-problem}
\end{equation}

A key aspect of this derivation is the consideration of the geometry that is implied by (\ref{eq:reduced-single-contact-dual-problem}), as depicted in Fig.\ref{fig:contact-geometry}. The friction cone constraint $\Vert \boldsymbol{\lambda}_{T} \Vert_{2} \leq \mu \, \lambda_{n}$ defines a conic surface boundary that extends in both half-spaces of the normal axis. The unilateral force constraint $\lambda_{n} \geq 0$ requires us to exclude the cone below the tangential plane. The inequality $\mathbf{d}_{n}^{T} \, \boldsymbol{\lambda} + \text{v}_{f,n} \geq 0$ of the unilateral contact hypothesis defines a plane that forms a closed upper bound on the contact reaction along the normal axis. If we know that the contact will remain closed, or in slip, then we can expect the contact reaction to necessarily lie on this plane, because (\ref{eq:post-event-velocity-normal-projection}) becomes an equality constraint (i.e. the constraint is active). The parameters (\ref{eq:plane-of-maximum-compression-parameters}) fully specify this plane, as $\mathbf{d}_{n}$ is parallel to the plane normal, and $\boldsymbol{\lambda}_{T=0}=[0 \, 0 \, D_{nn}^{-1}\,\text{v}_{f,n}]^{T}$ defines the point at which the plane intersects the normal axis. A mechanical interpretation of this plane serving as an upper-bound on the normal contact reaction gives us cause to refer to it as the \textit{plane of maximum compression} $\mathcal{P}_{mc}$. 
\begin{figure}[!t]
\centering
\includegraphics[width=1.0\linewidth]{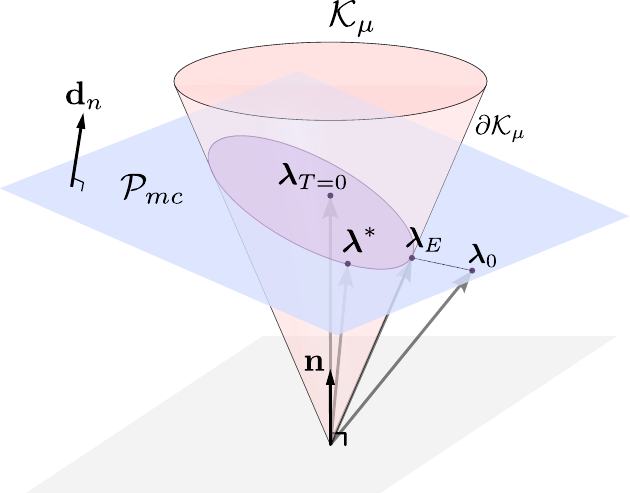}
\caption{A visualization of the geometry implied by the constraints of the single-contact dual problem (\ref{eq:reduced-single-contact-dual-problem}). It depicts the contact normal $\mathbf{n}$, the tangential plane (gray), the friction cone $\mathcal{K}_{\mu}$ (red) with boundary $\partial\mathcal{K}_{\mu}$, the plane of maximum compression $\mathcal{P}_{mc}$ (blue), and the elliptic conic-section $\mathcal{P}_{mc} \cap \partial\mathcal{K}_{\mu}$ (purple). $\mathcal{P}_{mc}$ intersects the local normal axis at $\boldsymbol{\lambda}_{T=0}$. The unconstrained solution $\boldsymbol{\lambda}_{0}$ is projected by the MDP as $\boldsymbol{\lambda}^{*}$, while the naive projection using the Euclidean minimum-distance metric is denoted as $\boldsymbol{\lambda}_{E}$.}
\label{fig:contact-geometry}
\end{figure}

Now, let us also consider when and how this plane intersects the friction cone $\mathcal{K}_{\mu}$. If the contact is sliding, then we know that the contact reaction will lie on the boundary $\partial\mathcal{K}_{\mu}$. This means that the \textit{conic section} $\mathcal{P}_{mc} \cap \partial\mathcal{K}_{\mu}$ defines the feasible set for this case. Thus, we can examine the possible contact states $m$ of the solution to (\ref{eq:reduced-single-contact-dual-problem}), and what they imply geometrically. For this analysis, we also need the unconstrained solution
\begin{equation}
\boldsymbol{\lambda}_{0} = - \mathbf{D}^{-1} \, \mathbf{v}_{f}
\,\, ,
\label{eq:single-contact-unconstrained-solution}
\end{equation}
Note that (\ref{eq:single-contact-unconstrained-solution}) always holds, since in the single-contact case, $\mathbf{D}$ is always invertible\footnote{i.e. as there are no couplings that can cause it to become rank-deficient. Only the nonphysical cases of zero mass, singular moment of inertia, or an infinite lever-arm can cause it to become singular.}. The three contact states defined in Sec.\ref{sec:models:contacts} are:
\begin{itemize}
\item $m=\text{Open}$: 
The Signorini conditions become $\mathbf{v}_{c,N}^{+} > 0$, $\boldsymbol{\lambda}_{N} = 0$, and therefore $\boldsymbol{\lambda}_{T} = 0$ as well. $\boldsymbol{\lambda} \in \mathcal{K}_{\mu}$ (i.e. in the interior of the cone), and thus $\boldsymbol{\lambda} = \boldsymbol{\lambda}_{0}$. Using (\ref{eq:post-event-velocity-normal-projection}), we can reduce this case to $\text{v}_{f,n} > 0$.\\
\item $m=\text{Stick}$: The Signorini conditions become $\mathbf{v}_{c,N}^{+} = 0$, $\boldsymbol{\lambda}_{N} > 0$, and due to Coulomb's law $\boldsymbol{\lambda}_{T} = 0$. $\boldsymbol{\lambda} \in \mathcal{K}_{\mu}$ and thus $\boldsymbol{\lambda} = \boldsymbol{\lambda}_{0}$. Using (\ref{eq:post-event-velocity-normal-projection}), this reduces to $\text{v}_{f,n} \leq 0$.\\
\item $m=\text{Slip}$: The Signorini conditions become $\mathbf{v}_{c,N}^{+} = 0$, $\boldsymbol{\lambda}_{N} > 0$, and because of Coulomb's law $\Vert \boldsymbol{\lambda}_{T} \Vert_{2} \leq \mu \, \lambda_{n}$. Therefore, the only conditions on the constrained and unconstrained solution are respectively
$\boldsymbol{\lambda} \in \mathcal{P}_{mc} \cap \partial\mathcal{K}_{\mu}$ and $\boldsymbol{\lambda}_{0} \in \mathcal{P}_{mc}$.
\end{itemize}
We now have a partial sketch of the semi-analytical solution. The open and sticking contact cases are solely determined by the given $\text{v}_{f,n}$, where in both $\boldsymbol{\lambda} = \boldsymbol{\lambda}_{0}$. The only non-trivial case is that of slipping contacts. Employing the approach of Preclik et al from \cite{preclik2018mdp}, we can apply the MDP, together with the geometry of conic sections, to re-phrase (\ref{eq:reduced-single-contact-dual-problem}) for the slipping case and yield the maximally dissipative contact reaction.  

As conic-sections can be expressed succinctly using \textit{polar coordinates}, let's re-parameterize the contact reaction as
\begin{subequations}
\begin{equation}
\lambda_{t} = r \, \cos\phi
\end{equation}
\begin{equation}
\lambda_{o} = r \, \sin\phi
\end{equation}
\begin{equation}
\lambda_{n} = - D_{nn}^{-1} \, (\text{v}_{f,n} + D_{nt} \, r \, \cos\phi + D_{no} \, r \, \sin\phi )
\,\, ,
\end{equation}
\begin{equation}
\boldsymbol{\lambda}_{p}(r,\phi) =
\begin{bmatrix}
\lambda_{t} & \lambda_{o} & \lambda_{n}
\end{bmatrix}^{T}
\end{equation}
\label{eq:polar-coordinates-contact-reaction}
\end{subequations}
where $r \in \mathbb{R}_{+}$ and $\phi \in \mathbb{R}$ are the radial and angular coordinates, respectively. Using (\ref{eq:polar-coordinates-contact-reaction}) in $\Vert \boldsymbol{\lambda}_{T} \Vert_{2} \leq \mu \, \lambda_{n}$ yeilds
\begin{equation}
r \, ( D_{nn} + \mu \, D_{nt} \, \cos\phi +  \mu \, D_{no} \, \sin\phi ) \leq -\mu \, \text{v}_{f,n}
\,\, .
\label{eq:polar-coordinates-fricion-cone-inequality}
\end{equation}
The multiplicand of $r$ is the \textit{radial scaling} function
\begin{subequations}
\begin{equation}
f_{r}(\phi) := 
D_{nn} + \mu \, D_{nt} \, \cos\phi +  \mu \, D_{no} \, \sin\phi
\label{eq:polar-function-f-def}
\end{equation}
\begin{equation}
= 
D_{nn} + \mu \, \sqrt{D_{nt}^{2} + D_{no}^{2}} \, \cos \left( \phi - \tan^{-1}\left(\frac{D_{no}}{D_{nt}}\right) \right)
\label{eq:polar-function-f-alt}
\,\, ,
\end{equation}
\label{eq:polar-function-f}
\end{subequations}
where $f_{r} : \mathbf{R} \rightarrow \mathbb{R}$. When $f_{r}(\phi) \leq 0$, (\ref{eq:polar-coordinates-fricion-cone-inequality}) holds trivially. However, when  $f_{r}(\phi) > 0$, we need to use (\ref{eq:polar-function-f-alt}) to determine the \textit{feasible range} of angles $\phi \in \mathcal{I}_{F} + 2\pi\mathbb{N}$, for which (\ref{eq:polar-coordinates-fricion-cone-inequality}) holds, i.e.
\begin{subequations}
\begin{equation}
\mathcal{I}_{F} =
\begin{cases}
(-\Delta{\phi},\Delta{\phi}] + \phi_{0},   & , D_{nn} \leq \mu \, \sqrt{D_{nt}^{2} + D_{no}^{2}} \\
[0, 2\pi) &, otherwise
\end{cases}
\,\,,
\label{eq:polar-angle-feasible-range-def}
\end{equation}
\begin{equation}
\Delta{\phi} = \cos^{-1}{\frac{-D_{nn}}{\mu \, \sqrt{D_{nt}^{2} + D_{no}^{2}}}}
\,\,,\,\,
\phi_{0} = \tan^{-1}{\frac{D_{no}}{D_{nt}}}
\,\,.
\label{eq:feasible-region-range}
\end{equation}
\label{eq:polar-angle-feasible-region}
\end{subequations}
According to (\ref{eq:polar-coordinates-fricion-cone-inequality}), $f_{r}(\phi)$ determines an upper-bound on $r$:
\begin{equation}
0 < r \leq \frac{-\mu\,\text{v}_{f,n}}{\max(0, f_{r}(\phi))}
\,\, .
\label{eq:polar-coordinates-radial-bounds}
\end{equation}
By defining the polar radius upper-bound as the function
\begin{equation}
\bar{r}(\phi) := \frac{-\mu\,\text{v}_{f,n}}{f_{r}(\phi)}
\,\,,
\label{eq:polar-angle-radius}
\end{equation}
we can express the tangential contact reaction as
\begin{equation}
\boldsymbol{\lambda}_{T}(\phi) := \bar{r}(\phi) \,
\begin{bmatrix}
    \cos\phi \\
    \sin\phi
\end{bmatrix}
\,\, .
\label{eq:polar-angle-tangent-reaction}
\end{equation}
Projecting the $\mathbf{D}$ and $\mathbf{v}_{f}$ onto the tangential plane as
\begin{equation}
\hat{\mathbf{D} }:=
\begin{bmatrix}
D_{tt} - D_{nn}^{-1}\,D_{nt}^{2} & D_{to} - D_{nn}^{-1}\,D_{nt}\,D_{no} \\
D_{to} - D_{nn}^{-1}\,D_{nt}\,D_{no} & D_{oo} - D_{nn}^{-1}\,D_{no}^{2}
\end{bmatrix}
\label{eq:tangential-delassus}
\end{equation}
\begin{equation}
\hat{\mathbf{v} }:=
\begin{bmatrix}
\text{v}_{f,t} - D_{nn}^{-1}\,D_{nt}\,\text{v}_{f,n} \\
\text{v}_{f,o} - D_{nn}^{-1}\,D_{no}\,\text{v}_{f,n}
\end{bmatrix}
\,\, ,
\label{eq:tangential-free-velocity}
\end{equation}
we can thus state the \textit{single-contact polar-angle dual problem} 
\begin{equation}
\begin{array}{rlclcl}
\displaystyle \displaystyle 
\min_{\phi} \,
    & \frac{1}{2} \, \boldsymbol{\lambda}_{T}(\phi)^{T} \, \hat{\mathbf{D}} \, \boldsymbol{\lambda}_{T}(\phi)
    + \boldsymbol{\lambda}_{T}(\phi)^{T} \, \hat{\mathbf{v}} \\[4pt]
\textrm{s.t.} & \phi \in \mathcal{I}_{F}
\end{array}
\,\, ,
\label{eq:polar-angle-dual-problem}
\end{equation}
where now the only constraint is that the solution must lie in the interval $\mathcal{I}_{F}$. The optimality condition on the gradient of the objective function of (\ref{eq:polar-angle-dual-problem}) is therefore
\begin{subequations}
\begin{equation}
\mathbf{t}(\phi)^{T} \, \left( \hat{\mathbf{D}} \, \boldsymbol{\lambda}_{T}(\phi) + \hat{\mathbf{v}} \right) = 0
\label{eq:polar-angle-optimality-condition-def}
\end{equation}
\begin{equation}
\mathbf{t}(\phi) := 
\begin{bmatrix}
\delta(\phi) \, \cos\phi - \bar{r}(\phi)\,\sin\phi \\
\delta(\phi) \, \sin\phi + \bar{r}(\phi)\,\cos\phi \\
\end{bmatrix}
\label{eq:polar-angle-tangent-vector}
\end{equation}
\begin{equation}
\delta(\phi) := 
\frac{\bar{r}(\phi)}{\text{v}_{f,n}} 
\, \sqrt{D_{nt}^{2} + D_{no}^{2}} 
\, \sin\left(\phi -\tan^{-1}\left( \frac{D_{no}}{D_{nt}} \right) \right)
\label{eq:polar-angle-tangent-cureve-gradient}
\end{equation}
\label{eq:polar-angle-optimality-condition}
\end{subequations}
At this point, plugging-in (\ref{eq:polar-coordinates-contact-reaction}) to (\ref{eq:polar-angle-optimality-condition}) and multiplying by $\frac{\mu^{2}\,\text{v}_{f,n}^{2}}{\bar{r}(\phi)^{3}}$, results in the \textit{polar-angle trigonometric equation} 
\begin{equation}
c_{0} + c_{1} \, \cos(\phi + \phi_{1}) + c_{2} \, \cos(2\,\phi + \phi_{2}) = 0
\label{eq:polar-angle-trigonometric-equation}
\end{equation}

The second-last step involves transforming the trigonometric equation to a complex polynomial. This requires defining the complex \textit{phasor} $x = e^{i\,\phi}$ and applying it to (\ref{eq:polar-angle-trigonometric-equation}). The result is the \textit{complex quartic polynomial equation}
\begin{equation}
c_{2}\,\beta^{*}\,{x}^{4} 
+ c_{1}\,\alpha^{*}\,{x}^{3} 
+ 2\,c_{0}\,{x}^{2}
+ c_{1}\,\alpha\,{x}
+ c_{2}\,\beta
= 0
\,\, ,
\label{eq:polar-angle-quartic-polynomial}
\end{equation}
where $\alpha = e^{i\,\phi_{1}}$ and $\alpha = e^{i\,\phi_{2}}$ are complex phasor constants, and $\alpha^{*}$ and $\beta^{*}$ are their respective complex conjugates. 
\begin{figure}[!t]
\centering
\includegraphics[width=1.0\linewidth]{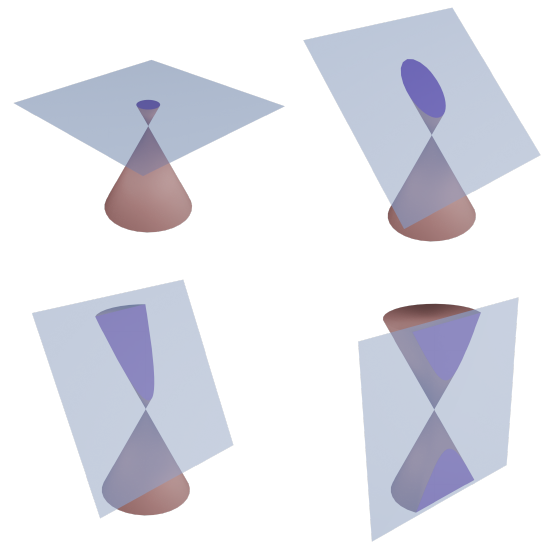}
\caption{Conic-section types: (top left) a circle when $r_{d/\mu} \rightarrow \infty$, i.e. the plane of maximum compression $\mathcal{P}_{mc}$ is exactly parallel to the contact normal $\mathbf{n}$, resulting in a circle. (top right) an ellipse when $r_{d/\mu} > 1$, i.e. $\mathbf{d}_{n}$ always points in directions that do not intersect the cone boundary $\partial\mathcal{K}_{\mu}$. (bottom left) a parabola when $r_{d/\mu} = 1$ , i.e. $\mathbf{d}_{n}$ is parallel to $\partial\mathcal{K}_{\mu}$. (bottom right) a hyperbola when $r_{d/\mu} < 1$, i.e. when $\mathcal{P}_{mc}$ intersects $\partial\mathcal{K}_{\mu}$ in both half-spaces.}
\label{fig:conic-sections}
\end{figure}
We have therefore reduced the dual problem (\ref{eq:reduced-single-contact-dual-problem}), for the case of a sliding contact, to a polynomial equation (\ref{eq:polar-angle-trigonometric-equation}). Solving\footnote{In our implementation we have used \texttt{Eigen::PolynomialSolver}.} it, however, often yields multiple solutions due its quartic order, but only one of these can be maximally dissipative. To determine the one we seek, we can trivially exclude those that do not lie in the interval $\mathcal{I}_{F}$, as well as those that do not satisfy (\ref{eq:polar-angle-optimality-condition}) and (\ref{eq:polar-angle-trigonometric-equation}). From those that remain, we select the one that minimizes the objective function of (\ref{eq:polar-angle-dual-problem}). This is necessarily unique because we know that the original problem in Cartesian coordinates is strictly convex. The maximally dissipative contact reaction $\lambda^{*}$ is computed using (\ref{eq:polar-coordinates-contact-reaction}) for the optimal polar-coordinate values $\phi^{*}$ and $r^{*} = \bar{r}(\phi^{*})$.

There are a few corner cases to mention though. In the geometric interpretation, the conic section exhibits degenerate cases that occur when $\text{v}_{f,n} = 0$, i.e. $\mathcal{P}_{mc}$ intersects the origin. To examine these cases, we'll utilize a quantity which tells us how 
$\mathcal{P}_{mc}$ is oriented w.r.t $\mathcal{K}_{\mu}$. Using the definition of $\Delta{\phi}$ from (\ref{eq:feasible-region-range}), we can define the \textit{conic-section inclination ratio}
\begin{equation}
r_{d/\mu} = (\cos\Delta{\phi})^{2} = \frac{D_{nn}^{2}}{\mu^{2} \, (D_{nt}^{2} + D_{no}^{2})}
\in \mathbb{R}_{+}
\,\, .
\end{equation}

Thus, we can use the values of $r_{d/\mu}$ to identify the four types of conic-sections that $\mathcal{P}_{mc} \cap \partial\mathcal{K}_{\mu}$ can form. These are depicted in Fig.\ref{fig:conic-sections}, and are defined precisely by the cases:
\begin{enumerate}
\item $r_{d/\mu} \rightarrow \infty$: \textit{circular}, when $\mathbf{D}$ is diagonal and/or $\mu = 0$.
\item $r_{d/\mu} > 1$: \textit{elliptic}, $\mathcal{P}_{mc}$ intersects $\partial\mathcal{K}_{\mu}$ in one half-space 
\item $r_{d/\mu} = 1$:  \textit{parabolic}, $\mathcal{P}_{mc}$ is parallel to $\partial\mathcal{K}_{\mu}$
\item $r_{d/\mu} < 1$: \textit{hyperbolic}, $\mathcal{P}_{mc}$ intersects $\partial\mathcal{K}_{\mu}$ in both half-spaces
\end{enumerate}
When $\text{v}_{f,n} = 0 \land r_{d/\mu} > 1$, $\mathcal{P}_{mc}$ intersects $\partial\mathcal{K}_{\mu}$ uniquely at the origin, and the solution is trivially $\lambda^{*} = 0$. This is the case of a degenerate elliptic conic-section. Conversely, when $\text{v}_{f,n} = 0 \land r_{d/\mu} < 1$, the problem admits a closed-form solution that is conditioned on whether the unconstrained solution lies within the friction cone:
\begin{equation}
\phi^{*} =
\begin{cases}
\phi_{0} - \Delta{\phi} & ,\lambda_{0} \in \mathcal{K}_{\mu}\\
\phi_{0} + \Delta{\phi} & , otherwise
\end{cases}
\,\,,
\label{eq:degenerate-hyperbolic-conic-section}
\end{equation}
$r^{*} = \max(0, \bar{r}(\phi^{*}))$, and $\lambda^{*}$ is computed from (\ref{eq:polar-coordinates-contact-reaction}). This is the case of a degenerate hyperbolic conic-section. For the special case of degenerate parabolic conic-section where $\text{v}_{f,n} = 0 \land r_{d/\mu} = 1$, (\ref{eq:degenerate-hyperbolic-conic-section}) still applies, but with $\Delta{\phi} = \pi$, because $\mathbf{d}_{n}$ is exactly perpendicular to the surface of $\mathcal{K}_{\mu}$. The final degenerate case is when $\mu = 0$ (i.e. frictionless contact), where $r_{d/\mu} \rightarrow \infty$ and the solution is trivially
\begin{equation}
\lambda_{\mu = 0}^{*} = 
\begin{bmatrix}
0 \\
0 \\
-D_{nn}^{-1}\,\text{v}_{f,n}
\end{bmatrix}
\label{eq:single-contact-fricionless-solution}
\end{equation}

In summary, we have seen that $\mathbf{d}_{n}$, $\text{v}_{f,n}$ and $\mu$ uniquely determine the geometry of the single-contact problem, i.e. $\mathcal{P}_{mc} \cap \mathcal{K}_{\mu}$. Moreover, only one particular combination requires solving the quartic polynomial (\ref{eq:polar-angle-quartic-polynomial}), while all other cases can be handled analytically.

\subsection{The Bisection-Search Method}
\label{sec:apndx:single-contact:bisection-search}

This derivation is based on that presented by Hwangbo et al in~\cite{hwangbo2018percontact} that describes the formulation used by RaiSim. Instead of solving the quartic polynomial (\ref{eq:polar-angle-quartic-polynomial}) to handle the slipping case, Hwangbo et al solved (\ref{eq:polar-angle-dual-problem}) by employing the well established \textit{bisection-method} to search for a zero-crossing that corresponds to the optimal polar-angle. In doing so, we can avoid performing most of the expensive trigonometric operations that semi-analytical approach requires. This alternate method significantly accelerates the operation that yields the maximally dissipative projection of a given contact reaction. This acceleration also comes without cost of accuracy, as the results are identical to former up to reasonable numerical tolerances, and thus serves as its drop-in replacement.

Recalling elements of the derivation of Appendix.~\ref{sec:apndx:single-contact:analytical-formulation}, the contact reaction necessarily lies on the plane of maximum compression $\mathcal{P}_{mc}$ defined by the quantities 
\begin{equation*}
\mathbf{d}_{n} := \boldsymbol{e}_{n}^{T} \, \mathbf{D}
=
\begin{bmatrix}
D_{tn} & D_{on} & D_{nn} 
\end{bmatrix}^{T}
\,\,\,,\,\,\,
\text{v}_{f,n} = \boldsymbol{e}_{n}^{T} \, \mathbf{v}_{f}
\,\, ,
\label{eq:plane-of-maximum-compression-parameters-recap}
\end{equation*}
where $\mathbf{d}_{n}$ is colinear to plane normal $\mathbf{n}$ and $\boldsymbol{\lambda}_{T=0}=[0 \, 0 \, D_{nn}^{-1}\,\text{v}_{f,n}]^{T}$ is the point at which it intersects the axis $\boldsymbol{e}_{n} = [0 \, 0 \, 1]^{T}$. When the contact is sliding, we seek solutions on the set  $\mathcal{P}_{mc}\cap\partial\mathcal{K}_{\mu}$ that represents the conic-section formed by $\mathcal{P}_{mc}$ as it intersects the boundary of the friction cone $\partial\mathcal{K}_{\mu}$. 
Exploiting the geometry of conic-sections, we parameterize the contact reaction using polar coordinates in the form of
\begin{subequations}
\begin{equation*}
\lambda_{t} = r \, \cos\phi
\,\,\,,\,\,\,\,
\lambda_{o} = r \, \sin\phi
\end{equation*}
\begin{equation*}
\lambda_{n} = - D_{nn}^{-1} \, (\text{v}_{f,n} + D_{nt} \, r \, \cos\phi + D_{no} \, r \, \sin\phi )
\,\, ,
\end{equation*}
\begin{equation*}
\boldsymbol{\lambda}_{p}(r,\phi) =
\begin{bmatrix}
\lambda_{t} & \lambda_{o} & \lambda_{n}
\end{bmatrix}^{T}
\,\, ,
\end{equation*}
\label{eq:polar-coordinates-contact-reaction-recap}
\end{subequations}
The conic-section induces a coupling between the polar coordinates such that the radius becomes a function of the angle, i.e.
\begin{equation*}
r(\phi) := \frac{-\mu\,\text{v}_{f,n}}{D_{nn} + \mu \, D_{nt} \, \cos\phi +  \mu \, D_{no} \, \sin\phi}
\,\, .
\end{equation*}
Depending on how $\mathcal{P}_{mc}\cap\partial\mathcal{K}_{\mu}$ is formed, the set of admissible values for the polar angle $\phi$, is defined as the feasible range
\begin{equation*}
\mathcal{I}_{F} =
\begin{cases}
(-\Delta{\phi},\Delta{\phi}] + \phi_{0},   & , D_{nn} \leq \mu \, \sqrt{D_{nt}^{2} + D_{no}^{2}} \\
[0, 2\pi) &, otherwise
\end{cases}
\,\,,
\label{eq:polar-angle-feasible-range-def-recap}
\end{equation*}
where the central angle $\phi_{0}$ and half-range $\Delta{\phi}$ are defined in (\ref{eq:feasible-region-range}). In the case of a sliding contact ($m=\text{Slip}$), the single-contact dual problem (\ref{eq:polar-angle-dual-problem}) can be formulated solely in terms of the polar-angle as
\begin{equation*}
\begin{array}{rlclcl}
\displaystyle \displaystyle 
\min_{\phi} \,
    & \frac{1}{2} \, \boldsymbol{\lambda}_{T}(\phi)^{T} \, \hat{\mathbf{D}} \, \boldsymbol{\lambda}_{T}(\phi)
    + \boldsymbol{\lambda}_{T}(\phi)^{T} \, \hat{\mathbf{v}} \\[4pt]
\textrm{s.t.} & \phi \in \mathcal{I}_{F}
\end{array}
\,\, ,
\label{eq:polar-angle-dual-problem-recap}
\end{equation*}
where $\boldsymbol{\lambda}_{T}(\phi) \in \mathbb{R}^{2}$ is the polar-angle-dependent tangential contact reaction, and $\hat{\mathbf{D}}\in\mathbb{R}^{2\times2}\,,\,\,\hat{\mathbf{v}}_{f}\in\mathbb{R}^{2}$, are respectively, the Delassus matrix and a free velocity projected onto $\mathcal{P}_{mc}$.

The conic-section, and thus the range of admissible angles, depends solely on the normal projection of the Delassus matrix and the friction cone aperture. Second, where along the conic-section, solutions actually form is also dependent on the free velocity, as it determines the unconstrained solution $\boldsymbol{\lambda}_{0}$. As sliding friction necessarily opposes motion, we can expect that the maximally dissipative contact reaction to lie on a section of the curve that shares some properties with $\boldsymbol{\lambda}_{0}$. Thus, the motivation behind this approach lies in three insightful observations regarding how and where solutions form:
\begin{enumerate}
\item the single-contact polar-angle dual problem (\ref{eq:polar-angle-dual-problem}) is strictly convex with a unique global minimum
\item the feasible region $\mathcal{I}_{F}$ is a bounded interval
\item when multiple solutions (i.e. zeros) of the stationarity condition (\ref{eq:polar-angle-optimality-condition-def}) exist, only the unique global minimum may lie in region of the conic-section curve that is within line-of-sight from the unconstrained solution (see p.31 of \cite{preclik2014models} for a formal proof of this property)
\end{enumerate}
From these, Hwangbo et al extrapolated the following ideas. Firstly, as the worst-case search interval $\mathcal{I}_{F}$ is finite, a direct search for zeros of (\ref{eq:polar-angle-optimality-condition-def}) is possible. Secondly, the line-of-site region presents a reasonable heuristic to both begin the search from as well as to narrow the effective search interval. Moreover, the Euclidean (i.e. minimum-distance) projection of the unconstrained solution provides an automatic mechanism to define a starting point for the search. Lastly, as floating-point trigonometric computations are expensive, an iterative method that only requires simple algebraic expressions can potentially reduce the time complexity of the search.

Next we describe how the bisection-method can be used to solve this polar-angle dual problem without solving the quartic polynomial (\ref{eq:polar-angle-quartic-polynomial}). The objective function of the dual problem (\ref{eq:reduced-single-contact-dual-problem}), and its gradient are
\begin{subequations}
\begin{equation}
h_{0}(\boldsymbol{\lambda}) 
= \boldsymbol{\lambda}^{T} \, \mathbf{D} \, \boldsymbol{\lambda} + \boldsymbol{\lambda}^{T} \, \mathbf{v}_{f}
\label{eq:polar-angle-dual-problem-objective-function}
\end{equation}
\begin{equation}
\nabla_{\boldsymbol{\lambda}} h_{0}(\boldsymbol{\lambda}) 
= \mathbf{D} \, \boldsymbol{\lambda} + \mathbf{v}_{f}
= \bar{\mathbf{D}} \, \boldsymbol{\lambda}_{T} + \bar{\mathbf{v}}_{f}
\,\,,
\label{eq:polar-angle-dual-problem-objective-gradient}
\end{equation}
\end{subequations}
where we've used 3D extensions of the tangentially-projected Delassus matrix and free velocity vector in the form of
\begin{equation}
\bar{\mathbf{D}} = 
\begin{bmatrix}
\hat{\mathbf{D}} \\
0
\end{bmatrix}
\,\in \mathbb{R}^{3\times2}
\,\,\,\,,\,\,\,\,
\bar{\mathbf{v}}_{f} = 
\begin{bmatrix}
\hat{\mathbf{v}}_{f} \\
0
\end{bmatrix}
\, \in \mathbb{R}^{3}
\,\,.
\end{equation}
Next, the constraint imposed by $\mathcal{P}_{mc}$ and its gradient are
\begin{subequations}
\begin{equation}
h_{1}(\boldsymbol{\lambda}) = \mathbf{d}_{n}^{T} \, \boldsymbol{\lambda} + \text{v}_{f,n} = 0
\end{equation}
\begin{equation}
\nabla_{\boldsymbol{\lambda}} h_{1}(\boldsymbol{\lambda}) = \mathbf{d}_{n}
\,\,,
\end{equation}
\end{subequations}
and equivalently, those of the friction cone $\mathcal{K}_{\mu}$ are
\begin{subequations}
\begin{equation}
h_{2}(\boldsymbol{\lambda}) = \lambda_{t}^{2} + \lambda_{o}^{2} - \mu^{2} \, \lambda_{n}^{2} = 0
\end{equation}
\begin{equation}
\nabla_{\boldsymbol{\lambda}} h_{2}(\boldsymbol{\lambda}) = 
\begin{bmatrix}
    2\,\lambda_{t} & 2\,\lambda_{o} & -2\,\mu^{2}\,\lambda_{n}
\end{bmatrix}^{T}
\,\,.
\end{equation}
\end{subequations}

Geometrically speaking, we know that $\nabla_{\boldsymbol{\lambda}} h_{1}(\boldsymbol{\lambda})$ and $\nabla_{\boldsymbol{\lambda}} h_{2}(\boldsymbol{\lambda})$ are always perpendicular to their respective constraint surfaces. We can exploit this to define a vector that is always tangent to the conic-section curve  by forming an orthogonal projection to the aforementioned constraint vectors. This results in the \textit{tangent vector} 
\begin{equation}
\boldsymbol{t}(\boldsymbol{\lambda}) := 
\nabla_{\boldsymbol{\lambda}} h_{1}(\boldsymbol{\lambda}) 
\times 
\nabla_{\boldsymbol{\lambda}} h_{2}(\boldsymbol{\lambda})
\,\,.
\label{eq:conic-section-tangent-vector}
\end{equation}
Furthermore, the KKT conditions require that the gradient of the objective function $\nabla_{\boldsymbol{\lambda}} h_{0}(\boldsymbol{\lambda})$ lie within the space spanned by the constraint gradients $\nabla_{\boldsymbol{\lambda}} h_{1}(\boldsymbol{\lambda})$ and $\nabla_{\boldsymbol{\lambda}} h_{1}(\boldsymbol{\lambda})$. Thus, the stationarity condition required on the Lagrangian of the dual problem can be expressed as
\begin{equation}
\nabla_{\boldsymbol{\lambda}} h_{0}(\boldsymbol{\lambda}^{*})^{T} \, \boldsymbol{t}(\boldsymbol{\lambda}^{*}) = 0
\,\,,
\label{eq:conic-section-gradient-product-optimality-condition}
\end{equation}
which is exactly (\ref{eq:polar-angle-optimality-condition-def}) from Sec.\ref{sec:apndx:single-contact:analytical-formulation}. We denote the function that evaluates this gradient product as
\begin{equation}
g(\boldsymbol{\lambda}) = \nabla_{\boldsymbol{\lambda}} h_{0}(\boldsymbol{\lambda})^{T} \, \boldsymbol{t}(\boldsymbol{\lambda})
\,\,.
\label{eq:conic-section-gradient-product}
\end{equation}

At this point we can introduce the bisection-method and explain how it's applicable. It operates by taking some initial interval $[\phi_{\text{min}}, \phi_{\text{max}}]$, and progressively dividing it in half while checking for changes in the sign of the function. The sign determines from which end of the interval the half-interval will be taken, i.e. start or end, and in this way, it very quickly converges to any detected zero-crossing. In our case, the function it operates on is the gradient product defined in (\ref{eq:conic-section-gradient-product}). Thus, we require two more operations to realize this bisection-based search. First is the function that yields the sign of the gradient product, which we'll denoted as
\begin{equation}
s_{g}(\boldsymbol{\lambda}) = 
\text{sign}(\,g(\boldsymbol{\lambda})\,)
\,\,.
\label{eq:conic-section-gradient-product-signum}
\end{equation}
Second, we require an operation that can provide a starting point for the search. We can use the radial projector
\begin{subequations}
\begin{equation}
\phi_{E} = \tan^{-1}\frac{\lambda_{0,o}}{\lambda_{0,t}} \,\,, r_{E} = \bar{r}(\phi_{E})
\end{equation}
\begin{equation}
\mathcal{P}_{\text{E}}^{\mathcal{S}\cap\mathcal{K}}(\boldsymbol{\lambda}_{0}) := \boldsymbol{\lambda}_{p}(\phi_{E}, \bar{r}(\phi_{E}))
\,\,.
\end{equation}
\end{subequations}
where $\mathcal{P}_{\text{E}}^{\mathcal{S}\cap\mathcal{K}} : \mathbb{R}^{3} \rightarrow \mathbb{R}^{3}$\footnote{$\mathcal{P}_{\text{E}}^{\mathcal{S}\cap\mathcal{K}}(\boldsymbol{\lambda}_{0})$ is a legitimate contact-consistent projector in its own right. However, in contrast to $\mathcal{P}_{\text{MDP}}^{\mathcal{S}\cap\mathcal{K}}(\boldsymbol{\lambda}_{0})$ and $\boldsymbol{P}_{\text{BS}}^{\mathcal{S}\cap\mathcal{K}}(\boldsymbol{\lambda}_{0})$, it does not yield maximally dissipative reactions.}. 
\begin{figure}[!t]
\centering
\includegraphics[width=1.0\linewidth]{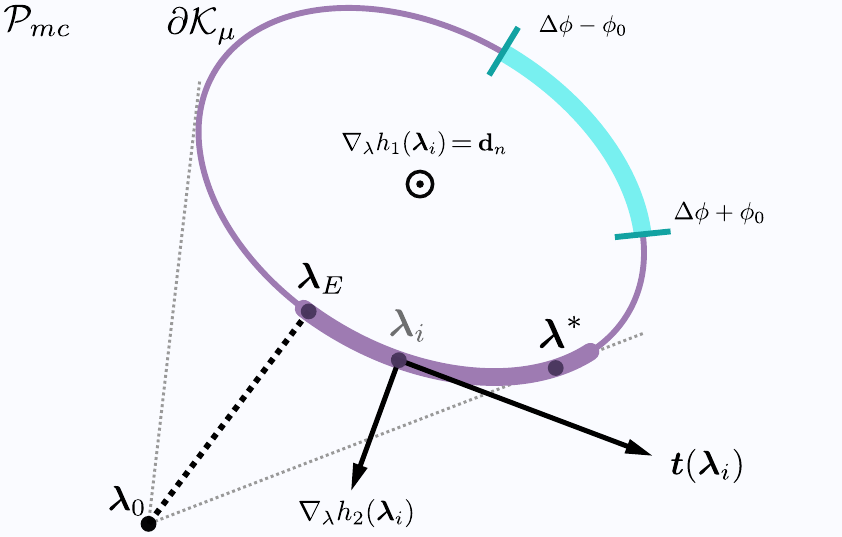}
\caption{The bisection-method applied to a sample conic-section (purple) defined by the intersection of the plane of maximum compression $\mathcal{P}_{mc}$ and the boundary of the friction cone $\partial\mathcal{K}_{\mu}$, depicted using $\mathcal{P}_{mc}$ as a cross-section via the plane normal $\mathbf{d}_{n}$. The unconstrained solution $\boldsymbol{\lambda}_{0}$ is first projected onto the conic-section as the Euclidean projection $\boldsymbol{\lambda}_{E}$. From the perspective of $\boldsymbol{\lambda}_{0}$, the line-of-sight region (gray dotted lines) extends onto each end of the contour of the curve, and the side opposite the gradient is taken as the zero-crossing search interval (think purple curve segment). At each step of the interval-search phase, the iterate $\boldsymbol{\lambda}_{i}$ is used to compute the tangent vector $\boldsymbol{t}(\boldsymbol{\lambda}_{i})$ w.r.t the local contour via the orthogonality of the constraint gradients $\nabla_{\boldsymbol{\lambda}} h_{1}(\boldsymbol{\lambda}_{i})$ and $\nabla_{\boldsymbol{\lambda}} h_{2}(\boldsymbol{\lambda}_{i})$. Thus, the search interval is progressively extended towards the clock-wise direction until the end reaching the end of the line-of-sight region, and while avoiding the infeasible region $[0, 2\pi) \,\setminus\, \mathcal{I}_{F}$ (cyan).}
\label{fig:bisection-method}
\end{figure}

We now have all the elements necessary to construct the bisection-method algorithm to realize $\boldsymbol{P}_{\text{BS}}^{\mathcal{S}\cap\mathcal{K}}(\boldsymbol{\lambda}_{0})$. The full algorithm is detailed in Alg.\ref{alg:bisection-method-optimal-polar-angle-search}, and a visual depiction using a mock example is provided in Fig.\ref{fig:bisection-method}. However, it's worth taking a moment to outline the key steps involved and provide some context. These are: 
\begin{enumerate}
\item compute $\boldsymbol{\lambda}_{E}$ using the Euclidean projection of $\boldsymbol{\lambda}_{0}$
\item compute the line-of-sight region $\mathcal{I}_{LoS}$
\item compute the feasible region $\mathcal{I}_{F}$ using $\mu$ and $\mathbf{d}_{n}$
\item determine the smallest feasible interval $\mathcal{I}_{BS} \subseteq \mathcal{I}_{F} \cap \mathcal{I}_{LoS}$
\item perform bisection on the interval $\mathcal{I}_{BS}$
\item if no solution was found, perform bisection on $\mathcal{I}_{F}$\\
\end{enumerate}

All but step (4) are straight-forward operations on the problem constants. The most delicate, and most important, is step (4) that seeks the smallest feasible interval to perform the bisection-search on. In a nutshell, this procedure proceeds as follows: starting from the angle $\phi_{E}$ corresponding to the Euclidean projection of the unconstrained solution $\boldsymbol{\lambda}_{0}$, we take make a series of exponentially increasing steps opposite of the direction (i.e. sign) of the projected gradient at $\phi_{E}$, and re-compute this gradient sign at each step until it inverts or until we've reached some margin beyond the line-of-sight region. The first case signifies that a zero-crossing has been detected, and the bisection search interval $[\phi_{\text{min}}, \phi_{\text{max}}]$ is collapsed to the previous and current angle of the interval search procedure.

However, there is a critical caveat to this procedure. As the interval necessarily needs to be a compact set for bisection to succeed, we must carefully handle the cases where the infeasible region $[0, 2\pi) \,\setminus\, \mathcal{I}_{F} $, overlaps with or is very close to, the line-of-sight interval $\mathcal{I}_{LoS}$. In fact, we found that the original algorithm described in \cite{hwangbo2018percontact} sometimes fails in such cases. To robustify it, we've incorporated a few modifications that ensure that: (a) we don't accidentally stop the interval expansion within or at the limit of the infeasible region, and (b) if this interval search fails, we use step (6) to perform a brute-force search of all zero-crossings and determine the optimal angle via exclusion similarly to how it was handled for the quartic polynomial solutions in Sec\ref{sec:apndx:single-contact:analytical-formulation}. The resulting operation, even in the worst case of resorting to step (6), drastically reduces the average time-complexity of the projection.

\begin{algorithm}[!t]
\caption{Bisection-Based Optimal Polar-Angle Search}
\label{alg:bisection-method-optimal-polar-angle-search}
\begin{algorithmic}
\Require $\boldsymbol{\lambda}_{0}\in\mathbb{R}^{3},\,\,\mathbf{D}\in\mathbb{R}^{3\times3},\,\,\mathbf{v}_{f}\in\mathbb{R}^{3},\,\,\mu\in\mathbb{R}_{+}$
\Require $N$, $\beta_{1}$, $\beta_{2}$, $\beta_{3}$, $\epsilon_{BS}$\\
\State $\phi_{E} \gets \text{atan2}(\lambda_{0,o}, \lambda_{0,t})$ \textcolor{gray}{\Comment{Euclidean projection}}
\State $\boldsymbol{\lambda}_{E} \gets \boldsymbol{\lambda}_{p}(\phi_{E}, r(\phi_{E}))$ \textcolor{gray}{\Comment{using (\ref{eq:polar-coordinates-contact-reaction})}}
\State $g_{0} \gets g(\boldsymbol{\lambda}_{E})$ \textcolor{gray}{\Comment{using (\ref{eq:conic-section-gradient-product})}}
\State $s_{0} \gets s_{g}(\boldsymbol{\lambda}_{E})$ \textcolor{gray}{\Comment{using (\ref{eq:conic-section-gradient-product-signum})}} \\
\State $\Delta{\phi}_{LoS} \gets \text{acos}\left(\frac{\Vert \boldsymbol{\lambda}_{0,T} \Vert_{2}}{r_{E}}\right)$ \textcolor{gray}{\Comment{line-of-sight range}}
\State $\mathcal{I}_{LoS} \gets [-\Delta{\phi}_{LoS}\,,\,\Delta{\phi}_{LoS}] + \phi_{E}$ \\
\State $\mathcal{I}_{F} \gets$ from $\mathbf{d}_{n}$ and $\mu$ using (\ref{eq:polar-angle-feasible-region}) \textcolor{gray}{\Comment{feasible range}}\\
\State $\phi_{0} \gets \phi_{E}$ \textcolor{gray}{\Comment{initial angle}}
\State $\alpha \gets - \min(|\Delta{\phi}_{LoS}|, \beta_{1}) \, s_{0}$ \textcolor{gray}{\Comment{initial step-size}}
\State $\text{found} \gets \text{false}$\\
\For{$i = 1$ to $N$} \textcolor{gray}{\Comment{find smallest feasible search interval}}
\State $\phi_{i} \gets \phi_{i-1} + \alpha$
\State $r_{i} \gets r(\phi_{i})$
\State $\boldsymbol{\lambda}_{i} \gets \boldsymbol{\lambda}_{p}(\phi_{i}, r(\phi_{i})))$
\State $g_{i} \gets g(\boldsymbol{\lambda}_{i})$
\State $s_{i} \gets s(\boldsymbol{\lambda}_{i})$
\State $p_{i} \gets |\phi_{E} - \phi_{i}| - \Delta{\phi}_{LoS}$ \\
\If{$r_{i} > 0 \land g_{0} \cdot g_{i} < 0$}
    \State $\phi_{f} \gets \phi_{i}$
    \State $\text{found} \gets \text{true}$ \textcolor{gray}{\Comment{zero-crossing detected}}
    \State \textbf{break}
\EndIf\\
\If{$p_{i} < 0 \lor r_{i} < 0$}
    \State $\phi_{i} \gets \phi_{i-1}$ \textcolor{gray}{\Comment{reset interval}}
    \State $\alpha \gets \alpha \, \beta_{2}$ \textcolor{gray}{\Comment{expand step-size}}
\Else
    \State $\alpha \gets -\beta_{3} \, s_{0}$ \textcolor{gray}{\Comment{step conservatively past line-of-sight}}
    \State $g_{0} \gets g_{i}$
    \State $s_{0} \gets s_{i}$
\EndIf
\EndFor\\
\If{found}
    \State $\phi_{\text{min}} \gets \min(\phi_{0}\,,\,\phi_{f})$
    \State $\phi_{\text{max}} \gets \max(\phi_{0}\,,\,\phi_{f})$
    \State $\mathcal{I}_{BS} \gets [\phi_{\text{min}}\,,\,\phi_{\text{max}}]$ \textcolor{gray}{\Comment{bisection interval}}
    \State $\phi^{*} \gets$ from Alg.\ref{alg:bisection-method-search-operation} using $\mathcal{I}_{BS}$ \textcolor{gray}{\Comment{bisection operation}}
\Else
    \State $\Phi_{zc} \gets$ zero-crossings in $\mathcal{I}_{F}$ using Alg.\ref{alg:feasible-region-bisection-search} \textcolor{gray}{\Comment{brute-force search}}
    \State $\phi^{*} \gets \text{arg}\min_{\phi \in \Phi_{zc}} h_{0}(\boldsymbol{\lambda}_{p}(\phi, r(\phi))))$ \textcolor{gray}{\Comment{using (\ref{eq:polar-angle-dual-problem-objective-function})}}
\EndIf\\
\State $\boldsymbol{\lambda}^{*} \gets \boldsymbol{\lambda}_{p}(\phi^{*}, r(\phi^{*})))$\\\\
\Return $\boldsymbol{\lambda}^{*}$
\end{algorithmic}
\end{algorithm}

\begin{algorithm}[!t]
\setstretch{1.1}
\caption{Bisection-Method Search Operation}
\label{alg:bisection-method-search-operation}
\begin{algorithmic}
\Require $\phi_{\text{min}}\in\mathbb{R}$, $\phi_{\text{max}}\in\mathbb{R}$, $\epsilon_{BS}\in\mathbb{R}_{+}$\\
\State $g_{0} \gets g(\boldsymbol{\lambda}_{p}(\phi_{\text{min}}, r(\phi_{\text{min}}))))$ \textcolor{gray}{\Comment{using (\ref{eq:conic-section-gradient-product})}}\\
\Repeat \textcolor{gray}{\Comment{bisection procedure}}\\
    \State $\phi_{\text{mid}} \gets \frac{1}{2}\,(\phi_{\text{min}} + \phi_{\text{max}})$ \textcolor{gray}{\Comment{midpoint}}
    \State $r_{\text{mid}} \gets r(\phi_{\text{mid}})$ \textcolor{gray}{\Comment{using (\ref{eq:polar-angle-radius})}}
    \State $\boldsymbol{\lambda}_{\text{mid}} \gets \boldsymbol{\lambda}_{p}(\phi_{\text{mid}}, r(\phi_{\text{mid}})))$ \textcolor{gray}{\Comment{using (\ref{eq:polar-coordinates-contact-reaction})}}\\
    \If{$g(\boldsymbol{\lambda}_{\text{mid}}) \cdot g_{0} < 0$} \textcolor{gray}{\Comment{check for zero-crossing}}
        \State $\phi_{\text{min}} \gets \phi_{\text{mid}}$ \textcolor{gray}{\Comment{contract left}}
    \Else
        \State $\phi_{\text{max}} \gets \phi_{\text{mid}}$ \textcolor{gray}{\Comment{contract right}}
    \EndIf\\
    \State $\boldsymbol{\lambda}_{\text{min}} \gets \boldsymbol{\lambda}_{p}(\phi_{\text{min}}, r(\phi_{\text{min}})))$
    \State $\boldsymbol{\lambda}_{\text{max}} \gets \boldsymbol{\lambda}_{p}(\phi_{\text{max}}, r(\phi_{\text{max}})))$\\
\Until{$\Vert \boldsymbol{\lambda}_{\text{min}} - \boldsymbol{\lambda}_{\text{max}} \Vert_{2} < \epsilon_{BS}$} \textcolor{gray}{\Comment{check for convergence}}\\\\
\Return $\phi_{\text{mid}}$
\end{algorithmic}
\end{algorithm}

\begin{algorithm}[!t]
\setstretch{1.1}
\caption{Feasible Region Bisection-Search}
\label{alg:feasible-region-bisection-search}
\begin{algorithmic}
\Require $\mathcal{I}_{F}$, $N\in\mathbb{N}_{+}$\\
\State $\Phi_{zc} \gets \emptyset$ \textcolor{gray}{\Comment{initialize set of zero-crossings}}
\State $\mathcal{I} \gets \emptyset$ \textcolor{gray}{\Comment{initialize set of intervals}}\\
\State $\phi_{0} \gets \min(\mathcal{I}_{F})$ \textcolor{gray}{\Comment{start}}
\State $\Delta{\phi} \gets |\phi_{f} - \phi_{0}|$ \textcolor{gray}{\Comment{range}}
\State $\alpha \gets \frac{\Delta{\phi}}{N}$ \textcolor{gray}{\Comment{step-size}}
\State $g_{0} \gets g(\boldsymbol{\lambda}_{p}(\phi_{0}, r(\phi_{0})))$\\
\For{$i = 1$ to $N$} \textcolor{gray}{\Comment{zero-crossing interval construction}}
    \State $\phi_{i} \gets \phi_{i-1} + \alpha$
    \State $\boldsymbol{\lambda}_{i} = \boldsymbol{\lambda}_{p}(\phi_{i}, r(\phi_{i}))$
    \State $g_{i} \gets g(\boldsymbol{\lambda}_{i})$
    \If{$g_{i} \cdot g_{i-1} < 0$}
        \State $\mathcal{I} \gets \mathcal{I} \cup \{[\phi_{i} - \alpha \,,\, \phi_{i}]\}$ \textcolor{gray}{\Comment{extend interval set}}
    \EndIf
\EndFor\\
\For{$i = 1$ to $n_{zc}$} \textcolor{gray}{\Comment{interval-wise bisection}}
    \State $\phi_{i} \gets$ from Alg.\ref{alg:bisection-method-search-operation} using $\mathcal{I}_{i}$ \textcolor{gray}{\Comment{bisection operation}}
    \State $\Phi_{zc} \gets \Phi_{zc} \cup \{\phi_{i}\}$ \textcolor{gray}{\Comment{extend zero-crossing set}}
\EndFor\\\\
\Return $\Phi_{zc}$
\end{algorithmic}
\end{algorithm}

\newpage
\onecolumn
\section{Additional Experiment Figures}
\label{sec:apndx:additional-figures}
\begin{figure*}[!ht]
\centering
\includegraphics[width=1.0\linewidth]{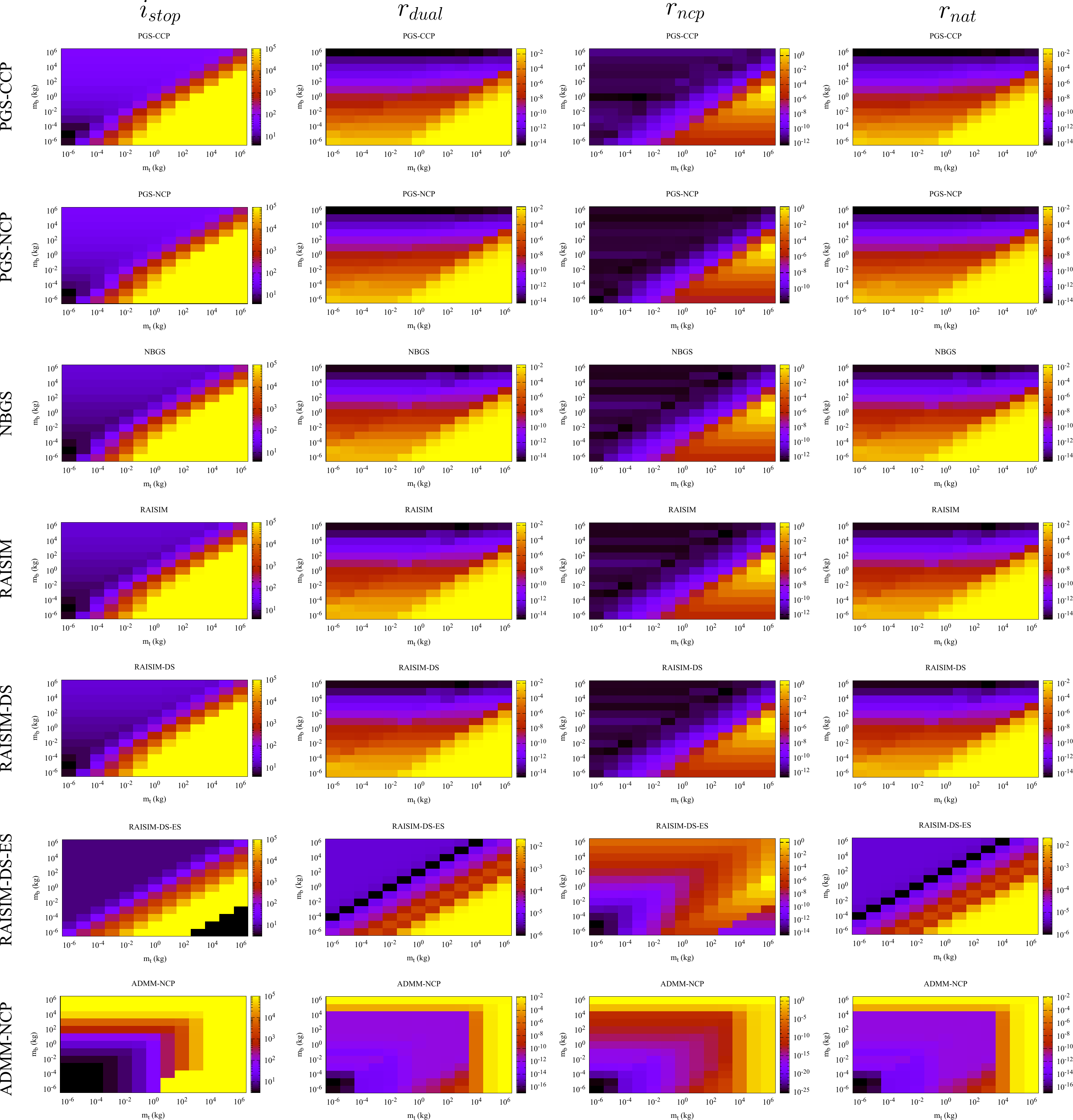}
\caption{\textit{Boxes-Fixed}: A heatmap-like rendering of solver iterations $i_{stop}$ and the complementarity residual $r_{ncp}$ as functions of the bottom and top bodies with respective masses $m_b$ and $m_t$. The data is generated while the bottom box is at rest on the plane.}
\label{fig:boxes-fixed-heatmap-rest}
\end{figure*}
\begin{figure*}[!ht]
\centering
\includegraphics[width=1.0\linewidth]{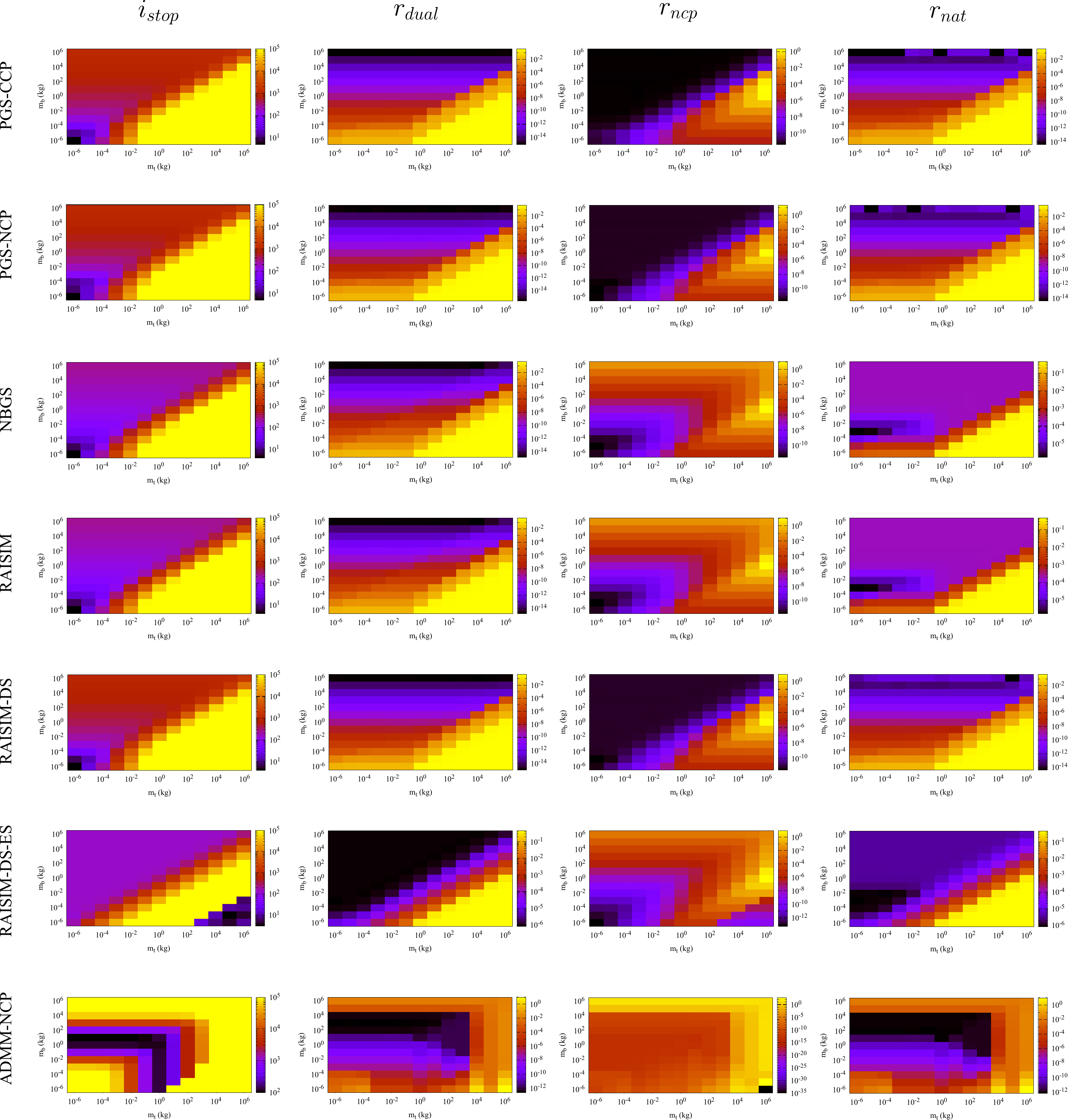}
\caption{\textit{Boxes-Fixed}: A heatmap-like rendering of solver iterations $i_{stop}$ and the complementarity residual $r_{ncp}$ as functions of the bottom and top bodies with respective masses $m_b$ and $m_t$. The data is generated while the bottom box is pushed by a force marginally exceeding the necessary stiction force.}
\label{fig:boxes-fixed-heatmap-push}
\end{figure*}
\begin{figure*}[!ht]
\centering
\includegraphics[width=0.95\linewidth]{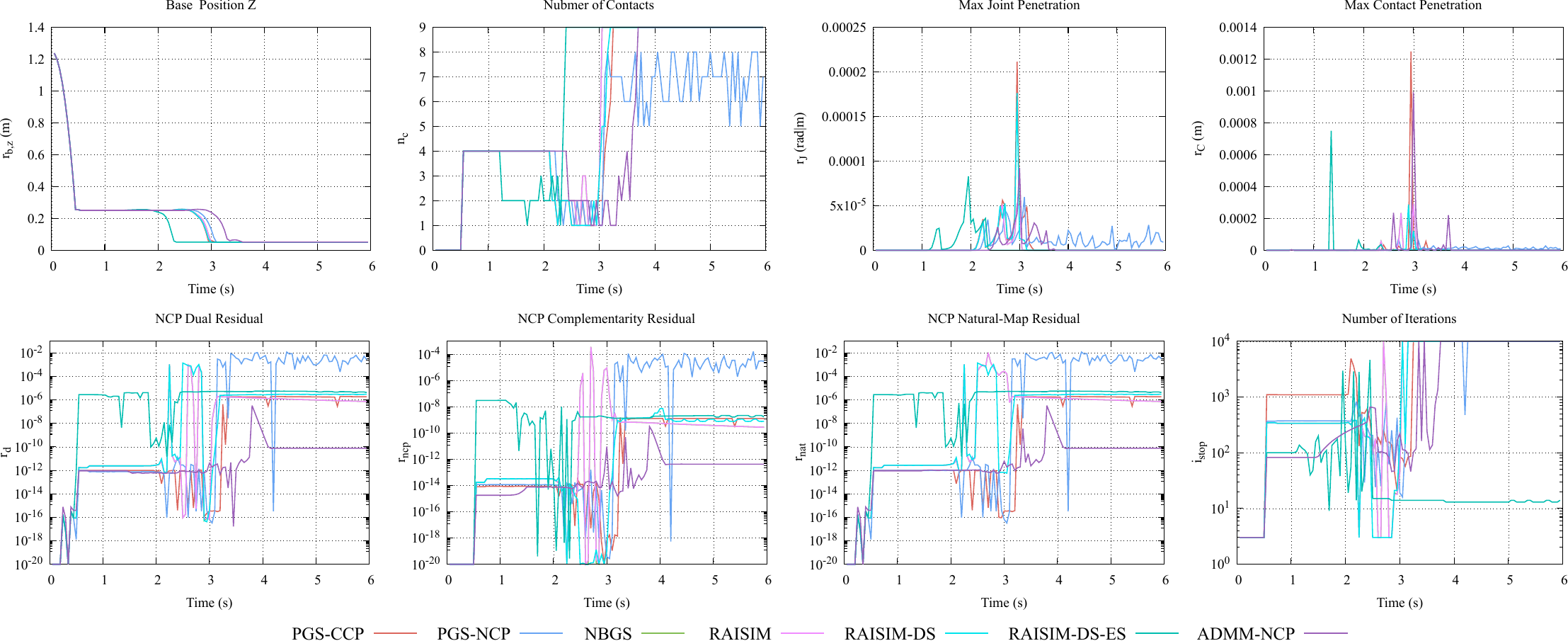}
\caption{\textit{Nunchaku}: Summary of the second run that enables constraint stabilization.}
\label{fig:nunchaku-hard-stbl}
\end{figure*}
\begin{figure*}[!ht]
\centering
\includegraphics[width=0.95\linewidth]{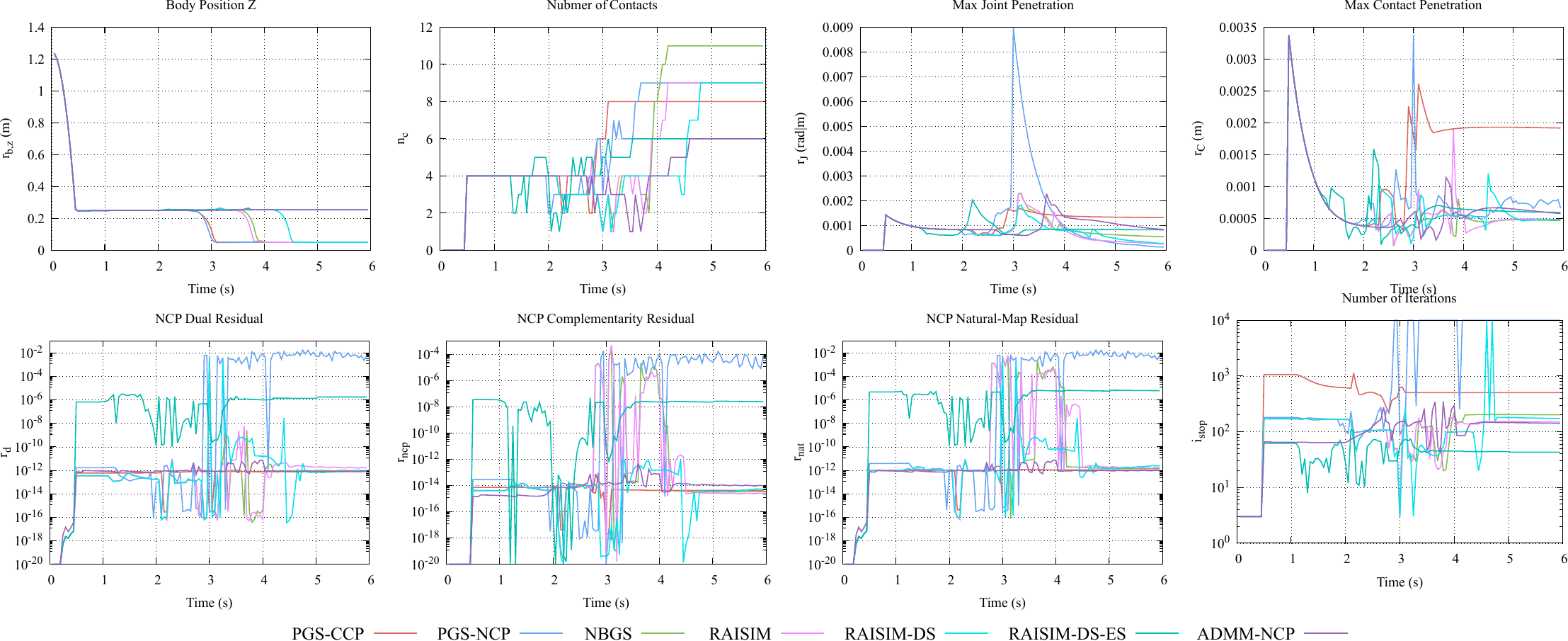}
\caption{\textit{Nunchaku}: Summary of the third run that enables constraint softening.}
\label{fig:nunchaku-soft}
\end{figure*}
\begin{figure*}[!ht]
\centering
\includegraphics[width=0.95\linewidth]{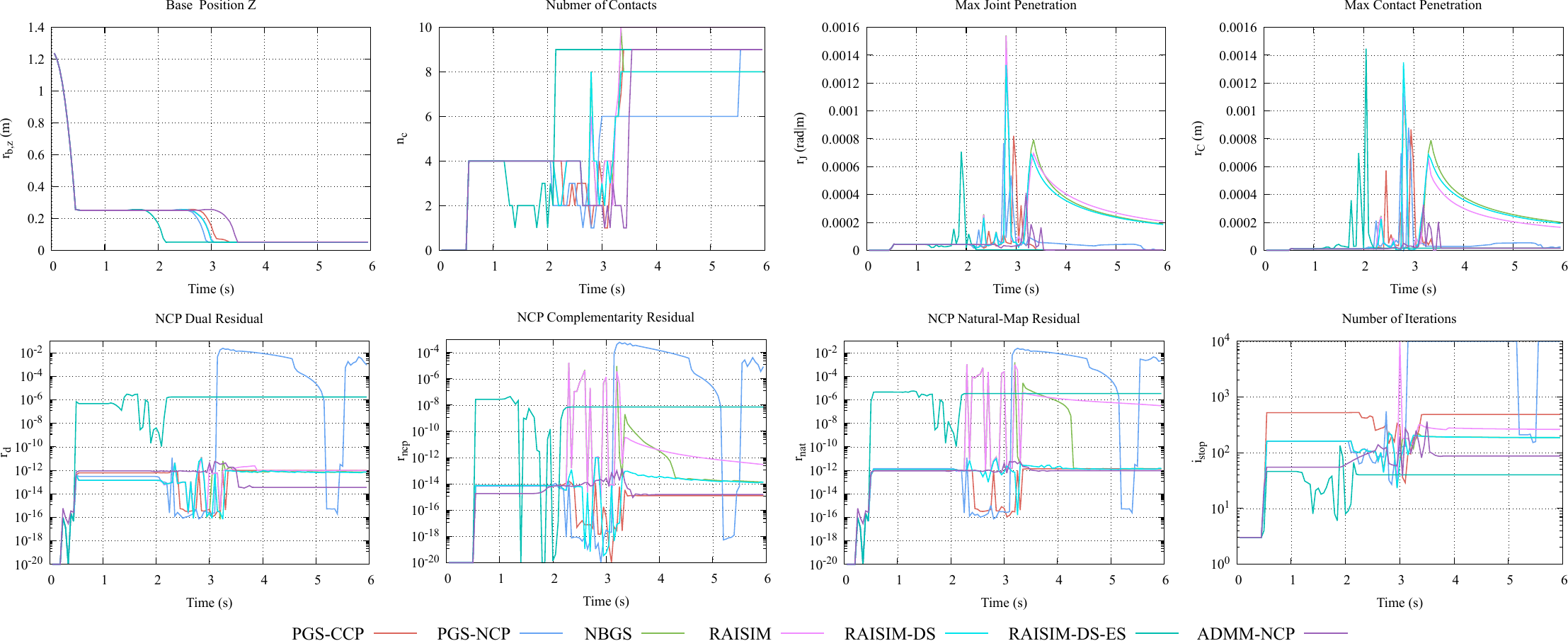}
\caption{\textit{Nunchaku}: Summary of the fourth run that enables both constraint stabilization and softening.}
\label{fig:nunchaku-soft-stbl}
\end{figure*}
\begin{figure*}[!ht]
\centering
\includegraphics[width=1.0\linewidth]{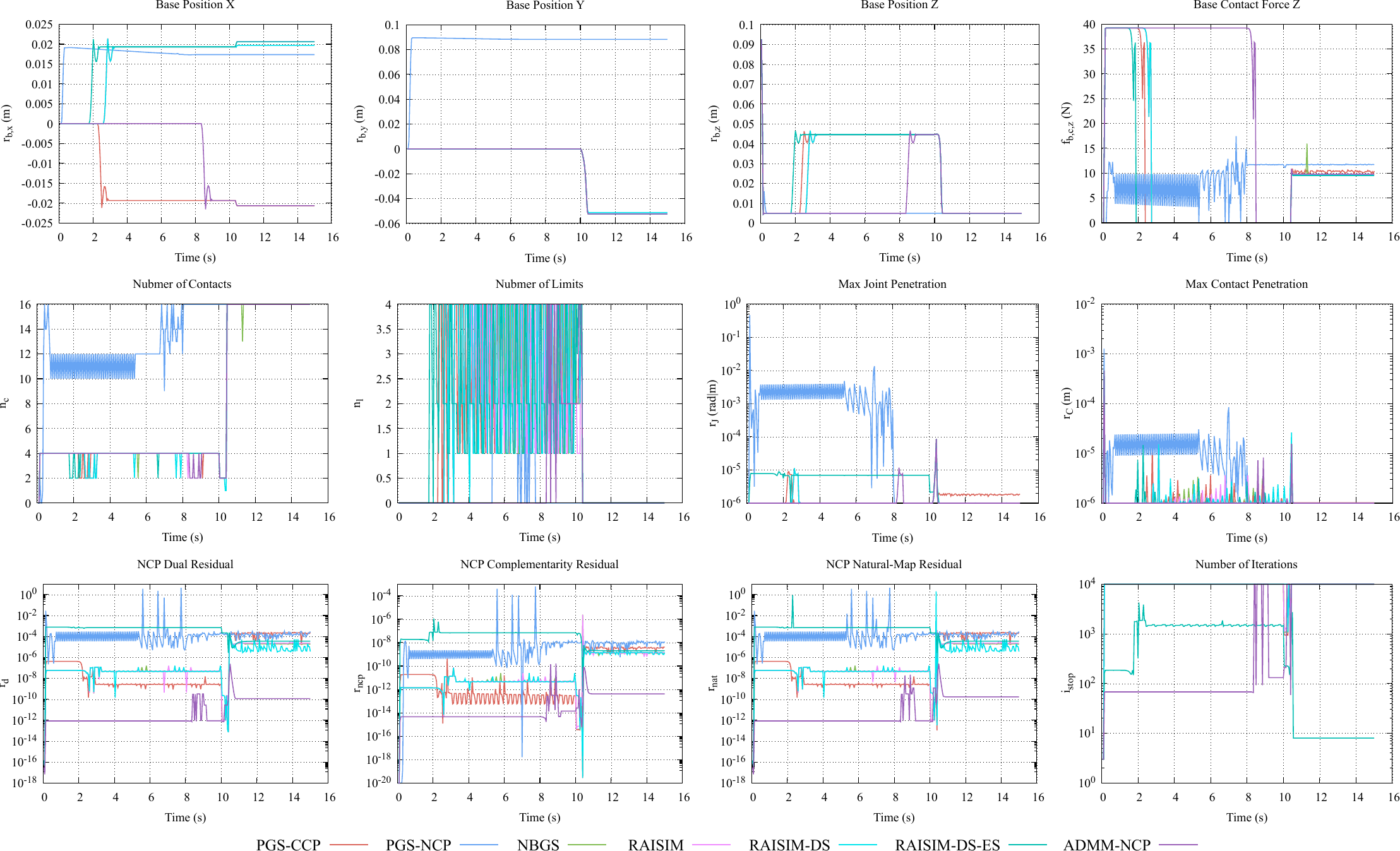}
\caption{\textit{Fourbar}: Summary of the second run that enables constraint stabilization.}
\label{fig:fourbar-hard-stbl}
\end{figure*}
\begin{figure*}[!ht]
\centering
\includegraphics[width=1.0\linewidth]{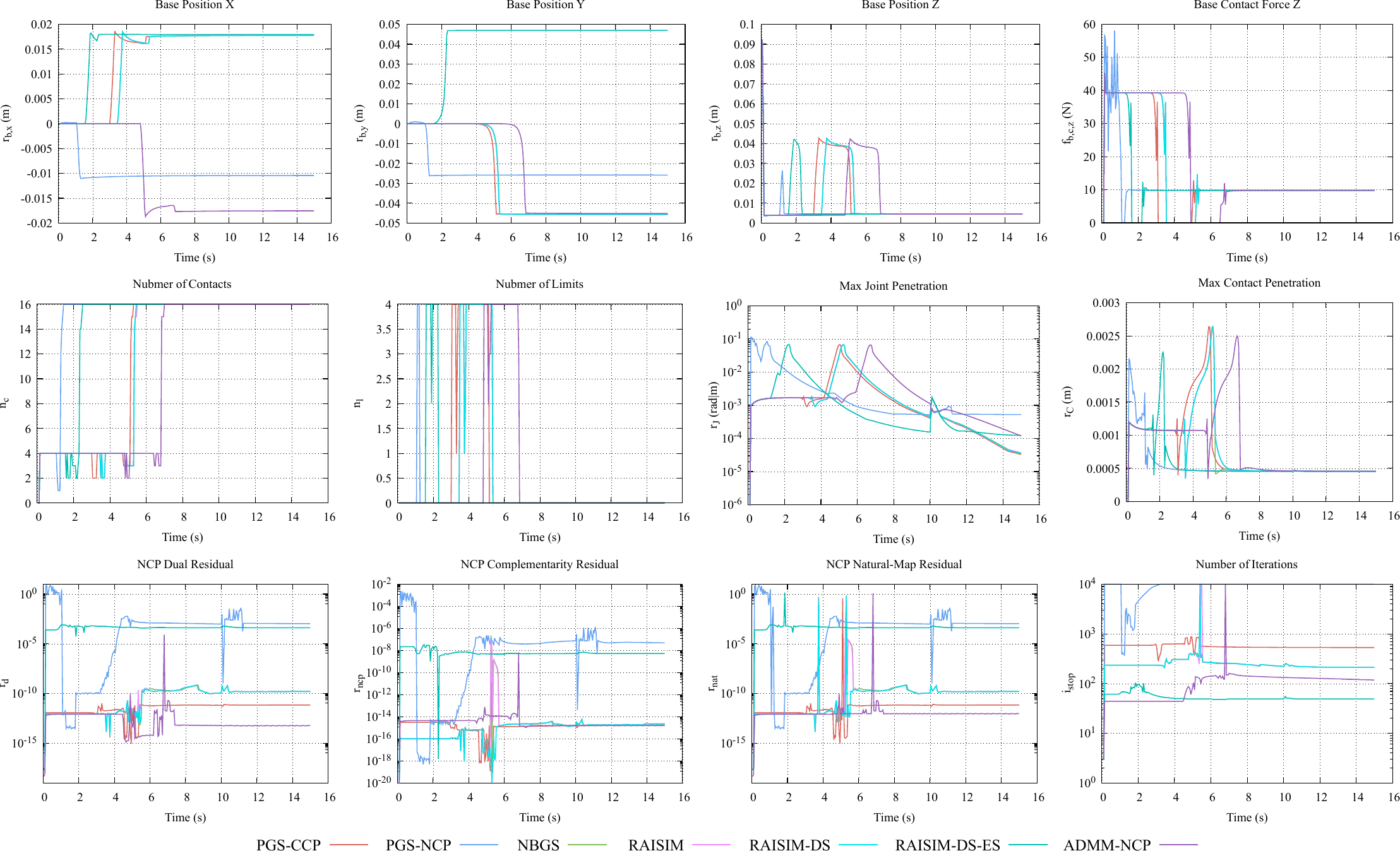}
\caption{\textit{Fourbar}: Summary of the third run that enables constraint softening.}
\label{fig:fourbar-soft}
\end{figure*}
\begin{figure*}[!ht]
\centering
\includegraphics[width=1.0\linewidth]{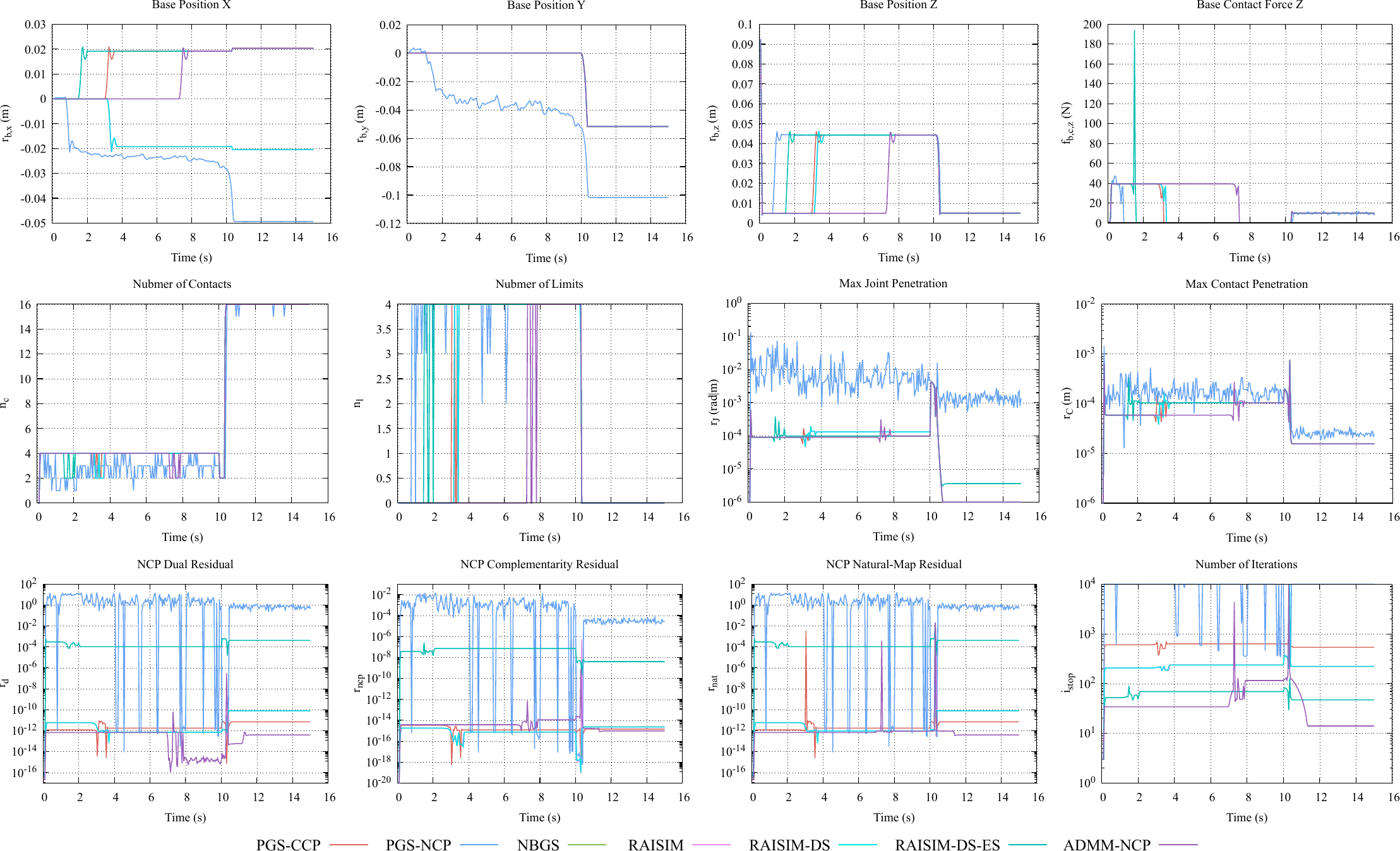}
\caption{\textit{Fourbar}: Summary of the fourth run that enables both constraint stabilization and softening.}
\label{fig:fourbar-soft-stbl}
\end{figure*}
\begin{figure*}[!ht]
\centering
\includegraphics[width=1.0\linewidth]{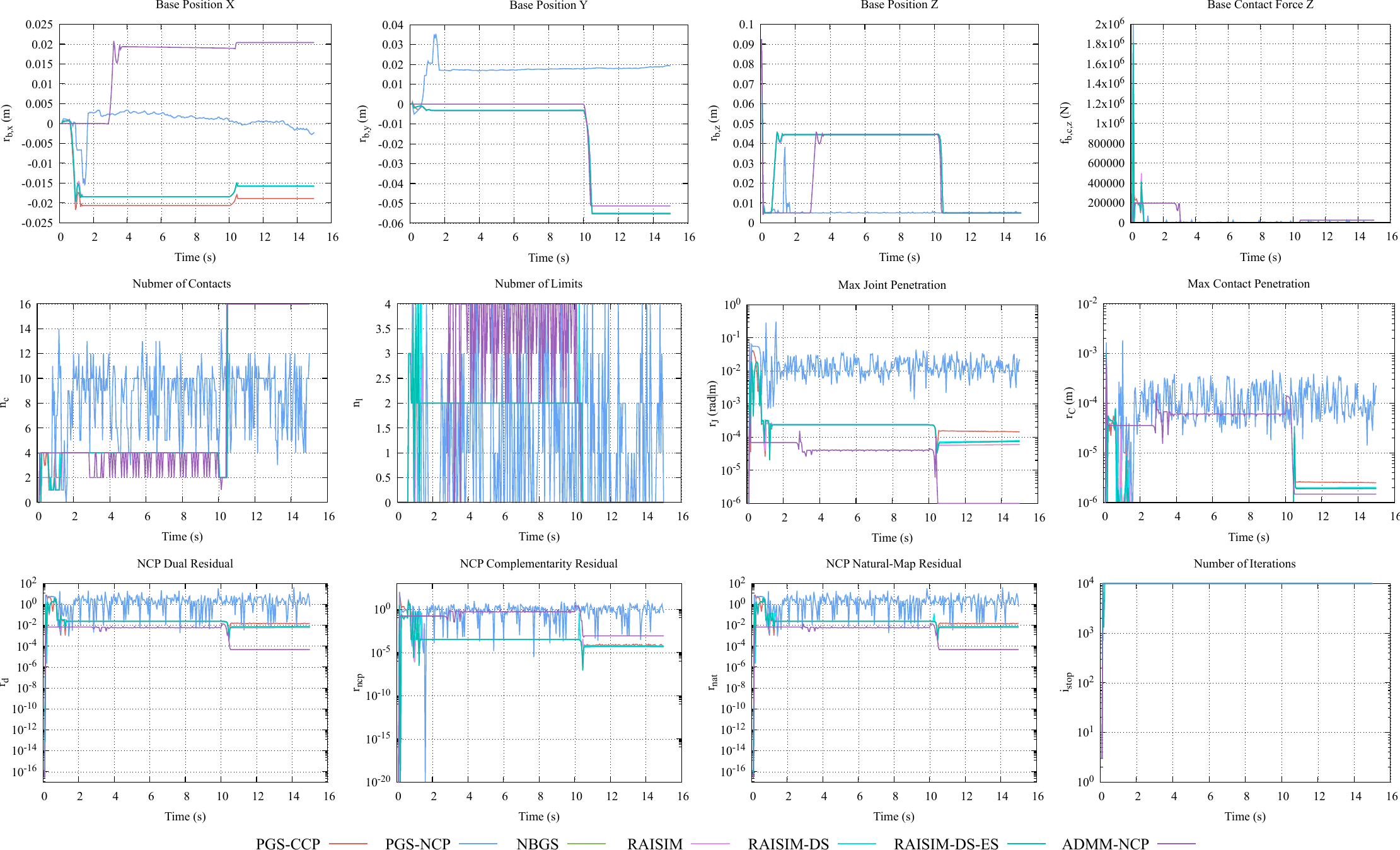}
\caption{\textit{Fourbar}: Summary of solving the un-augmented problem in the presence of a large mass-ratio.}
\label{fig:fourbar-hard-stbl-lmr}
\end{figure*}
\begin{figure*}[!ht]
\centering
\includegraphics[width=0.9\linewidth]{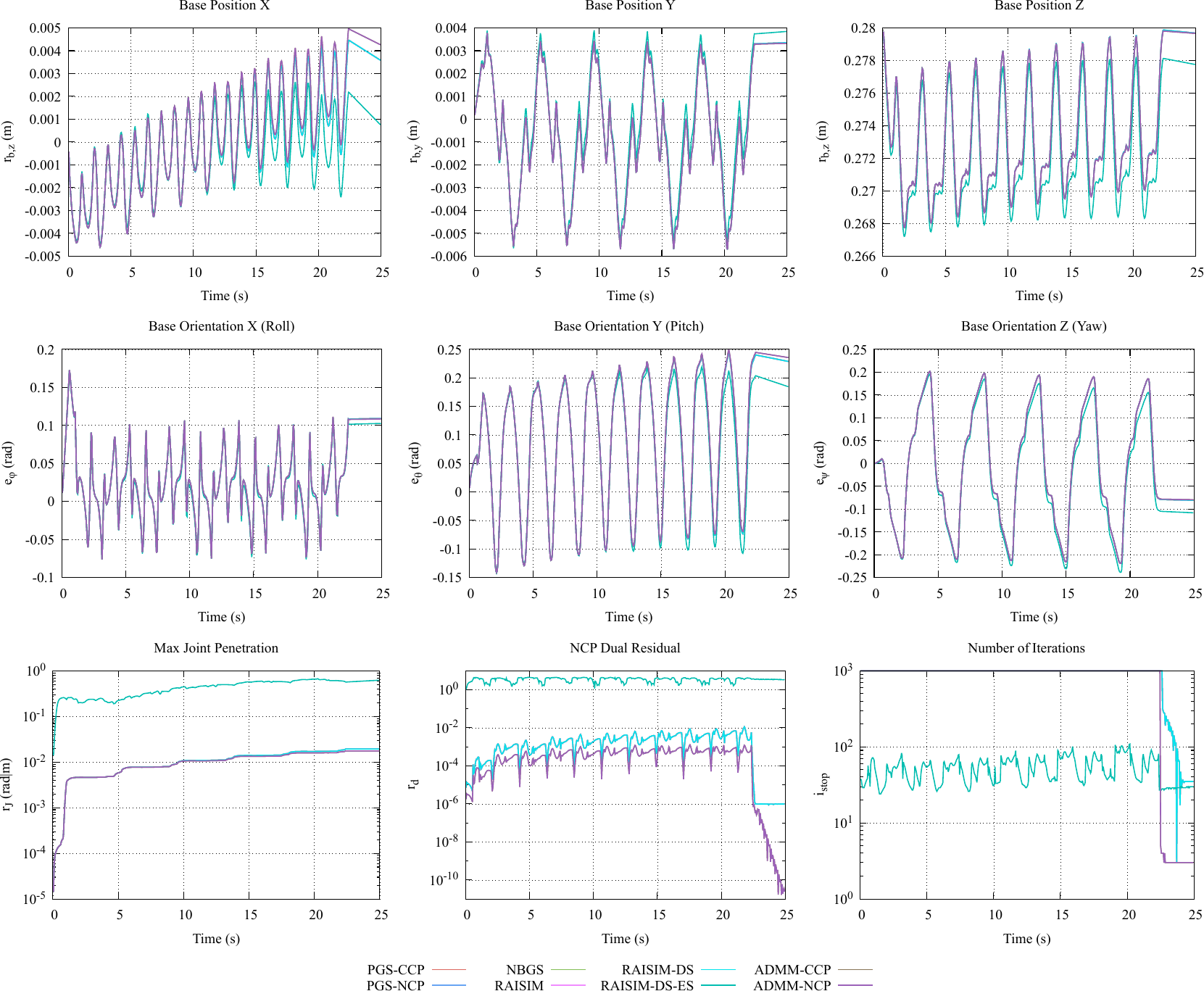}
\caption{\textit{Walker}: Time-series plots of solving the un-augmented problem with environment collisions and gravity disabled.}
\label{fig:walker-floating}
\end{figure*}
\begin{figure*}[!ht]
\centering
\includegraphics[width=1.0\linewidth]{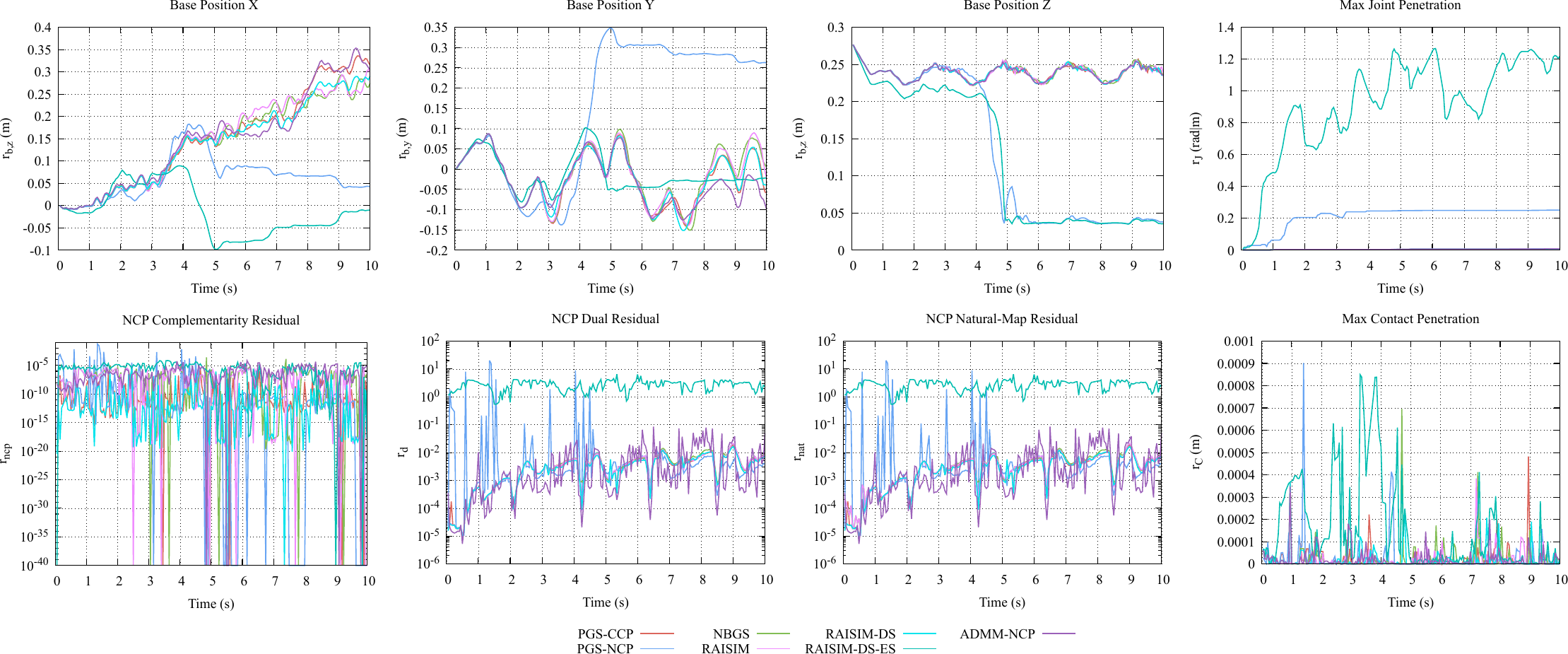}
\caption{\textit{Walker}: Time-series plots of solving the un-augmented problem with external collisions and gravity enabled.}
\label{fig:walker-hard}
\end{figure*}
\begin{figure*}[!ht]
\centering
\includegraphics[width=0.97\linewidth]{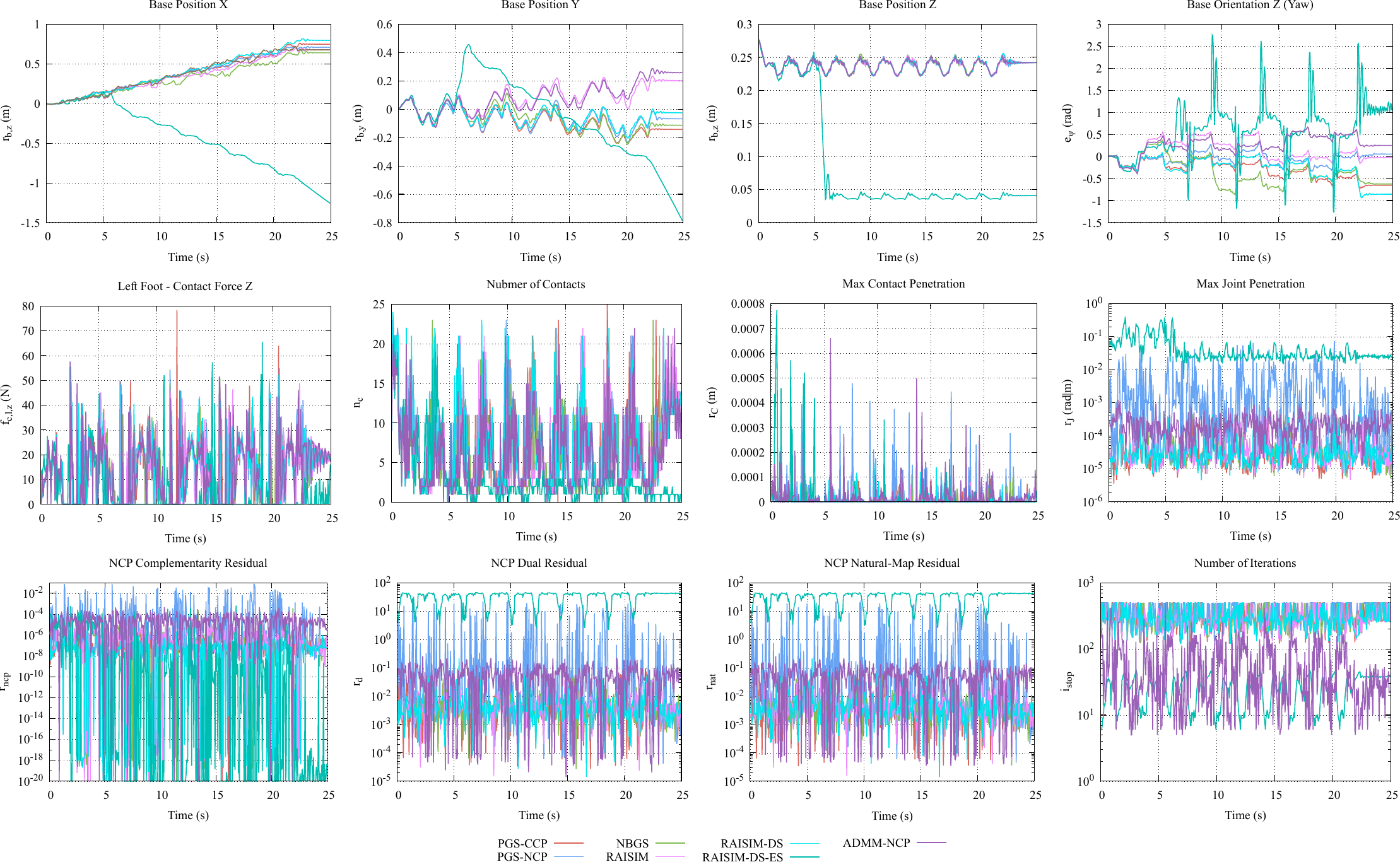}
\caption{\textit{Walker}: Time-series plots of solving the stabilized problem using high-throughput solver settings.}
\label{fig:walker-fast}
\end{figure*}
\begin{figure*}[!ht]
\centering
\includegraphics[width=0.97\linewidth]{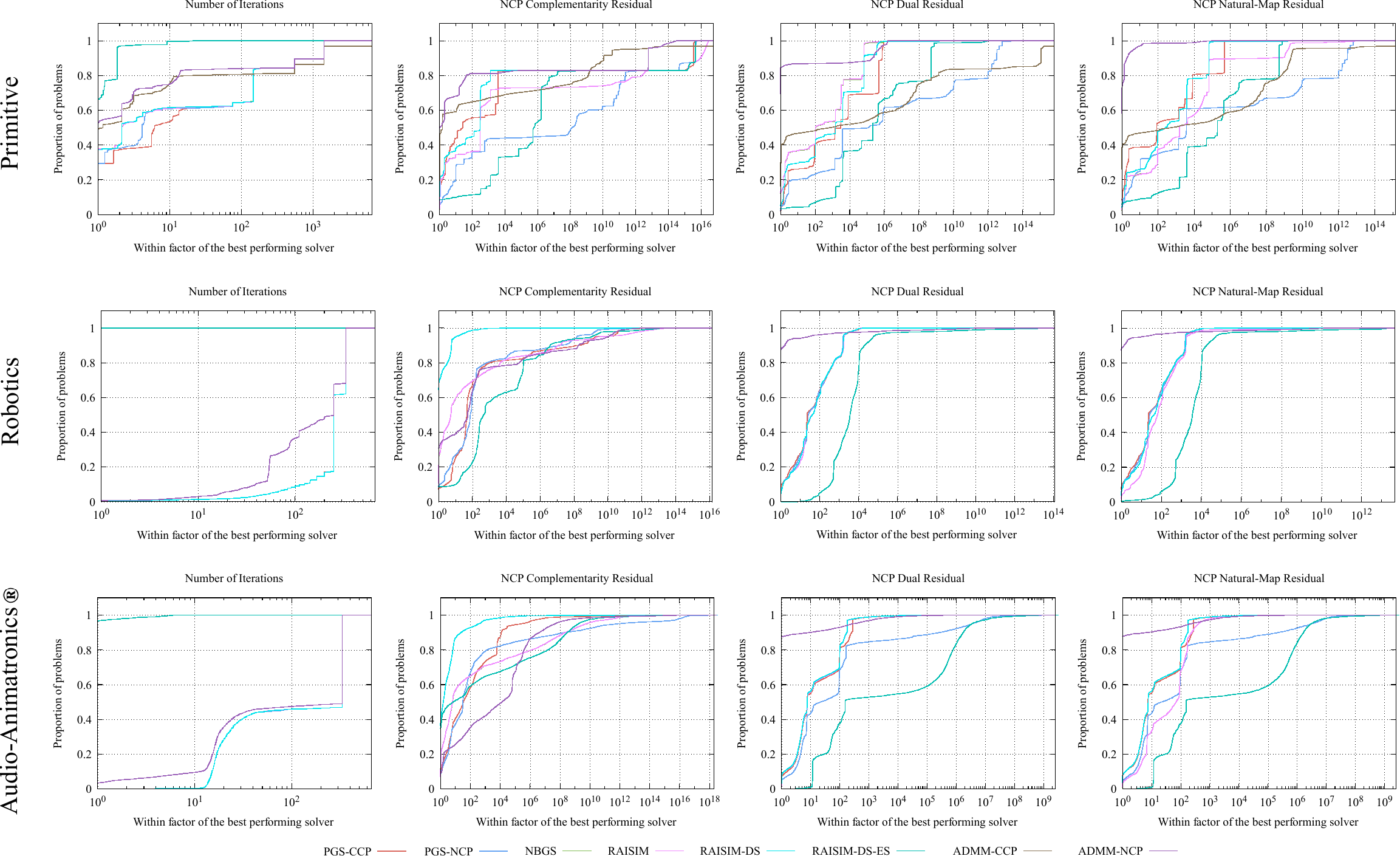}
\caption{Performance profiles w.r.t each benchmark system category.}
\label{fig:perfprof-problem-categories}
\end{figure*}
\begin{figure*}[!ht]
\centering
\includegraphics[width=0.97\linewidth]{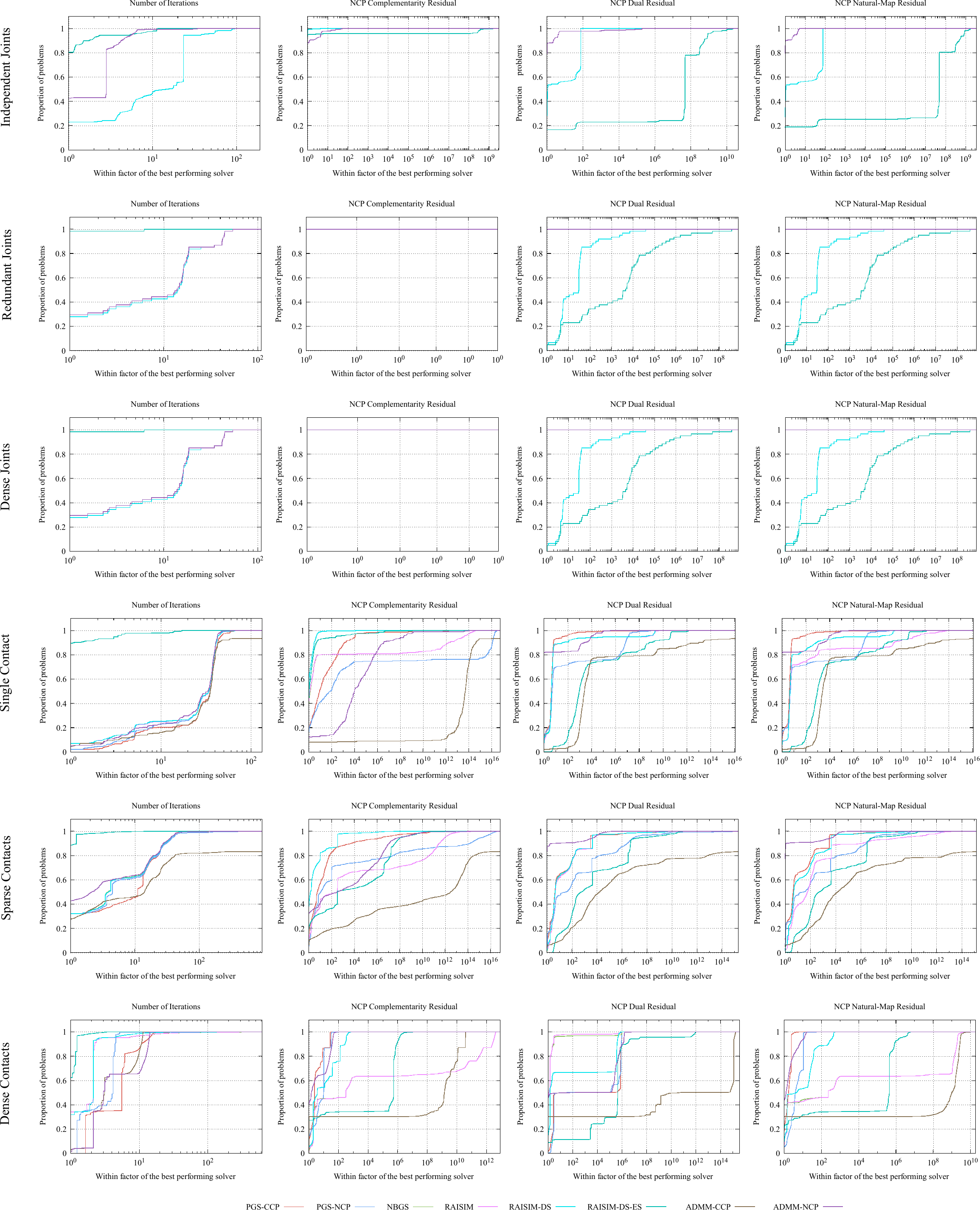}
\caption{Performance profiles w.r.t each benchmark sample category.}
\label{fig:perfprof-sample-categories}
\end{figure*}
%

\twocolumn
\clearpage
\ifCLASSOPTIONcaptionsoff
  \newpage
\fi

\bibliographystyle{IEEEtran}
\bibliography{main.bib}

\end{document}